\documentclass{article}

\usepackage[final]{neurips_2022}

\usepackage[utf8]{inputenc} % allow utf-8 input
\usepackage[T1]{fontenc}    % use 8-bit T1 fonts
\usepackage{url}            % simple URL typesetting
\usepackage{booktabs}       % professional-quality tables
\usepackage{amsfonts}       % blackboard math symbols
\usepackage{nicefrac}       % compact symbols for 1/2, etc.
\usepackage{microtype}      % microtypography
\usepackage{xcolor}         % colors
\usepackage{subcaption}

\usepackage[hidelinks,colorlinks=true,linkcolor=blue,citecolor=blue]{hyperref}
\usepackage{amsmath}
\usepackage{amssymb}

\usepackage{mathabx}
\usepackage{accents}

\usepackage[pdftex]{graphicx}
\usepackage{mathtools}
\usepackage{amsfonts}
\usepackage{amsthm,bm}
\usepackage{color}
\usepackage{enumitem}
\DeclareMathAlphabet{\mathpzc}{OT1}{pzc}{m}{it}

\usepackage{mathrsfs}

\newtheorem{theorem}{Theorem}[section]
\newtheorem{definition}{Definition}[section]

\newtheorem{lemma}{Lemma}[section]
\newtheorem{corollary}{Corollary}[section]
\newtheorem{assumption}{Assumption}%[section]

\newtheorem{induction}{Induction hypothesis}[section]

\usepackage{mathtools}

\newtheorem{parametrization}{Parametrization}[section]

\newtheorem{simplification}{Simplification}[section]

\usepackage{thm-restate}

\title{Vision Transformers provably learn spatial structure}

\author{%
  Samy Jelassi\\
  Princeton University\\
  \texttt{sjelassi@princeton.edu} \\
   \And
   Michael E. Sander\\
   Ecole Normale Supérieure \\
   \texttt{michael.sander@ens.fr} \\
   \And
    Yuanzhi Li \\
   Carnegie Mellon University \\
  \texttt{yuanzhil@andrew.cmu.edu} \\
}

\begin{document}

\maketitle

\begin{abstract}
  %Convolutional neural networks (CNNs) have so far been the standard model for computer vision tasks. Recently, Vision Transformers (ViTs) have achieved comparable or superior in these tasks. This empirical breakthrough is even more remarkable since a) ViTs process information in a non-intuitive way b) they a priori lack of visual inductive biases such as translation invariance and locally restricted receptive fields.  This raises a central question: how do ViTs learn a useful inductive bias for computer vision? We propose a simple structured classification dataset and a simplified ViT model to analyze the mechanism by which attention-based models learn a useful inductive bias using gradient descent. The key insight in our analysis is that ViTs learn a more general operation that convolutions that we refer to as \textit{patch association}: these models are able to associate spatially distant patches that \textcolor{red}{belong to the same feature}. Lastly, we prove that learning patch association helps to transfer in a sample-efficient way to downstream datasets that share the same structure as the pre-training one but do not share the same features.

%in order to generalize, an algorithm needs to learn how to map each patch to its underlying set.
Vision Transformers (ViTs) have achieved comparable or superior performance than Convolutional Neural Networks (CNNs) in computer vision. This empirical breakthrough is even more remarkable since, in contrast to CNNs, ViTs 
%discard spatial information by mixing patch embeddings and positional encodings and 
do not embed any visual inductive bias of spatial locality. Yet, recent works have shown that while minimizing their training loss, ViTs specifically learn spatially localized patterns. This raises a central question: how do ViTs learn these patterns by solely minimizing their training loss using gradient-based methods from \emph{random initialization}? In this paper, we provide some theoretical justification of this phenomenon. We propose a spatially structured dataset and a simplified ViT model. In this model, the attention matrix solely depends on the positional encodings. We call this mechanism the positional attention mechanism. On the theoretical side, we consider a binary classification task and show that while the learning problem admits multiple solutions that generalize, our model implicitly learns the spatial structure of the dataset while generalizing: we call this phenomenon patch association.
We prove that patch association helps to  sample-efficiently transfer to downstream datasets that share the same structure as the pre-training one but differ in the  features.
Lastly, we empirically verify that a ViT with positional attention performs similarly to the original one on CIFAR-10/100, SVHN and ImageNet.
\end{abstract}

\section{Introduction}

%\mic{that rely on attention mechanisms}
% \mic{poorly understood?}
%\paragraph{Problem setup.}
Transformers are deep learning models built on self-attention  \citep{vaswani2017attention}, and in the past several years they have increasingly formed the backbone for state-of-the-art models in domains ranging from Natural Language Processing (NLP)  \citep{vaswani2017attention,devlin2018bert} to computer vision \citep{dosovitskiy2020image}, reinforcement learning \citep{chen2021decision,janner2021offline}, program synthesis \citep{austin2021program} and symbolic tasks \citep{lample2019deep}. Beyond their remarkable performance, several works reported the ability of transformers to simultaneously minimize their training loss and  learn inductive biases tailored to specific datasets e.g.\ in computer vision \citep{raghu2021vision}, in NLP   \citep{brown2020language,warstadt2020can} or in mathematical reasoning \citep{wu2021lime}. 
%This capacity is not common as for instance multi-layer perceptrons (MLPs) are also able to fit computer vision datasets such as CIFAR-10 \citep{krizhevsky2009learning} without learning the structure of vision tasks.
%Despite their widespread use, the fundamental mechanisms behind their empirical success remain obscure. {Written by Boris}
%At a high level, the success of CNNs was often ascribed to their biologically inspired inductive biases \cite{fukushima} of spacial locality and weight sharing in convolutional filters. This seems intuitively reasonable: nearby pixels encode the presence of small scale features, whose patterns in turn determine more abstract features at longer and longer length scales. 
In this paper, we focus on computer vision where convolutions are considered to be an adequate and biologically plausible inductive bias since they capture local spatial information \citep{fukushima2003neocognitron} by imposing a sparse local connectivity pattern. This seems intuitively reasonable: nearby pixels encode the presence of small scale features, whose patterns in turn determine more abstract features at longer and longer length scales. Several seminal works \citep{cordonnier2019relationship,dosovitskiy2020image,raghu2021vision} \textit{empirically} show that although randomly initialized, the positional encodings in Vision transformers (ViTs) \cite{dosovitskiy2020image} actually learn this local connectivity: closer patches have more similar positional encodings, as shown in \autoref{fig:pos_intro}.
%We refer to this structure as spatially localized patterns.
A priori, learning such spatial structure is surprising. 
Indeed, in contrast to convolutional neural networks (CNNs), ViTs are not built with the inductive bias of local connectivity and weight sharing.
They start by replacing an image by a collection of $D$ patches $(\bm{X}_1,\dots,\bm{X}_D)\in\mathbb{R}^{d\times D}$, each of dimension $d$. While each $\bm{X}_i$ represents (an embedding of) a spatially localized portion of the original image, the relative positions of the patches $\bm{X}_i$ in the image are disregarded. Instead, relative spatial information is supplied through \textit{image-independent} positional encodings $\bm{P}=(\bm{p}_1,\dots,\bm{p}_D)\in\mathbb{R}^{d\times D}$. Unlike CNNs, each layer of a ViT then learns, via trainable self-attention, a non-local set of filters that non-linearly depend on both the values of all patches $\bm{X}_j$ and their positional encodings $\bm{p}_j$. 

\begin{figure}[tbp]
\vspace*{-1.5cm}
\begin{subfigure}{0.43\textwidth}%0.33
 \hspace{1cm}
 \includegraphics[width=.97\linewidth]{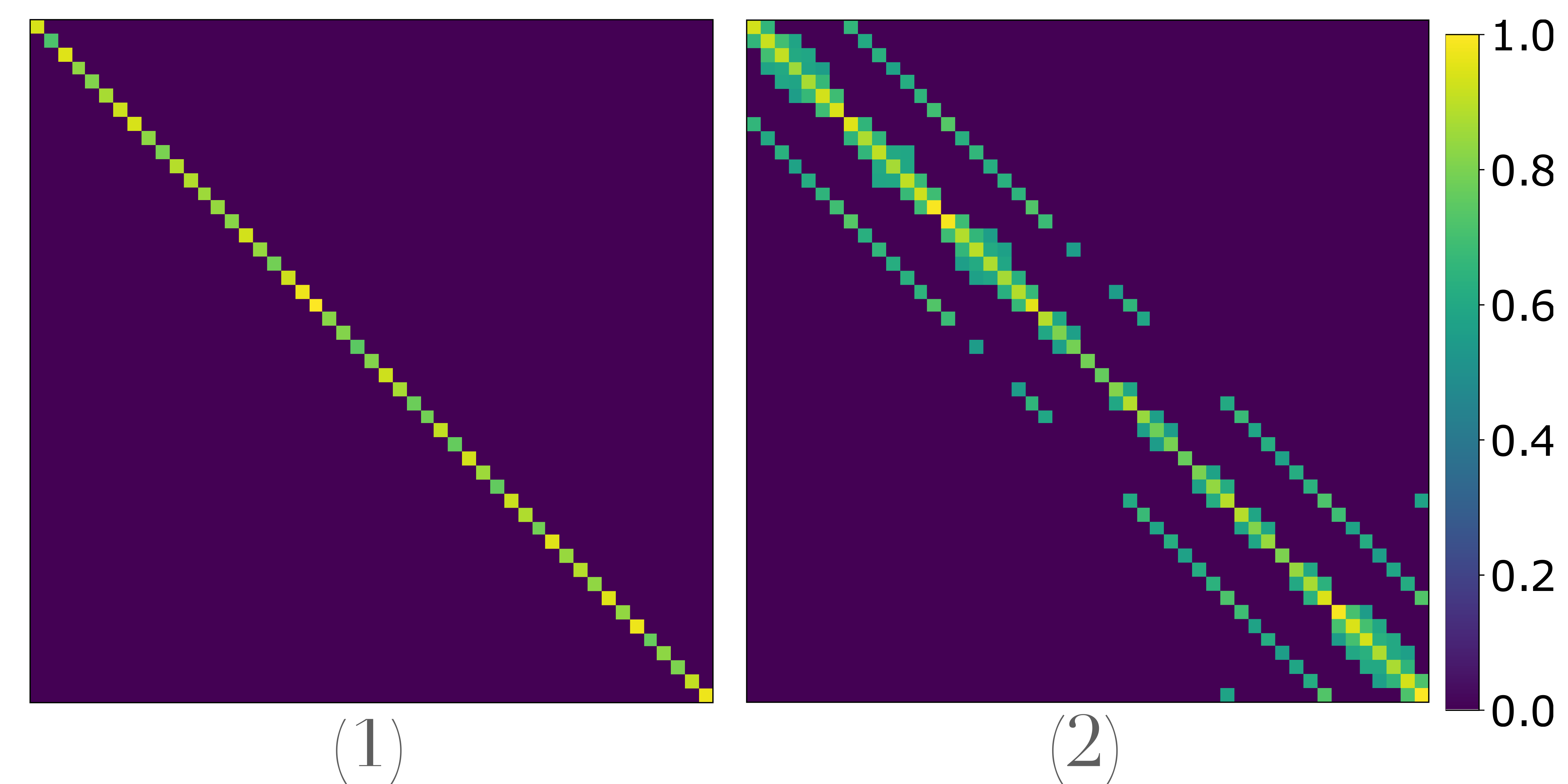}
 \caption{}\label{fig:pos_intro}
\end{subfigure}
\begin{subfigure}{0.5\textwidth}%0.33
\hspace*{1.cm}
%\centering
 \includegraphics[width=.85\linewidth]{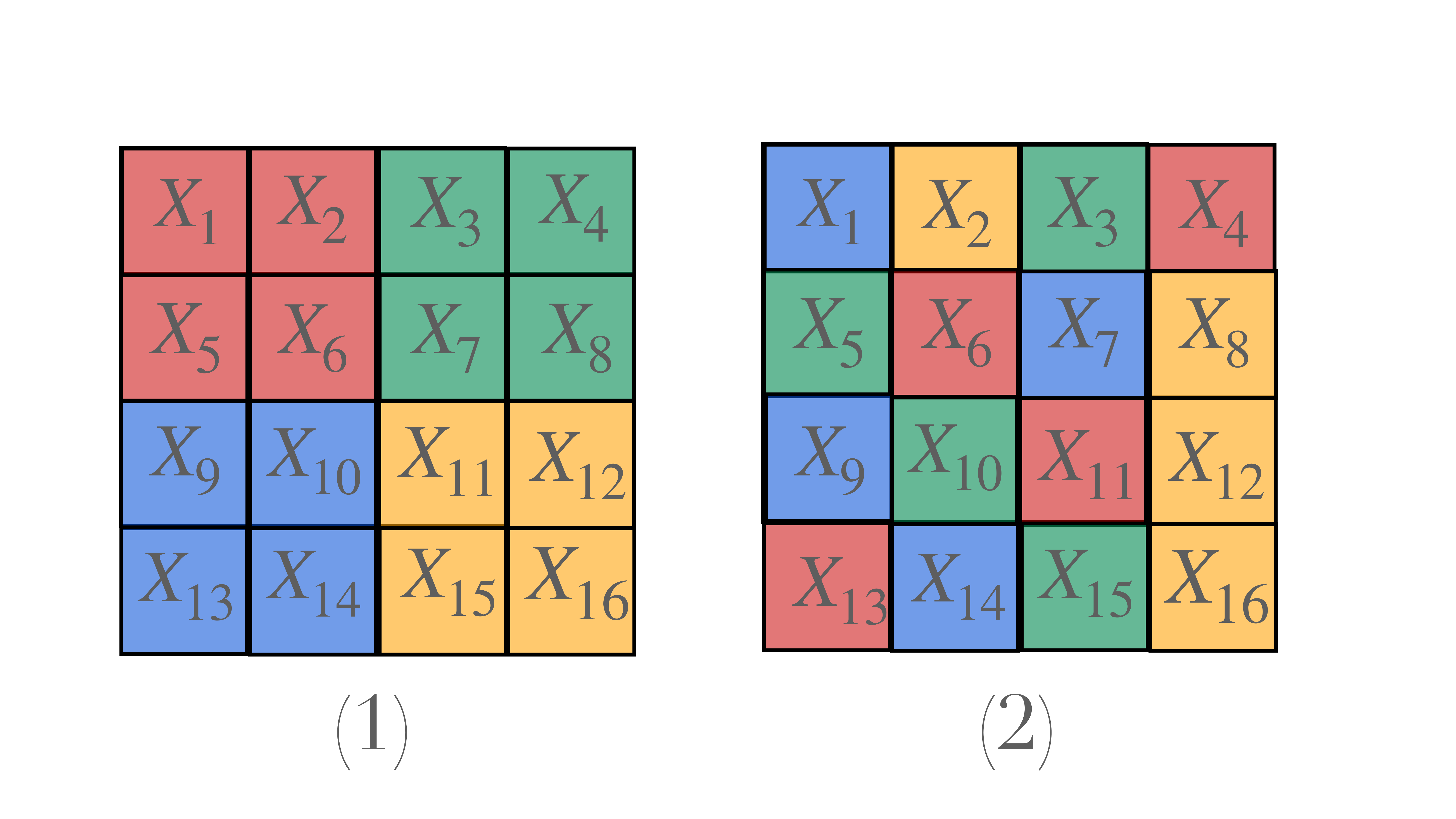}
\caption{}\label{fig:sets}
\end{subfigure}
\vspace{-3mm}
\caption{\small (a) Visualization of the positional encodings similarities $\bm{P}^\top \bm{P} = (\langle \bm{p}_i,\bm{p}_j\rangle)_{(i,j)\in [D]^2}$ at initialization (1) and after training on Imagenet (2) using a "ViT-small-patch32-224" \citep{dosovitskiy2020image}. %$\{\langle \bm{p}_i,\bm{p}_j\rangle\}_{j=1}^D$ for $i=60,80,130,150$  Originally, $\{\langle \bm{p}_i,\bm{p}_j\rangle\}_{j=1}^D$ is a vector of size $196$ and we display their reshape into an array of size $14\times 14.$ 
%Bright (resp.\ dark) entries correspond to high (resp.\ low) values in the plot. 
We normalise the values $\bm{P}^\top \bm{P}$ between $-1$ and $1$ and apply a threshold of $0.55$.
In contrast with the initial  arrays that are random, the final ones show local connectivity patterns: nearby patches have similar positional encodings. (b) Partition of the patches into sets $\mathcal{S}_{\ell}$ as in \autoref{fact}.  Squares in the same color belong to the same set $\mathcal{S}_{\ell}$. {We refer to (1) as a "spatially localized set" since all the elements in a $\mathcal{S}_{\ell}$ are spatially contiguous. This is the type of sets appearing in \autoref{fig:pos_intro} at the end of training.} \autoref{fact} also covers sets with non-contiguous elements as (2).  %is general as it does not assume any structure on the $\mathcal{S}_{\ell}$'s%Our framework includes convolutions as (1) and sets with non-contiguous elements as (2). 
}\label{tab:intro_fig}
\vspace*{-.5cm}
\end{figure}

\paragraph{Contributions.} The empirical observation of \autoref{fig:pos_intro} sets a central question: from a \textit{theoretical perspective}, how do ViTs manage to learn these local connectivity patterns by simply minimizing their training loss using gradient descent from random initialization?  
While it is known that attention can express local operations as convolution \citep{cordonnier2019relationship}, it remains unclear how ViTs learn it.
In this paper, we present a simple spatially-structured classification dataset for which it is sufficient (but not necessary) to learn the structure in order to generalize. We also present a simplified ViT model which we prove implicitly learns sparse spatial connectivity patterns when it minimizes its training loss via gradient descent (GD). We name this implicit bias \textit{patch association} (defined in \autoref{lem:ind_bias_tf}). We prove that our ViT model leverages this bias to generalize. More precisely, we make the following contributions:

\begin{itemize}[leftmargin=*, itemsep=1pt, topsep=1pt, parsep=1pt]
    \item[--] In \autoref{sec:patchassoc}, we formally define the concept of performing patch association, which refer to the ability of learning spatial connectivity patterns on a dataset.
    \item[--] In \autoref{sec:setting}, we introduce a structured classification dataset and a simplified ViT model. This model is simplified in the sense that its attention matrix only depends on the positional encodings. We then present the learning problems we are interested in: empirical risk (realistic setting) and population risk (idealized setting) minimization for binary classification. 
    \item[--] In \autoref{sec:theory}, we prove that a one-layer single-head ViT model trained with gradient descent on our synthetic dataset performs patch association and generalizes, in the idealized (\autoref{thm:ideal}) and realistic (\autoref{thm:real}) settings. We present a detailed proof, based on invariance and symmetries of coefficients in the attention matrix throughout the learning process.
    \item[--] In \autoref{sec:transfer}, we show (\autoref{thm:transfer}) that after pre-training in our synthetic dataset, our model can be  sample-efficiently fine-tuned to transfer to a downstream dataset that shares the same structure as the source dataset (and may have different features). %Such transfer requires dramatically more samples for any other algorithm.
    \item[--]  On the experimental side, we validate in \autoref{sec:num_exp} that ViTs learn spatial structure in images from the CIFAR-100 dataset, even when the pixels of the images are permuted. This result validates that, in contrast to CNNs, ViTs learn a more general form of spatial structure that is not limited to local patterns (\autoref{tab:perm_fig}). We finally show that our ViT model {--where the attention matrix only depends on the positional encodings--} is competitive with the vanilla ViT on the ImageNet, CIFAR-10/100 and SVHNs datasets (\autoref{tab:imgnet_fig} and  \autoref{tab:barplots_fig}).
\end{itemize}

\paragraph{Notation.} We use lower case letters for scalars, lower case bold for vectors and upper case bold for matrices.  Given an integer $D$, we define $[D]=\{1,\dots,D\}.$ Any statement made "with high probability" holds with probability at least $1-1/\mathrm{poly}(d).$ Given a vector $\bm{a}\in\mathbb{R}^d$ and $k\leq d$, we define $\mathrm{Top}_k\{a_j\}_{j=1}^d=\{a_{i_1},\dots,a_{i_k}\}$ where $a_{i_1},\dots,a_{i_k}$ are the $k$-largest elements. For a function $F$ that implicitly depend on parameters $\bm{A}$ and $\bm{v}$, we often write $F_{\bm{A},\bm{v}}$ to highlight its parameters. We use the asymptotic complexity notations when defining the different constants. 
%\mic{Here we should also define the notations for the asymptotic comparison, such as $\omega, \Theta; \Omega, ...$}.

\iffalse
\paragraph{Remark:} The goal of our paper is not to fully characterize when can ViTs learn convolution patterns. We are simply making a preliminary step to \emph{theoretically showing} that there exists simple (yet natural) data distribution where simple ViTs can \emph{provably} learn convolution-like patterns via minimizing the training objective to match the (binary classification) labels using gradient descent on this distribution, both sample and running time efficiently. 
\fi

\section*{Related work}
%\paragraph{Transformers for computer vision.} \mic{(maybe put in the intro)}
%Self-attention based deep learning models inspired by the Transformer \citep{vaswani2017attention} have recently achieved state of the art performance in computer vision. In particular, Vision Transformers (ViTs) \citep{dosovitskiy2020image} rely entirely on attention mechanisms and have achieved impressive results on several supervised learning tasks \citep{touvron2022deit, touvron2021training}. Thus, ViTs appear as an alternative to Convolutional Neural Networks (CNNs).

%\textcolor{red}{Learn features. Ability of learning features. Learned sparsity prior. Feature learning. Patch association as feature learning. Identify the subset of features to make classification. Subset depends on X. If you knew the correct set, doing the is simple. However, too much noise so it is difficult. Sparse regression. Briefly bring the feature learning} or include downsampling
%to be more like traditional pyramid-shaped convolutional networks \citep{wang2021pyramid}.

%\paragraph{Attention mechanism.}

\paragraph{CNNs and ViTs.} Many computer vision architectures can be considered as a form of hybridization between Transformers and CNNs. For example, DeTR \citep{carion2020end} use a CNN to generate features that are fed to a Transformer. \citep{d2021convit} show that self-attention can be initialized or regularized to behave like a convolution and \citep{dai2021coatnet,guo2021cmt} add
convolution operations to Transformers. Conversely, \citep{bello2019attention,ramachandran2019stand,bello2021lambdanetworks} introduce
self-attention or attention-like operations to supplement or replace convolution in ResNet-like
models. %These works are orthogonal to our approach that is mainly theoretical. 
In contrast, our paper does not consider any form of hybridization with CNN, but rather a simplification of the original ViT to  explain how ViTs learn spatially structured patterns using GD. 
%consider a setting in which a simplified ViT model trained with gradient descent is shown to recover the
%underlying structure  of the labeling function – including (disjoint) convolutions of arbitrary shape.

\paragraph{Empirical understanding of ViTs.} A long line of work consists in analyzing the properties of ViTs, such as robustness
\citep{bhojanapalli2021understanding,paul2021vision,naseer2021intriguing} or the effect of self-supervision \citep{caron2021emerging,chen2021empirical}. Closer to our work, some papers investigate why ViTs perform so well. \cite{raghu2021vision} compare the representations of ViTs and CNNs and  \cite{melas2021you,trockman2022patches} argue that the patch embeddings could explain the performance of ViTs. We empirically show in \autoref{sec:num_exp} that applying the attention matrices to the positional encodings -- which contains the structure of the dataset -- approximately recovers the baselines. Hence, our work rather suggests that the structural learning performed by the attention matrices may explain the success of ViTs.

\paragraph{Theory for attention models.} Early theoretical works have focused on the expressivity of attention. \citep{vuckovic2020mathematical,edelman2021inductive} addressed this question in the context of  self-attention blocks and  \citep{dehghani2018universal,wei2021statistically,hron2020infinite} for Transformers. On the optimization side, \citep{zhang2020adaptive} investigate the role of adaptive methods  in attention models and \citep{snell2021approximating} analyze the dynamics of a single-head attention head to approximate the learning of a Seq2Seq architecture. In our work, we also consider a single-head ViT trained with gradient descent and exhibit a setting where it provably learns convolution-like patterns and generalizes. 

\paragraph{Algorithmic regularization.} The question we address concerns algorithmic regularization which characterizes the generalization of an optimization algorithm when multiple global solutions exist in over-parametrized  models. This regularization arises in deep learning mainly due to the \textit{non-convexity} of the objective function. Indeed, this latter potentially creates multiple global minima scattered in the space that vastly differ in terms of generalization.  Algorithmic regularization appears in  binary classification  \citep{soudry2018implicit,lyu2019gradient,chizat2020implicit}, matrix factorization \citep{gunasekar2018implicit,arora2019implicit}, convolutional neural networks \citep{gunasekar2018implicit,jagadeesan2022inductive}, generative adversarial networks \citep{allen2021forward}, contrastive learning \citep{pmlr-v139-wen21c} and mixture of experts \citep{chen2022towards}.  Algorithmic regularization is induced by and depends on many factors such as learning rate and batch size \citep{goyal2017accurate,hoffer2017train,keskar2016large,smith2018dont,li2019towards}, initialization \cite{allen2020towards}, momentum \citep{v22}, adaptive step-size \citep{kingma2014adam,neyshabur2015path,daniely2017sgd,wilson2017marginal,zou2021understanding,jelassi2022adam}, batch normalization \citep{arora2018theoretical,hoffer2019norm,ioffe2015batch} and dropout \citep{srivastava2014dropout,wei2020implicit}. However, all these works consider the case of feed-forward neural networks which does not apply to ViTs.

\vspace{-3mm}
\section{Defining patch association}\label{sec:patchassoc}
\vspace{-2mm}

{The goal of this section is to formalize the way ViTs learn sparse spatial connectivity patterns. We thus introduce the concept of performing patch association for a spatially structured dataset.}

% Written by Michael before
%The goal of this section is to formally define the concept of performing patch association for a spatially structured dataset.

\begin{definition}[Data distribution with spatial structure]\label{fact}\label{ass:data_dist} 
Let $\mathcal{D}$ be a  distribution over $\mathbb{R}^{d\times D}\times\{-1,1\}$ where each patch $\bm{X} =(\bm{X}_1,\dots,\bm{X}_D)\in\mathbb{R}^{d\times D}$ has label $y\in \{-1,1\}$.  We say that $\mathcal{D}$ is spatially structured if
\begin{itemize}[leftmargin=*, itemsep=1pt, topsep=1pt, parsep=1pt]
    \item[--] there exists a partition of $[D]$  into $L$ disjoint subsets i.e.\ $[D]=\bigcup_{\ell=1}^L \mathcal{S}_{\ell}$ with $\mathcal{S}_{\ell} \subsetneq D$ and $|\mathcal{S}_{\ell}|=C$. 
   \item[--]  there exists a labeling function $f^*$ satisfying $\mathbb{P}[yf^*(\bm{X})>0]=1 - d^{-\omega(1)}$  and,
\begin{align}\label{eq:label_fn}
   f^*(\bm{X}) := \sum_{\ell \in [L]} \phi( \big(\bm{X}_i \big)_{i \in \mathcal{S}_{\ell}}) , \quad\text{where } \phi\colon \mathbb{R}^{ d\times C} \rightarrow \mathbb{R} \text{ is an arbitrary function}.  
\end{align}
\end{itemize}
\end{definition}
 %$\mathcal{S}_1=\{1,\dots,C\},\dots,\mathcal{S}_{\lfloor D/C\rfloor}=\{D-C,\dots,D\}$.
 %We depict this example in  \autoref{fig:sets}-(1) in the case $D=16$ and $C=4$.
 \begin{wrapfigure}[13]{r}{0.45\textwidth}
 
 \vspace{-1cm}
 
\includegraphics[width=0.45\textwidth]{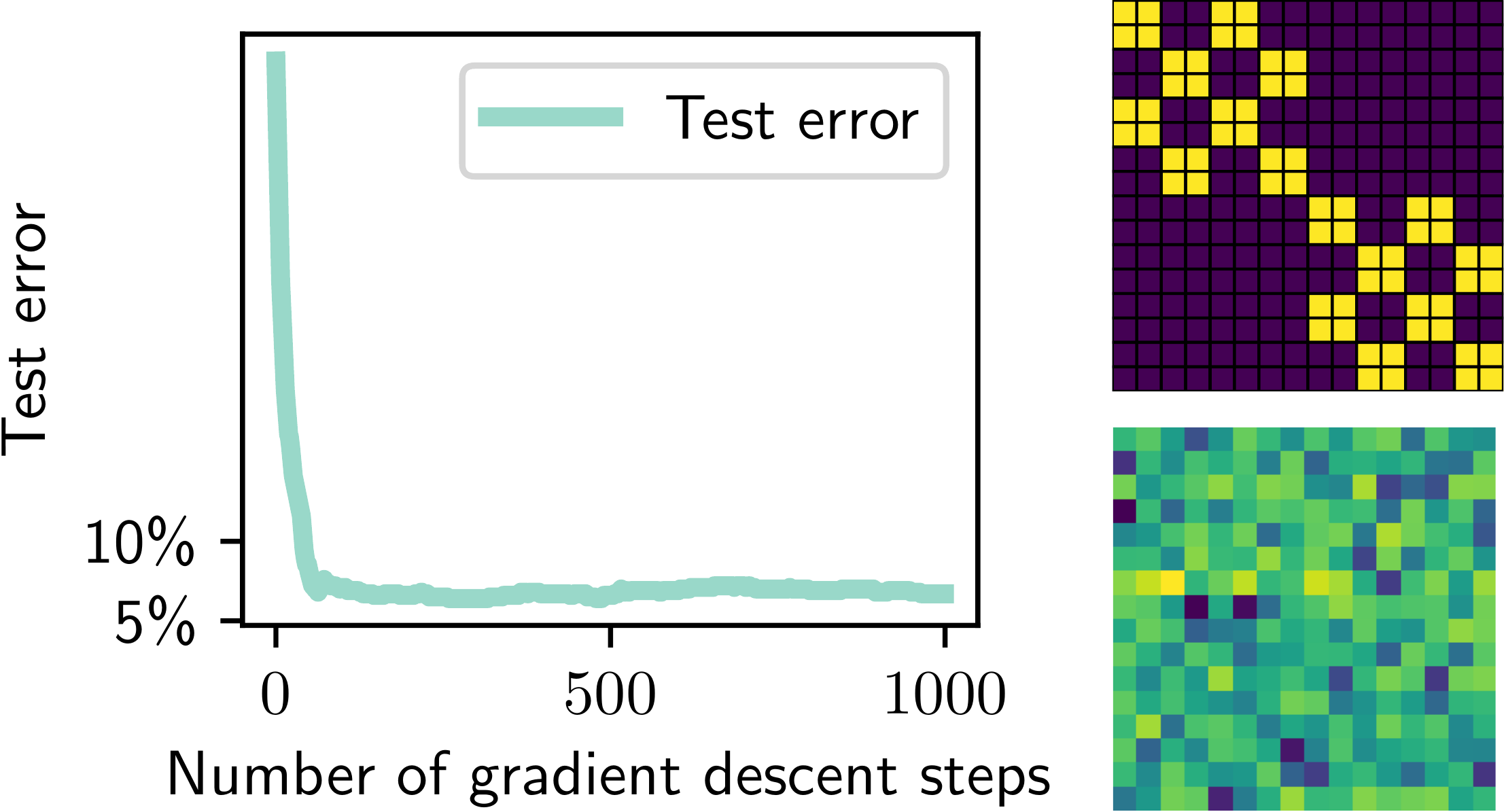}
\caption{\small{Left: Test error of the ViT on the convolution structured dataset. Upper Right: Grid displaying the input patches. Yellow squares represent spatially localized sets $\mathcal{S}_{\ell}$.  Those sets are taken into account when computing the convolutional function $f^*$.}  Lower Right: \mbox{Learnt $\bm{P}^{\top}\bm{P}$ looks random compared to upper one.}}\label{fig:conv_gaussian}
\end{wrapfigure}
\paragraph{Examples.} A particular case for the sets $\mathcal{S}_{\ell}$'s is the one of spatially localized sets as in \autoref{fig:sets}-(1). {In this case, we have $D=16$, $C=4$ and $\mathcal{S}_1=\{1,2,5,6\}, \; \mathcal{S}_2=\{3,4,7,8\},\; \mathcal{S}_3=\{9,10,13,14\},\; \mathcal{S}_4=\{11,12,15,16\}.$}
 We emphasize that \autoref{fact} is not limited to spatially localized sets and also covers non-contiguous sets as \autoref{fig:sets}-(2). 
%\emph{Our final results actually do not rely on the particular structure inside each $\mathcal{S}_{\ell}$.}

\paragraph{Labelling function}
%Since we aim to show that ViTs learn the sets $\mathcal{S}_{\ell}$ while fitting the labels $y$, we need to assume the existence of a labeling function $f^*$ that respects the partition in \autoref{fact}.
\autoref{ass:data_dist} states that there exists a labelling function that preserves the underlying structure by applying the same function $\phi$ to each $\mathcal{S}_{\ell}$ as in \eqref{eq:label_fn}. 
For instance, when the sets  $\mathcal{S}_{\ell}$'s  are spatially localized, $f^*$ can be a one-hidden layer convolutional network. In this paper, we are interested in \textit{patch association} which refers to the ability of an algorithm to identify the sets $\mathcal{S}_{\ell}$'s, and is formally defined as follow.
%Following ~\citep{cordonnier2019relationship,dosovitskiy2020image,raghu2021vision}, we investigate whether ViTs can learn patch association using its positional encodings. 

\begin{definition}[Patch association for ViTs]\label{lem:ind_bias_tf} Let $\mathcal{D}$ be as in \autoref{ass:data_dist}. Let $\mathcal{M}\colon\mathbb{R}^{d\times D}\rightarrow\{-1,1\}$ be a transformer and $\bm{P}^{(\mathcal{M})}$ its positional encodings matrix. We say that $\mathcal{M}$ performs patch association on $\mathcal{D}$ if for all $\ell\in[L]$ and $i\in\mathcal{S}_{\ell}$, we have $ \mathrm{Top}_{C}\; \{\langle \bm{p}_i^{(\mathcal{M})},\bm{p}_j^{(\mathcal{M})}\rangle\}_{j=1}^D = \mathcal{S}_{\ell}.$ %after training.
\end{definition}
\autoref{lem:ind_bias_tf} states that patch association is learned when for a given $i\in\mathcal{S}_{\ell},$ its positional encoding mainly attends those of $j$ such that $i,j\in\mathcal{S}_{\ell}$. In this way, the transformer groups the $\bm{X}_i$ according to $\mathcal{S}_{\ell}$ just like the true labeling function. 
\autoref{lem:ind_bias_tf} formally describes the empirical findings in \autoref{fig:pos_intro}-(2), where nearby patches have similar positional encodings. 
%and thus belong to the same sets $\mathcal{S}_{\ell}$'s.}
%\textcolor{blue}{Indeed, if we set the  $\mathcal{S}_{\ell}$un peut dur de re's in \autoref{lem:ind_bias_tf} to be the local connectivity patterns  (colored in green), then \autoref{fig:pos_intro}-(2) is an illustration of patch association.} 
A natural question is then: would ViTs really learn those $\mathcal{S}_{\ell}$ after training to match the labeling function $f^*$? {Without further assumptions on the data distribution,  we next show that the answer is no.}
\paragraph{ViTs do not always learn  patch association under Assumption 1.} %via minimizing the training objective.}
%\autoref{fig:pos_intro} shows that ViTs learn spatially localized patterns on ImageNet. This raises a question: while minimizing their loss, do ViTs always learn the structure of $\mathcal{S}_{\ell}$ for any data distribution satisfying \autoref{ass:data_dist}?
We  give a negative answer through the following synthetic experiment.
Consider the case where all the patches $\bm{X}_j$ are i.i.d.\ standard Gaussian and $f^*$ is a one-hidden layer CNN with cubic activation. The label $y$ of any $\bm{X}$ is then given by $y = \mathrm{sign}(f^*(\bm{X}))$. As shown in \autoref{fig:conv_gaussian}, one-layer ViT reaches small test error on the binary classification task. However,  $\bm{P}^{\top}\bm{P}$ does not match the convolution pattern encoded in $f^*$. This is not surprising, since the data distribution $\mathcal{D}$ is Gaussian, and thus lacks spatial structure. Thus, in order to prove that ViTs learn patch association, we need additional assumptions on $\mathcal{D}$, which we discuss in the next section.

\section{Setting to learn patch association}\label{sec:setting}

In this section, we introduce our theoretical setting to analyze how ViTs learn patch association. We first  define our binary classification dataset and finally present the ViT model we use to classify it.

%\paragraph{Data distribution.} 
%The binary classification problem analyzed in this paper is defined as follows. 
%We focus on the labeling function of form:
%$f^*(\bm{X}) = \sum_{\ell \in [L]} \mathrm{Threshold}_{0.1C}( \sum_{i \in \mathcal{S}_{\ell}} \langle \bm{w}^*, \bm{X}_i \rangle )$, where $\mathrm{Threshold}_{C} (z) = z $ if $|z| > {C}$ and $0$ otherwise. We consider the data distribution:

\begin{assumption}
[Data distribution with specific spatial structure]\label{def:datadist}
Let $\mathcal{D}$ be a distribution as in \autoref{ass:data_dist} and $\bm{w}^*\in\mathbb{R}^d$  be an underlying feature. 
We suppose that each data-point $\bm{X}$ is defined as follow
\begin{itemize}[leftmargin=*, itemsep=1pt, topsep=1pt, parsep=1pt]
    \item[--] Uniformly sample an index $\ell(\bm{X})$ from $[L]$ and for  $j\in\mathcal{S}_{\ell(\bm{X})}$, $\bm{X}_j=y\bm{w}^*+\bm{\xi}_j$, where $y\bm{w}^*$ is the informative feature and  $\bm{\xi}_j\overset{i.i.d.}{\sim}\mathcal{N}(0,\sigma^2(\mathbf{I}_D-\bm{w}^*\bm{w}^{*\;\top}))$ (\textbf{signal set}).
    %$\bm{X}_j=\rho_j\bm{w}^*+\bm{\xi}_j$, where $\rho_j = y $ w.p. $1- \rho$ and $-y$ otherwise (for $\rho \in (0.6, 1]$ as a constant), and $\bm{\xi}_j\overset{i.i.d.}{\sim}\mathcal{N}(0,\sigma^2\mathbf{I}_D)$.
    \vspace{.2cm}
    \item[--] For $\ell\in[L]\backslash\{\ell(\bm{X})\}$ and $j\in\mathcal{S}_{\ell}$, $\bm{X}_j=\delta_j \bm{w}^*+\bm{\xi}_j$, where  $\delta_j=1$ with probability $q/2$, $-1$ with  same probability and $0$ otherwise, and  $\bm{\xi}_j\overset{i.i.d.}{\sim}\mathcal{N}(0,\sigma^2(\mathbf{I}_D-\bm{w}^*\bm{w}^{*\;\top}))$ (\textbf{random sets}). 
\end{itemize}
\vspace{.1cm}
\end{assumption}
 %\textcolor{red}{Our data distribution is a  model to describe the patch association phenomenon in \autoref{tab:intro_fig}-(a). Indeed, we can divide each image in sets $\mathcal{S}_{\ell}$'s that contain adjacent patches. For instance, the patches in yellow in \autoref{tab:intro_fig}-(a) belong to the same $\mathcal{S}_{\ell}$. Thus, in our setting, proving patch association amounts to show that the positional encodings in the same set $\mathcal{S}_{\ell}$ are highly correlated.} We make the following assumptions on the parameters of $\mathcal{D}.$
 \begin{wrapfigure}[20]{r}{0.35\textwidth}
\vspace*{-.5cm}
\hspace*{-.9cm}
\includegraphics[width=1.3\linewidth]{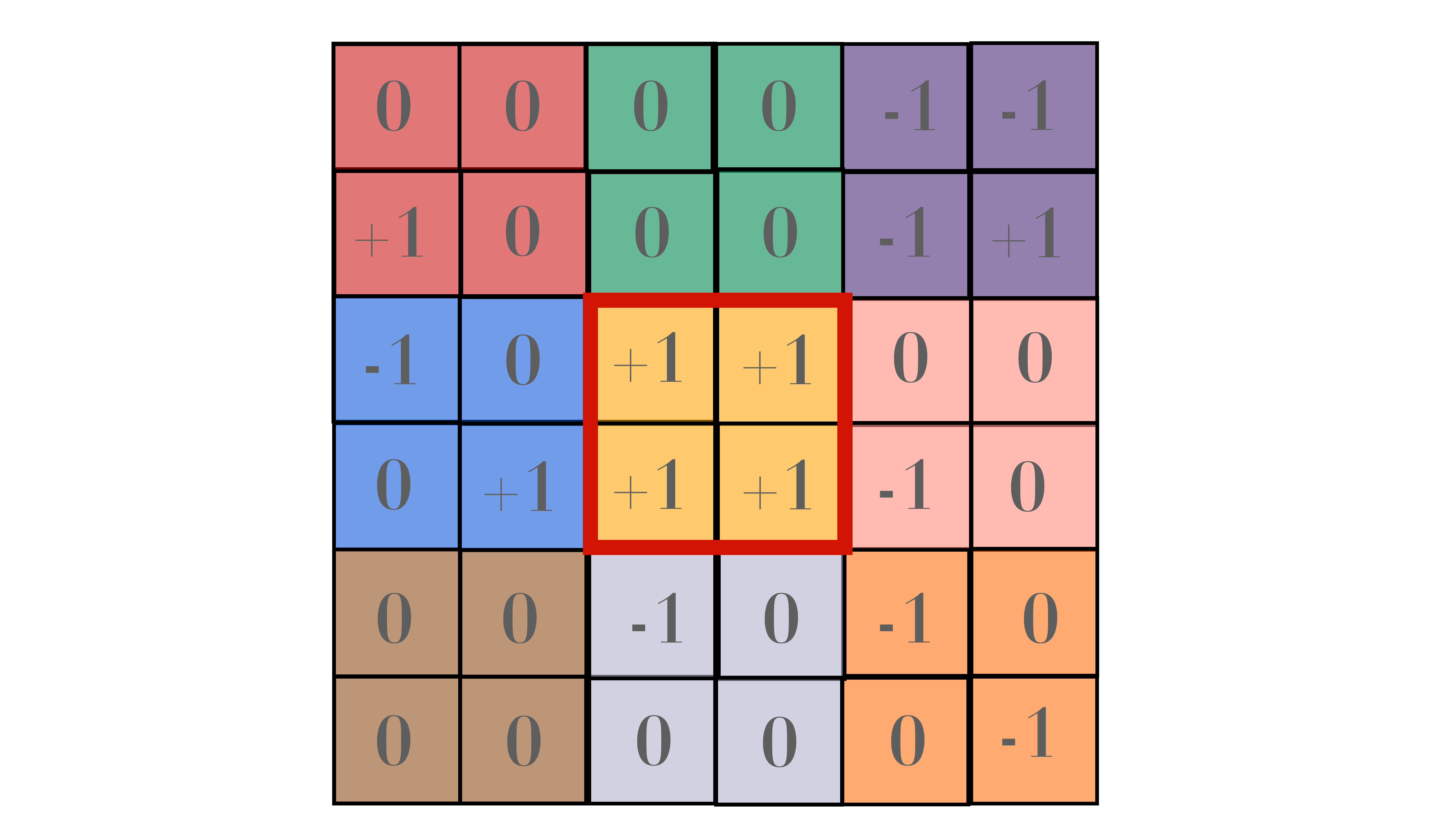}
\caption{\small Visualization of a data-point $\bm{X}$ in $\mathcal{D}$ when the $\mathcal{S}_{\ell}$'s are spatially localized. Each square depicts a patch $\bm{X}_j$ and squares of the same color belong to the same set $\mathcal{S}_{\ell}.$ "0" indicates that the patch does not have a feature, "1" stands for feature $1\cdot\bm{w}^*$ and "-1"  for feature $-1\cdot\bm{w}^*$. The large red square depicts the signal set $\ell(\bm{X}).$ Although there are more "-1"'s than "+1"'s, 
the label of $\bm{X}$ is $+1$ since there are only "+1"'s inside the signal set. }\label{fig:dataset}
\vspace{10cm}
\end{wrapfigure}
To keep the analysis simple, the noisy patches are sampled from the orthogonal complement of $\bm{w}^*.$ Note that $\mathcal{D}$ admits the labeling function $f^*(\bm{X}) = \sum_{\ell \in [L]} \mathrm{Threshold}_{0.9C}( \sum_{i \in \mathcal{S}_{\ell}} \langle \bm{w}^*, \bm{X}_i \rangle )$, where $\mathrm{Threshold}_{C} (z) = z $ if $|z| > {C}$ and $0$ otherwise. %\textcolor{red}{Yuanzhi: is there a simpler labelling function?}

We sketch a data-point of $\mathcal{D}$ in \autoref{fig:dataset}. Our dataset can be viewed as an extreme simplification of real-world image datasets where there is a set of adjacent patches that contain a useful feature (e.g.\ the nose of a dog) and many patches that have uninformative or spurious features e.g.\ the background of the image. We make the following assumption on the parameters of the data distribution.

\begin{assumption}\label{ass:paramdatadist}
We suppose that  
  $d=\mathrm{poly}(D)$, $C=\mathrm{polylog}(d)$, $q=\mathrm{poly}(C)/D$, $\|\bm{w}^*\|_2=1$ and $\sigma^2=1/d$. This implies $C\ll D$ and $q\ll1.$
\end{assumption}

\vspace{-.3cm}

\autoref{ass:paramdatadist} may be justified by considering a "ViT-base-patch16-224" model \cite{dosovitskiy2020image} on ImageNet. In this case, $d=384$, $D=196$. $\sigma$ is set to have $\|\bm{\xi}_j\|_2\approx \|\bm{w}^*\|_2$.  
$q$ is chosen so that there are more spurious features than informative ones (low signal-to-noise regime) which makes the data \emph{non-linearly separable}.
Our dataset is non-trivial to learn since generalized linear networks fail to generalize, as shown in the next theorem (see Appendix \ref{sec:data_distrib} for a proof).

\begin{restatable}{theorem}{linearmodel} \label{lem:linear_models}
Let $\mathcal{D}$ be as in \autoref{def:datadist}. 
Let    $g(\bm{X})= \phi \left(\sum_{j=1}^D \langle \bm{w}_j,\bm{X}_j\rangle \right)$ be a generalized linear model. Then, $g$ does not fit the labeling function i.e.\  $\mathbb{P}[f^*(\bm{X})g(\bm{X})\leq 0]\geq 1/8$.
\end{restatable}
%We refer the reader to \autoref{sec:useful} for a proof of \autoref{thm:linearconv}. Intuitively, we set the parameters of $\mathcal{D}$ such that the signal-to-noise ratio is low (see \autoref{ass:paramdatadist} for details) which leads $g$ to poor generalization.
Intuitively, $g$ fails to generalize because it does not have any knowledge on the underlying partition and the number of random sets is much higher than those with signal.  Thus, a model must  have a minimal knowledge about the $\mathcal{S}_{\ell}$'s  in order to generalize. In addition, the following
\autoref{thm:nopatchlearning} states the existence of a transformer that generalizes without learning spatial structure (see Appendix \ref{sec:data_distrib} for a proof), thus showing that the learning process has a priori no straightforward reason to lead to patch association.
\begin{restatable}{theorem}{spurious}\label{thm:nopatchlearning}
Let $\mathcal{D}$ be defined as in \autoref{def:datadist}. There exists a (one-layer) transformer $\mathcal{M}$ so that $\mathbb{P}[f^*(\bm{X})\mathcal{M}(\bm{X})\leq 0] =  d^{-\omega(1)}$ but for all $\ell \in [L]$, $i \in \mathcal{S}_{\ell}$, $ \mathrm{Top}_{C}\; \{\langle \bm{p}_i^{(\mathcal{M})},\bm{p}_j^{(\mathcal{M})}\rangle\}_{j=1}^D \cap \mathcal{S}_{\ell} = \emptyset$.
\end{restatable}

\paragraph{Simplified ViT model.} We now define our simplified ViT model for which we show in \autoref{sec:theory} that it implicitly learns patch association via minimizing its training objective.  We first remind the self-attention mechanism that is ubiquitously used in transformers. 

\begin{definition}[Self-attention \citep{bahdanau2014neural,vaswani2017attention}] The attention mechanism \citep{bahdanau2014neural,vaswani2017attention} in the single-head case is defined as follow. Let $\bm{X}\in\mathbb{R}^{d\times D}$ a data point and $\bm{P}\in\mathbb{R}^{d\times D}$ its positional encoding. The self-attention mechanism computes
\begin{enumerate}[leftmargin=*, itemsep=1pt, topsep=1pt, parsep=1pt]
    \item the sum of patches and positional encodings i.e.\ ${\large \pmb{\mathpzc{X}}}=\bm{X}+\bm{P}.$ %\in\mathbb{R}^{d\times (D+1)}$.
    \item the attention matrix $\bm{A}=\bm{Q}\bm{K}^{\top}$  where $\bm{Q}={\large \pmb{\mathpzc{X}}}^{\top}\bm{W}_{\bm{Q}}$,  $\bm{K}={\large \pmb{\mathpzc{X}}}^{\top}\bm{W}_{\bm{K}}$,\; $\bm{W}_{\bm{Q}},\bm{W}_{\bm{K}}\in\mathbb{R}^{d\times d}$.
    \item the score matrix $\bm{S}\in\mathbb{R}^{D\times D}$ with coefficients $S_{i, j} = \exp(A_{i,j}/\sqrt{d})/ \sum_{r=1}^D \exp(A_{i,r}/\sqrt{d})$.
    \item the matrix $\bm{V}={\large \pmb{\mathpzc{X}}}^{\top}\bm{W}_{\bm{V}}$, where $\bm{W}_{\bm{V}}\in\mathbb{R}^{d\times d}.$
\end{enumerate}
It finally outputs $\mathrm{SA}((\bm{X};\bm{P}))=\bm{S}\bm{V}\in\mathbb{R}^{d\times D}.$
\end{definition}

{In this paper, our ViT model relies on a different attention mechanism --the "positional attention"-- that we define as follows. }
%Given its highly non-linear nature with respect to the input,
%directly analyzing self-attention is difficult. Instead, our model only rely on positional attention which is defined as follow:

\begin{definition}[Positional attention]\label{def:pos_attn}
Let $\bm{X}\in\mathbb{R}^{d\times D}$ and $\bm{P}\in\mathbb{R}^{d\times D}$ the positional encoding. The positional attention mechanism takes as input the pair $(\bm{X};\bm{P})$ and computes:
\begin{enumerate}[leftmargin=*, itemsep=1pt, topsep=1pt, parsep=1pt]
    \item the attention matrix $\bm{A}=\bm{Q}\bm{K}^{\top}$  where $\bm{Q}=\bm{P}^{\top}\bm{W}_{\bm{Q}}$,  $\bm{K}=\bm{P}^{\top}\bm{W}_{\bm{K}}$ and $\bm{W}_{\bm{Q}},\bm{W}_{\bm{K}}\in\mathbb{R}^{d\times d}$.
        \item the score matrix $\bm{S}\in\mathbb{R}^{D\times D}$ with coefficients $S_{i, j} = \exp(A_{i,j}/\sqrt{d})/ \sum_{r=1}^D \exp(A_{i,r}/\sqrt{d})$.
    \item the matrix $\bm{V}=\bm{X}^{\top}\bm{W}_{\bm{V}}$, where $\bm{W}_{\bm{V}}\in\mathbb{R}^{d\times d}.$
\end{enumerate}
It outputs $\mathrm{PA}((\bm{X};\bm{P}))=\bm{S}\bm{V}.$
\end{definition}
Positional attention \textit{isolates} positional encoding $\bm{P}$ from data $\bm{X}$: $\bm{A}$ encodes the dynamics of $\bm{P}$ and tracks whether patch association is learned. $\bm{V}$ encodes the data-dependent part and monitors whether the feature is learned. {Indeed, given its highly non-linear nature with respect to the input, directly analyzing self-attention is difficult.} 
Yet, positional attention is similar to self-attention. As this latter, positional attention is also permutation-invariant and processes all tokens simultaneously. Besides, positional attention also computes a score matrix between the different tokens. This similarity matrix is also normalized in a sparse manner with the Softmax operator. The only aspect that positional attention misses from self-attention is the fact that $\bm{S}$ does not depend on the input. Nevertheless, we empirically show that our positional attention model competes  with self-attention  in \autoref{sec:num_exp}. Lastly, we make the following simplification in the parameters to ease our analysis.

%monitor the learning of patch association
\begin{simplification}\label{ass:transformer}\label{ass:trainable} 
    In the positional attention mechanism, we set $d=D$,  $\bm{W}_{\bm{K}}=\bm{I}_D$ and $\bm{W}_{\bm{Q}}=\bm{I}_D$ which implies $\bm{A}=\bm{P}^{\top}\bm{P}.$ %Without loss of generality, we train all $A_{i,j}$ for $i\neq j$ and leave the diagonals of $\bm{A}$ fixed. \\
    We set $\bm{W}_{\bm{V}}=[\bm{v},\dots,\bm{v}]\in\mathbb{R}^{d\times D}$ where $\bm{v}\in\mathbb{R}^d.$ Finally, we set $\bm{A}$ and $\bm{v}$ as trainable parameters. Besides, without loss of generality, we train all $A_{i,j}$ for $i\neq j$ and leave the diagonals of $\bm{A}$ fixed.
\end{simplification}
In \autoref{ass:transformer}, we set $\bm{W}_{\bm{K}}$ and $\bm{W}_{\bm{Q}}$ to the identity so that $\bm{A}=\bm{P}^{\top}\bm{P}.$ This Gram matrix encodes the spatial patterns learned by the ViT as shown in \autoref{fig:pos_intro}. Besides, since fitting the labeling function requires to learn one feature $\bm{w}^*$, it is sufficient to parameterize $\bm{W}_{\bm{V}}$ with a  vector $\bm{v}$. Also, although $\bm{A}=\bm{P}^{\top}\bm{P}$ and $\bm{P}$ is trainable, we choose for simplicity to only optimize over $\bm{A}.$ Besides, we leave the $A_{i,i}$'s fixed because Softmax is invariant under the uniform shift of
the input. Under \autoref{ass:transformer}, our simplified ViT model is then a two attention layer with a single head:
\begin{equation}\label{eq:transformer}\tag{T}
F(\bm{X}) =\sum_{i=1}^D \sigma\bigg(D\sum_{j=1}^D S_{i,j}\langle \bm{v},\bm{X}_j\rangle\bigg) \quad \mathrm{with} \quad  S_{i, j} = \exp(A_{i,j}/\sqrt{d})/ \sum_{r=1}^D \exp(A_{i,r}/\sqrt{d}),
\end{equation}
where $\sigma$ is an activation function.  Since we aim to the simplest ViT model, we opt for a polynomial activation i.e.\  $\sigma(x)=x^p+\nu x$ where $p\geq 3$ is an odd integer and $\nu=1/\mathrm{poly}(d)$. Note that this choice of polynomial activation is common in the deep learning theory literature   -- see e.g.\ \citep{li2018algorithmic,allen2020towards,woodworth2020kernel} among others.
The degree $p$ is odd to make the ViT model compatible with the labeling function and strictly larger than 1 because the data is not linearly separable (\autoref{lem:linear_models}).   We add a linear part in the activation function to ensure that the gradient is non-zero when $\bm{v}$ has small coefficients. With these simplifications, we formally prove that $F$ is able to learn patch association and generalize, in the two following settings.

\paragraph{Idealized and realistic learning problems.} Given a dataset $\mathcal{Z}=\{(\bm{X}[i],y[i])\}_{i=1}^N$ sampled from $\mathcal{D}$, we solve the empirical risk minimization problem for the logistic loss defined by:
\begin{align}\label{eq:empirical}\tag{E}
    \min_{\widehat{\bm{A}},\hat{\bm{v}}}\;\; \frac{1}{N}\sum_{i=1}^N \log\big(1+e^{-y[i]F(\bm{X}[i])}\big):=\widehat{\mathcal{L}}(\widehat{\bm{A}},\widehat{\bm{v}}).
\end{align}
Instead of directly analyzing (\ref{eq:empirical}), we introduce a proxy where we minimize the population risk
\begin{align}\label{eq:population}\tag{P}
    \min_{\bm{A},\bm{v}}\;\; \mathbb{E}_{\mathcal{D}}\big[\log\big(1+e^{-yF(\bm{X})}\big)\big]:=\mathcal{L}(\bm{A},\bm{v}).
\end{align}
We refer to (\ref{eq:empirical}) as the \textit{realistic} problem while (\ref{eq:population}) as the \textit{idealized} problem. 
%We consider two different learning problems. In \autoref{sec:theory}, we analyze the \textit{idealized} case where we  minimize the population risk for the logistic loss with respect to $\bm{A},\bm{v}$ ,

% In \autoref{sec:theory_real}, we study the \textit{realistic} case where we are given a training dataset  sampled from $\mathcal{D}$ and minimize the empirical risk

\paragraph{Algorithm.} We solve (\ref{eq:population}) and (\ref{eq:empirical}) using gradient descent (GD) for $T$ iterations. The update rule in the case of (\ref{eq:population}) for $t\in[T]$ and $i,j\in [D]$ is
\begin{equation}\label{eq:GD}\tag{GD}
%\begin{aligned}
    A_{i,j}^{(t+1)}=A_{i,j}^{(t)}-\eta\partial_{A_{i,j}} \mathcal{L}(\bm{A}^{(t)},\bm{v}^{(t)}), \quad 
    \bm{v}^{(t+1)}= \bm{v}^{(t)} -\eta\nabla_{\bm{v}}\mathcal{L}(\bm{A}^{(t)},\bm{v}^{(t)}),
%\end{aligned}
\end{equation}
where $\eta>0$ is the learning rate. A similar update may be written for (\ref{eq:empirical}).  We now detail how to set the parameters in (\ref{eq:GD}).

\begin{parametrization}\label{param}
When running GD on (\ref{eq:population}) and (\ref{eq:empirical}), the number of iterations is any $T\geq \mathrm{poly}(d)/\eta.$ We set the learning rate as $\eta\in\big(0,\frac{1}{\mathrm{poly}(d)}\big)$.  The diagonal coefficient of the attention  matrix are set for $i\in[D]$ as $A_{i,i}^{(0)}=\widehat{A}_{i,i}^{(0)}=\sigma_{\bm{A}}\mathbf{I}_{D}$ where $\sigma_{\bm{A}}=\mathrm{polyloglog}(d).$ The off-diagonal coefficients of $\bm{A}$ and the value vector are initialized as: 
\begin{enumerate}
    \item Idealized case: $\bm{v}^{(0)}=\alpha^{(0)}\bm{w}^*$ where  $\alpha^{(0)}=\nu^{1/(p-1)}$ and $A_{i,j}^{(0)}=0$ for $i\neq j.$ %1/\mathrm{poly}(d).$  %small \mic{How small?}  \mic{This supposes we now $w$, is this realistic?}.
    \item Realistic case: $\widehat{\bm{v}}^{(0)}\sim\mathcal{N}(0,\omega^2\mathbf{I}_d)$ and $\widehat{A}_{i,j}^{(0)}\sim\mathcal{N}(0,\omega^2)$ where $i\neq j$ and $\omega=1/\mathrm{poly}(d)$. %\bm{0}_{d}$.  
\end{enumerate}
\end{parametrization}
We remind that in \autoref{ass:transformer}, we have $\bm{A}=\bm{P}^{\top}\bm{P}$. If one initializes $\bm{P}\sim\mathcal{N}(0,\sigma_A\mathbf{I}_D/D)$, then with high probability, $ A_{i,i}^{(0)}=\|\bm{p}_i^{(0)}\|_2^2=\Theta(\sigma_A)$ and $A_{i,j}^{(0)}=\langle \bm{p}_i^{(0)}, \bm{p}_j^{(0)}\rangle = \Theta(\sigma_A/\sqrt{D})$ for $i\neq j$. Since $D\gg 1$, it is then reasonable to set $A_{i,j}^{(0)}=0$. 
 Note that, also in the idealized setting, we initialize  $\bm{v}^{(0)}$ in $\mathrm{span}(\bm{w}^*)$, even though this latter should be unknown to the algorithm. We remind that the idealized case is a proxy to ultimately characterize the realistic dynamics.

\section{Learning spatial structure via matching the labeling function}\label{sec:theory}

\vspace{-.3cm}

As announced above, we show that our ViT (\ref{eq:transformer}) implicitly learns patch association and fits the labeling function by minimizing the training objective. We first study the dynamics in  (\ref{eq:population}). Using the analysis in the idealized case, we then characterize the solution found in the realistic problem (\ref{eq:empirical}).

\vspace{-.2cm}

\subsection{Learning process in the idealized case}

%In this section, we analyze the idealized learning process where we train our model  (\ref{eq:transformer}) using GD on infinitely many examples sampled from $\mathcal{D}$. As announced above, we show that (\ref{eq:transformer}) generalizes while learning patch association. 

In this section, we analyze the dynamics of (\ref{eq:population}). Our main result is that after minimizing (\ref{eq:population}),  our model (\ref{eq:transformer}) performs patch association while generalizing.

\begin{theorem}\label{thm:ideal}
Assume that we run GD on (\ref{eq:population}) for $T$ iterations with parameters set as in \autoref{param}. With high probability, the ViT model (\ref{eq:transformer})
\begin{enumerate}
    \item learns patch association i.e.\ for all $\ell\in[L]$ and  $i\in\mathcal{S}_{\ell}$, $\mathrm{Top}_{C}\;\{A_{i,j}^{(T)}\}_{j=1}^D=\mathcal{S}_{\ell} .$
    %\mic{Here i would put $Q_{i,j}$ rather than the dot product (recall that $p_k = e_k$.)} 
    %\mic{ $\mathrm{Top}_{\textcolor{red}{C}}\;\{Q_{i,j}^{(T)}\}_{j=1}^D=\mathcal{S}_{\ell}$?}
    \item learns the labeling function $f^*$ i.e.\ $\mathbb{P}_{\mathcal{D}}[f^*(\bm{X})F_{\bm{A}^{(T)},\bm{v}^{(T)}}(\bm{X})>0]\geq 1-o(1).$ 
    %\textcolor{red}{How do we write the dependence on params?}%yF_{\bm{Q}^{(T)},\bm{v}^{(T)}}(\bm{X})>0
\end{enumerate}

\end{theorem}

\begin{figure}[tbp]
\vspace*{-1.5cm}

 \includegraphics[width=\linewidth]{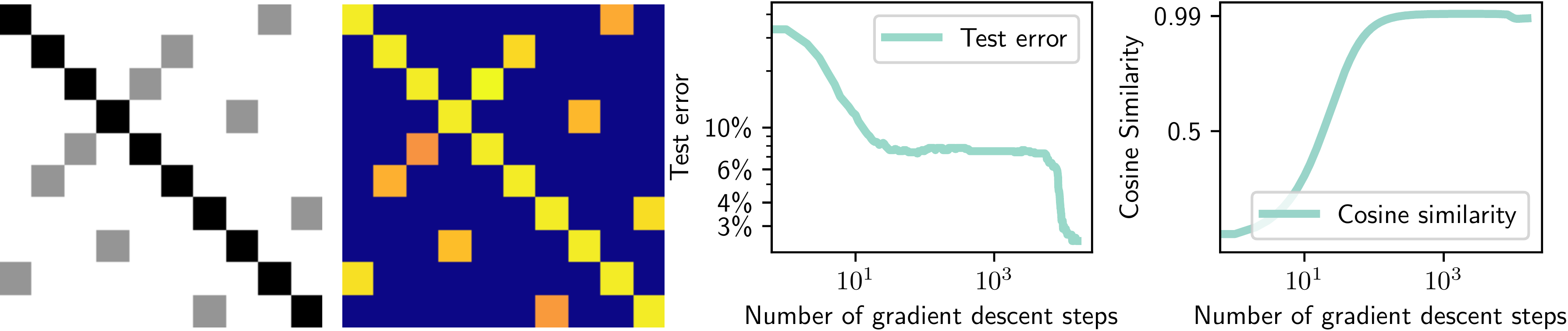}

\caption{\small Illustration of \autoref{thm:real}. We consider the exact same setting (data generation, parameter settings...) as for the realistic case. From left to right, we first display in grey the tuples $(i,j)$ such that $(i, j) \in \mathcal{S}_{\ell}$. We then plot the learned matrix $\bm{A}$ and see that coefficients with high value exactly correspond to their grey scale counterpart in the left plot. We also display test error and cosine similarity between $\bm{w}^*$ and $\bm{v}$ w.r.t the number of training steps.}\label{fig:fig_synthetic}

\vspace{-.6cm}

\end{figure}

We now sketch the main ideas to prove the theorem for which one can refer to \autoref{sec:ideal} for a complete proof. 

\paragraph{Invariance and symmetries.} In (\ref{eq:population}), we take the expectation over $\mathcal{D}$. Since (\ref{eq:transformer}) is permutation-invariant and the data distribution is symmetric, we can thus dramatically simplify the variables in (\ref{eq:population}). An illustration of this is the next lemma that shows that $\bm{A}$ can be reduced to three variables in (\ref{eq:population}).
\begin{restatable}{lemma}{lemQreducvar}\label{lem:Qreducvar}%\begin{lemma}\label{lem:Qreducvar}
There exist $\beta=\sigma_{\bm{A}}$, $\gamma^{(t)},\rho^{(t)}\in\mathbb{R}$ such that for all $t\geq 0$: 
\begin{enumerate}
    \item for all $i\in [D]$, $A_{i,i}^{(t)}=\beta.$
    \item for all $i,j\in[D]$ such that $i,j\in\mathcal{S}_{\ell}$ \; for some  $\ell\in[L]$,  $A_{i,j}^{(t)}=\gamma^{(t)}.$
    \item for all $i,j\in[D]$ such that $i\in\mathcal{S}_{\ell}$ and $j\in\mathcal{S}_m$\; for some $\ell,m\in[L]$ with $\ell\neq m$, $A_{i,j}^{(t)}=\rho^{(t)}.$
\end{enumerate}
\end{restatable}
%We leverage the structures of the data distribution, the learner model and  the initialization to simplify (\ref{eq:population}).
 %\textcolor{red}{Correct order to put the lemmas?}
Besides, using the initialization in \autoref{param}, we can show that $\bm{v}$ always lies in $\mathrm{span}(\bm{w}^*).$
\begin{restatable}{lemma}{lemvspanw}\label{lem:v_spanw}
For all $t\in[T]$, there exists $\alpha^{(t)}\in\mathbb{R}$ such that $\bm{v}^{(t)}=\alpha^{(t)}\bm{w}^*.$
\end{restatable}

\vspace{-.2cm}

In summary, \autoref{lem:Qreducvar} and \autoref{lem:v_spanw} imply that instead of optimizing over $\bm{A}$ and $\bm{v}$ in (\ref{eq:population}), we can instead consider the scalar variables $\alpha^{(t)}$, $\gamma^{(t)}$ and $\rho^{(t)}$. The remaining of this section consists in analyzing the dynamics of these three quantities.

\paragraph{Learning patch association.} We first analyze the dynamics of $\gamma^{(t)}$ and $\rho^{(t)}$. To this end, we introduce the following terms:
\begin{align*}
    \Lambda^{(t)}&= \frac{e^{\beta}}{e^{\beta}+(C-1)e^{\gamma^{(t)}}+(D-C)e^{\rho^{(t)}}}, &\Gamma^{(t)}=\frac{e^{\gamma^{(t)}}}{e^{\beta}+(C-1)e^{\gamma^{(t)}}+(D-C)e^{\rho^{(t)}}},\\
    \Xi^{(t)}&= \frac{e^{\rho^{(t)}}}{e^{\beta}+(C-1)e^{\gamma^{(t)}}+(D-C)e^{\rho^{(t)}}}, &G^{(t)}= D(\Lambda^{(t)}+(C-1)\Gamma^{(t)}).\hspace{2cm}
\end{align*}
Note that $\Lambda^{(t)}$, $\Gamma^{(t)}$ and $\Xi^{(t)}$ respectively correspond to the  coefficients on the diagonal, those for which $i, j\in \mathcal{S}_{\ell}$ for some $\ell \in [L]$ and all the other coefficients of the attention matrix $\bm{S}$. Using these notations, we first derive the GD updates of $\gamma^{(t)}$ and $\rho^{(t)}.$

\vspace{.2cm}

\begin{lemma}\label{lem:rhogammaupd}
Let $t\leq T$. The attention weights $\gamma^{(t)}$ and $\rho^{(t)}$ satisfy:
\begin{align*}
    \gamma^{(t+1)}&=\gamma^{(t)}+\eta\mathrm{polylog}(d) (\alpha^{(t)})^p\cdot \Gamma^{(t)}(G^{(t)})^{p-1},\\
    |\rho^{(t+1)}|&\leq |\rho^{(t)}|+\eta\mathrm{polylog}(d) (\alpha^{(t)})^p\Big(\frac{1}{D}+\frac{1}{D}\Gamma^{(t)}(G^{(t)})^{p-1}\Big).
\end{align*}

    \vspace{-.5cm}

\end{lemma}
\autoref{lem:rhogammaupd} shows that the increment of $\gamma^{(t)}$ is larger than the one of $\rho^{(t)}$. Since $\gamma^{(0)}=\rho^{(0)}=0$, this implies that $\gamma^{(t)}\geq\rho^{(t)}$ for all $t\geq0.$ This observation proves the first item  of \autoref{thm:ideal}. We now explain how learning patch association leads to $\bm{v}$ highly correlated with $\bm{w}^*.$

\begin{itemize}[leftmargin=*, itemsep=1pt, topsep=1pt, parsep=1pt]
    \item[--]  \textbf{Event I}:
At the beginning of the process, the update of $\bm{v}^{(t)}$ is larger than the one of $A_{i,j}^{(t)}$ which implies that only $\bm{v}^{(t)}$ updates during this first phase. We show that $\alpha^{(t)}=\langle \bm{v}^{(t)},\bm{w}^*\rangle$ increases until a time $\mathcal{T}_0>0$ where it reaches some threshold  (\autoref{lem:event1}). At this point, the model is nothing else than a generalized linear model that would not generalize because there are much more noisy tokens than signal ones (see \autoref{lem:linear_models}).  %Because the $A_{i,j}^{(t)}$'s stay very small and  $\bm{v}^{(t)}$ is initialized in $\mathrm{span}(\bm{w}^*)$, $v$ updates in the direction of $\bm{w}^*$. At this point, the model is nothing else than a generalized linear model that would not generalize because there are much more noisy tokens than signal ones (see \autoref{lem:linear_models}). 
\item[--] \textbf{Event II}: During this phase, the attention weights must update. Indeed, assume by contradiction that the $A_{i,j}^{(t)}$ stay around initialization and that $\bm{v}^{(t)}$ is optimal i.e.\  $\bm{v}^{(t)}=a^{(t)}\bm{w}^*$ where $a^{(t)}\gg 1.$ Then, the predictor $g$ we would have is 
\begin{align}
    g(\bm{X})= \sum_{i=1}^D\sum_{j=1}^D S_{i,j}^{(0)}\langle \bm{v}^{(t)},\bm{X}_j\rangle \propto \sum_{i=1}^D\sum_{j=1}^D e^{A_{i,j}^{(0)}}\langle \bm{w}^*,\bm{X}_j\rangle
\end{align}
Such predictor $g$ would yield high population loss because there many more data with random labels ($qD=\mathrm{poly}(C)$) than with the exact label. Therefore,  $A_{i,j}^{(t)}$'s start to update. The  gradient increment for $\gamma^{(t)}$ (which corresponds to $i$ and $j$ in the same set $\mathcal{S}_{\ell}$) is much larger than the one for $\rho^{(t)}$ (\autoref{lem:rhogammaupd}). Thus, $\gamma^{(t)}$ increases until a time $\mathcal{T}_1\in [\mathcal{T}_0,T]$ such  that $\gamma^{(\mathcal{T}_1)}>\max_{t\in [T]}|\rho^{(t)}|$.

\item[--] \textbf{Event III}: Because we have $\gamma^{(\mathcal{T}_1)}>\max_{t\in [T]}|\rho^{(t)}|$, we again have  $\alpha^{(t+1)} > \alpha^{(t)}$ as in Phase I (\autoref{lem:event3}). Thus, $\alpha^{(t)}$ increases again until the population risk becomes a $o(1)$.%Since the gradient of $\mathcal{L}$ is in the direction of $\bm{w}^*$, $\bm{v}^{(t)}$ updates again in the direction of $\bm{w}^*$ until the population risk becomes a $\Omega(1)$ and we fit the labelling function.
\end{itemize}

%\vspace{-1.5cm}
\paragraph{Main insights of our analysis.}
Our mechanism highlights two important aspects that are proper to attention models:
\begin{itemize}[leftmargin=*, itemsep=1pt, topsep=1pt, parsep=1pt]
    \item[--]  because of the initialization and  the data structure, we have patch association for any time $t$ (\autoref{lem:rhogammaupd}).
    \item[--] our ViT model uses patch association to minimize the population loss (Event III). Without patch association, the model would only be a generalized linear model that does not minimize the loss.
\end{itemize}

\subsection{From the idealized to the realistic learning process}\label{sec:theory_real}

The real learning process differs from the idealized one in that we have a finite number of samples and we initialize both $\widehat{\bm{A}}$ and $\widehat{\bm{v}}$ as Gaussian random variables. Using a polynomial number of samples, we show that (\ref{eq:transformer}) still learns patch association and generalizes.

\begin{theorem}\label{thm:real}
Assume that we run GD on (\ref{eq:empirical}) for $T$ iterations with parameters set as in \autoref{param}. Assume that the number of samples is $N=\mathrm{poly}(d).$ With high probability, the model 
\begin{enumerate}
   \item learns patch association i.e.\ for all $\ell\in[L]$ and  $i\in\mathcal{S}_{\ell}$, $\mathrm{Top}_{C}\;\{\widehat{A}_{i,j}^{(T)}\}_{j=1}^D=\mathcal{S}_{\ell}. $
    %\mic{ $\mathrm{Top}_{\textcolor{red}{C}}\;\{Q_{i,j}^{(T)}\}_{j=1}^D=\mathcal{S}_{\ell}$?}
    \item  fits the labeling function i.e.\ $\mathbb{P}_{\mathcal{D}}[f^*(\bm{X})F_{\widehat{\bm{A}}^{(T)},\widehat{\bm{v}}^{(T)}}(\bm{X})>0]\geq 1-o(1).$
\end{enumerate}
\end{theorem}
%\autoref{thm:real} imply that running GD on the empirical risk with a polynomial number of samples ensures that the Transformer learns the underlying structure of $\mathcal{D}$ and achieves small generalization error. 
Similarly to \cite{li2020learning},  the proof introduces a "semi-realistic" learning process that is a mid-point between the idealized and realistic processes. We show that $\widehat{\bm{A}}^{(T)}$ and $\widehat{\bm{v}}^{(T)}$ are close to their semi-realistic counterparts -- see \autoref{sec:real_process} for a complete proof. \autoref{fig:fig_synthetic} numerically illustrates \autoref{thm:real}.

\section{Patch association yields sample-efficient fine-tuning with ViTs}\label{sec:transfer}

%%We examine a downstream dataset $\mathring{S}=\{(\widetilde{\bm{X}}^{(i)},\Tilde{y}^{(i)}\}_{i=1}^{\tilde{N}}$ sampled from the distribution  $\widetilde{\mathcal{D}}$ where the data-points share a \textit{similar} structure as $\mathcal{D}$ but with a \textit{different} feature $\widetilde{\bm{w}}^*.$ The generative process of $\widetilde{\mathcal{D}}$ is similar to the one of $\mathcal{D}$ except the following steps: 

A fundamental byproduct of our theory is that after pre-training on a dataset sampled from $\mathcal{D}$, our model  \eqref{eq:transformer} sample-efficiently transfers to datasets that are structured as $\mathcal{D}$ but differ in their features.

\vspace*{-.4cm}
\paragraph{Downstream dataset.} Let $\widetilde{\mathcal{D}}$ a downstream data distribution  defined as in \autoref{def:datadist} such that its underlying feature is $\widetilde{\bm{w}}^*$ with $\|\widetilde{\bm{w}}^*\|_2=1$ and $\widetilde{\bm{w}}^*$ potentially different from $\bm{w}^*$. In other words, the downstream $\widetilde{\mathcal{D}}$ and source $\mathcal{D}$ distributions share the same structure but not necessarily the same feature. We sample a downstream dataset $\widetilde{\mathcal{Z}}=\{(\widetilde{\bm{X}}[i],\widetilde{y}[i])\}_{i=1}^{\widetilde{N}}$ from $\widetilde{\mathcal{D}}$.

%$\widetilde{\bm{w}}^*\in\mathbb{R}^d$ such that $\|\widetilde{\bm{w}}^*\|_2=1$ and  $\{\widetilde{\bm{\xi}}_j\}_{j=1}^D$ be such that $\widetilde{\bm{\xi}}_i\sim\mathcal{N}(0,\sigma^2\mathbf{I}_d).$ Let  $\{\widetilde{\delta}_j\}_{j=1}^D$ be a collection of random variables such that $\widetilde{\delta}_j=1$ with probability $q/2$, $\widetilde{\delta}_j=-1$ with the same probability and $0$ otherwise. The downstream distribution $\widetilde{\mathcal{D}}$ follows the same generative process as $\mathcal{D}$ except for the following steps:  
%\begin{enumerate}
% \setcounter{enumi}{3}
%    \item  \textbf{Signal set}: for $i\in\mathcal{S}_{\ell(\bm{X})}$, $ \bm{X}_i=y\widetilde{\bm{w}}^*+\widetilde{\bm{\xi}}_i$.
    %%%%, where $\widetilde{\bm{\xi}}_i\sim\mathcal{N}(0,\sigma^2\mathbf{I}_d).$
%    \item \textbf{Random set}: for $j\in\mathcal{S}_{\ell}$ with $\ell\neq\ell(\bm{X})$, $\bm{X}_j=\widetilde{\delta}_{j}\widetilde{\bm{w}}^*+\widetilde{\bm{\xi}}_{j}$.
    %%%%, where $\widetilde{\delta}_{j}=1$ with probability $q/2$, $-1$ with the same probability and $0$ otherwise and $\widetilde{\bm{\xi}}_{j}\sim\mathcal{N}(0,\sigma^2\mathbf{I}_d).$
%\end{enumerate}
%Using this process, we generate a downstream dataset $\widetilde{S}=\{(\widetilde{\bm{X}}[i],\widetilde{y}[i]\}_{i=1}^{\widetilde{N}}$.

%\vspace*{-.4cm}

\paragraph{Learning problem.} We consider the model \eqref{eq:transformer} pre-trained as in  \autoref{sec:theory_real}. We assume that $\widehat{\bm{A}}$ is  kept \textit{fixed} from the pre-trained model and we only optimize the value vector $\widetilde{\bm{v}}$ to solve:

\vspace{-.4cm}

\begin{align}\label{eq:empirical_transfer}\tag{$\widetilde{\text{E}}$}
    \min_{\widetilde{\bm{v}}}\;\; \frac{1}{\widetilde{N}}\sum_{i=1}^{\widetilde{N}} \log\big(1+e^{-\widetilde{y}[i]F(\widetilde{\bm{X}}[i])}\big):=\widetilde{\mathcal{L}}(\widetilde{\bm{v}}).
\end{align}

\vspace{-.2cm}

We run GD on \eqref{eq:empirical_transfer} with parameters set as in \autoref{param} except that the $\widehat{A}_{i,j}$'s are fixed and $\widetilde{\bm{v}}^{(0)}\sim\mathcal{N}(0,\omega^2\mathbf{I}_d)$ with $\omega=1/\mathrm{poly}(d)$. Our main results states that this fine-tuning procedure requires a few samples to achieve high test accuracy in $\widetilde{\mathcal{D}}$. In contrast, any algorithm without patch association needs a large number of samples to generalize.
%Our main result states that this fine-tuning procedure requires a few samples to achieve high test accuracy in $\widetilde{\mathcal{D}}$. %[Sample efficient adaptation with pre-training]
\begin{restatable}{theorem}{thmtransfer}\label{thm:transfer}
 Let $\widehat{\bm{A}}$ be the attention matrix obtained after pre-training as in \autoref{sec:theory_real}. Assume that we run GD for $T$ iterations on \eqref{eq:empirical_transfer} to  fine-tune the value vector. Using $\widetilde{N}\leq\mathrm{polylog}(D)$ samples, the model \eqref{eq:transformer} transfers to $\widetilde{\mathcal{D}}$ i.e.\ $\mathbb{P}_{\widetilde{\mathcal{D}}}[f^*(\bm{X})F_{\widehat{\bm{A}},\widetilde{\bm{v}}^{(T)}}(\bm{X})>0]\geq 1-o(1).$
\end{restatable}
 
% During the optimization, we have $\widetilde{\bm{v}}^{(t)} = \widetilde{\alpha}^{(t)}\widetilde{\bm{w}}^*+\widetilde{\varepsilon}_{\bm{v}}^{(t)}\widetilde{\bm{u}}^{(t)}.$ The dynamics of $\widetilde{\bm{v}}$ is similar to the one of $\hat{\bm{v}}$ in \autoref{sec:theory_real}. Therefore, \autoref{lem:event3} and \autoref{lem:epsv} respectively provide a bound on $\widetilde{\alpha}^{(t)}$ and $\widetilde{\varepsilon}_{\bm{v}}^{(t)}.$ We finally resort to classical generalization bounds based on the Rademacher complexity to obtain the sample complexity in \autoref{thm:transfer}. 

%any algorithm without pre-training or prior knowledge of the underlying structure

\begin{restatable}{theorem}{thmtransfernegative}\label{thm:samplcomplexA}
 Let $\mathcal{A}\colon\mathbb{R}^{d\times D}\rightarrow\{-1,1\}$ be a binary classification algorithm without patch association knowledge.  Then, it needs  $ D^{\Omega(1)}$ training samples to get test error $\leq o(1)$ on $\widetilde{\mathcal{D}}$. %$\mathbb{P}_{(\bm{X},y)\sim\widetilde{\mathcal{D}}}[y\mathcal{A}(\bm{X})>0]\geq 1-o(1).$
\end{restatable}
%\textcolor{red}{Intuitively, when an algorithm $\mathcal{A}$ is given less than $D^{\Omega(1)}$ training samples,   there exists some set $\mathcal{S}_{k}$ that $\mathcal{A}$ has not seen during training. Therefore, $\mathcal{A}$ cannot classify test samples such that $\ell(\bm{X})=k$ and thus incurs high generalization error.} \autoref{thm:transfer} and \autoref{thm:samplcomplexA} show that our model is sample-efficient for transfer. %without knowing the patch associations $\mathcal{S}_{\ell}$ as prior.
The proofs of \autoref{thm:transfer} and \autoref{thm:samplcomplexA} are in \autoref{sec:app_transfer}. These theorems hightlight that learning patch association is required for efficient transfer. We believe that they offer a new perspective on explaining why ViTs are widely used in transferring to downstream tasks. While it is possible that ViTs learn shared (with the downstream dataset) features during pretraining, our theory hints that learning the inductive bias of the labeling function is also central for transfer.

\begin{figure}[tbp]
\vspace*{-1.5cm}

\begin{subfigure}{0.54\textwidth}
\centering
 \includegraphics[width=.8\linewidth]{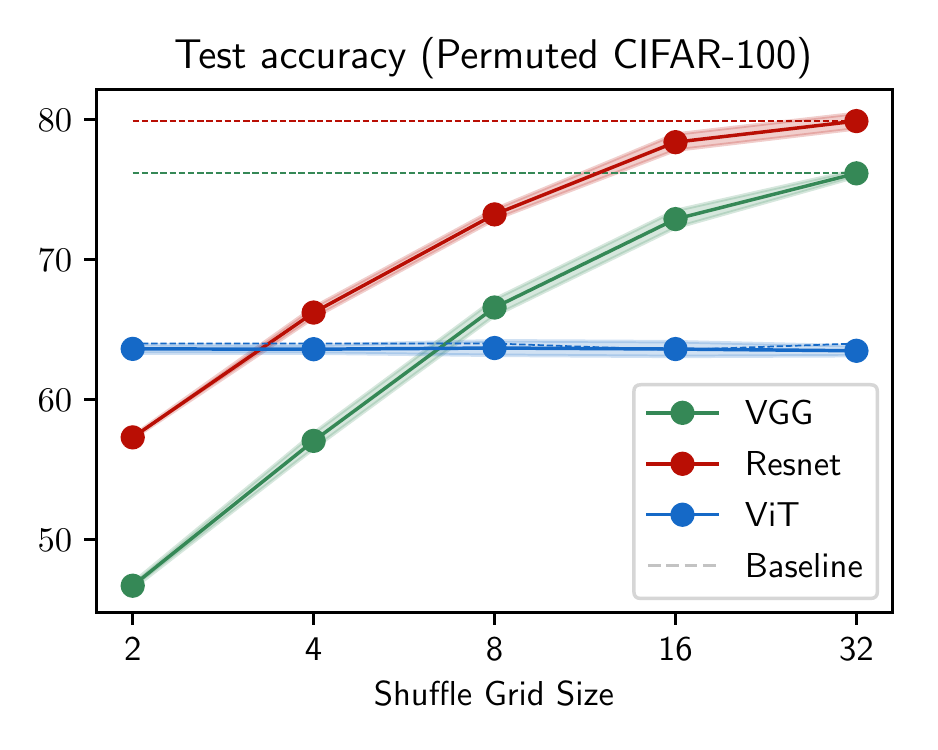}
%  \vspace{-8mm}

  \vspace{-2mm}
\caption{}\label{fig:cifar100perm}
\end{subfigure}
\begin{subfigure}{0.4\textwidth}%0.33
%\hspace*{1cm}
\begin{subfigure}{1\textwidth}
\includegraphics[width=1\linewidth]{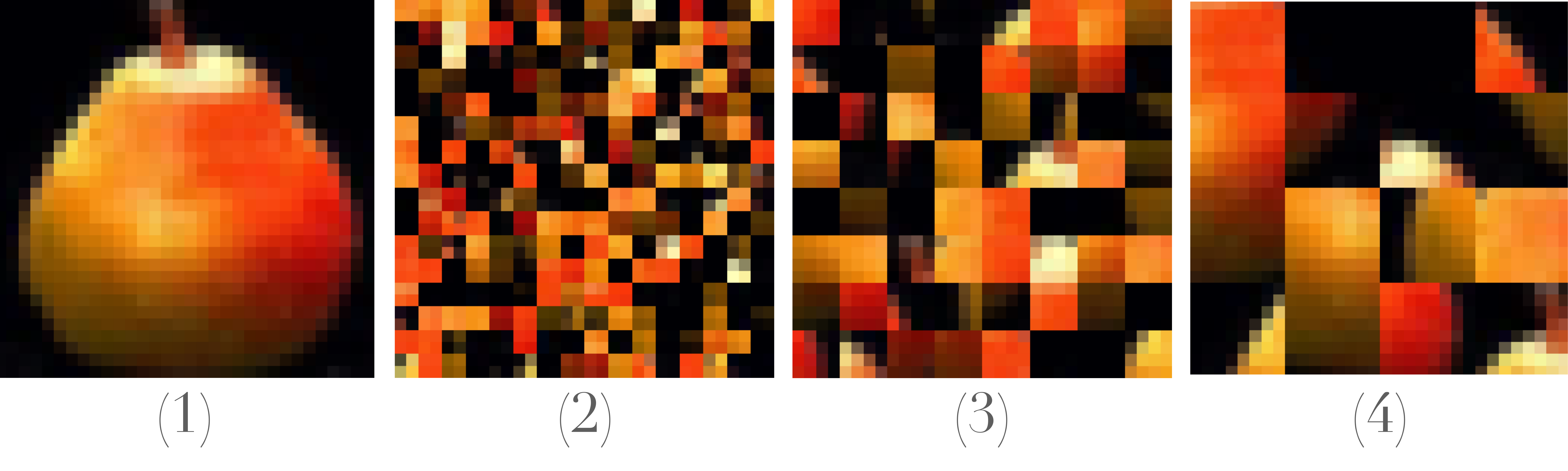}
\vspace*{-7mm}
\caption{}\label{fig:image_perm}
\end{subfigure}
\begin{subfigure}{1\textwidth}
\includegraphics[width=1\linewidth]{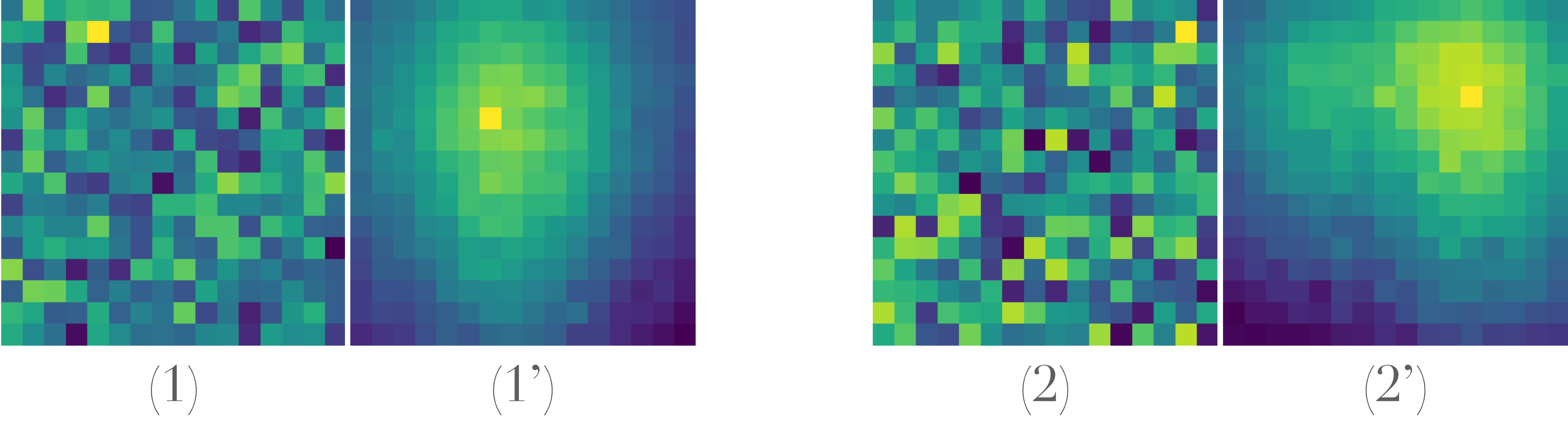}
\vspace*{-7mm}
 \caption{}\label{fig:perm_pos}
\end{subfigure}
\end{subfigure}
%\begin{subfigure}{0.5\textwidth}%0.33
%\hspace*{1cm} \includegraphics[width=.7\linewidth]{images/pear_perm.pdf}
 %\vspace{1mm}
% \caption{}\label{fig:perm_attn}
%\end{subfigure}
\caption{\small  (a): Test accuracy obtained with ViT (patch size 2), ResNet-18 and VGG-19 on permuted (in solid lines) and on original (in dashed lines) CIFAR-100. While convolutional models are very sensitive to permutations, the ViT performs equally whether the dataset is permuted or not. (b): (2) CIFAR-100 image (1) and Permuted CIFAR-100 image when shuffle grid size is 2 (2), 4 (3) and 8 (4). (c): (1-2) Visualization of positional encoding similarities after training a ViT (patch size 2) on permuted CIFAR-100 (shuffle grid size 2). Here, we display $\bm{p}_i^\top \bm{P}$ where $i$ is some fixed index and reshape such vector into a matrix $16\times 16$. We observe that these similarities (1-2) do not have any spatially localized structure.  However, when applying the inverse of the permutation, we recover spatially localized patterns in (1'-2').    }\label{tab:perm_fig}
%\mic{detail what it is exactly, it is not $P^\top P$, but rather some $\bm{p}_i^\top P$.}
% Visualization of rows of $\bm{P}^{\top}\bm{P}$ after training a ViT (patch size 2) on permuted CIFAR-100 (shuffle grid size 2). Each of this matrices is obtained by selecting one row of $\bm{P}^{\top}\bm{P}$ and reshaping it into a matrix of size $256\times256.$ These arrays do not have a non-convolutional structure. However, when applying the inverse of the permutation to the selected row, we recover a convolution-like pattern in (1'-2'). 

\vspace{-.5cm}

\end{figure}

\section{Numerical experiments}\label{sec:num_exp}

\vspace{-.3cm}

%\begin{figure}[tbp]
%\vspace*{-.4cm}

% \includegraphics[width=.9\linewidth]{images/vit_barplots.pdf}

%\caption{\small Test accuracy obtained with a ViT using vanilla attention (ViT) and positional attention (Ours) on CIFAR-10 (1), CIFAR-100 (2) and SVHN (3). In small datasets, our model also competes with the vanilla ViT. Results are averaged over 10 seeds for this experiment. }\label{fig:barplots_fig}
%\end{figure}

%Regarding 2: Someone might say that "convolution" is not learned, it's just a neural network tends to group pixels nearby together. We want to show that the inductive bias of convolution is indeed learned, in the sense that if we permute the image, then now VIT will learn attentions that attend to very far away pixels (following the convolution structure after permutation). Thus, VIT is indeed learning the induction bias instead of just grouping nearby pixels together. 
%Moreover,  random permuted images should be very hard for a human to learn -- We probably can not detect the permutation + convolution structure just by looking at the data. So inductive bias learning here is very non-trivial.
%We used Pytorch and Nvidia Tesla V100 GPUs.

In this section, we first empirically verify that ViTs learn patch association while miniziming their training loss. We then numerically show that the positional attention mechanism  competes with the vanilla one on small-scale datasets such as CIFAR-10/100 \citep{krizhevsky2009learning}, SVHN \citep{netzer2011reading} and large-scale ones such as ILSVRC-2012 ImageNet \citep{deng2009imagenet}. For the small datasets, we use a ViT with 7 layers,  12 heads and hidden/MLP dimension 384. For ImageNet, we train a "ViT-tiny-patch16-224"  \cite{dosovitskiy2020image}. Both models are trained with standard augmentations techniques \citep{cubuk2018autoaugment} and using AdamW with a cosine learning rate scheduler. We run all the experiments for 300 epochs, with batch size 1024 for Imagenet and 128 otherwise and average our results over 5 seeds. We refer to \autoref{sec:appexp} for the training details. 
%\vspace{-1em}

\paragraph{ViTs learn patch association.} 
We consider the CIFAR-100 dataset where we divide each image into grids of size $s\times s$ pixels. For a fixed $s\in\{2,4,8,16,32\}$, we permute the grids according to $\pi_s$ 
to create the permuted CIFAR-100 dataset. We call $s$ the grid shuffle size. 
\autoref{fig:image_perm}-(1) shows a CIFAR-100 image and its corresponding shuffling in the permuted CIFAR-100 dataset  \autoref{fig:image_perm}-(2-3-4).
We train a ViT and CNNs ResNet18 \citep{he2016identity} and VGG-19 \citep{simonyan2014very} on the permuted CIFAR-100 dataset. 
\begin{wrapfigure}[12]{r}{0.48\textwidth}
\vspace{-.4cm}
\includegraphics[width=0.48\textwidth]{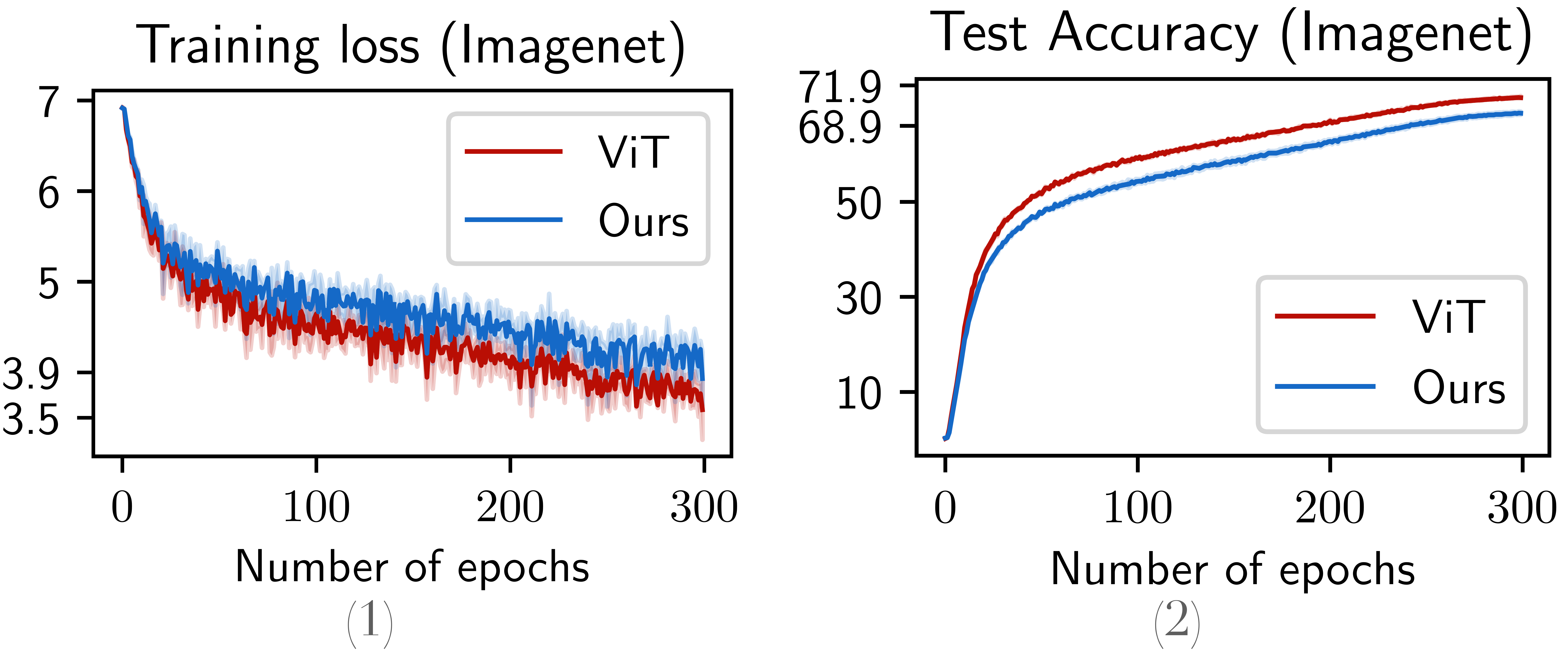}
\caption{\small Training loss  (1) and test accuracy (2) obtained using a ViT-tiny-patch16-224 on Imagenet. ViT using positional attention (Ours) gets $68.9\%$ test accuracy while vanilla ViT (ViT) gets  $71.9\%$.  }\label{tab:imgnet_fig}
\end{wrapfigure}
For the ViT, we set the patch size to $2$, although this is sub-optimal in terms of accuracy, because the patch size needs to stay smaller or equal to $s$.
Indeed, intuitively, when we permute the grids in \autoref{fig:image_perm}, we lose the local aspect of the spatial structure and create new sets $\mathcal{S}_{\ell}$'s and a new labeling function $f^*.$ 
\autoref{fig:cifar100perm} reports the test accuracy of these three models for different values of $s$.
When $s$ is small, the image does not have a coherent structure e.g. \autoref{fig:image_perm}-(2) and thus, CNNs struggle to generalize. As $s$ increases e.g. \autoref{fig:image_perm}-(4), the information inside a patch is meaningful and thus, the CNNs well-perform. Unsurprisingly, since ViTs are \textit{permutation invariant}, their performance remains unchanged for all $s$ -- see \autoref{fig:cifar100perm}. Despite this change, we verify that the ViT is able to recover the new $\mathcal{S}_{\ell}$'s: we feed the ViT with the shuffled pear image ( \autoref{fig:image_perm}-(2)) and consider for some $i$ the similarity matrix $\bm{p}_i^{\top}\bm{P}$. We see that it does not exhibit a local spatial structure in \autoref{fig:perm_pos}-(1,2). We then apply $\pi_s^{-1}$ to $\bm{p}_i^{\top}\bm{P}$ and observe that we recover the spatially localized patterns \autoref{fig:perm_pos}-(1',2'). This experiment highlights that ViTs do not just group nearby pixels together as convolutions. They learn a more general spatial structure, in accordance to our theoretical results.
\vspace{-0.3cm}
\paragraph{ViTs with positional attention are competitive.} 
\begin{wrapfigure}[10]{r}{0.48\textwidth}

\vspace{-.5cm}

\includegraphics[width=0.5\textwidth]{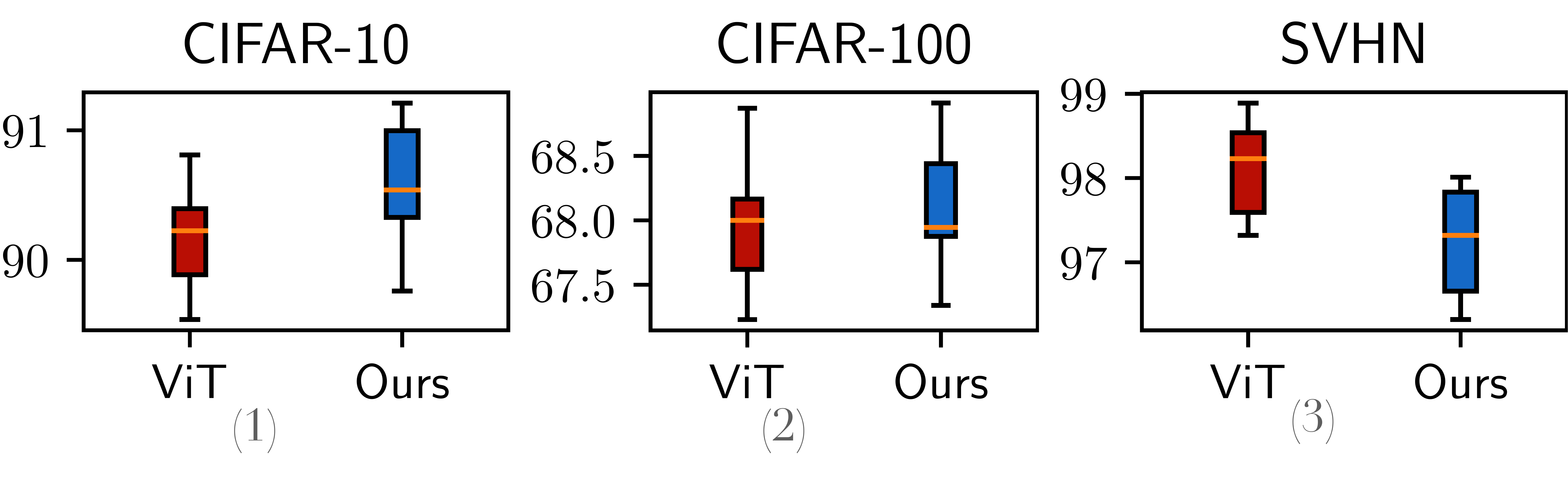}
\caption{\small Test accuracy obtained with a ViT using vanilla attention (ViT) and positional attention (Ours) on CIFAR-10 (1), CIFAR-100 (2) and SVHN (3). Our model competes with the vanilla ViT. Patch size 4 and average over 10 seeds for this experiment.}\label{tab:barplots_fig}
%compared to the upper matrix
%However, it is also possible to apply positional attention to real-world ViTs (see \autoref{sec:appexp} for an explanation)
\end{wrapfigure}
We numerically verify that ViTs using positional attention compete with those with  vanilla attention. In \autoref{sec:setting}, we introduced positional attention to define our theoretical learner model. \autoref{tab:imgnet_fig} and  \autoref{tab:barplots_fig} show that ViTs using positional attention compete with vanilla ViTs on a range of datasets. These experiments strengthen our intuition that for images, having an attention matrix that only depends on the positional encodings is sufficient to have a good test accuracy. 

\section*{Conclusion, limitations and future works}\label{sec:conclusion}

Our work is a first step towards understanding how Transformers learn tailored inductive biases when trained with gradient descent. %Inspired by computer vision datasets and ViTs,  we propose a theoretical framework to address this question. 
Our analysis heavily relies on the positional attention mechanism that disentangles patches and positional encodings. In practice, self-attention mixes these two quantities.  An interesting direction is to understand the impact of patch embeddings on the inductive bias learned by ViTs. Moreover, our experiment on the Gaussian data shows that ViTs do not always learn the correct inductive bias under  \autoref{ass:data_dist}: characterizing the distributions under which ViTs recover the structure of the function is an important question. Lastly, this work also paves the way to many extensions beyond convolution. For example, can ViTs learn other inductive biases? What are the inductive biases learnt by Transformers in NLP? Answering those questions is central to  better understand the underlying mechanism of attention.

\begin{ack}
The authors would like to thank Boris Hanin for helpful discussions and feedback on this work.
\end{ack}
%Our data distribution model (\autoref{def:datadist}) is general as it models any arbitrary sets. It would be interesting to see whether our data model can be used to characterize the behavior of Transformers in other tasks such as NLP. }

%Inspired by computer vision datasets and ViTs,  we propose a theoretical framework  to understand this question. 

% work is a first step towards understanding how transformers learn tailored inductive biases when trained on large datasets using gradient descent. Our framework is general as  \autoref{def:datadist} cristallizes  Our framework  While convolutional sets are one of the main modelling  emphasize that our framework is general and not limited to convolutional sets. It would be interesting to see whether our framework   analysis relies on the positional attention mechanism to disentangle patch embeddings and positional encodings in the ViT model. However, in practice, self-attention mixes positional encodings and patch embeddings.  An interesting direction is to understand the impact of the patch embeddings on the inductive bias learned by ViTs. 

%Besides, our work highlight the importance of the positional encoding to learn spatial information in computer vision. 
%how these latter It would be interesting to  understand what inductive biases are important in NLP, whether they rely so heavily on the positional encoding and how they are learned.
% focuses on the positional encoding to characterize the bias learnt by ViTs. For this reason, we introduced

\bibliography{references}
\bibliographystyle{unsrtnat}

%\newpage

%%%%%%%%%%%%%%%%%%%%%%%%%%%%%%%%%%%%%%%%%%%%%%%%%%%%%%%%%%%%
\section*{Checklist}

%%% BEGIN INSTRUCTIONS %%%

\begin{enumerate}

\item For all authors...
\begin{enumerate}
  \item Do the main claims made in the abstract and introduction accurately reflect the paper's contributions and scope?
    \answerYes{See \autoref{sec:theory} and \autoref{sec:transfer}.}
  \item Did you describe the limitations of your work?
    \answerYes{See \autoref{sec:conclusion}.}
  \item Did you discuss any potential negative societal impacts of your work?
    \answerNA{This is a theory paper.}
  \item Have you read the ethics review guidelines and ensured that your paper conforms to them?
   \answerYes{}
\end{enumerate}

\item If you are including theoretical results...
\begin{enumerate}
  \item Did you state the full set of assumptions of all theoretical results?
    \answerYes{See \autoref{sec:setting}.}
        \item Did you include complete proofs of all theoretical results?
    \answerYes{See Appendix.}
\end{enumerate}

\item If you ran experiments...
\begin{enumerate}
  \item Did you include the code, data, and instructions needed to reproduce the main experimental results (either in the supplemental material or as a URL)?
    \answerYes{See supplementary material.}
  \item Did you specify all the training details (e.g., data splits, hyperparameters, how they were chosen)?
    \answerYes{See Appendix.}
        \item Did you report error bars (e.g., with respect to the random seed after running experiments multiple times)?
    \answerYes{See \autoref{sec:num_exp}}
        \item Did you include the total amount of compute and the type of resources used (e.g., type of GPUs, internal cluster, or cloud provider)?
    \answerYes{See Appendix.}
\end{enumerate}

\item If you are using existing assets (e.g., code, data, models) or curating/releasing new assets...
\begin{enumerate}
  \item If your work uses existing assets, did you cite the creators?
    \answerYes{}
  \item Did you mention the license of the assets?
   \answerNA{}
  \item Did you include any new assets either in the supplemental material or as a URL?
\answerNA{}
\item Did you discuss whether and how consent was obtained from people whose data you're using/curating?
\answerNA{}
  \item Did you discuss whether the data you are using/curating contains personally identifiable information or offensive content?
\answerNA{}
\end{enumerate}

\item If you used crowdsourcing or conducted research with human subjects...
\begin{enumerate}
  \item Did you include the full text of instructions given to participants and screenshots, if applicable?
 \answerNA{}
  \item Did you describe any potential participant risks, with links to Institutional Review Board (IRB) approvals, if applicable?
   \answerNA{}
  \item Did you include the estimated hourly wage paid to participants and the total amount spent on participant compensation?
   \answerNA{}
\end{enumerate}

\end{enumerate}

\vfill

%%%%%%%%%%%%%%%%%%%%%%%%%%%%%%%%%%%%%%%%%%%%%%%%%%%%%%%%%%%%
\pagebreak

%\newpage
\appendix
\section{Additional experimental details}\label{sec:appexp}

In this section, we provide additional details on our experiments and additional plots. 

\subsection{Details on the implementation}

We used Pytorch and Nvidia Tesla V100 GPUs. We conduct experiments on small-scale (CIFAR-10/100 and SVHN) and large-scale datasets (ImageNet). The choice of architecture and training parameters depend on the size of the dataset as we detail below. 

\paragraph{Small-scale experiments.} We use the code available at \url{https://github.com/omihub777/ViT-CIFAR}. The model is made of 7 layers, 12 heads, hidden and MLP dimension 384, dropout 0. We use "mean-pooling" and not the CLS pooling. We set the patch size to 2 in the experiment \autoref{tab:perm_fig} and to 4 in the experiment \autoref{tab:barplots_fig}. Indeed, we empirically found that setting patch size 4 was the optimal choice. We apply label smoothing \citep{szegedy2016rethinking} with coefficient 0.1 and do not apply any cutmix \citep{zhang2017mixup} nor mixup \citep{yun2019cutmix}. We use Adam \citep{kingma2014adam} as optimizer and set the learning rate to $10^{-3}$, minimum learning rate to $10^{-5}$, $\beta_1$ to $0.9$, $\beta_2$ to $0.999$, batch size to $128$, weight decay to $5\cdot10^{-5}$, number of warmup epochs to 5 and number of total epochs to 200. The scheduler is a cosine learning rate. We used the AutoAugment procedure \citep{cubuk2018autoaugment} as in the repository to generate data augmentations. The model has been trained over a single GPU. 

Regarding the convolutional models in the experiment \autoref{tab:perm_fig}, we trained a ResNet-18 and a VGG-19 with batch normalization. We trained the two architectures using the same training procedure and hyperparameters as for the ViT.

\paragraph{Large-scale experiments.} We use the code available at \url{https://github.com/facebookresearch/deit}. Due to limited computational resources, we train a ViT-tiny-patch16-224 \citep{dosovitskiy2020image} where "CLS-pooling" is applied. A detailed table with the hyperparameters used for the ImageNet experiment may be found in Table 9 (column "DeiT-B") in \citep{touvron2021training}. We set no dropout but set stochastic depth \citep{huang2016deep} 0.1. We used label smoothing  0.1. Regarding the augmentations, we set RandAugment \citep{cubuk2018autoaugment} 9/0.5, mixup  0.8, cutmix  1, erasing probability \citep{zhong2020random} 0.25. Lastly, we trained the model using AdamW \citep{loshchilov2017decoupled} and set the batch size to 1024, learning rate to $5\cdot10^{-4}\cdot\frac{\mathrm{batch size}}{512}$ as in \citep{goyal2017accurate}, weight decay to $0.05$, warmup epochs $5$ and number of total epochs to 300. The total number of epochs is 300. The model has been trained over 16 GPUs (8 nodes and 2 GPUs per node) and batch size for each device is 64. 

\subsection{Additional plots}

In \autoref{tab:perm_fig}, we plot the positional encoding similarities for a few patches. \autoref{tab:pos_full} provides these plots for all the patches. One should think of \autoref{tab:perm_fig} as a Figure displaying just two of the arrays present in \autoref{tab:pos_full}. We consistently verify that the ViT is always able to recover the convolution-like patterns which shows that it is able to learn the right patch association.

\begin{figure}[tbp]
\vspace*{-1.5cm}

\begin{subfigure}{0.48\textwidth}%0.33
 %\vspace{-1mm}
 
 \includegraphics[width=.8\linewidth]{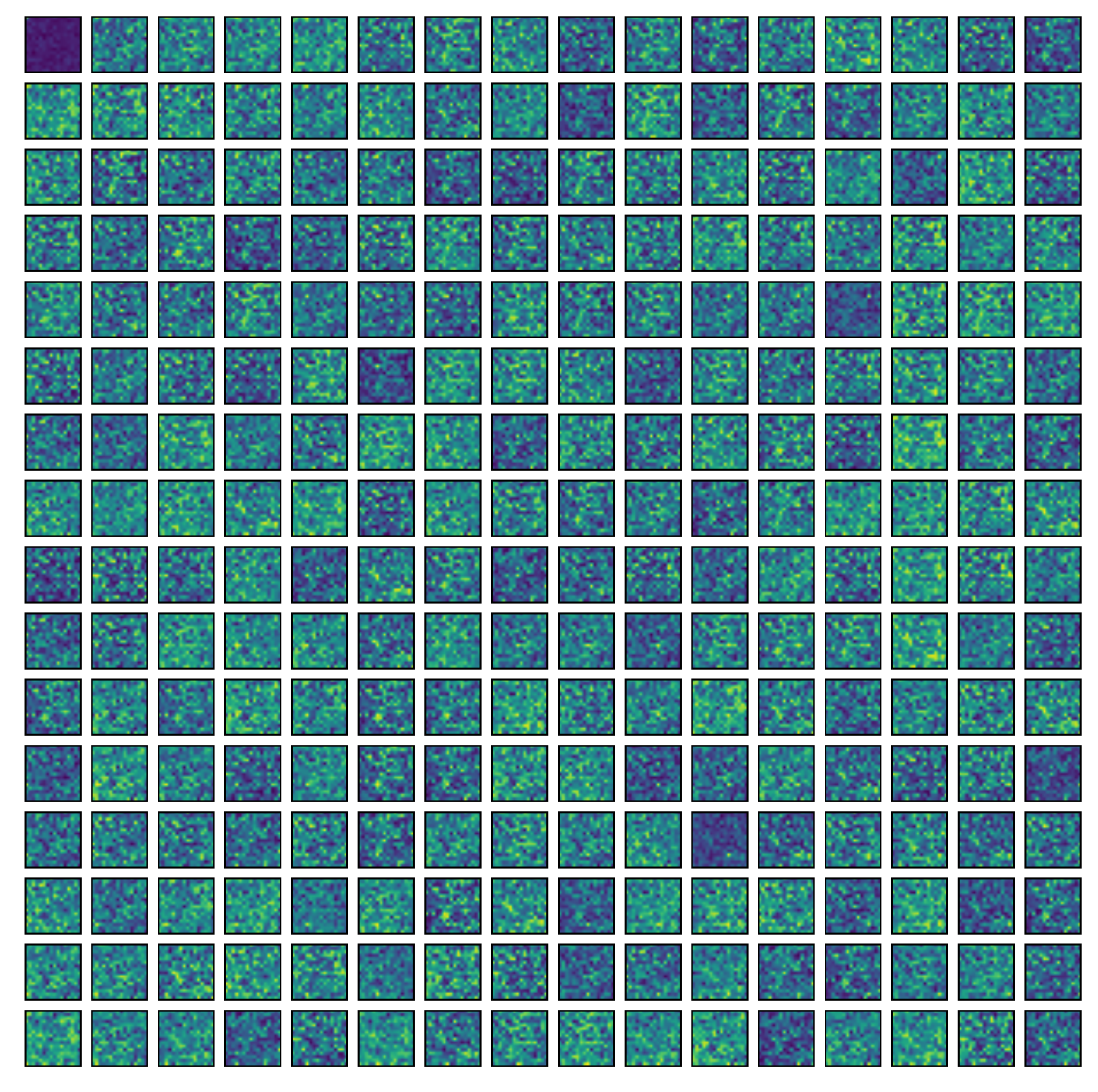}
 \vspace{-2mm}
 \caption{}\label{fig:attn_perm_full}
\end{subfigure}
\begin{subfigure}{0.48\textwidth}%0.33
\hspace*{1.cm}
%\centering
 \includegraphics[width=.8\linewidth]{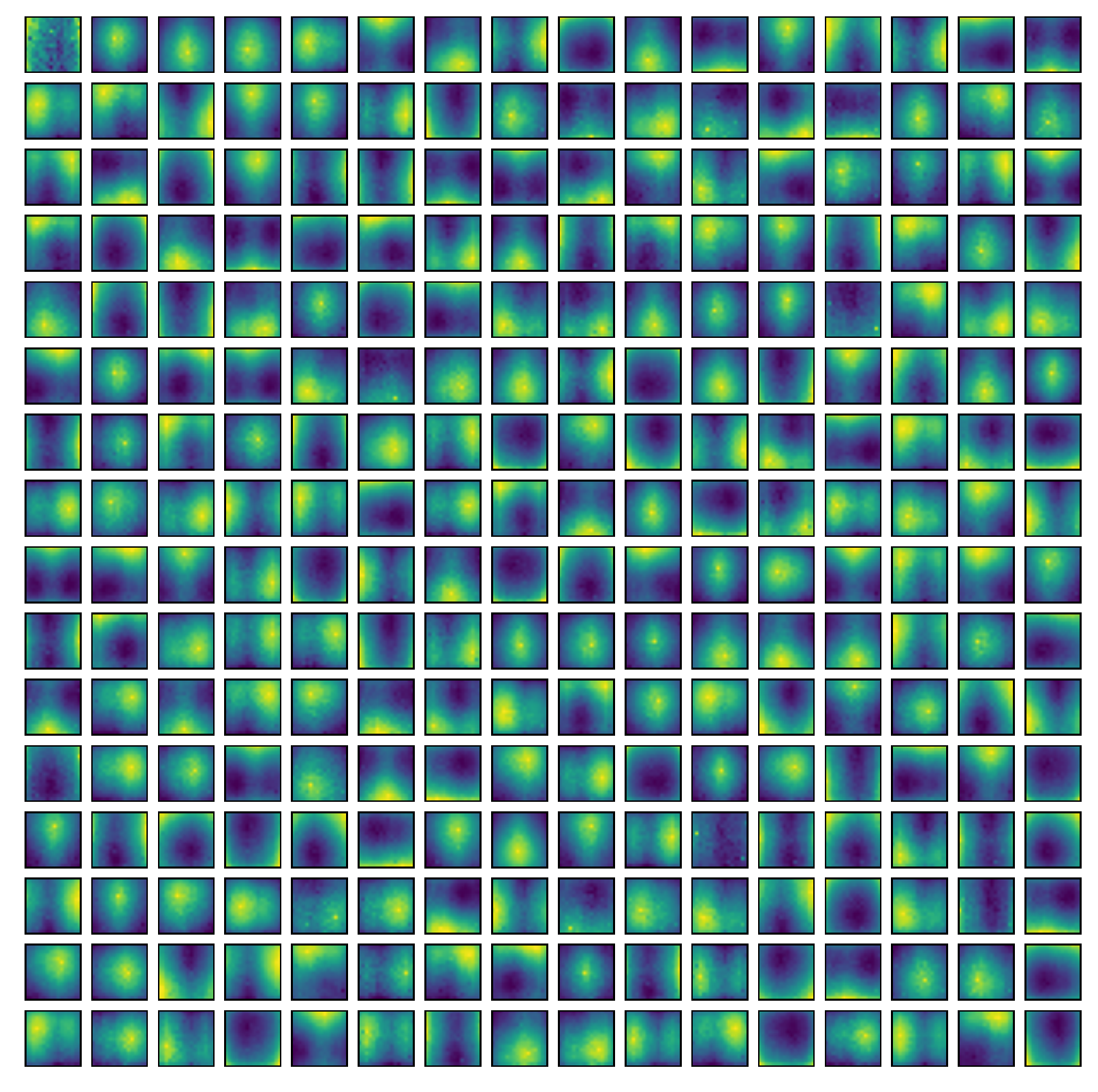}
  %\vspace{2mm}
\caption{}\label{fig:attn_vanilla_full}
\end{subfigure}
\vspace{-3mm}
\caption{\small (a) Visualization of the positional encodings similarities when feeding the ViT with Permuted CIFAR-100 data. Each array represents $\{\langle \bm{p}_{i},\bm{p}_j\rangle\}_{j=1}^D$ for a fixed $i.$ (b) displays the positional encodings similarities obtained after inverting the permutation. We see that the ViT is able to recover the convolution-like structure in all the cases.
}\label{tab:pos_full}
\end{figure}
\section{Induction hypothesis}

In this section, we present the induction hypothesis that we use in the analysis of the idealized case. This hypothesis is ultimately proved in \autoref{sec:indhideal}.

\begin{induction}\label{indh:lambgamxi} During the idealized learning process, the following holds for $t\leq T$.
\begin{itemize}
    \item[--] the sofmax denominator  is large i.e. $\Lambda^{(t)}+(C-1)\Gamma^{(t)}+(D-C)\Xi^{(t)}=\Theta(D).$
    \item[--] $\Xi^{(t)}$ is not too small i.e. $\Xi^{(t)}=\Theta(1/D).$
    \item[--] $\Gamma^{(t)}$ and $\Lambda^{(t)}$ are in a good range i.e. 
    \begin{align*}
        \Lambda^{(t)}=\frac{e^{\mathrm{polyloglog}(d)}}{D},\quad(C-1)\Gamma^{(t)}\in\left[\frac{\Omega(C)}{D},\frac{\lambda_0}{D}\right],
    \end{align*}
where $\lambda_0=\Theta(D^{0.01}).$ 
\end{itemize}

\end{induction} 
\section{Notations}

In this section, we introduce the different notations used in the proofs. 
\paragraph{General purpose.} We first define  notations that are used everywhere in the appendix.
\begin{itemize}
    \item[--] Sigmoid function: Given $z\in\mathbb{R}$, $\mathfrak{S}(z) = 1/(1+e^{-z}). $
    \item[--] Softmax function: Given $\bm{z}=(z_1,\dots,z_D)\in\mathbb{R}^D$,  $(\mathrm{softmax}(z_1,\dots,z_D))_m = \frac{e^{z_m}}{\sum_{j=1}^D e^{z_j}}.$
    \item[--] Loss for a data-point $(\bm{X},y)$: $L(\bm{X})=\log(1+e^{-yF(\bm{X})}).$
\end{itemize}
\paragraph{Analysis in idealized case.} We now provide  notations used in the analysis of the idealized case.
\begin{itemize}
    \item[--] $\kappa(\bm{X})=\sum_{\ell\neq\ell(\bm{X})}\sum_{j\in\mathcal{S}_{\ell}}\delta_j +C.$
    \item[--] for $i\in[D]$, $\bm{O}_i^{(t)}=\sum_{j=1}^D S_{i,j}^{(t)}\bm{X}_j.$
\end{itemize}

\paragraph{Analysis in realistic case.} We now provide  notations used in the analysis of the realistic case. 
\begin{itemize}
    \item[--] Score matrix: $\bm{S}\in\mathbb{R}^{D\times D}$ with coefficients $\widehat{S}_{i, j} = \exp(\widehat{A}_{i,j}/\sqrt{d})/ \sum_{r=1}^D \exp(\widehat{A}_{i,r}/\sqrt{D})$.
    \item[--] Given a data-point $(\bm{X}[i],y[i])$ and $j\in[D]$, $\bm{O}_j^{(t)}[i]=\sum_{k=1}^D S_{j,k}^{(t)}\bm{X}_k[i]$
\end{itemize}

%\begin{itemize}
%    \item[--] %$\kappa(\bm{X})=\sum_{\ell\neq\ell(\bm{X})}\sum_{j\in\mathcal{S}_{\ell}}\delta_j +C.$
    %\item[--] $\widehat{\bm{v}}_{w}^{(t)}$
   % \item[--] 
    %\item[--] $\bm{O}_j^{(t)}[i]$
    %\item[--] $\widehat{G}^{(t)} = \frac{1}{N}\sum_{i=1}^N \sum_{j\in\mathcal{S}_{\ell(\bm{X}[i])}}\widehat{G}_j^{(t)}$
    %\item[--] With high probability
   % \item[--] $L(\bm{X})$
   % \item[--] $\bm{O}_i^{(t)}$
    %\item[--] $\widehat{S}$
    %\item[--] $\mathfrak{S}$
%\end{itemize} 

\section{Learning process in the idealized setting}\label{sec:ideal}

\subsection{Roadmap of the proof}

From \autoref{lem:v_spanw}, we know that $\bm{v}^{(t)}\in\mathrm{span}(\bm{w}^*)$ for all $t\in[T]$. The main idea of the proof consists in analyzing the GD dynamics of $\alpha^{(t)}=\langle \bm{v}^{(t)},\bm{w}^*\rangle$ that satisfy 

\begin{align}
    &\hspace{-.5cm}\alpha^{(t+1)}= \alpha^{(t)}+\eta\underbrace{\mathbb{E}\Bigg[y\mathfrak{S}(-yF_{\bm{A}^{(t)},\bm{v}^{(t)}}(\bm{X}))\hspace{-.3cm}\sum_{i\in\mathcal{S}_{\ell(\bm{X})}}\hspace{-.3cm} \sigma'\Big(\sum_{j=1}^D S_{i,j}^{(t)}\langle \bm{v}^{(t)},\bm{X}_j\rangle\Big)\sum_{k=1}^D S_{i,k}^{(t)}\langle \bm{w}^*,\bm{X}_m\rangle\Bigg]}_{{\Large \pmb{\mathpzc{S}}}^{(t)}}\nonumber\\
    &+\eta\underbrace{\mathbb{E}\Bigg[y\mathfrak{S}(-yF_{\bm{A}^{(t)},\bm{v}^{(t)}}(\bm{X}))\hspace{-.3cm}\sum_{i\not\in\mathcal{S}_{\ell(\bm{X})}}\hspace{-.3cm} \sigma'\Big(\sum_{j=1}^D S_{i,j}^{(t)}\langle \bm{v}^{(t)},\bm{X}_j\rangle\Big)\sum_{m=1}^D S_{i,m}^{(t)}\langle \bm{w}^*,\bm{X}_m\rangle\Bigg]}_{{\Large \pmb{\mathpzc{N}}^{(t)}}}.\label{eq:upd_alpha}\tag{GD-$\alpha$}
\end{align}

We divide the idealized learning process as follows. 
\begin{itemize}[leftmargin=*, itemsep=1pt, topsep=1pt, parsep=1pt]
    \item[--] \textbf{Event I} ($t\in[0,\mathcal{T}_0]$, \autoref{sec:initalphaincr}): at initialization, $\alpha^{(0)}$ is small. Therefore, the sigmoid $\mathfrak{S}(-yF(\bm{X}))$ is large. Besides, around $\alpha^{(0)}$, it stays constant i.e.\ $\mathfrak{S}(-yF_{\bm{A}^{(t)},\bm{v}^{(t)}}(\bm{X}))\approx \mathfrak{S}(-yF_{\bm{A}^{(0)},\bm{v}^{(0)}}(\bm{X}))$ in \eqref{eq:upd_alpha}. This implies that ${\large \pmb{\mathpzc{N}}}^{(t)}=0$ which yields $\alpha^{(t)}$ to increase until reaching a specific value where the sigmoid is not constant anymore.
    \item[--] \textbf{Event II} ($t\in[\mathcal{T}_0,\mathcal{T}_1]$, \autoref{sec:event2}): at time $\mathcal{T}_0$, $\alpha^{(t)}$ is large. This fact along with  \autoref{lem:rhogammaupd} imply that $\gamma^{(t)}$ increases. Eventually,  $\Gamma^{(\mathcal{T}_1)}$ becomes large enough so that ${\large \pmb{\mathpzc{S}}}^{(t)}\geq\max_{\tau\leq T}{\large |\pmb{\mathpzc{N}}}^{(\tau)}|$.
  %\autoref{lem:gd_gamma} and \autoref{lem:gd_rho} imply that $\gamma^{(t)}$ (and potentially $\rho^{(t)}$) increase. $\Gamma^{(t)}$ .
  \item[--] \textbf{Event III} ($t\in[\mathcal{T}_1,T]$, \autoref{sec:event3}): Since ${\large \pmb{\mathpzc{S}}}^{(t)}\geq\max_{\tau\leq T}{\large |\pmb{\mathpzc{N}}}^{(\tau)}|$,  $\alpha^{(t)}$ increases again. It increases until the population risk is at most $o(1).$
\end{itemize}
After $T$ iterations, $\alpha^{(t)}$ is large and the population risk thus converges (\autoref{sec:convloss}). Since the logistic loss is a surrogate for the 0-1 loss, we prove that the learner model fits the labeling function (\autoref{sec:generalization_ideal}) which implies the first statement of \autoref{thm:ideal}.

\textbf{Remark :} Since we initialize $\alpha^{(0)}\geq \nu^{1/(p-1)}$, \autoref{lem:activ_simple} implies that we can overlook the linear part of the activation in this section. Therefore, we only consider $\sigma(x)=x^p$ in the idealized process.

\subsection{Event I: $\alpha^{(t)}$ initially increases}\label{sec:initalphaincr}

A first question that arises is: starting from $\alpha^{(0)}$, what is the value of $\alpha^{(t)}$  that makes the sigmoid non-constant? The following lemma addresses this question.

\begin{lemma}\label{lem:signonconstant}
The value $\alpha^{(t)}$ at which the sigmoid $\mathfrak{S}(-yF(\bm{X}))$ becomes non-constant is: 
\begin{align*}
    \tilde{\alpha}=\frac{\Theta(1)}{C^{2}\lambda_0}. %\frac{\Theta(1)}{C(e^{\mathrm{polyloglog}(d)}+(C-1)\lambda_0)}. %({\large \pmb{\mathpzc{S}}}^{(t)}+{\large \pmb{\mathpzc{N}}}^{(t)})
\end{align*}
\end{lemma}
\begin{proof}[Proof of \autoref{lem:signonconstant}]  The update of  $\alpha^{(t)}$ is
\begin{equation}\label{eq:alphaupd1}
\begin{aligned}
    \alpha^{(t+1)}&=\alpha^{(t)}+\eta\mathbb{E}[y\mathfrak{S}(-yF(\bm{X}))\mathscr{G}^{(t)}]\\
    &=\alpha^{(t)}+\eta\mathbb{E}[y\mathfrak{S}(-\alpha^{(0)})\mathscr{G}^{(t)}]+\eta\mathbb{E}\big[y\big(\mathfrak{S}(-yF(\bm{X}))-\mathfrak{S}(-\alpha^{(0)})\big)\mathscr{G}^{(t)}\big],
\end{aligned}  
\end{equation}
where $\mathscr{G}^{(t)}:=\sum_{j=1}^D \langle \bm{O}_j^{(t)},\bm{w}^*\rangle^p.$ Since $x\mapsto\mathfrak{S}(-x)$ is $1/4$-Lipschitz,  we rewrite  \eqref{eq:alphaupd1} as: 
\begin{equation}
\begin{aligned}\label{eq:oonwoenwwdnedde}
    \hspace{-.3cm}\big| \alpha^{(t+1)}-\alpha^{(t)}-\eta\mathbb{E}[y\mathfrak{S}(-\alpha^{(0)})\mathscr{G}^{(t)}]\big|&\leq \frac{\eta}{4}\mathbb{E}\big[|F(\bm{X})-y\alpha^{(0)}|\cdot|y\mathscr{G}^{(t)}|\big]\\
    &\leq \Theta(C\eta)(\alpha^{(t)})^{p}\mathbb{E}\big[(C^{1+1/p} \lambda_0)^p|y\mathscr{G}^{(t)}|\big]\\
    &\leq\Theta(\eta)(\alpha^{(t)})^{p}\mathbb{E}\big[(C^{2} \lambda_0)^p|y\mathscr{G}^{(t)}|\big],
\end{aligned}
\end{equation}

where we applied \autoref{lem:Fbdalph}, \autoref{indh:lambgamxi} and the fact that $\alpha^{(0)}$ is small in the penultimate inequality. Using \autoref{lem:singbdT1} and \autoref{lem:noisbdT1}, we have  $\mathbb{E}[|y\mathscr{G}^{(t)}|]\leq\Theta(1)\mathbb{E}[y\mathscr{G}^{(t)}]$ which yields
\begin{align}
    \big| \alpha^{(t+1)}-\alpha^{(t)}-\eta\mathbb{E}[y\mathfrak{S}(-\alpha^{(0)})\mathscr{G}^{(t)}]\big|&\leq\Theta(\eta)(\alpha^{(t)})^{p}\mathbb{E}\big[(C^{2} \lambda_0)^p y\mathscr{G}^{(t)}\big].
\end{align}

\eqref{eq:oonwoenwwdnedde} shows that when $\alpha^{(t)}$ is small, we have $\alpha^{(t+1)}\approx \alpha^{(t)}+\eta\mathbb{E}[y\mathfrak{S}(-\alpha^{(0)})A]$. Besides, we have  $\mathbb{E}[y\mathscr{G}^{(t)}]\geq0$ so $\alpha^{(t)}$ increases. However, during this increase, the right-hand side of \eqref{eq:oonwoenwwdnedde} increases and this approximation does not hold anymore. Therefore,  the sigmoid is approximately constant when $\alpha^{(t)}$ satisfies: 
\begin{align}\label{eq:fjendwwe}
   \Theta(1)(\alpha^{(t)})^{p}\mathbb{E}[(C^2\lambda_0)^py\mathscr{G}^{(t)}]\leq   \mathbb{E}[\mathfrak{S}(-\alpha^{(0)})y\mathscr{G}^{(t)}] \iff  \alpha^{(t)}\leq\frac{\Theta(1)}{C^{2}\lambda_0}.
\end{align}
\end{proof}

%Using \autoref{lem:signonconstant}, we can simplify $\alpha^{(t)}$'s update rule and analyze it. 
%\begin{lemma}\label{lem:phase1}
%Let $\mathcal{T}_0=\Theta\left(\frac{1}{\eta C(\alpha^{(0)})^{p-1}e^{\mathrm{polyloglog}(d)}}\right)$ and $t\in[0,\mathcal{T}_0].$ During this phase, $\alpha^{(t)}$ is updated as
%\begin{align}
%    \alpha^{(t+1)}\geq \alpha^{(t)}+B(\alpha^{(t)})^{p-1},
%\end{align}
%where $B>0$ is some constant and  $\alpha^{(t)}\geq\frac{\Omega(1)}{C^2\lambda_0}$ for $t\geq \mathcal{T}_0.$
%\end{lemma}
\begin{restatable}{lemma}{lemeventone}\label{lem:event1}
Let $\mathcal{T}_0=\Theta\left(\frac{1}{\eta C(\alpha^{(0)})^{p-1} }\right)$. For all $t\in[0,\mathcal{T}_0]$, we have ${\large \pmb{\mathpzc{N}}^{(t)}}=0.$ Therefore, $\alpha^{(t)}$ is updated as
\begin{align*}
    \alpha^{(t+1)}&=\alpha^{(t)}+\Theta(\eta C) ( G^{(t)} )^{p}(\alpha^{(t)})^{p-1}.
\end{align*}
Consequently, $\alpha^{(t)}$ is non-decreasing and after $\mathcal{T}_0$ iterations, we have $\alpha^{(t)}\geq\frac{\Omega(1)}{C^2\lambda_0}$ for $t\geq \mathcal{T}_0.$
\end{restatable}
\begin{proof}[Proof of \autoref{lem:event1}] For $t\in[0,\mathcal{T}_0]$, we know that the sigmoid $\mathfrak{S}(-yF(\bm{X}))$ is constant. We apply \autoref{lem:singbdT1} and \autoref{lem:exppowdeltas} to respectively bound ${\large \pmb{\mathpzc{S}}}^{(t)}$ and ${\large \pmb{\mathpzc{N}}}^{(t)}$ in the update of $\alpha^{(t)}$.
\begin{align}\label{eq:fhewbwei}
    \alpha^{(t+1)}&=\alpha^{(t)}+\Theta(\eta C) (\alpha^{(t)})^{p-1}( G^{(t)} )^{p}.
\end{align} 
We apply \autoref{indh:lambgamxi} and $\Omega(C)\geq 0$ and $e^{\mathrm{polyloglog}(d)}\leq \lambda_0$ in \eqref{eq:fhewbwei} and obtain: 
\begin{align}\label{eq:eecjowe}
    \begin{cases}
      \alpha^{(t+1)}\geq \alpha^{(t)}+\Theta(\eta C)e^{\mathrm{polyloglog}(d)} (\alpha^{(t)})^{p-1}\\
      \alpha^{(t+1)}\leq \alpha^{(t)}+\Theta(\eta C) \lambda_0^p(\alpha^{(t)})^{p-1}
    \end{cases}.
\end{align}
%\begin{align}\label{eq:eecjowe}
%    \begin{cases}
%      \alpha^{(t+1)}\geq \alpha^{(t)}+\Theta(\eta C)(e^{\mathrm{polyloglog}(d)}+\Omega(C))^p(\alpha^{(t)})^{p-1}\\
%      \alpha^{(t+1)}\leq \alpha^{(t)}+\Theta(\eta C)(e^{\mathrm{polyloglog}(d)}+\lambda_0)^p(\alpha^{(t)})^{p-1}
%    \end{cases}.
%\end{align}
\eqref{eq:eecjowe} indicates that $\alpha^{(t)}$ is a non-decreasing sequence. Therefore, there exists a time $\mathcal{T}_0$ such that $\alpha^{(\mathcal{T}_0)}=\frac{\Theta(1)}{C^{2}\lambda_0}$.
%$\alpha^{(\mathcal{T}_0)}=\frac{\Theta(1)}{C(e^{\mathrm{polyloglog}(d)}+\lambda_0)}$.
Using \autoref{lem:pow_method}, the time $\mathcal{T}_0$ is equal to: 
\begin{align}
    \mathcal{T}_0=\frac{1}{\eta C(\alpha^{(0)})^{p-1}e^{\mathrm{polyloglog}(d)}}+(\lambda_0)^pe^{-\mathrm{polyloglog}(d)}\Big\lceil -\log\big(C^2\lambda_0\alpha^{(0)}\big)\Big\rceil.
\end{align}
\end{proof}

\subsubsection{Auxiliary lemmas}

In this section, we present the auxiliary lemmas needed to prove the main results of \autoref{sec:initalphaincr}. We first present a lemma that bounds the learner model.

\begin{lemma}\label{lem:Fbdalph} Let $t\in[0,T]$.
The learner model $F$ is bounded for all $(\bm{X},y)\sim\mathcal{D}$ as:
\begin{align*}
    yF_{\bm{A}^{(t)},\bm{v}^{(t)}}(\bm{X})\leq \Theta(1)(C^{1+1/p}\alpha^{(t)}\lambda_0)^p.
\end{align*}
\end{lemma}
\begin{proof}[Proof of \autoref{lem:Fbdalph}]
By definition, the learner model is: 
\begin{align}\label{eq:eojwdnowe}
   yF_{\bm{A}^{(t)},\bm{v}^{(t)}}(\bm{X}) = y\sum_{i=1}^D \langle \bm{v}^{(t)},\bm{O}_i^{(t)}\rangle^p=\sum_{i\in\mathcal{S}_{\ell(X)}}\langle \bm{v}^{(t)},\bm{O}_i^{(t)}\rangle^p+\sum_{i\not\in\mathcal{S}_{\ell(X)}}\langle v^{(t)},\bm{O}_i^{(t)}\rangle^p.
\end{align}
We successively apply \autoref{lem:singbdT1}, \autoref{indh:lambgamxi} and \autoref{lem:noisbdT1} to bound \eqref{eq:eojwdnowe}. 
\begin{align}\label{eq:jenwodn}
   yF_{\bm{A}^{(t)},\bm{v}^{(t)}}(\bm{X}) \leq (\alpha^{(t)})^p\Big(C \Theta(D C\Gamma^{(t)} )^p + o(\lambda_0)^p\Big).
\end{align}
Finally, we apply  \autoref{indh:lambgamxi} in \eqref{eq:jenwodn} to obtain the desired result.
\end{proof}
We now present lemmas that bound ${\large \pmb{\mathpzc{S}}}^{(t)}$ and ${\large \pmb{\mathpzc{N}}}^{(t)}$.
\begin{lemma}\label{lem:singbdT1} Let $t\in[T]$ and $i\in\mathcal{S}_{\ell(\bm{X})}.$ We have $ y\langle \bm{w}^{*},\bm{O}_i^{(t)}\rangle^p =\Theta(G^{(t)} )^p.$ 
%\begin{align*}
%    y\langle \bm{w}^{*},\bm{O}_i^{(t)}\rangle^p &=\Theta(G^{(t)} )^p.
%\end{align*}
As long as the population risk is not $o(1)$, we have ${\large \pmb{\mathpzc{S}}}^{(t)}= C\Theta(G^{(t)} )^p$ for all $t\leq T.$
\end{lemma}
\begin{proof}[Proof of \autoref{lem:singbdT1}] We apply \autoref{lem:bd_signnoisesig} and obtain:
\begin{equation*}
\begin{aligned}%\label{eq:fsewnwef}
  y\langle \bm{w}^{*},\bm{O}_i^{(t)}\rangle^p&= D^p \bigg( (\Lambda^{(t)}+(C-1)\Gamma^{(t)})+y\Xi^{(t)}\hspace{-.6cm}\sum_{h\in[L]\backslash\{\ell(\bm{X})\}}\sum_{r\in\mathcal{S}_h}\delta_{r}\bigg)^{p}\\
  &=\Theta \Big(D \big(\Lambda^{(t)}+(C-1)\Gamma^{(t)}\big)\Big)^p\\
  &=\Theta(G^{(t)} )^p.
    % +&\Theta(\eta )(\alpha^{(t)})^{p-1}D^p\mathbb{E}\left[y\mathfrak{S}(-yF(X))\hspace{-.6cm}\sum_{k\in[L]\backslash\{\ell(X)\}}\sum_{u\in\mathcal{S}_k}\left(\Lambda^{(t)}\delta_{k,u}(X)+\Gamma^{(t)}\hspace{-.4cm}\sum_{s\in\mathcal{S}_k\backslash\{u\}}\delta_{k,s}(X)+\right.\right.\\
    %   &\left.\left.\hspace{3.4cm}\Xi^{(t)}\left[Cy+\hspace{-.6cm}\sum_{h\in[L]\backslash\{\ell(X),k\}}\sum_{r\in\mathcal{S}_h}\delta_{h,r}(X)\right]\right)^{p}\right].
\end{aligned}
\end{equation*}
%where we used \autoref{indh:lambgamxi} in the last inequality of \eqref{eq:fsewnwef}.
\end{proof}

\begin{lemma}\label{lem:noisbdT1}
Let $t\in[T]$. We have $ \sum_{i\not\in\mathcal{S}_{\ell(\bm{X})}}y\langle \bm{w}^*,\bm{O}_i^{(t)}\rangle^p   \leq o(\lambda_0)^p.$
%\begin{align*}
%   \sum_{i\not\in\mathcal{S}_{\ell(\bm{X})}}y\langle \bm{w}^*,\bm{O}_i^{(t)}\rangle^p   \leq o(\lambda_0)^p.
%\end{align*}
In particular, this implies ${\large \pmb{\mathpzc{N}}}^{(t)}\leq o(\lambda_0)^p$ for all $t\leq T.$
\end{lemma}
\begin{proof}[Proof of \autoref{lem:noisbdT1}] We have:
\begin{equation}
\begin{aligned}\label{eq:jwfejfw}
 &\sum_{k\in[L]\backslash\{\ell(X)\}}\sum_{i\in\mathcal{S}_k}y\langle \bm{w}^*,\bm{O}_i^{(t)}\rangle^p \\
 \hspace{-.3cm}=& D^p\hspace{-.5cm}\sum_{k\in[L]\backslash\{\ell(X)\}}\sum_{m\in\mathcal{S}_k}y\bigg(\Lambda^{(t)}\delta_{m}+\Gamma^{(t)}\hspace{-.4cm}\sum_{s\in\mathcal{S}_k\backslash\{m\}}\hspace{-.3cm}\delta_{s}+\Xi^{(t)}\Big[Cy+\hspace{-.6cm}\sum_{h\in[L]\backslash\{\ell(X),k\}}\sum_{r\in\mathcal{S}_h}\delta_{r}\Big]\bigg)^{p}.
\end{aligned}  
\end{equation}
We distinguish two cases.
\begin{itemize}
    \item[--] $\delta_{r}=0$ for all $r\in\mathcal{S}_{\ell}:$ we apply \autoref{lem:chernoff_sumnoise} and obtain: 
\begin{equation}
\begin{aligned}
      \hspace{-1cm} \sum_{k\neq\ell(\bm{X})}\sum_{i\in\mathcal{S}_k} y\langle \bm{w}^*,\bm{O}_i^{(t)}\rangle^p &=\hspace{-.5cm} \sum_{k\neq \ell(\bm{X})}\sum_{m\in\mathcal{S}_k}y\Big(D\Xi^{(t)}\Big[Cy+\hspace{-.6cm}\sum_{h\neq\{\ell(\bm{X}),k\}}\sum_{r\in\mathcal{S}_h}\delta_{r}\Big]\Big)^{p}\\
        &\leq  D\cdot \Theta(D\Xi^{(t)}qD\log(d))^p.\label{eq:ewkfkwef}
\end{aligned}
\end{equation} 
\item[--] $\exists r\in \mathcal{S}_{\ell}$ such that $\delta_{r}\neq 0$: let $i\in\mathcal{S}_{k}$. We apply  \autoref{lem:smallsumdeltar} and obtain:
\begin{equation}
\begin{aligned}\label{eq:jcneiencae}
    &y\langle \bm{w}^*,\bm{O}_i^{(t)}\rangle^p \leq \Theta\big(D(\Lambda^{(t)}+O(1)\Gamma^{(t)})\big)^{p}\cdot \mathbf{1}_{\exists r\in\mathcal{S}_{k},\delta_{r}\neq 0}.
    %\cdot\\
    % &D\left[\Lambda^{(t)}|\delta_{k,u}(X)|+\Gamma^{(t)}\hspace{-.4cm}\sum_{s\in\mathcal{S}_k\backslash\{u\}}|\delta_{k,s}(X)|+\Xi^{(t)}\left(C+\hspace{-.6cm}\sum_{h\in[L]\backslash\{\ell(X),k\}}\sum_{r\in\mathcal{S}_h}|\delta_{h,r}(X)|\right)\right].
\end{aligned}
\end{equation}
We now sum \eqref{eq:jcneiencae} and apply \autoref{lem:chernoff_sumnoise} to obtain:
\begin{equation}
\begin{aligned}\label{eq:venejcneiencdcwae}
    \hspace{-.5cm}\sum_{k\neq\ell(\bm{X})}\sum_{i\in\mathcal{S}_k} y\langle \bm{w}^*,\bm{O}_i^{(t)}\rangle^p
     & \leq \Theta\big(D(\Lambda^{(t)}+O(1)\Gamma^{(t)})\big)^{p} \sum_{k\neq\ell(\bm{X})}\sum_{i\in\mathcal{S}_k}|\delta_{r}|\\
      & \leq \Theta\big(D(\Lambda^{(t)}+O(1)\Gamma^{(t)})\big)^{p}qD\log(d).
\end{aligned}
\end{equation}
We finally apply \autoref{indh:lambgamxi} to have $\Gamma^{(t)}\leq \lambda_0/C$ in \eqref{eq:venejcneiencdcwae} and get 
\begin{align}\label{eq:fejnnew}
    \hspace{-.5cm}\sum_{k\neq \ell(\bm{X})}\sum_{i\in\mathcal{S}_k} y\langle \bm{w}^*,\bm{O}_i^{(t)}\rangle^p
     & \leq  \Theta(\lambda_0/C)^pqD\log(d).
\end{align}
\end{itemize}
We finally plug \eqref{eq:ewkfkwef} and \eqref{eq:fejnnew} in \eqref{eq:jwfejfw} and obtain: 
\begin{align*}
    \sum_{k\neq\ell(\bm{X})}\sum_{i\in\mathcal{S}_k} y\langle \bm{w}^*,\bm{O}_i^{(t)}\rangle^p
     & \leq  D\cdot \Theta(D\Xi^{(t)}qD\log(d))^p + \Theta(\lambda_0/C)^pqD\log(d)\leq o(\lambda_0)^p.
\end{align*}
\end{proof}

\begin{lemma} \label{lem:exppowdeltas}
Let $t\in[0,\mathcal{T}_0].$ We have  ${\large \pmb{\mathpzc{N}}}^{(t)}=0.$
\end{lemma}
\begin{proof}[Proof of \autoref{lem:exppowdeltas}] %For the linear part of the activation, it's obvious that the expectation is zero. We therefore only focus on the power part. 
By definition of ${\large \pmb{\mathpzc{N}}}^{(t)}$, we have:
\begin{equation}
\begin{aligned}\label{eq:foewjjewf}
    &{\large \pmb{\mathpzc{N}}}^{(t)}\\
    =&\frac{1}{2}\bigg(\mathbb{E}\bigg[\mathfrak{S}(-F(\bm{X}))\hspace{-.3cm}\sum_{j\not\in\mathcal{S}_{\ell(\bm{X})}} \hspace{-.2cm}\langle \bm{w}^*,\bm{O}_j^{(t)}\rangle^{p}\Big|y=1\bigg] - \mathbb{E}\bigg[\mathfrak{S}(F(\bm{X}))\hspace{-.3cm}\sum_{j\not\in\mathcal{S}_{\ell(\bm{X})}} \hspace{-.2cm}\langle \bm{w}^*,\bm{O}_j^{(t)}\rangle^{p}\Big|y=-1\bigg]  \bigg)\\
    =&\Theta(1)\bigg(\mathbb{E}\bigg[\sum_{j\not\in\mathcal{S}_{\ell(\bm{X})}} \hspace{-.2cm}\langle \bm{w}^*,\bm{O}_j^{(t)}\rangle^{p}\Big|y=1\bigg] - \mathbb{E}\bigg[\sum_{j\not\in\mathcal{S}_{\ell(\bm{X})}} \hspace{-.2cm}\langle \bm{w}^*,\bm{O}_j^{(t)}\rangle^{p}\Big|y=-1\bigg]  \bigg),
\end{aligned}
\end{equation}
where we use  $\mathfrak{S}(-F(\bm{X}))\approx \mathfrak{S}(-F_{\bm{A}^{(0)},\bm{v}^{(0)}}(\bm{X})) $ for $t\in[0,\mathcal{T}_0]$ in the last equality of \eqref{eq:foewjjewf}.
We now show that each of the summands in \eqref{eq:foewjjewf} is zero. Without loss of generality, let's focus on the first summand. The same reasoning holds for the second one. In particular, for $j\in\mathcal{S}_{k}$ with $k\neq j$, we now compute $\mathbb{E}[\langle \bm{w}^*,\bm{O}_j^{(t)}\rangle^{p}|y=1].$ Using the binomial theorem and the independence of the $\delta_{r}$'s, we have:
\begin{equation}
\begin{aligned}\label{eq:fidekowe}
    &\mathbb{E}[\langle \bm{w}^*,\bm{O}_j^{(t)}\rangle^{p}|y=1]\\
    \hspace{-.5cm}=&\sum_{a=0}^{p}\binom{p}{a} (\Lambda^{(t)})^{p-a}\mathbb{E}[(\delta_{j})^{p-a}]\mathbb{E}\bigg[\bigg(\Gamma^{(t)}\hspace{-.4cm}\sum_{s\in\mathcal{S}_k\backslash\{j\}}\delta_{s}+\Xi^{(t)}\bigg[C+\hspace{-.6cm}\sum_{h\neq\{\ell(\bm{X}),k\}}\sum_{r\in\mathcal{S}_h}\delta_{r}\bigg]\bigg)^a\bigg].
\end{aligned}   
\end{equation}
For $a$ even, we have $p-a$ odd which implies  $\mathbb{E}[(\delta_{u})^{p-a}]=0$. Therefore, the summands with even $a$ are zero. We now focus on the case $a$ odd. We again apply the binomial theorem and the independence of the $\delta_{r}$'s to get: 
\begin{equation}
\begin{aligned}\label{eq:jcenjnac}
   &\mathbb{E}\bigg[\bigg(\Gamma^{(t)}\hspace{-.4cm}\sum_{s\in\mathcal{S}_k\backslash\{j\}}\delta_{s}+\Xi^{(t)}\bigg[C+\hspace{-.6cm}\sum_{h\neq\{\ell(\bm{X}),k\}}\sum_{r\in\mathcal{S}_h}\delta_{r}\bigg]\bigg)^a\bigg] \\
   =& \sum_{b=0}^{a}\binom{a}{b}(\Gamma^{(t)})^{a-b}\mathbb{E}\bigg[\bigg(\sum_{s\in\mathcal{S}_k\backslash\{j\}}\delta_{s}\bigg)^{a-b}\bigg](\Xi^{(t)})^{b}\mathbb{E}\bigg[\bigg(Cy+\hspace{-.6cm}\sum_{h\neq\{\ell(\bm{X}),k\}}\sum_{r\in\mathcal{S}_h}\delta_{r}\bigg)^b\bigg].
\end{aligned}
\end{equation}
For $b$ even, we have $a-b$ odd. This implies that $\mathbb{E}\big[\big(\sum_{s\in\mathcal{S}_k\backslash\{j\}}\delta_{s}\big)^{a-b}\big]=0$. In the case $b$ odd, we exactly use  the same argument and obtain: 
\begin{align}\label{eq:fjwef}
    \mathbb{E}\bigg[\bigg(Cy+\hspace{-.6cm}\sum_{h\neq\{\ell(\bm{X}),k\}}\sum_{r\in\mathcal{S}_h}\delta_{r}\bigg)^b\bigg]=0.
\end{align}
\eqref{eq:fjwef} implies \eqref{eq:jcenjnac} is zero and which lastly implies \eqref{eq:fidekowe} is zero. We conclude that $\mathbb{E}[\langle \bm{w}^*,\bm{O}_j^{(t)}\rangle^{p}|y=1]=\mathbb{E}[\langle \bm{w}^*,\bm{O}_j^{(t)}\rangle^{p}|y=-1]=0$ and thus ${\large \pmb{\mathpzc{N}}}^{(t)}=0.$

\end{proof}

\begin{lemma}\label{lem:activ_simple}
Let $(\bm{X},\cdot)\sim\mathcal{D}$ and $j\in[D]$.  Assume that $\alpha^{(t)}\geq \nu^{1/(p-1)}$. Then, we have:
\begin{align*}
    \sigma'\Big(D\sum_{k=1}^D \bm{S}_{j,k}^{(t)}\langle \bm{v},\bm{X}_k\rangle\Big) &=  \Theta(1)\Big(D\sum_{k=1}^D \bm{S}_{j,k}^{(t)}\langle \bm{v},\bm{X}_k\rangle\Big)^{p-1}. %\begin{cases}
          %D\alpha^{(t)}\sum_{k=1}^D \bm{S}_{j,k}^{(t)}\langle \bm{w}^*,\bm{X}_k\rangle\geq \nu^{1/(p-1)}\\
        %\Theta(\nu) & \text{otherwise}
    %\end{cases}.
\end{align*}
\end{lemma}
\begin{proof}[Proof of \autoref{lem:activ_simple}] We remind that the derivative of the activation function $\sigma'(x)= px^{p-1}+\nu$. We first remark that for all $x,$ $\sigma'(x)\geq px^{p-1}$. Besides, we have: 
\begin{align}\label{eq:condderpowerp}
    px^{p-1} +\nu \leq  \frac{3px^{p-1}}{2} \iff x\geq \Big(\frac{2\nu}{p}\Big)^{1/(p-1)}.
\end{align}
In our case, we have $x=D\alpha^{(t)}\sum_{k=1}^D \bm{S}_{j,k}^{(t)}\langle \bm{w}^*,\bm{X}_k\rangle.$ Using \autoref{indh:lambgamxi}, we have $x\leq C\lambda_0 \alpha^{(t)}.$ Therefore, a sufficient condition for \eqref{eq:condderpowerp} to hold is $\alpha^{(t)}\geq \nu^{1/(p-1)}/(C\lambda_0).$ Since $\lambda_0,C\ll \mathrm{poly}(d)$, we can simplify this condition as $\alpha^{(t)}\geq \nu^{1/(p-1)}$.

\end{proof}

\subsection{Event II: $\Gamma^{(t)}$  significantly increases}\label{sec:event2} 

 In this section, we show the increase of $\alpha^{(t)}$ for $t\in[0,\mathcal{T}_0]$ leads to the increase of  $\Gamma^{(t)}$. At time $\mathcal{T}_1>\mathcal{T}_0,$ $\Gamma^{(t)}$ is significantly large.
 
\begin{restatable}{lemma}{lemeventwo}\label{lem:event2} Let   $\mathcal{T}_1=\mathcal{T}_0+ \Theta\Big( \frac{\lambda_0^p}{\eta  e^{\mathrm{polyloglog}(d)}} \Big)$.
For all $t\in[\mathcal{T}_1,T]$, we have $\Gamma^{(t)}\geq \frac{\Omega(\lambda_0)}{D}$. This implies ${\large \pmb{\mathpzc{S}}}^{(t)}\geq\max_{\tau\leq T}|{\large \pmb{\mathpzc{N}}}^{(\tau)}|$.
\end{restatable}

%\begin{lemma}\label{lem:gammaxiT0} Let   $\mathcal{T}_1=\mathcal{T}_0+ \Theta\left( \frac{C^{2p}\lambda_0^p}{\eta  e^{\mathrm{polyloglog}(d)}} \log\left(\frac{\lambda_0}{C^2}\right) \right)$.
%Then, for all $t\in[\mathcal{T}_1.T]$, we have:
%\begin{align*}
%   C\Gamma^{(t)}\geq \Omega\left(\frac{\lambda_0}{D}\right).
%\end{align*}
%\end{lemma}
\begin{proof}[Proof of \autoref{lem:event2}] Let $t\in[\mathcal{T}_0,T]$ and $\tau\in[\mathcal{T}_0,t].$ Using \autoref{cor:gamupdate} and \autoref{indh:lambgamxi}, $\gamma^{(\tau)}$ satisfies:
\begin{align}\label{eq:updgamlwbdT0}
   \gamma^{(\tau+1)}&\geq\gamma^{(\tau)}+\Omega(\eta C)(\alpha^{(\tau)})^{p}e^{\mathrm{polyloglog}(d)}.
\end{align}
Summing \eqref{eq:updgamlwbdT0} for $\tau=\mathcal{T}_0,\dots,t-1$ yields
\begin{align}\label{lem:ejnerref}
    \gamma^{(t)}&\geq\gamma^{(\mathcal{T}_0)}+e^{\mathrm{polyloglog}(d)}\Omega(\eta C)\sum_{\tau=\mathcal{T}_0}^{t-1}(\alpha^{(\tau)})^{p}.
\end{align}
We successively apply \autoref{lem:alphapos} and $(a-b)^p\geq a^p-pb$ for $a\ll b$ to lower bound \eqref{lem:ejnerref} to obtain: 
\begin{align}\label{lem:ejnerref111}
    \gamma^{(t)}&\geq\gamma^{(\mathcal{T}_0)}+e^{\mathrm{polyloglog}(d)}\Omega(\eta C)\sum_{\tau=\mathcal{T}_0}^{t-1}\Big(\frac{\Omega(1)}{C^{2p}\lambda_0^p}-o(\eta)\Big).
\end{align}

We apply \autoref{indh:lambgamxi} in \eqref{lem:ejnerref111} to obtain a bound on $\Gamma^{(t)}$.
\begin{equation}
\begin{aligned}\label{eq:fjnreje}
    C\Gamma^{(t)}&\geq C\Gamma^{(\mathcal{T}_0)} \exp\Big(e^{\mathrm{polyloglog}(d)}\Omega(\eta C)\sum_{\tau=\mathcal{T}_0}^{t-1}\big(\frac{\Omega(1)}{C^{2p}\lambda_0^p}-o(\eta)\big)\Big)\\
    &\geq \frac{\Omega(C^2)}{D} \exp\Big(e^{\mathrm{polyloglog}(d)}\Omega(\eta C)\sum_{\tau=\mathcal{T}_0}^{t-1}\big(\frac{\Omega(1)}{C^{2p}\lambda_0^p}-o(\eta)\big)\Big).
\end{aligned}    
\end{equation}
 \eqref{eq:fjnreje} shows that $\Gamma^{(t)}$ is an non-decreasing sequence. We thus deduce the time $\mathcal{T}_1$ such that $C\Gamma^{(t)}\geq \Omega(\lambda_0)/D.$
\begin{equation}
\begin{aligned}
   &\frac{C^2}{D} \exp\Big(e^{\mathrm{polyloglog}(d)}\Omega(\eta C)\big(\frac{\Omega(1)}{C^{2p}\lambda_0^p}-o(\eta)\big)(\mathcal{T}_1-\mathcal{T}_0)\Big) = \frac{\lambda_0}{D}\\
   &\implies \mathcal{T}_1 = \mathcal{T}_0+ \Theta\Bigg( \frac{C^{2p}\lambda_0^p}{\eta  e^{\mathrm{polyloglog}(d)}} \log\Big(\frac{\lambda_0}{C^2}\Big) \Bigg).
\end{aligned}    
\end{equation}
We now prove the second part of the lemma. We respectively apply \autoref{lem:singbdT1} and \autoref{lem:noisbdT1} to bound ${\large \pmb{\mathpzc{S}}}^{(t)}$ and $\max_{\tau\leq T}|{\large \pmb{\mathpzc{N}}}^{(\tau)}|$.
\begin{align}\label{eq:snbd}
    {\large \pmb{\mathpzc{S}}}^{(t)}\geq \Omega(\lambda_0)^p \quad \text{and} \quad \max_{\tau\leq [T]} |{\large \pmb{\mathpzc{N}}}^{(\tau)}|\leq o(\lambda_0)^p.
\end{align}
\eqref{eq:snbd} implies for all $t\in [ \mathcal{T}_1,T]$, ${\large \pmb{\mathpzc{S}}}^{(t)}\geq \max_{\tau\leq [T]} |{\large \pmb{\mathpzc{N}}}^{(\tau)}|.$

\end{proof}
 
\subsubsection{Auxiliary lemmas}

In this section, we present the auxiliary lemmas needed to prove the main results in \autoref{sec:event2}.

\begin{lemma}\label{lem:alphapos}
Let $t\geq \mathcal{T}_0.$ Then, we always have $\alpha^{(t)}\geq \frac{\Omega(1)}{C^2\lambda_0}-o(\eta).$
\end{lemma}
\begin{proof}[Proof of \autoref{lem:alphapos}] For $t\in[0,\mathcal{T}_0]$, $\alpha^{(t)}$ increases and eventually satisfies $\alpha^{(t)}\geq\frac{\Omega(1)}{C^2\lambda_0}$ (\autoref{lem:event1}). However, for $t\geq \mathcal{T}_0$, $\alpha^{(t)}$ may be non-increasing. Here, we want to quantify the maximum amount of decrease for $t>\mathcal{T}_0$. The worst-case scenario is when $\alpha^{(t)}=\frac{\Omega(1)}{C^2\lambda_0}$. We bound $\alpha^{(t+1)}$ by using \autoref{lem:singbdT1}, \autoref{lem:sigmfunct} and \autoref{lem:noisbdT1}.
%\begin{align}
%    &\alpha^{(t+1)}-\alpha^{(t)}\nonumber\\
%    =&\eta(\alpha^{(t)})^p
%    \mathbb{E}\Bigg[y\mathfrak{S}(-yF(X))\Big(\sum_{i\in\mathcal{S}_{\ell(X)}} \langle v^{(t)},O_i^{(t)}\rangle^{p-1}\langle O_i^{(t)},w^*\rangle +\sum_{i\not\in\mathcal{S}_{\ell(X)}} \langle v^{(t)},O_i^{(t)}\rangle^{p-1} \langle O_i^{(t)},w^*\rangle \Big) \bigg]\nonumber\\
%    \geq&  \eta(\alpha^{(t)})^p \Bigg(\Theta(1)
%    \mathbb{E}\bigg[ \sum_{i\in\mathcal{S}_{\ell(X)}} \langle v^{(t)},O_i^{(t)}\rangle^{p-1}\langle O_i^{(t)},w^*\rangle\bigg]-\mathbb{E}\bigg[\sum_{i\not\in\mathcal{S}_{\ell(X)}} \langle v^{(t)},O_i^{(t)}\rangle^{p-1} \langle O_i^{(t)},w^*\rangle  \bigg]\Bigg),\label{eq:fjwjwf}
%\end{align}
%\hspace{-.6cm}where we used \autoref{lem:sigmfunct} in \eqref{eq:fjwjwf}. We apply \autoref{lem:singbdT1} and \autoref{lem:noisbdT1}  in \eqref{eq:fjwjwf} to obtain:
\begin{align}\label{eq:woenewef}
     \alpha^{(t+1)}&\geq \alpha^{(t)}+\Theta(C\eta)(\alpha^{(t)}G^{(t)} )^p -\eta (\alpha^{(t)})^p o(\lambda_0)^p.
\end{align}
We now apply \autoref{indh:lambgamxi} in \eqref{eq:woenewef} and get: 
\begin{equation}
\begin{aligned}
    \alpha^{(t+1)}&\geq \frac{\Omega(1)}{C^2\lambda_0}+\Theta(\eta) \frac{e^{\mathrm{polyloglog}(d)}}{C^{2p-1}\lambda_0^p} - \frac{o(\eta)}{C^{2p}}\\
    &\geq \frac{\Omega(1)}{C^2\lambda_0} - \frac{o(\eta)}{C^{2p}}.
\end{aligned}
\end{equation}
%Given the values of $C,\eta,\lambda_0$, we conclude that $\alpha^{(t+1)} \in (0, \frac{\Omega(1)}{C^2\lambda_0}).$ 
At time $t+1,$ we potentially have $ \alpha^{(t+1)}< \frac{\Omega(1)}{C^2\lambda_0}$. In this case, $\alpha^{(t+1)}$ starts to increase again because it is in the range of $\alpha$'s that satisfies Event I (and therefore the update rule in \autoref{lem:event1} holds). Thus, for all $t\geq\mathcal{T}_0,$ we have $\alpha^{(t)}\geq \frac{\Omega(1)}{C^2\lambda_0}-o(\eta).$
\end{proof}

\begin{lemma}\label{lem:sigmfunct}
 When the population risk is $\Omega(1)$, we have $\mathbb{E}[\mathfrak{S}(-yF(\bm{X}))]\geq\Omega(1).$

\end{lemma}
\begin{proof}[Proof of \autoref{lem:sigmfunct}]
Let $(\bm{X},y)$ be a data-point. We distinguish two cases:
\begin{itemize}
    \item[--] $yF(\bm{X})>0$: we apply \autoref{lem:logsigm2} which implies  $\mathfrak{S}(-yF(\bm{X})\geq \log(1+e^{yF(\bm{X})}).$ Since the population loss is $\Omega(1)$, this implies the aimed result.
    \item[--] $yF(\bm{X})\leq 0$: we have necessarily $\mathfrak{S}(-yF(\bm{X}))\geq \Omega(1)$ since the sigmoid function is large for non-positive values.
\end{itemize}
Therefore, we have $\mathbb{E}[\mathfrak{S}(-yF(\bm{X}))]\geq\Omega(1).$
\end{proof}

\subsection{Event III:  $\alpha^{(t)}$ keeps increases again}\label{sec:event3} %for $t\geq \mathcal{T}_1$

For $t\in[\mathcal{T}_0,\mathcal{T}_1]$, $\Gamma^{(t)}$ increases until reaching $C\Gamma^{(t)}\geq \Omega(\lambda_0)/D.$ In this section, we show that this  implies that $\alpha^{(t)}$ increases again.

\begin{restatable}{lemma}{lemeventhree}\label{lem:event3}
 Let   $\mathcal{T}_1=\mathcal{T}_0+ \Theta\Big( \frac{\lambda_0^p}{\eta  e^{\mathrm{polyloglog}(d)}} \Big)$ and $t\in[\mathcal{T}_1,T]$. Since ${\large \pmb{\mathpzc{S}}}^{(t)}\geq\max_{\tau\leq T}|{\large \pmb{\mathpzc{N}}}^{(\tau)}|$, $\alpha^{(t)}$ updates as
\begin{align*}
    \alpha^{(t+1)}&=\alpha^{(t)}+\Theta(\eta C) ( G^{(t)} )^{p}(\alpha^{(t)})^{p-1}.
\end{align*}
Consequently, $\alpha^{(t)}$ is non-decreasing until the population risk satisfies $\mathcal{L}(\bm{A}^{(t)},\bm{v}^{(t)})\leq o(1).$ Eventually, $\alpha^{(T)}$ is as large as $\alpha^{(T)}=\frac{\mathrm{polylog}(d)}{C^2\lambda_0}.$% $\alpha^{(t)}\leq\frac{\mathrm{polylog}(d)}{C^2\lambda_0}$ for $t\leq T.$
\end{restatable}

%\begin{lemma}\label{lem:phase3}
%Let $t\in[\mathcal{T}_1,T].$ During this phase, $\alpha^{(t)}$ is updated as
%\begin{align*}
%    \alpha^{(t+1)}\geq \alpha^{(t)}+\Theta(\eta C) (\alpha^{(t)})^{p-1}(G^{(t)})^p.
%\end{align*}
% $\alpha^{(t)}$ increases until the population loss is $o(1)$ which happens at time $\mathcal{T}_2=\mathcal{T}_1+\Theta\big(\frac{\lambda_0^p}{\eta e^{\mathrm{polyloglog}(d)}}\big)$. Thus, for $t\in[\mathcal{T}_1,T]$,  $\alpha^{(t)}\leq \frac{\tilde{O}(1)}{C^2\lambda_0}$.
%\end{lemma}
\begin{proof}[Proof of \autoref{lem:event3}]  
Since ${\large \pmb{\mathpzc{S}}}^{(t)}\geq\max_{\tau\leq T}{\large |\pmb{\mathpzc{N}}}^{(\tau)}|$ (\autoref{lem:event2}), the update of $\alpha^{(t)}$ is:
\begin{equation}
\begin{aligned}\label{eq:fejdnjwjbbje}
    &\alpha^{(t+1)}-\alpha^{(t)}=\Theta(\eta C) (\alpha^{(t)})^{p-1}(G^{(t)})^p\mathbb{E}[\mathfrak{S}(-yF(\bm{X}))], 
\end{aligned}    
\end{equation}
 Since the population loss is at least $\Omega(1)$ for $t\in[\mathcal{T}_1,T]$, \autoref{lem:sigmfunct} implies that   $\mathbb{E}[\mathfrak{S}(-yF(\bm{X}))]\geq \Omega(1)$. Besides, we apply \autoref{indh:lambgamxi} and $\Omega(C)\geq 0$ and $e^{\mathrm{polyloglog}(d)}\leq \lambda_0$ in \eqref{eq:fejdnjwjbbje} and obtain: 
\begin{align}\label{eq:eecjowerfew}
    \begin{cases}
      \alpha^{(t+1)}\geq \alpha^{(t)}+\Theta(\eta C)e^{\mathrm{polyloglog}(d)} (\alpha^{(t)})^{p-1}\\
      \alpha^{(t+1)}\leq \alpha^{(t)}+\Theta(\eta C) \lambda_0^p(\alpha^{(t)})^{p-1}
    \end{cases}.
\end{align}
\eqref{eq:eecjowerfew} and \autoref{lem:alphalimsigmoidsmall} show that $\alpha^{(t)}$ increases until reaching  $\alpha^{(t)}\leq \frac{\mathrm{polylog}(d)}{C^2\lambda_0}$.

\end{proof}

\subsubsection{Auxiliary results}

\begin{lemma}\label{lem:alphalimsigmoidsmall}
The values of $\alpha $ such that $\mathbb{E}[\mathfrak{S}(-yF(\bm{X}))]\geq \Omega(1)$ is 
\begin{align*}
    \alpha \leq \frac{\mathrm{polylog}(d)}{C^2\lambda_0}.
\end{align*}
\end{lemma}
\begin{proof}[Proof of \autoref{lem:alphalimsigmoidsmall}]  We say that the sigmoid term is small for a constant $\kappa$ that satisfies
\begin{align}\label{eq:sum_sigmoid}% \mathfrak{S}(\kappa)
    \sum_{\tau=0}^T \frac{1}{1+\exp(\kappa)}&\leq \mathrm{polylog}(d) \implies  \kappa  \geq \log(T) \iff \kappa \geq \mathrm{polylog}(d).
\end{align}
Intuitively, \eqref{eq:sum_sigmoid} means that the sum of the sigmoid terms for all time steps is bounded (up to a logarithmic dependence).
In our case, by using \eqref{eq:sum_sigmoid} and \autoref{lem:Fbdalph}, the sigmoid $\mathfrak{S}(-yF(\bm{X}))$ is small when 
\begin{align}\label{eq:owedwe}
    (C^2\alpha \lambda_0)^p\geq \mathrm{polylog}(d)\implies \alpha \geq \frac{\mathrm{polylog}(d)}{C^2\lambda_0}.
\end{align}
\end{proof}

\subsection{Convergence rate of the population loss}\label{sec:convloss}

\begin{lemma}\label{lem:cvrate} Let $t\in[\mathcal{T}_1,T]$. Then, the population loss linearly converges to zero i.e. 
\begin{align}
  \mathcal{L}(\bm{A}^{(t)},\bm{v}^{(t)}) \leq \frac{\mathrm{polylog}(d)}{\eta\lambda_0^{2p} (t-\mathcal{T}_1+1)}
\end{align}

\end{lemma}
\begin{proof}[Proof of \autoref{lem:cvrate}] To ease the explanation in this proof, we use the $\tilde{\Omega}$, $\tilde{\Theta}$, $\tilde{O}$ notations to hide the logarithmic dependence. We hide for instance the constant $C$ under this notation. 
From \autoref{lem:event3}, we know that $\alpha^{(t)}$ is lower bounded as: 
\begin{align}\label{eq:efvervrt}
    \alpha^{(t+1)}\geq \alpha^{(t)}+\Theta(\eta C) (\alpha^{(t)})^{p-1}(G^{(t)})^p\mathbb{E}[\mathfrak{S}(-yF(\bm{X}))].
\end{align}
Using \autoref{lem:Fbdalph}, we have $\mathbb{E}[\mathfrak{S}(-yF(\bm{X}))]\geq \mathfrak{S}(-(\alpha^{(t)}G^{(t)})^p).$ Plugging this in \eqref{eq:efvervrt} yields: 
\begin{align}\label{eq:nkibibih}
    \alpha^{(t+1)}\geq \alpha^{(t)}+\Theta(\eta ) G^{(t)}\frac{(\alpha^{(t)})^{p-1}(G^{(t)})^{p-1}}{1+\exp((\alpha^{(t)}G^{(t)})^p)}.
\end{align}
Since $0<\alpha^{(t)}G^{(t)}\leq \tilde{O}(\lambda_0^{p-1})$,%%\tilde{O}(C^{p-2}\lambda_0^{p-1})
we apply \autoref{lem:logsigm}
and get: 
\begin{align}\label{eq:nkibibllmlih}
    \alpha^{(t+1)}\geq \alpha^{(t)}+\frac{\tilde{\Omega}(\eta)}{\lambda_0^{p}} G^{(t)}\log(1+e^{-(\alpha^{(t)}G^{(t)})^p}). %\textbf{\frac{\tilde{\Omega}(\eta)}{C^{p-3}\lambda_0^{p-1}} }
\end{align}
\autoref{lem:event2} implies that $G^{(t)}\geq \tilde{\Omega}(\lambda_0).$ Therefore, we have: 
\begin{align}\label{eq:nkibibknkknih}
    \alpha^{(t+1)}\geq \alpha^{(t)}+\frac{\tilde{\Omega}(\eta)}{\lambda_0^{p}} \log(1+e^{-(\alpha^{(t)}G^{(t)})^p}).
\end{align}
Let's now assume by contradiction that for $t\in[\mathcal{T}_1,T]$, we have:
\begin{align}\label{eq:idwenkfvfevfevewkn}
    \log(1+e^{-(\alpha^{(t)}G^{(t)})^p})> \frac{\tilde{\Omega}(1)}{\eta\lambda_0^{2p} (t-\mathcal{T}_1+1)}.
\end{align}
For $t\in[\mathcal{T}_1,T]$, we know that $\alpha^{(t)}G^{(t)}$ is non-decreasing which implies that $(\alpha^{(t)}G^{(t)})^p $ is also non-decreasing. Since $x\mapsto \log(1+\exp(-x))$ is non-increasing, this implies for $s\leq t$ that
\begin{align}\label{eq:foeperf}
 \frac{\tilde{\Omega}(1)}{\eta\lambda_0^{2p} (t-\mathcal{T}_1)}  < \log(1+e^{-(\alpha^{(t)}G^{(t)})^p}) \leq \log(1+e^{-(\alpha^{(s)}G^{(s)})^p}).
\end{align}
Plugging \eqref{eq:foeperf} in the update \eqref{eq:nkibibih} yields for $s\in[\mathcal{T}_1,t]$:
\begin{align}\label{eq:icfhrewihvdfvfwfewfwc}
    \alpha^{(s+1)}> \alpha^{(s)}+\frac{\tilde{\Omega}(1)}{ \lambda_0^p(t-\mathcal{T}_1+1)}.
\end{align}
Let $t\in[\mathcal{T}_1,T]$. We now sum \eqref{eq:icfhrewihvdfvfwfewfwc} for $s=\mathcal{T}_1,\dots,t$ and obtain:
\begin{align}\label{eq:iejfefreeifced}
    \alpha^{(t+1)}> \alpha^{(\mathcal{T}_1)}+\frac{\tilde{\Omega}(1)(t-\mathcal{T}_1+1)}{\lambda_0^p(t-\mathcal{T}_1+1)}> \frac{\tilde{\Omega}(1)}{\lambda_0^p},
\end{align}
where we used the fact that $\alpha^{(\mathcal{T}_1)}\geq\alpha^{(\mathcal{T}_0)} \geq \tilde{\Omega}(1)/(C\lambda_0)>0$ (\autoref{lem:event1}) in the last inequality.  Therefore, we have for $t\in[\mathcal{T}_1,T],$ $\alpha^{(t)}\geq \tilde{\Omega}(1/\lambda_0^p)> 0$. Let's now show that \eqref{eq:iejfefreeifced} implies  a contradiction. Indeed, we have:
\begin{align}
   \eta \lambda_0^{2p}(t-\mathcal{T}_1+1) \log(1+e^{-(\alpha^{(t)}G^{(t)})^p})
   &\leq  \eta \lambda_0^{2p}T \log(1+e^{-(\alpha^{(t)}G^{(t)})^p}) \nonumber\\
   &\leq  \eta \lambda_0^{2p}T \log(1+e^{-\tilde{\Omega}(1)}), \label{eq:efevvoekmeifriw3jiepr}
\end{align}
where we used $G^{(t)}\geq\Omega(\lambda_0)$ (\autoref{lem:event2}) and \eqref{eq:iejfefreeifced} in the last inequality. We now apply \autoref{lem:logsigm2} and obtain: 
\begin{align}\label{eq;frrikriknnkn}
   \eta \lambda_0^{2p}(t-\mathcal{T}_1+1) \log(1+e^{-(\alpha^{(t)}G^{(t)})^p})\leq \frac{\eta \lambda_0^{2p}T}{1+\exp(\tilde{\Omega}(1))}.
\end{align}
Given the values of $T,\eta,\lambda_0$, we finally have:
\begin{align}
      \eta \lambda_0^{2p}(t-\mathcal{T}_1+1) \log(1+e^{-(\alpha^{(t)}G^{(t)})^p})<\tilde{O}(1),
\end{align}
which contradicts \eqref{eq:idwenkfvfevfevewkn}. Therefore, we obtain the convergence rate:
\begin{align}\label{eq:ltbd1}
    \log(1+e^{-(\alpha^{(t)}G^{(t)})^p})\leq \frac{\tilde{O}(1)}{\eta\lambda_0^{2p} (t-\mathcal{T}_1+1)}
\end{align}
We apply \autoref{lem:ofeofjewew} to bound the left-hand side of \eqref{eq:ltbd1} and  get the aimed result.
\end{proof}

\subsubsection{Auxiliary lemmas}

\begin{lemma}\label{lem:ofeofjewew}
Let $t\in[\mathcal{T}_1,T]$. We have: 
\begin{align*}
    \Theta(1)\log(1+e^{-(\alpha^{(t)}G^{(t)})^p})\geq\mathcal{L}(\bm{A}^{(t)},\bm{v}^{(t)}).
\end{align*}
\end{lemma}
\begin{proof}[Proof of \autoref{lem:ofeofjewew}]
The proof is similar to the one of \autoref{lem:Fbdalph}. We apply \autoref{lem:singbdT1}, \autoref{lem:event2} and  \autoref{lem:noisbdT1} and get: 
\begin{align}\label{eq:ddedede}
   yF_{\bm{A}^{(t)},\bm{v}^{(t)}}(\bm{X}) \geq (\alpha^{(t)})^p\left(C (G^{(t)})^p - o(\lambda_0)^p\right)\geq \Theta(1)(\alpha^{(t)}G^{(t)})^p.
\end{align}
Using \eqref{eq:ddedede}, we deduce: 
\begin{align}
    \mathbb{E}[\log(1+e^{-yF_{\bm{A}^{(t)},\bm{v}^{(t)}}(\bm{X})})]\leq \log(1+e^{-\Theta(1)(\alpha^{(t)}G^{(t)})^p})\leq \Theta(1) \log(1+e^{-(\alpha^{(t)}G^{(t)})^p}),
\end{align}
where we applied \autoref{lem:logfzd} in the last inequality.
\end{proof}

\subsection{Fitting the labeling function}\label{sec:generalization_ideal}

%\begin{lemma}\label{lem:testerror}
%The model generalizes well i.e.
%\begin{align*}
%    \mathbb{P}_{(X,y)\sim\mathcal{D}}[yF^{(T)}(X)>0]\geq 1-\frac{1}{\mathrm{poly}(d)}.
%\end{align*}
%\end{lemma}

We now show that the learner model fits the labeling function.

\begin{restatable}{lemma}{lemcvpop}\label{lem:cv_pop}
After $T$ iterations, the population risk converges i.e.\ $\mathcal{L}(\bm{A}^{(T)},\bm{v}^{(T)})\leq O(1/\mathrm{poly}(d)).$  Therefore, $\mathbb{P}_{\mathcal{D}}[f^*(\bm{X})F_{\bm{A}^{(t)},\bm{v}^{(t)}}(\bm{X})>0]\geq 1-o(1).$
\end{restatable}

\begin{proof}[Proof of \autoref{lem:cv_pop}] Since the logistic loss is a surrogate for the 0-1 loss, we have: 
\begin{align}\label{eq:jewdneq}
    \mathbb{E}_{\mathcal{D}}[\mathbf{1}_{yF_{\bm{A}^{(t)},\bm{v}^{(t)}}(\bm{X})<0}]\leq \mathbb{E}_{\mathcal{D}}[\log(1+e^{-yF_{\bm{A}^{(t)},\bm{v}^{(t)}}(\bm{X})})].
\end{align}
We now apply \autoref{lem:cvrate} to bound the right-hand side of \eqref{eq:jewdneq}. Given the value of $T$, we have: 
\begin{align}\label{eq:jejknjwdneq}
    \mathbb{P}_{\mathcal{D}}[yF_{\bm{A}^{(T)},\bm{v}^{(T)}}(\bm{X})]=\mathbb{E}_{\mathcal{D}}[\mathbf{1}_{yF_{\bm{A}^{(T)},\bm{v}^{(T)}}(\bm{X})<0}]\leq \frac{\mathrm{polylog}(d)}{\eta\lambda_0^{2p} T}\leq \frac{1}{\mathrm{poly}(d)}.
\end{align} 
We now use \eqref{eq:jejknjwdneq} and \autoref{ass:data_dist} to show that the learner model fits the labeling function. Indeed, we rewrite  $\mathbb{P}_{\mathcal{D}}[f^*(\bm{X})F_{\bm{A}^{(T)},\bm{v}^{(T)}}(\bm{X})]$ as 
\begin{align}
   \mathbb{P}_{\mathcal{D}}[f^*(\bm{X})F_{\bm{A}^{(T)},\bm{v}^{(T)}}(\bm{X})]&\geq \mathbb{P}_{\mathcal{D}}[yf^*(\bm{X})>0]\mathbb{P}_{\mathcal{D}}[yF_{\bm{A}^{(T)},\bm{v}^{(T)}}(\bm{X})>0]\nonumber\\
   &\geq (1-d^{-\omega(1)})\Big(1-\frac{1}{\mathrm{poly}(d)}\Big)\nonumber\\
   &=1-o(1).
\end{align}
\end{proof}

\subsection{Proof of the induction hypothesis}\label{sec:indhideal}

In this section, we prove  \autoref{indh:lambgamxi}. 
\begin{proof}[Proof of \autoref{indh:lambgamxi}] We start by proving that $|\rho^{(t)}|=\Theta(1)$ for all $t\in[T].$\\
Let $t\in[\mathcal{T}_1,T]$ and $\tau\in[t].$ Using \autoref{cor:rhoupdate} and  \autoref{indh:lambgamxi}, we upper bound $|\rho^{(\tau)}|$ as: \begin{align}\label{eq:refpfer}
   |\rho^{(\tau+1)}|&\leq|\rho^{(\tau)}|+\frac{\eta\Theta(C)}{D}(\alpha^{(\tau)})^{p}\bigg(1 +\frac{\lambda_0^{p+1}}{D}\bigg).
   \end{align}
Summing \eqref{eq:refpfer} for $\tau=0,\dots,t-1$ and using $\rho^{(0)}=0$ lead to
\begin{align}\label{eq:oewnffw}
    |\rho^{(t)}|&\leq \frac{\eta\Theta(C)}{D}\bigg(1 +\frac{\lambda_0^{p+1}}{D}\bigg) \sum_{\tau=0}^{T}(\alpha^{(\tau)})^{p} 
\end{align}
We now apply \autoref{lem:sumalphat} to bound the sum of $\alpha^{(t)}$'s in \eqref{eq:oewnffw}.
\begin{equation}
\begin{aligned}\label{eq:ierjersdcsj}
     |\rho^{(t)}|&\leq \Theta\left( \frac{1}{D(\alpha^{(0)})^{p-1}(\lambda_0)^p e^{\mathrm{polyloglog}(d)}}    + \frac{1}{ e^{\mathrm{polyloglog}(d)}}\right).
\end{aligned}
\end{equation}
Given the values of the different parameters, \eqref{eq:ierjersdcsj} implies that $|\rho^{(t)}|\leq \Theta(1).$ \newline

We now prove  $e^{\gamma^{(t)}}\in [\Omega(1),\lambda_0].$ Since $e^{\gamma^{(t)}}$ is non-decreasing (\autoref{cor:gamupdate}), we have $e^{\gamma^{(t)}}\geq e^{\gamma^{(0)}}\geq\Omega(1)$ for all $t\geq 0.$ We now prove the upper bound on $e^{\gamma^{(t)}}$. We assume that for all $\tau\leq t$, $e^{\gamma^{(\tau)}}\leq \lambda_0.$ Let's show this inequality for $t+1$. Using \autoref{cor:gamupdate}, we have:
\begin{equation}
\begin{aligned}
     e^{\gamma^{(t+1)}}&\leq e^{\gamma^{(t)}}\exp\Big(\Theta(C\eta)(\alpha^{(t)})^{p}\Gamma^{(t)}(G^{(t)})^{p-1}\Big)\\
     &=\exp\Big(\Theta(C\eta) \sum_{\tau=0}^{t}(\alpha^{(\tau)})^{p}\Gamma^{(\tau)}(G^{(\tau)})^{p-1}\Big).
    % &=\exp\left(\frac{\Theta(\lambda_0^pC^p\eta)}{D^p}\sum_{\tau=0}^{t-1}(\alpha^{(\tau)})^{p}\right)
     \label{eq:wijf}
\end{aligned}
\end{equation}
We now apply the induction hypothesis in \eqref{eq:wijf} and get: 
\begin{equation}
\begin{aligned}
     e^{\gamma^{(t+1)}}&\leq e^{\gamma^{(t)}}\exp\big(\Theta(C\eta)(\alpha^{(t)})^{p}\Gamma^{(t)}(G^{(t)})^{p-1}\big)\\
     &=\exp\Big(\frac{\Theta(\lambda_0^pC\eta)}{D^p} \sum_{\tau=0}^{t}(\alpha^{(\tau)})^{p} \Big)\\
     &\leq \exp\Big(\frac{\Theta(\lambda_0^pC\eta)}{D^p} \sum_{\tau=0}^{T}(\alpha^{(\tau)})^{p} \Big).
     \label{eq:wijecef}
\end{aligned}
\end{equation}
We apply \autoref{lem:sumalphat}  in \eqref{eq:wijecef}
and obtain: 
\begin{align}
     e^{\gamma^{(t+1)}}&\leq 
     \exp\bigg(\Theta\Big(\frac{1}{D^p(\alpha^{(0)})^{p-1}  e^{\mathrm{polyloglog}(d)}}    + \frac{\lambda_0^p}{D^{p-1} e^{\mathrm{polyloglog}(d)}}\Big) \bigg)\nonumber\\
     &=\Theta(1)\exp\left(  \frac{\lambda_0^p}{D^{p-1} e^{\mathrm{polyloglog}(d)}} \right)\nonumber\\
     &\leq \Theta\bigg(1+\frac{\lambda_0^p}{D^{p-1} e^{\mathrm{polyloglog}(d)}}+\frac{\lambda_0^{2p}}{D^{2(p-1)} e^{\mathrm{polyloglog}(d)}}\bigg),
     \label{eq:wjnojnoijecef}
\end{align}
where we used the inequality $e^x\leq 1+x+x^2$ for $x\leq 1$ in \eqref{eq:wjnojnoijecef}.
Given the values of the different parameters, we deduce that $ e^{\gamma^{(t+1)}}\leq \lambda_0.$\newline

We now prove $e^{\beta^{(t)}}=e^{\mathrm{polyloglog}(d)}$ for $t\in[0,T].$ Since $\beta^{(t)}$ is not updated i.e. $\beta^{(t)}=\beta^{(0)}$ and $\beta^{(0)}=\sigma_{\bm{M}}=\mathrm{polyloglog}(d)$, we therefore have the aimed result.\newline

Lastly, we prove that $e^{\beta^{(t)}}+(C-1)e^{\gamma^{(t)}}+(D-C)e^{\rho^{(t)}}=\Theta(D)$ for $t\in[T].$ Since $e^{\gamma^{(t)}}\geq \Theta(1)$, $e^{\rho^{(t)}}=\Theta(1)$ and $e^{\beta^{(t)}}\geq \Theta(1)$, we have: 
\begin{align}
    e^{\beta^{(t)}}+(C-1)e^{\gamma^{(t)}}+(D-C)e^{\rho^{(t)}}\geq \Theta(D).
\end{align}
On the other hand, we have $e^{\gamma^{(t)}}\leq \lambda_0$, $e^{\beta^{(t)}}= e^{\mathrm{polyloglog}(d)}$ and $e^{\rho^{(t)}}=\Theta(1)$ which imply: 
\begin{equation}
\begin{aligned}
    e^{\beta^{(t)}}+(C-1)e^{\gamma^{(t)}}+(D-C)e^{\rho^{(t)}}&\leq e^{\mathrm{polyloglog}(d)}+(C-1)\lambda_0+(D-C)\Theta(1)\\ 
    &\leq C\lambda_0+(D-C)\Theta(1)\\
    &\leq \Theta(D).
\end{aligned}
\end{equation}
\end{proof}

\subsubsection{Auxiliary lemmas}

\begin{lemma}\label{lem:sumalphat} 
The sum of the $\alpha^{(t)}$'s is bounded as:
\begin{align*}
    \sum_{\tau=0}^T(\alpha^{(\tau)})^p&= \Theta\left( \frac{1}{\eta C(\alpha^{(0)})^{p-1}(\lambda_0)^p e^{\mathrm{polyloglog}(d)}}    + \frac{D}{C\eta e^{\mathrm{polyloglog}(d)}}\right) .
\end{align*}
\end{lemma}
\begin{proof}[Proof of \autoref{lem:sumalphat}] We first decompose the sum of $\alpha^{(t)}$'s. 
\begin{align}\label{eq:fjnenew}
     \sum_{\tau=0}^T(\alpha^{(\tau)})^p&=\sum_{\tau=0}^{\mathcal{T}_0-1}(\alpha^{(\tau)})^{p}+\sum_{\tau=\mathcal{T}_0}^{\mathcal{T}_1-1}(\alpha^{(\tau)})^{p}+\sum_{\tau=\mathcal{T}_1}^{T}(\alpha^{(\tau)})^{p}.
\end{align}
We apply \autoref{lem:event1} and  \autoref{lem:event2} to rewrite \eqref{eq:fjnenew}.
\begin{align}\label{eq:fjnenebjw}
     \sum_{\tau=0}^T(\alpha^{(\tau)})^p&=\frac{\Theta(1)}{\eta C(\alpha^{(0)})^{p-1}(\lambda_0)^p e^{\mathrm{polyloglog}(d)}}+\frac{\log\left(\lambda_0\right)}{\eta e^{\mathrm{polyloglog}(d)}} +\sum_{\tau=\mathcal{T}_1}^{T}(\alpha^{(\tau)})^{p}.
\end{align}
Now, we aim to obtain the value of the last summand in \eqref{eq:fjnenebjw}. Using \autoref{cor:gamupdate}, we have
\begin{align}\label{eq:foeenwe}
    C\Gamma^{(T)}\geq C\Gamma^{(\mathcal{T}_1)}\exp\left(\frac{C\eta}{D} e^{\mathrm{polyloglog}(d)}\sum_{\tau=\mathcal{T}_1}^T(\alpha^{(\tau)})^p\right).
\end{align}
We finally apply \autoref{lem:event2} and \autoref{indh:lambgamxi} in \eqref{eq:foeenwe}
to get: 
\begin{align}\label{eq:ewdodw}
    \sum_{\tau=\mathcal{T}_1}^T(\alpha^{(\tau)})^p\leq \frac{\Theta(D)}{C\eta e^{\mathrm{polyloglog}(d)}}.
\end{align}
To obtain the aimed result, we plug \eqref{eq:ewdodw} in \eqref{eq:fjnenebjw} and use $\frac{\log\left(\lambda_0\right)}{\eta e^{\mathrm{polyloglog}(d)}}\leq \frac{\Theta(D)}{C\eta e^{\mathrm{polyloglog}(d)}}$.
\end{proof}
\section{From idealized to real learning process}\label{sec:real_process}

In \autoref{sec:ideal}, we analyzed the ideal learning process. We now aim to bridge the gap between the idealized and realistic cases. Given our initialization, $\widehat{\bm{v}}^{(t)}$ has a component in $\mathrm{span}(\bm{w}^*)^{\perp}$ i.e.\ 
\begin{align*}
\widehat{\bm{v}}^{(t)}=\widehat{\alpha}^{(t)}\bm{w}^*+\varepsilon_{\bm{v}}^{(t)}\bm{u}^{(t)}
\end{align*}
where $\bm{u}^{(t)}\in\mathbb{R}^d$ such that $\bm{u}^{(t)}\perp \bm{w}^*$ and $\|\bm{u}^{(t)}\|_2=1.$ Thus, one main difference between the two cases is that we initialize $\alpha^{(0)}\geq \nu^{1/(p-1)}$ in the idealized case while $\widehat{\alpha}^{(0)}\leq \omega$. Thus, the proof strategy consists in i) $t\in[0,\mathscr{T}]$, $\widehat{\alpha}^{(t)}$ increases until having $\widehat{\alpha}^{(t)}\geq \nu^{1/(p-1)}$ (\autoref{subsubsec:alpha_real}) while $\varepsilon_{\bm{v}}^{(t)}$ (\autoref{subsubsec:eps_v_real}) and $\widehat{A}_{i,j}^{(t)}$ (\autoref{sec:subsubsecAbdinti}) stay tiny. ii) $t\in [\mathscr{T},T]$, compare the realistic and idealized iterates. We remind the GD update of $\widehat{\alpha}^{(t)}.$
%We have to say that the $alpha^t$ satisfies the update in the previous section. 
\begin{equation}
\begin{aligned}
    &\widehat{\alpha}^{(t+1)}-\widehat{\alpha}^{(t)}\\
    &\hspace{-.5cm}=\eta\underbrace{\frac{ D}{N} \sum_{i=1}^Ny[i]\mathfrak{S}\big(-y[i]F(\bm{X}[i])\big)\hspace{-.3cm}\sum_{j\in\mathcal{S}_{\ell(\bm{X}[i])}}\hspace{-.3cm} \sigma'\Big(\sum_{k=1}^D \widehat{S}_{j,k}^{(t)}\langle \widehat{\bm{v}}^{(t)},\bm{X}_k[i]\rangle\Big)\sum_{r=1}^D \widehat{S}_{j,r}^{(t)}\langle \bm{w}^*,\bm{X}_r[i]\rangle}_{{\Large \widehat{\pmb{\mathpzc{S}}}^{(t)}}}\\
    &\hspace{-.5cm}+\eta\underbrace{\frac{D}{N} \sum_{i=1}^Ny[i]\mathfrak{S}\big(-y[i]F(\bm{X}[i])\big)  \sigma'\Big(\sum_{k=1}^D \widehat{S}_{j,k}^{(t)}\langle \widehat{\bm{v}}^{(t)},\bm{X}_k[i]\rangle\Big)\sum_{r=1}^D \widehat{S}_{j,r}^{(t)}\langle \bm{w}^*,\bm{X}_r[i]\rangle}_{{\Large \widehat{\pmb{\mathpzc{N}}}^{(t)}}}.
\end{aligned}\label{eq:upd_alpha_emp}\tag{GD-$\widehat{\alpha}$}
\end{equation}

\subsection{Bound on the iterates during the initial steps ($t\in[0,\mathscr{T}]$)}

Since we randomly initialize $\widehat{\bm{v}}^{(0)}$ with tiny variance, we need to take into account the linear part of the activation function. \autoref{lem:activ_simple_real} shows that we can overlook the power part of the activation and consider $\sigma(x)=\nu x$ as long as $\widehat{\alpha}^{(t)}\geq\nu^{1/(p-1)}$. %In this setting, we show that a) $\widehat{\alpha}^{(t)}$ increases until reaching $\nu^{1/(p-1)}$ while $\widehat{A}_{i,j}^{(t)}$ (for $i\neq j$) and $\varepsilon_{\bm{v}}^{(t)}$ stay small.
\subsubsection{$\widehat{\alpha}^{(t)}$ initially increases}\label{subsubsec:alpha_real}

\begin{lemma}\label{lem:eventinit_real}
Let $\mathscr{T}=\Theta\Big(\frac{1}{\eta\nu^{(p-2)/(p-1)} e^{\mathrm{polyloglog}(d)}}\Big)$. For all $t\in[0,\mathscr{T}]$,  $\widehat{\alpha}^{(t)}$ is updated as
\begin{align*}
       \widehat{\alpha}^{(t+1)}&=\widehat{\alpha}^{(t)}+\Theta(\eta\nu ) e^{\beta}.
\end{align*}
Consequently, $\widehat{\alpha}^{(t)}$ is non-decreasing and after $\mathscr{T}$ iterations, we have $\widehat{\alpha}^{(t)}\geq\nu^{1/(p-1)}$ for $t\geq \mathscr{T}.$
\end{lemma}

\begin{proof}[Proof of \autoref{lem:eventinit_real}] Let $t\geq 0$. %We know that $\widehat{\alpha}^{(t)}\leq \nu^{1/p}$ so the linear component of the activation dominates. 
We apply \autoref{lem:singbdT1_real} and \autoref{lem:noisbdT1_real} to respectively bound ${\large \widehat{\pmb{\mathpzc{S}}}}^{(t)}$ and $\widehat{\large \pmb{\mathpzc{N}}}^{(t)}$ in the update of $\widehat{\alpha}^{(t)}$.
\begin{align}\label{eq:fhewbkjjkwei}
    \widehat{\alpha}^{(t+1)}&= \widehat{\alpha}^{(t)}+\Theta(C\eta\nu ) e^{\beta}.
\end{align} 
\eqref{eq:fhewbkjjkwei} indicates that $\widehat{\alpha}^{(t)}$ is a non-decreasing sequence.  Therefore, there exists a time $\mathscr{T}$ such that $\widehat{\alpha}^{(\mathscr{T})}=\nu^{1/(p-1)}$. Summing \eqref{eq:fhewbkjjkwei} for $t=0,\dots,\mathscr{T}-1$ yields $\mathscr{T}=\Theta\Big(\frac{1}{\eta\nu^{(p-2)/(p-1)} e^{\beta}}\Big)$.
\end{proof}

\subsubsection{\texorpdfstring{Bound on    $\varepsilon_{\bm{v}}$}{Bound on epsilonv}}\label{subsubsec:eps_v_real}

We now show that for $t\in [0,\mathscr{T}]$, the orthogonal component $\varepsilon_{\bm{v}}^{(t)}$ stays small.

\begin{lemma}\label{lem:epsv_init}
 Assume that we run GD on the empirical risk \eqref{eq:empirical} for $T$ iterations with parameters set as in \autoref{param}. For $t\in[0,\mathscr{T}]$, the orthogonal component $\varepsilon_{\bm{v}}$ satisfies
    \begin{align*}
        \varepsilon_{\bm{v}}^{(t+1)}&\leq \Big(1+\eta \nu\frac{\mathrm{poly}(D)}{\sqrt{d}}\Big)\varepsilon_{\bm{v}}^{(t)}+\eta\zeta,
    \end{align*}
    where $\zeta=\frac{\mathrm{poly}(D)}{\sqrt{N}}.$ 
\end{lemma}
\begin{proof}[Proof of \autoref{lem:epsv_init}] Let $\mathbf{P}= (\mathbf{I}-\bm{w}^*\bm{w}^{*\top})$ and $\accentset{\circ}{\bm{v}}^{(t)}=\widehat{\alpha}^{(t)}\bm{w}^*$.  The projected update of $\widehat{\bm{v}}$ satisfies: 
\begin{align}
     &\|\mathbf{P}\widehat{\bm{v}}^{(t+1)}-\mathbf{P}\widehat{\bm{v}}^{(t)}\|_2\nonumber\\
     \leq&D\eta\nu\biggr\|\frac{1}{N}\sum_{i=1}^N y[i]\mathfrak{S}(-y[i]F_{\widehat{\bm{v}}}(\bm{X}[i]))\sum_{m=1}^D \sum_{b=1}^D\widehat{\bm{S}}_{m,b}^{(t)}\mathbf{P}\bm{X}_b[i]\label{eq:fewpwpkjkje}\\ 
     &-\mathbb{E}\biggr[y\mathfrak{S}(-yF_{\widehat{\bm{v}}}(\bm{X}))\sum_{m=1}^D \sum_{b=1}^D \widehat{\bm{S}}_{m,b}^{(t)}\mathbf{P}\bm{X}_b\biggr]\biggr\|_2 \label{eq:fewefdfvdfewpe}\\
     %
     %
     %&\nonumber\\
     +&D\eta\nu\biggr\|\mathbb{E}\biggr[y\mathfrak{S}(-yF_{\widehat{\bm{v}}}(\bm{X}))\sum_{m=1}^D \sum_{b=1}^D \widehat{\bm{S}}_{m,b}^{(t)}\mathbf{P}\bm{X}_b\biggr]\label{eq:fewkkjjpwpkjkje}\\ 
     &-\mathbb{E}\biggr[y\mathfrak{S}(-yF_{\accentset{\circ}{v}}(\bm{X}))\sum_{m=1}^D \sum_{b=1}^D \widehat{\bm{S}}_{m,b}^{(t)}\mathbf{P}\bm{X}_b\biggr]\biggr\|_2 \label{eq:fewefdfvdfoiookewpe}.
\end{align}

\paragraph{Summand 1: $ \|\eqref{eq:fewpwpkjkje}-\eqref{eq:fewefdfvdfewpe}\|_2$.} Using the matrix Hoeffding inequality, we have with high probability,
    $\|\eqref{eq:fewpwpe}-\eqref{eq:fewefewpe}\|_2\leq 16\sqrt{\log(d)\sum_{i=1}^N M_i^2},$
where\\
$\Big\|\frac{\eta \nu D y[i]}{N} \mathfrak{S}(-y[i]F_{\widehat{\bm{v}}}(\bm{X}^{(i)}))\sum_{m=1}^D \sum_{b=1}^D\widehat{\bm{S}}_{m,b}^{(t)}\mathbf{P}\bm{X}_b^{(i)}\Big\|_2^2\leq M_i^2.$  \autoref{indh:lambgamxi} and $\|\mathbf{P}\bm{X}_u^{(i)}\|_2\leq \sigma^2 d \log(d)$ imply $M_i^2 \leq \frac{\eta^2\nu^2\mathrm{poly}(D)}{N^2}.$ We deduce that $ \|\eqref{eq:fewpwpkjkje}-\eqref{eq:fewefdfvdfewpe}\|_2\leq \eta\nu \frac{\mathrm{poly}(D)}{\sqrt{N}}.$

\paragraph{Summand 2: $ \|\eqref{eq:fewkkjjpwpkjkje}-\eqref{eq:fewefdfvdfoiookewpe}\|_2$.}  We use the 1-Lipschitzness of the sigmoid function and get:
\begin{align}
    &\|\eqref{eq:fewkkjjpwpkjkje}-\eqref{eq:fewefdfvdfoiookewpe}\|_2\nonumber\\
    \leq& \eta \nu\Biggr\|\mathbb{E}\Biggr[Dy \sum_{m=1}^D \sum_{b=1}^D \widehat{\bm{S}}_{m,b}^{(t)}\mathbf{P}\bm{X}_u\cdot\nonumber\\
    &\biggr[\mathfrak{S}\biggr(\hspace{-.1cm}-y\nu D\hspace{-.1cm}\sum_{a=1}^D\sum_{a'=1}^D \widehat{\bm{S}}_{a,a'}^{(t)}\langle \widehat{\bm{v}}^{(t)},\bm{X}_{a'}\rangle \biggr)-\mathfrak{S}\biggr(\hspace{-.1cm}-y\nu D\hspace{-.1cm}\sum_{a=1}^D \sum_{a'=1}^D \widehat{\bm{S}}_{a,a'}^{(t)}\langle \accentset{\circ}{\bm{v}}^{(t)},\bm{X}_{a'}\rangle \biggr)\biggr] \Biggr]\Biggr\|_2\nonumber\\
    \leq&\Theta(\eta\nu D) \mathbb{E}\biggr[ \sum_{m=1}^D \sum_{r=1}^D \widehat{\bm{S}}_{m,r}^{(t)}    \sum_{a=1}^D\biggr|D\sum_{a'=1}^D \widehat{\bm{S}}_{a,a'}^{(t)}\langle \accentset{\circ}{\bm{v}}^{(t)}-\widehat{\bm{v}}^{(t)},\bm{X}_{a'}\rangle\biggr|  \sum_{b=1}^D \widehat{\bm{S}}_{m,u}^{(t)}\|\mathbf{P}\bm{X}_b\|_2\biggr]\nonumber\\
    \leq&\Theta(\eta\nu D) \mathbb{E}\biggr[\sum_{m=1}^D \lambda_0\varepsilon_{\bm{v}}^{(t)}\cdot\biggr|\langle \bm{u}^{(t)},\sum_{a'=1}^D \bm{\xi}_{a'}\rangle \biggr|\cdot\sum_{b=1}^D \lambda_0 \|\bm{\xi}_b\|_2\biggr].\label{eq:fwkjewewe}
\end{align}
where we applied \autoref{indh:lambgamxi} in \eqref{eq:fwkjewewe}. Since with high probability,  $\|\bm{\xi}_b\|_2\leq \sigma\sqrt{d\log(d)}$, $\big|\langle \bm{u}^{(t)},\sum_{r=1}^D \bm{\xi}_r\rangle \big|\leq \sqrt{D\log(d)}\sigma$, we finally have: 
\begin{align*} 
    \|\eqref{eq:fewkkjjpwpkjkje}-\eqref{eq:fewefdfvdfoiookewpe}\|_2\leq \eta\nu\cdot\mathrm{poly}(D) \varepsilon_{\bm{v}}^{(t)} \sigma = \eta\nu\frac{\mathrm{poly}(D)}{\sqrt{d}} \varepsilon_{\bm{v}}^{(t)}.
\end{align*}
Combining the bounds on Summands 1 and 2 yields the aimed result.
\end{proof}

We now use \autoref{lem:epsv_init} to show that $\varepsilon_{\bm{v}}$ stays small.

\begin{lemma}\label{cor:epsv_init} For all $t\leq \mathscr{T}$, $\varepsilon_{\bm{v}}^{(t)}\leq  \omega\sqrt{d\log(d)}+ \nu^{1/(p-1)}\frac{\mathrm{poly}(D)}{\sqrt{N}}.$ By setting $N=\mathrm{poly}(d)$, we have: $\varepsilon_{\bm{v}}^{(t)}\leq1/\mathrm{poly}(d)$.
\end{lemma}
\begin{proof}[Proof of \autoref{cor:epsv_init}]  Unraveling  \autoref{lem:epsv_init} for $t=0,\dots,\mathscr{T}$ and using $\varepsilon_{\bm{v}}^{(0)}\leq \omega\sqrt{d\log(d)}$ (with high probability) leads to: 
\begin{equation}
\begin{aligned}\label{eq:jojojoj}
 \varepsilon_{\bm{v}}^{(\mathscr{T})} &\leq \omega\sqrt{d\log(d)}+ \eta\nu\zeta \frac{\Big(1+\eta\nu\frac{\mathrm{poly}(D)}{\sqrt{d}}\Big)^{\mathscr{T}}-1}{\eta\nu\frac{\mathrm{poly}(D)}{\sqrt{d}}}\leq  \omega\sqrt{d\log(d)}+ 2\mathscr{T}\eta\nu\zeta,
\end{aligned}  
\end{equation}
where we used $(1+x)^y\leq 1+2yx$ for $x\ll 1$ and $y\geq 0.$ Plugging the value of $\mathscr{T}$ in \eqref{eq:jojojoj} yields the aimed result.
\end{proof}

\subsubsection{$\widehat{A}_{a,b}^{(t)}$ stays small}\label{sec:subsubsecAbdinti}

We finally show that $\widehat{A}_{a,b}^{(t)}$ remains tiny for $t\in[0,\mathscr{T}].$

\begin{lemma}\label{lem:Qbd}
  Let $a,b\in[D]$. We have $|\widehat{A}_{a,b}^{(t)}|\leq \omega\sqrt{d\log(d)} +  \frac{\Theta(\nu^{2/(p-1)}) }{ D  e^{\beta}}.$
\end{lemma}
\begin{proof}[Proof of \autoref{lem:Qbd}] We remind that the GD update of $\widehat{A}_{a,b}^{(t)}$ is
\begin{align}\label{eq:jfoewojew}
    \widehat{A}_{a,b}^{(t+1)}&= \widehat{A}_{a,b}^{(t)} + \frac{\nu\eta}{N}\sum_{i=1}^N \mathfrak{S}\big(-y[i]F(\bm{X}[i])\big) \widehat{S}_{a,b}^{(t)}\sum_{m\neq b} \widehat{S}_{a,m}^{(t)} \langle \widehat{\bm{v}}^{(t)},\bm{X}_{m}[i]-\bm{X}_b[i]\rangle.
\end{align}
The proof is by induction. We assume that $\widehat{A}_{a,b}^{(t)}\leq \omega\sqrt{d\log(d)} +  \frac{\Theta(\nu^{2/(p-1)}) }{ D  e^{\beta}}.$ We first apply Cauchy-Schwarz on \eqref{eq:jfoewojew} and get:
\begin{align}
    |\widehat{A}_{a,b}^{(t+1)}|&\leq  |\widehat{A}_{a,b}^{(t)}| +  \nu\eta  \widehat{S}_{a,b}^{(t)}\sum_{m\neq b} \widehat{S}_{a,m}^{(t)}  \|\widehat{\bm{v}}^{(t)}\|_2.
\end{align}
Using the induction hypothesis, we have $\widehat{S}_{a,b}^{(t)}\leq \Theta(1)/D.$ Thus, we have
\begin{align}\label{eq:jowjfwoj}
    |\widehat{A}_{a,b}^{(t+1)}|&\leq  |\widehat{A}_{a,b}^{(t)}| + \frac{\Theta(\nu\eta)}{D}    \|\widehat{\bm{v}}^{(t)}\|_2\leq |\widehat{A}_{a,b}^{(t)}| + \frac{\Theta(\nu\eta)}{D}    \sqrt{(\widehat{\alpha}^{(t)})^2+(\varepsilon_{\bm{v}}^{(t)})^2}.
\end{align}
We sum \eqref{eq:jowjfwoj} and get:
\begin{align}\label{eq:jwojfw}
    |\widehat{A}_{a,b}^{(\mathscr{T})}|&\leq |\widehat{A}_{a,b}^{(0)}| + \frac{\Theta(\nu\eta)}{D}    \sum_{t=0}^{\mathscr{T}-1}\sqrt{(\widehat{\alpha}^{(t)})^2+(\varepsilon_{\bm{v}}^{(t)})^2}.
\end{align}
We now use  \autoref{lem:eventinit_real} and  \autoref{cor:epsv_init} in \eqref{eq:jwojfw} and get:
\begin{align}
     |\widehat{A}_{a,b}^{(\mathscr{T})}|&\leq \omega\sqrt{d\log(d)} +  \frac{\Theta(\nu^{1/(p-1)}) }{ D  e^{\beta}}  \sqrt{ \nu^{2/(p-1)}+ \bigg(\omega\sqrt{d\log(d)}+\nu^{1/(p-1)}\frac{\mathrm{poly}(D)}{\sqrt{N}}\bigg)^2}\nonumber\\
     &\leq \omega\sqrt{d\log(d)} +  \frac{\Theta(\nu^{2/(p-1)}) }{ D  e^{\beta}}.%\frac{\Theta(\nu^{1/(p-1)}) }{ D  e^{\beta}}  \sqrt{ \omega\log(d)+ \nu^{1/(p-1)}\Big(1+\frac{\mathrm{poly}(D)}{\sqrt{N}}\Big)}\\
     %&\leq \omega\sqrt{d\log(d)}+\frac{\Theta(\nu^{1/(p-1)}) }{ D  e^{\beta}} \Big( \omega\log(d)+ \nu^{1/(p-1)}\Big).
\end{align}

\end{proof}

\subsubsection{Auxiliary lemmas}

%Since $\widehat{\alpha}^{(t)}$ is small, the update of $\widehat{\alpha}^{(t)}$  is much larger than the one of $\widehat{\bm{A}}_{i,j}^{(t)}$. Therefore, $\widehat{\bm{A}}_{i,j}^{(t)}=\widehat{\bm{A}}_{i,j}^{(0)}+\omega_Q^{(t)}=\omega_Q^{(t)}$ where $\omega_Q^{(t)}$ very small  for $i\neq j.$

\begin{lemma}\label{lem:singbdT1_real} Let $\mathscr{T}$ be the time where $\widehat{\alpha}^{(t)}\geq \nu^{1/(p-1)}$. Let $t\in[0,\mathscr{T}]$, $i\in[N]$  and $j\in\mathcal{S}_{\ell(\bm{X}[i])}.$ We have: 
\begin{align*}
    %y\frac{\langle v^{(t)},O_i^{(t)}\rangle^p}{(\alpha^{(t)})^{p}}= y\frac{\langle v^{(t)},O_i^{(t)}\rangle^{p-1}\langle O_i^{(t)},w^*\rangle}{(\alpha^{(t)})^{p-1}}= 
    y[i]\langle \bm{w}^{*},\bm{O}_j^{(t)}[i]\rangle &= e^{\beta}.
\end{align*}
%For all $i\in[N]$, 
This implies $\widehat{{\large \pmb{\mathpzc{S}}}}^{(t)}=C e^{\beta}$ for all $t\in [0,\mathscr{T}].$
\end{lemma}
\begin{proof}[Proof of \autoref{lem:singbdT1_real}] We successively apply \autoref{lem:Qbd} and  \autoref{lem:chernoff_sumnoise} to get:
\begin{align}%\label{eq:fsewnwef}
 y[i]\langle \bm{w}^{*},\bm{O}_j^{(t)}[i]\rangle &= D \bigg( \widehat{S}_{j,j}^{(t)}+ \sum_{k\in\mathcal{S}_{\ell(\bm{X}[i])}}\widehat{S}_{j,k}^{(t)}+y[i]\sum_{h\neq\ell(\bm{X}[i])}\sum_{r\in\mathcal{S}_h}\widehat{S}_{j,r}^{(t)}\delta_{r}\bigg)\nonumber\\
 &\geq   \Theta\bigg(  e^{\mathrm{polyloglog}(d)}+ \Theta(C-1)  -\Theta(1)\sum_{h\neq\ell(\bm{X}[i])}\sum_{r\in\mathcal{S}_h}\delta_{h,r}\bigg)\label{eq:jdvde}\\
 &\geq   \Theta\bigg(  e^{\mathrm{polyloglog}(d)}+ \Theta(C-1)  -\Theta(qD\log(d))\bigg)\label{eq:dweojwe}\\
  &\geq e^{\mathrm{polyloglog}(d)}.\nonumber
\end{align}
Similarly, we also have  $ y[i]\langle \bm{w}^{*},\bm{O}_j^{(t)}[i]\rangle\leq e^{\mathrm{polyloglog}(d)}.$
%where we used \autoref{indh:lambgamxi} in the last inequality of \eqref{eq:fsewnwef}.
\end{proof}

\begin{lemma} \label{lem:noisbdT1_real} 
\mbox{Let $t\in[0,\mathscr{T}].$ Assume that $N=\mathrm{poly}(d).$ With high probability, ${\large \widehat{\pmb{\mathpzc{N}}}}^{(t)}\leq 1/\mathrm{poly}(d).$}
\end{lemma}
\begin{proof}[Proof of \autoref{lem:noisbdT1_real}] During this time phase, the sigmoid stays constant. Therefore, we have $\mathbb{E}[{\large \widehat{\pmb{\mathpzc{N}}}}^{(t)}]=0.$ Therefore, we apply Hoeffding inequality and \autoref{lem:exppowdeltas} to get:
\begin{align}
    {\large \widehat{\pmb{\mathpzc{N}}}}^{(t)}\leq   \sqrt{ 8N\log(d) \max_{i\in[N]}\;M_i^2},
\end{align}
where $A_i$ is a constant such that $\Big|\frac{1}{N}\sum_{j\not\in\mathcal{S}_{\ell(\bm{X}[i])}} \hspace{-.1cm} y[i]\langle \bm{w}^*,\bm{O}_j^{(t)}[i]\rangle\Big| \leq M_i\leq \frac{\mathrm{poly}(D)}{N^2}.$ Since $N=\mathrm{poly}(d)$, we finally proved ${\large \widehat{\pmb{\mathpzc{N}}}}^{(t)}\leq 1/\mathrm{poly}(d). $
\end{proof}

\begin{lemma}\label{lem:sumalphaT_init}
Let $t\leq\mathscr{T}$. The sum of $\widehat{\alpha}^{(t)}$'s is bounded as: 
\begin{align}
    \sum_{\tau=0}^t \widehat{\alpha}^{(\tau)}&= t\widehat{\alpha}^{(0)} + \Theta(\eta\nu ) e^{\beta} t^2.
\end{align}
\end{lemma}
\begin{proof}[Proof of \autoref{lem:sumalphaT_init}] Let $\tau\in[0,\mathscr{T}].$ We sum the update rule of $\widehat{\alpha}^{(t)}$ (\autoref{lem:eventinit_real}) and obtain: $\widehat{\alpha}^{(\tau)}=\widehat{\alpha}^{(0)}+\Theta(C\eta\nu ) e^{\beta}\tau.$ Summing again this update yields the aimed result.

\end{proof}

\begin{lemma}\label{lem:activ_simple_real}
Let $(\bm{X},\cdot)\sim\mathcal{D}$ and $j\in[D]$.  Assume that $\widehat{\alpha}^{(t)}\leq \nu^{1/(p-1)}$. Then, we have:
\begin{align*}
    \sigma'\Big(D\sum_{k=1}^D \widehat{S}_{j,k}^{(t)}\langle \widehat{\bm{v}},\bm{X}_k\rangle\Big) &=  \begin{cases}
        \Theta(\nu) & \text{if } \widehat{\alpha}^{(t)}\leq \nu^{1/(p-1)}\\
        \Theta(p)\Big(D\sum_{k=1}^D \widehat{S}_{j,k}^{(t)}\langle \widehat{\bm{v}},\bm{X}_k\rangle\Big)^{p-1} & \text{otherwise}
    \end{cases}.
\end{align*}
\end{lemma}
\begin{proof}[Proof of \autoref{lem:activ_simple_real}] We remind that the derivative of the activation function $\sigma'(x)= px^{p-1}+\nu$. We first remark that for all $x,$ $\sigma'(x)\geq \nu$. Besides, we have since $p-1$ is even, 
\begin{align}\label{eq:conddmmnerpowerp}
    px^{p-1} +\nu \leq  3\nu \iff  |x|\leq \Big(\frac{2\nu}{p}\Big)^{1/(p-1)}.
\end{align}
In our case, we have $x=D\widehat{\alpha}^{(t)}\sum_{k=1}^D \widehat{S}_{j,k}^{(t)}\langle \bm{w}^*,\bm{X}_k\rangle.$ Using \autoref{indh:lambgamxi}, we have $|x|\geq e^{\mathrm{polyloglog}(d)} \alpha^{(t)}.$ Therefore, a sufficient condition for \eqref{eq:conddmmnerpowerp} to hold is $\widehat{\alpha}^{(t)}\leq \nu^{1/(p-1)} e^{-\mathrm{polyloglog}(d)}.$ Since $e^{\mathrm{polyloglog}(d)}\ll \mathrm{poly}(d)$, we can simplify this condition as $\widehat{\alpha}^{(t)}\leq \nu^{1/(p-1)}$. 
Proving the second part of the lemma can be done as in the proof of \autoref{lem:activ_simple}.
\end{proof}

%\begin{lemma}\label{lem:epsv_real}
%   During the training process, the error $\varepsilon_{\bm{v}}^{(t)}$ updates as: 
%   \begin{align*}
%       \varepsilon_{\bm{v}}^{(t+1)}&\leq \biggr(1 + \eta (\alpha^{(t)})^{p-1}\frac{\mathrm{poly}(D)}{\sqrt{d}}\biggr)\varepsilon_{\bm{v}}^{(t)} +\eta\zeta, \qquad \text{where }\zeta=\frac{\mathrm{poly}(D)}{\sqrt{N}}.
%   \end{align*}
%\end{lemma}

\subsection{Coupling between the semi-idealized and realistic processes ($t\in[\mathscr{T},T]$)}

In this section, we aim to bound the realistic iterates $\widehat{A}_{i,j}^{(t)}$ and $\widehat{\bm{v}}^{(t)}$ for $t\in[\mathscr{T},T].$ For this reason, we introduce a "semi-idealized" learning process (\autoref{sec:semi_ideal}) which may be viewed as a mid-point between the idealized and realistic process. We first bound the iterates in this process. Then, using this process, we show that $\varepsilon_{\bm{v}}^{(t)}$ (\autoref{sec:epsilonv_bd}) and $\Delta_{\bm{A}}^{(t)}:=\max_{i\neq j}|\widehat{A}_{i,j}^{(t)}-\widecheck{A}_{i,j}^{(t)}|$ (\autoref{sec:deltaA}) stay small. Here, $\widecheck{A}_{i,j}^{(t)}$ is the semi-idealized attention matrix coefficient. Finally, since $\varepsilon_{\bm{v}}^{(T)}$ and $\Delta_{\bm{A}}^{(T)}$ are small, the final iterates $\widehat{\alpha}^{(T)}$ and $\alpha^{(T)}$ are equal (\autoref{sec:dynv}) and thus, the model fits the labeling function (\autoref{sec:general_realistic}).

\subsubsection{Defining the semi-idealized process}\label{sec:semi_ideal}

We define an intermediate learning process that we refer to as the "semi-idealized" process. This process starts at time $t=\mathscr{T}$ involves two parameters: the semi-idealized value vector $\widecheck{\bm{v}}$ and semi-idealized attention matrix $\widecheck{\bm{A}}$ defined as
\begin{itemize}
    \item[--] the value vector $\widecheck{\bm{v}}$ is fixed and satisfies  $\widecheck{\bm{v}}^{(t-\mathscr{T})}=\widehat{\alpha}^{(t)}\bm{w}^*$ for $t\in[\mathscr{T},T].$
    \item[--] $\widecheck{A}_{i,j}^{(t-\mathscr{T})}$ is a trainable parameter and is initialized as $\widecheck{A}_{i,j}^{(0)}=0$ for $i\neq j$.
\end{itemize}
%For $t\in[\mathscr{T},T]$,  $\widecheck{\bm{v}}^{(t-\mathscr{T})}=\widehat{\alpha}^{(t)}\bm{w}^*$ where $\widehat{\alpha}^{(\mathscr{T})}\geq \nu^{1/(p-1)}$ (\autoref{lem:eventinit_real})  and attention weights $\widecheck{A}_{i,j}^{(t-\mathscr{T})}$ initialized as $\widecheck{A}_{i,j}^{(0)}=0$ for $i\neq j$. 
Therefore, the only trainable parameter in this process is $\widecheck{\bm{A}}$. In the semi-idealized process, we minimize the population risk
\begin{align}\label{eq:population_tilde}\tag{$\tilde{\text{P}}$}
    \min_{\widecheck{\bm{A}}}\;\; \mathbb{E}_{\mathcal{D}}\big[\log\big(1+e^{-yF(\bm{X})}\big)\big]:=\widecheck{\mathcal{L}}(\widecheck{\bm{v}},\widecheck{\bm{A}}).
    %\widecheck{\mathcal{L}}(\widecheck{\bm{v}},\widecheck{\bm{A}}).
\end{align}
We remark that such process present similarities to the idealized case. In particular, it satisfies all the invariance and symmetry properties from \autoref{lem:Qreducvar}. We thus define 
\begin{itemize}
    \item[--] $\widecheck{A}_{i,j}^{(t)}=\widecheck{\gamma}^{(t)}$ for all $\ell\in[L]$ and $i,j\in\mathcal{S}_{\ell}.$
    \item[--] $\widecheck{A}_{i,j}^{(t)}=\widecheck{\rho}^{(t)}$ for all $\ell,m\in[L]$ such that $m\neq\ell$ and  $i\in\mathcal{S}_{\ell}$ and $j\in\mathcal{S}_{m}.$
    
\end{itemize}

Therefore,  $ \widecheck{\gamma}^{(t)}$ and $ \widecheck{\rho}^{(t)}$ are  respectively updated as in \autoref{lem:gd_gamma} and  \autoref{lem:gd_rho}. We define also the softmax terms
\begin{align*}
    \widecheck{\Lambda}^{(t)}&= \frac{e^{\beta}}{e^{\beta}+(C-1)e^{\widecheck{\gamma}^{(t)}}+(D-C)e^{\widecheck{\rho}^{(t)}}}, &\widecheck{\Gamma}^{(t)}=\frac{e^{ \widecheck{\gamma}^{(t)}}}{e^{\beta}+(C-1)e^{ \widecheck{\gamma}^{(t)}}+(D-C)e^{ \widecheck{\rho}^{(t)}}},\\
    \widecheck{\Xi}^{(t)}&= \frac{e^{ \widecheck{\rho}^{(t)}}}{e^{\beta}+(C-1)e^{ \widecheck{\gamma}^{(t)}}+(D-C)e^{ \widecheck{\rho}^{(t)}}}, & \widecheck{G}^{(t)}= D( \widecheck{\Lambda}^{(t)}+(C-1) \widecheck{\Gamma}^{(t)}).\hspace{2cm}
\end{align*}
 We finally assume \autoref{indh:lambgamxi} for this process. This latter can be proved using the same arguments as in \autoref{sec:indhideal}. %Using this semi-idealized process, we prove \autoref{lem:eps_v} and  \autoref{lem:q_real}. 

\subsubsection{\texorpdfstring{Realistic dynamics are mainly on  $\mathrm{span}(\bm{w}^*)$}{Realistic dynamics are mainly on spanw**}}\label{sec:epsilonv_bd}

We previously showed in \autoref{cor:epsv_init} that $\varepsilon_{\bm{v}}^{(t)}$ is small in the initial steps. We  now show that it stays small during the whole process.

\begin{lemma}\label{lem:eps_v}
    Assume that we run GD on the empirical risk \eqref{eq:empirical} for $T$ iterations with parameters set as in \autoref{param} and  the number of samples is $N=\mathrm{poly}(d).$ Then, $\widehat{\bm{v}}^{(t)}$ mainly lies in $\mathrm{span}(\bm{w}^*)$ i.e.\ for $t\leq T$,  $\varepsilon_{\bm{v}}^{(t)}\leq 1/\mathrm{poly}(d).$
\end{lemma}

We now proceed to the proof of \autoref{lem:eps_v}. We first characterize the recursion satisfied by $\varepsilon_{\bm{v}}^{(t)}.$

\begin{restatable}{lemma}{lemepsv}\label{lem:epsv}
    Assume that we run GD on the empirical risk \eqref{eq:empirical} for $T$ iterations with parameters set as in \autoref{param}. Then,  $\varepsilon_{\bm{v}}^{(t)}$ satisfies for $t\in[\mathscr{T},T]$
    \begin{align*}
       \varepsilon_{\bm{v}}^{(t+1)}&\leq  
          \big(1+\eta (\widehat{\alpha}^{(t)})^{p-1}\frac{\mathrm{poly}(D)}{\sqrt{d}}\big)\varepsilon_{\bm{v}}^{(t)}+\eta\zeta, \qquad \text{where } \zeta=\frac{\mathrm{poly}(D)}{\sqrt{N}}.
    \end{align*}
 %Consequently, as long as $N=\mathrm{poly}(d),$ we have $\varepsilon_{\bm{v}}^{(t)}\leq1/\mathrm{poly}(d)$ for $t\leq T.$
\end{restatable}

\begin{proof}[Proof of \autoref{lem:epsv}] Let $\mathbf{P}:= (\mathbf{I}-\bm{w}^*\bm{w}^{*\top})$, $\accentset{\circ}{\bm{v}}^{(t)}:=\widehat{\alpha}^{(t)}\bm{w}^*$ and $t\in[\mathscr{T},T].$ The projected update of $\widehat{\bm{v}}$ satisfies: 
\begin{align}
     &\|\mathbf{P}\widehat{\bm{v}}^{(t+1)}-\mathbf{P}\widehat{\bm{v}}^{(t)}\|_2\nonumber\\
     \leq&\eta\biggr\|\frac{1}{N}\sum_{i=1}^N Dy[i]\mathfrak{S}(-y[i]F_{\widehat{\bm{v}}}(\bm{X}[i]))\sum_{m=1}^D\biggr(D\sum_{r=1}^D \widehat{S}_{m,r}^{(t)}\langle \widehat{\bm{v}}^{(t)},\bm{X}_r^{(i)}\rangle\biggr)^{p-1}\sum_{b=1}^D\widehat{S}_{m,b}^{(t)}\mathbf{P}\bm{X}_b[i]\label{eq:fewpwpe}\\ 
     -&\mathbb{E}\biggr[Dy\mathfrak{S}(-yF_{\widehat{\bm{v}}}(\bm{X}))\sum_{m=1}^D\biggr(D\sum_{r=1}^D \widehat{S}_{m,r}^{(t)}\langle \widehat{\bm{v}}^{(t)},\bm{X}_r\rangle\biggr)^{p-1}\sum_{b=1}^D \widehat{S}_{m,b}^{(t)}\mathbf{P}\bm{X}_b\biggr]\biggr\|_2 \label{eq:fewefewpe}\\
     &\nonumber\\
     +&\eta\biggr\|\mathbb{E}\biggr[Dy\mathfrak{S}(-yF_{\widehat{\bm{v}}}(\bm{X}))\sum_{m=1}^D\biggr(D\sum_{r=1}^D \widehat{S}_{m,r}^{(t)}\langle \widehat{\bm{v}}^{(t)},\bm{X}_r\rangle\biggr)^{p-1}\sum_{b=1}^D \widehat{S}_{m,b}^{(t)}\mathbf{P}\bm{X}_b\biggr]\label{eq:woesddsfnfeo}\\
     -&\mathbb{E}\biggr[Dy\mathfrak{S}(-yF_{\widehat{\bm{v}}}(\bm{X}))\sum_{m=1}^D\biggr(D\sum_{r=1}^D \widehat{S}_{m,r}^{(t)}\langle \accentset{\circ}{\bm{v}}^{(t)},\bm{X}_r\rangle\biggr)^{p-1}\sum_{b=1}^D \widehat{S}_{m,b}^{(t)}\mathbf{P}\bm{X}_b\biggr]\biggr\|_2\label{eq:woenfeofrwed}\\
     &\nonumber\\
     +&\eta\biggr\|\mathbb{E}\biggr[Dy\mathfrak{S}(-yF_{\widehat{\bm{v}}}(\bm{X}))\sum_{m=1}^D\biggr(D\sum_{r=1}^D \widehat{S}_{m,r}^{(t)}\langle \accentset{\circ}{\bm{v}}^{(t)},\bm{X}_r\rangle\biggr)^{p-1}\sum_{b=1}^D \widehat{S}_{m,b}^{(t)}\mathbf{P}\bm{X}_b\biggr]\label{eq:woenfeo}\\
     -&\mathbb{E}\biggr[Dy\mathfrak{S}(-yF_{\accentset{\circ}{\bm{v}}}(\bm{X}))\sum_{m=1}^D\biggr(D\sum_{r=1}^D \widehat{S}_{m,r}^{(t)}\langle \accentset{\circ}{\bm{v}}^{(t)},\bm{X}_r\rangle\biggr)^{p-1}\sum_{b=1}^D \widehat{S}_{m,b}^{(t)}\mathbf{P}\bm{X}_b\biggr]\biggr\|_2.\label{eq:ewjfnweonj}
\end{align} 
Remark that  \eqref{eq:ewjfnweonj} is equal to zero because $\mathbb{E}[\mathbf{P}\bm{X}_u]=0$ for all $u\in[D].$ %We now bound the two summands above. 

\paragraph{Summand 1: $ \|\eqref{eq:fewpwpe}-\eqref{eq:fewefewpe}\|_2$.} Using the matrix Hoeffding inequality, we have with high probability,
    $\|\eqref{eq:fewpwpe}-\eqref{eq:fewefewpe}\|_2\leq 16\sqrt{\log(d)\sum_{i=1}^N M_i^2},$
where\\
$\Big\|\frac{\eta Dy[i]}{N}\mathfrak{S}(-y[i]F_{\widehat{\bm{v}}}(\bm{X}[i]))\sum_{m=1}^D\left(D\sum_{r=1}^D \widehat{S}_{m,r}^{(t)}\langle \widehat{\bm{v}}^{(t)},\bm{X}_r[i]\rangle\right)^{p-1}\sum_{u=1}^D\widehat{S}_{m,u}^{(t)}\mathbf{P}\bm{X}_u[i]\Big\|_2^2\leq M_i^2.$  \autoref{indh:lambgamxi},  \autoref{lem:event3} and $\|\mathbf{P}\bm{X}_u[i]\|_2\leq \sigma^2 d \log(d)$ imply $M_i^2 \leq \frac{\eta^2\mathrm{poly}(D)}{N^2}.$ We deduce that $ \|\eqref{eq:fewpwpe}-\eqref{eq:fewefewpe}\|_2\leq \eta \frac{\mathrm{poly}(D)}{\sqrt{N}}.$

\paragraph{Summand 2: $\|\eqref{eq:woesddsfnfeo}-\eqref{eq:woenfeofrwed}\|_2$.} The function $x\mapsto x^{p-1}$ is $(p-1)M^{p-2}$ Lipschitz on a bounded domain $[0,M]$. We apply this property and $(p-1)\max_{m\in[D]}\big(D\sum_{r=1}^D \widehat{S}_{m,r}^{(t)}\langle \widehat{\bm{v}}^{(t)},\bm{X}_r\rangle\big)^{p-2}\leq  (\widehat{\alpha}^{(t)})^{p-2} \mathrm{poly}(D)$ to get: \begin{align}
   &\hspace{-.3cm}\|\eqref{eq:woesddsfnfeo}-\eqref{eq:woenfeofrwed}\|_2\nonumber\\
   \leq&\eta (\widehat{\alpha}^{(t)})^{p-2}\mathrm{poly}(D) \mathbb{E}\Biggr[\sum_{m=1}^D \sum_{r=1}^D \widehat{S}_{m,r}^{(t)}|\langle \widehat{\bm{v}}^{(t)}-\accentset{\circ}{\bm{v}}^{(t)},\bm{X}_r\rangle|\cdot \sum_{b=1}^D \widehat{S}_{m,b}^{(t)}\|\mathbf{P}\bm{X}_b\|_2\Biggr]\nonumber\\
   \leq& \eta (\widehat{\alpha}^{(t)})^{p-2}\mathrm{poly}(D) \mathbb{E}\Biggr[\sum_{m=1}^D \sum_{r=1}^D \widehat{S}_{m,r}^{(t)}\varepsilon_{\bm{v}}^{(t)} |\langle\bm{u}^{(t)},\bm{X}_r\rangle|\cdot \sum_{b=1}^D \widehat{S}_{m,b}^{(t)}\|\mathbf{P}\bm{X}_b\|_2\Biggr].
\end{align}
With high probability,  $\|\bm{\xi}_b\|_2\leq \sigma\sqrt{\log(d)}\leq \sqrt{\log(d)/d}$, $|\langle \bm{u}^{(t)}, \bm{\xi}_r\rangle |\leq \sqrt{\log(d)}\sigma$. We thus get: 
\begin{align}
    \|\eqref{eq:woesddsfnfeo}-\eqref{eq:woenfeofrwed}\|_2
    %&\leq \eta (\alpha^{(t)})^{p-1}\mathrm{poly}(D) \varepsilon_v^{(t)} \mathbb{E}\Biggr[D^2\sum_{m=1}^D \sum_{r=1}^D \widehat{S}_{m,r}^{(t)}|\langle u^{(t)},\xi_r\rangle|\cdot \sum_{b=1}^D \widehat{S}_{m,b}^{(t)}\|\xi_b\|_2\Biggr]\nonumber\\
    &\leq \eta (\widehat{\alpha}^{(t)})^{p-2}\frac{\mathrm{poly}(D)}{\sqrt{d}}\varepsilon_{\bm{v}}^{(t)}.\label{eq:oewfojefej}
\end{align}

%\textcolor{red}{Since $\varepsilon_v^{(t)}$ is small, we have $\biggr(D\sum_{r=1}^D \widehat{S}_{m,r}^{(t)}\langle \widehat{v}^{(t)},X_r\rangle\biggr)^{p-1}=\biggr(D\sum_{r=1}^D \widehat{S}_{m,r}^{(t)}\langle v^{(t)},X_r\rangle\biggr)^{p-1}.$}
\paragraph{Summand 3: $ \|\eqref{eq:woenfeo}-\eqref{eq:ewjfnweonj}\|_2$.}  We have:
\begin{equation}\label{eq:fwjenewjd}
\begin{aligned}
    &\|\eqref{eq:woenfeo}-\eqref{eq:ewjfnweonj}\|_2\\
    \leq& \eta \Biggr\|\mathbb{E}\Biggr[Dy \sum_{m=1}^D\biggr(D\sum_{r=1}^D \widehat{S}_{m,r}^{(t)}\langle \accentset{\circ}{\bm{v}}^{(t)},\bm{X}_r\rangle\biggr)^{p-1}\sum_{b=1}^D \widehat{S}_{m,b}^{(t)}\mathbf{P}\bm{X}_u\cdot\\
    &\biggr[\mathfrak{S}\biggr(\hspace{-.1cm}-y\hspace{-.1cm}\sum_{a=1}^D\hspace{-.1cm}\big(D\sum_{a'=1}^D \widehat{S}_{a,a'}^{(t)}\langle \accentset{\circ}{\bm{v}}^{(t)},\bm{X}_{a'}\rangle\big)^{p}\biggr)-\mathfrak{S}\biggr(\hspace{-.1cm}-y\hspace{-.1cm}\sum_{a=1}^D\hspace{-.1cm}\big(D\sum_{a'=1}^D \widehat{S}_{a,a'}^{(t)}\langle \widehat{\bm{v}}^{(t)},\bm{X}_{a'}\rangle\big)^{p}\biggr)\biggr] \Biggr]\Biggr\|_2.
\end{aligned}
\end{equation}
%Since for all $a\in[D]$, $\big(D\sum_{a'=1}^D \widehat{S}_{a,a'}^{(t)}\langle \widehat{\bm{v}}^{(t)},\bm{X}_{a'}\rangle\big)^p \geq -1/D$ and $\big(D\sum_{a'=1}^D \widehat{S}_{a,a'}^{(t)}\langle \widehat{\bm{v}}^{(t)},\bm{X}_{a'}\rangle\big)^p \geq 0$, 
We apply \autoref{lem:lippsi} to bound the local change of the sigmoid in \eqref{eq:fwjenewjd} which yields: 
\begin{align}
     &\|\eqref{eq:woenfeo}-\eqref{eq:ewjfnweonj}\|_2\nonumber\\
    \leq&\Theta(\eta D) \mathbb{E}\biggr[ \sum_{m=1}^D\biggr(D\sum_{r=1}^D \widehat{S}_{m,r}^{(t)}\langle \accentset{\circ}{\bm{v}}^{(t)},\bm{X}_r\rangle\biggr)^{p-1} \hspace{-.2cm} \sum_{a=1}^D\biggr|D\sum_{a'=1}^D \widehat{S}_{a,a'}^{(t)}\langle \accentset{\circ}{\bm{v}}^{(t)}-\widehat{\bm{v}}^{(t)},\bm{X}_{a'}\rangle\biggr|\cdot \sum_{b=1}^D \widehat{S}_{m,u}^{(t)}\|\mathbf{P}\bm{X}_b\|_2\biggr]\nonumber\\
    \leq&\Theta(\eta D^2) \mathbb{E}\biggr[ \sum_{m=1}^D\biggr(D\sum_{r=1}^D \widehat{S}_{m,r}^{(t)}\langle \accentset{\circ}{\bm{v}}^{(t)},\bm{X}_r\rangle\biggr)^{p-1}  \hspace{-.2cm}  \sum_{a=1}^D \sum_{a'=1}^D \widehat{S}_{a,a'}^{(t)} \varepsilon_{\bm{v}}^{(t)}|\langle \bm{u}^{(t)},\bm{X}_{a'}\rangle| \sum_{b=1}^D \widehat{S}_{m,u}^{(t)}\|\mathbf{P}\bm{X}_b\|_2\biggr]\label{eq:fjefej}.
    %\leq& \Theta(\eta) \mathbb{E}\biggr[\sum_{m=1}^D\biggr(\sum_{r=1}^D \lambda_0 \alpha^{(t)}\biggr)^{p-1}D\lambda_0\varepsilon_{\bm{v}}^{(t)}\cdot\biggr|\langle \bm{u}^{(t)},\sum_{a'=1}^D \bm{\xi}_{a'}\rangle \biggr|\cdot\sum_{b=1}^D \lambda_0 \|\bm{\xi}_b\|_2\biggr],\label{eq:fjefej}
\end{align}
We apply \autoref{indh:lambgamxi} to bound the softmax terms
 in \eqref{eq:fjefej}. Besides,  with high probability, we have $\|\bm{\xi}_b\|_2\leq \sigma\sqrt{d\log(d)}$, $|\langle \bm{u}^{(t)}, \bm{\xi}_{a'}\rangle |\leq \sqrt{\log(d)}\sigma$. Thus,  we have: 
\begin{align}\label{eq:sndepsv}
    \|\eqref{eq:woenfeo}-\eqref{eq:ewjfnweonj}\|_2&\leq \eta\cdot\mathrm{poly}(D) (\widehat{\alpha}^{(t)})^{p-1}\varepsilon_{\bm{v}}^{(t)} \sigma \leq \eta\frac{\mathrm{poly}(D)}{\sqrt{d}} (\widehat{\alpha}^{(t)})^{p-1} \varepsilon_{\bm{v}}^{(t)}.
\end{align}
We combine the bounds on the three summands to obtain the recursion of $\varepsilon_{\bm{v}}^{(t)}$. 
\end{proof}

We now prove \autoref{cor:epsv} that gives the final bound on $\varepsilon_{\bm{v}}^{(t)}$ for $t\leq T.$

\begin{lemma}\label{cor:epsv}  For all $t\leq T$, $\varepsilon_{\bm{v}}^{(t)}\leq O(\frac{\mathrm{poly}(d)}{\sqrt{N}}).$ By setting $N=\mathrm{poly}(d)$,   $\varepsilon_{\bm{v}}^{(t)}\leq \frac{1}{\mathrm{poly}(d)}.$
\end{lemma}
\begin{proof}[Proof of \autoref{cor:epsv}] %In this proof, we maintain the following hypotheses: 
%\begin{itemize}
%    \item[--] $\widehat{\alpha}^{(t)}$ satisfies the same recursion bound as the idealized $\alpha^{(t)}$ in \autoref{lem:event1} and \autoref{lem:event3}.
%    \item[--] $\Delta_{\bm{A}}^{(t)}\leq 1/\mathrm{poly}(d)$ for $t\leq T.$
%\end{itemize}
%We respectively prove these two facts in **** and \autoref{lem:epsq}. 
We bound $\varepsilon_{\bm{v}}^{(t)}$ in the following two regimes: $t\in[\mathscr{T},\mathscr{T}+\widecheck{\mathcal{T}}_0]$ and $t\in [\mathscr{T}+\widecheck{\mathcal{T}}_0,T].$ 

\paragraph{First phase: $t\in[\mathscr{T},\mathscr{T}+\widecheck{\mathcal{T}}_0]$.} Unraveling  \autoref{lem:epsv} for $t=\mathscr{T},\dots,\mathscr{T}+\widecheck{\mathcal{T}}_0$  leads to: 
\begin{equation}
\begin{aligned}\label{eq:jewfiferf}
    \hspace{-.4cm} \varepsilon_{\bm{v}}^{(t)}&\leq   \Big[\varepsilon_{\bm{v}}^{(\mathscr{T})}+ \eta\zeta\widecheck{\mathcal{T}}_0 \Big]\prod_{\tau=\mathscr{T}}^{\mathscr{T}+\widecheck{\mathcal{T}}_0}\biggr(1+\eta(\widehat{\alpha}^{(\tau)})^{p-1}\frac{\mathrm{poly}(D)}{\sqrt{d}}\biggr)\\
    &\leq  [\varepsilon_{\bm{v}}^{(\mathscr{T})}+ \eta\zeta\widecheck{\mathcal{T}}_0 ]\prod_{\tau=0}^{\widecheck{\mathcal{T}}_0-1}\biggr(1+\eta(\widehat{\alpha}^{(\tau)})^{p-1}\frac{\mathrm{poly}(D)}{\sqrt{d}}\biggr).
\end{aligned}  
\end{equation}

\autoref{lem:alphahatupdinit} provides the update of $\widehat{\alpha}^{(t)}$ during this time phase. We thus apply \autoref{lem:pow_method_prod} to bound the product term in \eqref{eq:jewfiferf}. 
\begin{align}
   \prod_{\tau=0}^{\widecheck{\mathcal{T}}_0-1}\biggr(1+\eta(\widehat{\alpha}^{(\tau)})^{p-1}\frac{\mathrm{poly}(D)}{\sqrt{d}}\biggr)&\leq  \biggr(1+\frac{\Theta(1)}{C^{2(p-2)}\lambda_0^{p-2}} \biggr)^{\frac{\eta\mathrm{poly}(D)}{\sqrt{d}}}\leq O(1).\label{eq:oefwoefff}
\end{align}
%\begin{align}
%   \prod_{\tau=0}^{\mathcal{T}_0-1}\biggr(1+\eta(\alpha^{(\tau)})^{p-1}\frac{\mathrm{poly}(D)}{\sqrt{d}}\biggr)&\leq \biggr(1+\frac{\Theta(1)}{C^{2(p-2)}\lambda_0^{p-2}} \biggr)^{\frac{\lambda_0^{p-1}|\log(C^2\lambda_0\alpha^{(0)})|}{C^2e^{\mathrm{polyloglog}(d)}}\cdot\frac{\eta\mathrm{poly}(D)}{\sqrt{d}}}\nonumber\\
%    &= \biggr(1+\frac{\Theta(1)}{C^{2(p-2)}\lambda_0^{p-2}} \biggr)^{\frac{\eta\mathrm{poly}(D)}{\sqrt{d}}}.\label{eq:oefwoefff}
%\end{align}
Plugging \eqref{eq:oefwoefff} in \eqref{eq:jewfiferf} yields a bound on $\varepsilon_{\bm{v}}^{(\mathscr{T}+\widecheck{\mathcal{T}}_0)}.$
\begin{align}\label{eq:oewfjewf}
    \varepsilon_{\bm{v}}^{(\mathscr{T}+\widecheck{\mathcal{T}}_0)}&\leq O(1)\big(\varepsilon_{\bm{v}}^{(\mathscr{T})}+\eta\zeta \widecheck{\mathcal{T}}_0\big). 
\end{align}
%and using the bound on  $\varepsilon_{\bm{v}}^{(\mathscr{T})}$ (\autoref{cor:epsv_init})

\paragraph{Second phase: $t\in(\mathscr{T}+\widecheck{\mathcal{T}}_0,T]$.} \autoref{lem:timeT2} shows that  $\widehat{\alpha}^{(t)}$ gets updated until $t= \widecheck{\mathcal{T}}_2<T.$ Therefore, we have $\varepsilon_{\bm{v}}^{(T)}=\varepsilon_{\bm{v}}^{(\mathscr{T}+\widecheck{\mathcal{T}}_2)}.$ Unraveling \autoref{lem:epsv} for $t=\mathscr{T}+\widecheck{\mathcal{T}}_0,\dots,\mathscr{T}+\widecheck{\mathcal{T}}_2$ and using $\widehat{\alpha}^{(t)}\leq \frac{\mathrm{polylog}(d)}{C^2\lambda_0}$ for $t\leq T$ leads to: 
\begin{align}
   \hspace{-.3cm} \varepsilon_{\bm{v}}^{(\mathscr{T}+\widecheck{\mathcal{T}}_2)}&\leq \biggr(1+\frac{\eta\mathrm{polylog}(d)}{\lambda_0^{p-1}}\frac{\mathrm{poly}(D)}{\sqrt{d}}\biggr)^{\widecheck{\mathcal{T}}_2-\widecheck{\mathcal{T}}_0} \varepsilon_{\bm{v}}^{(\mathscr{T}+\widecheck{\mathcal{T}}_0)}+ \eta\zeta\frac{\biggr(1+\frac{\eta\mathrm{polylog}(d)}{\lambda_0^{p-1}}\frac{\mathrm{poly}(D)}{\sqrt{d}}\biggr)^{\widecheck{\mathcal{T}}_2-\widecheck{\mathcal{T}}_0}  -1}{\frac{\eta\mathrm{polylog}(d)}{\lambda_0^{p-1}}\frac{\mathrm{poly}(D)}{\sqrt{d}}}.\label{eq:ewdoedw}
\end{align}
Since $\widecheck{\mathcal{T}}_2-\widecheck{\mathcal{T}}_0\leq \frac{\lambda_0^p}{\eta  e^{\mathrm{polyloglog}(d)}}$, we have $\Big(1+\frac{\eta\mathrm{polylog}(d)}{\lambda_0^{p-1}}\frac{\mathrm{poly}(D)}{\sqrt{d}}\Big)^{\widecheck{\mathcal{T}}_2-\widecheck{\mathcal{T}}_0}\leq O(1).$ %1+\frac{\lambda_0 \mathrm{polylog}(d)}{ e^{\mathrm{polyloglog}(d)} }\frac{\mathrm{poly}(D)}{\sqrt{d}}\leq 
%Therefore, using $(\widehat{\alpha}^{(\mathscr{T})})^{p-1}\geq \nu$ and and plugging the value of $\mathcal{T}_0$ and  \eqref{eq:oewfjewf} in \eqref{eq:ewdoedw} yields:
Simplifying \eqref{eq:ewdoedw} yields:
\begin{align}
    \varepsilon_{\bm{v}}^{(\mathscr{T}+\widecheck{\mathcal{T}}_2)}
    %&\leq O \biggr(\frac{\mathrm{poly}(D)}{\sqrt{N} (\alpha^{(0)})^{p-1}e^{\mathrm{polyloglog}(d)}} + \sqrt{\frac{d}{N}}\lambda_0^{p-1} \biggr)\nonumber\\
    %&\leq O\biggr(\frac{1}{\textcolor{red}{\bm{(\alpha^{(0)})^{p-1}}}\sqrt{N}}\biggr) 
    &\leq O\biggr(\frac{\mathrm{poly}(d)}{\sqrt{N}}\biggr).\label{eq:fiewfi}
\end{align}
\eqref{eq:fiewfi} implies that we need $N= \mathrm{poly}(d)$ samples to have $ \varepsilon_{\bm{v}}^{(T)}\leq 1/\mathrm{poly}(d).$ 
% \eqref{eq:fiewfi} suggests that we need $N\geq \frac{\mathrm{poly}(d)}{\delta^2}$ samples to have $ \varepsilon_{\bm{v}}^{(T)}\leq\delta.$ Setting $\delta=1/\mathrm{poly}(d)$ yields the aimed result.
\end{proof}

\subsubsection{Auxiliary lemmas}

In this section, we prove the Lipschitzness of the function appearing in the proof of \autoref{lem:epsv}.

\begin{lemma}\label{lem:lippsi}
Let $\psi\colon\mathbb{R}^d\rightarrow\mathbb{R}$ defined as $\psi(\bm{x}):=\mathfrak{S}(-\sum_{m=1}^D x_m^p)$ and $p\geq 3$ be an odd integer. Assume that $ x_m^p\geq -1/D$. Then, $\psi$ is $p$-Lipschitz i.e.\ for all $\bm{x},\bm{y}\in\mathbb{R}^d,$ $|\psi(\bm{x})-\psi(\bm{y})|\leq p\|\bm{x}-\bm{y}\|_1.$
\end{lemma}
\begin{proof}[Proof of \autoref{lem:lippsi}] Let $l\in[D].$ The derivative of $\psi$ with respect to a variable $x_l$ is:
 \begin{align}\label{eq:ewfewfee}
     \left|\frac{\partial \psi}{\partial x_l}\right|&=p \frac{x_l^{p-1}\exp(\sum_{m=1}^D x_m^p)}{(1+\exp(\sum_{m=1}^D x_m^p))^2} \leq  \frac{px_l^{p-1} }{1+\exp(\sum_{m=1}^D x_m^p) } .
 \end{align}
\eqref{eq:ewfewfee} implies a bound on $\|\nabla\psi(\bm{x})\|_1$. Indeed, since $\sum_{m=1}^D x_m^p\geq -1$, we have:
\begin{align}\label{eq:ewfe}
    \left\|\nabla \psi(\bm{x})\right\|_1&\leq p\frac{\sum_{m=1}^D (yx_m)^{p-1} }{1+\exp(\sum_{m=1}^D (yx_m)^p) }\leq p.
\end{align}
\eqref{eq:ewfe} shows that $\psi$ is $p$-Lipschitz.
\end{proof}

\subsubsection{\texorpdfstring{$\Delta_{\bm{A}}^{(t)}$ stays small during the learning process}{DELTAAT stays small during the learning process}}\label{sec:deltaA}

%\subsection{$\Delta_{\bm{A}}^{(t)}$ stays small during the process}\label{sec:deltaA}

Here, we bound the gap in attention coefficients between the realistic and semi-idealized cases.

\begin{lemma}\label{lem:q_real}
    Assume that we run GD on the empirical risk \eqref{eq:empirical} for $T$ iterations with parameters set as in \autoref{param} and the number of samples is $N=\mathrm{poly}(d).$ Then, the attention matrix in the realistic case is very close to the semi-idealized one i.e.\ for $t\in[\mathscr{T},T]$, 
    \begin{align*}
        \Delta_{\bm{A}}^{(t)}:=\max_{i\neq j}|\widehat{A}_{i,j}^{(t)}-A_{i,j}^{(t)}|\leq \frac{1}{\mathrm{poly}(d)}.
     \end{align*}
\end{lemma}

%\begin{lemma}\label{lem:epsQ_real}
%   During the training process, the error $\varepsilon_Q^{(t)}$ updates as: 
%   \begin{align*}
%       \varepsilon_Q^{(t+1)}&\leq \big(1 + \eta R^{(t)}\big)\varepsilon_Q^{(t)} +\eta(\mathrm{poly}(D)\varepsilon_v^{(t)}+\zeta),
%   \end{align*}
%where $R^{(t)}=\frac{D\alpha^{(t)}}{L}\mathbb{E}[\mathfrak{S}(-yF(X))]O\big(D\alpha^{(t)}(\Lambda^{(t)}+(C-1)\Gamma^{(t)})\big)^{p-1}\Gamma^{(t)}$ and $\zeta=\frac{\mathrm{poly}(D)}{\sqrt{N}}$. 
%\end{lemma}
We now detail the steps to prove \autoref{lem:q_real}. We first provide the recursion that $\Delta_{\bm{A}}^{(t)}$ satisfies. 
\begin{restatable}{lemma}{lemepsq}\label{lem:epsq}
    Assume that we run GD on the empirical risk \eqref{eq:empirical} for $T$ iterations with parameters set as in \autoref{param}. Then, the discrepancy $\Delta_{\bm{A}}$ satisfies for $t\in(\mathscr{T},T]$,
    \begin{align*}
        \Delta_{\bm{A}}^{(t+1)}&\leq 
        \big(1+\eta R^{(t)}\big)\Delta_{\bm{A}}^{(t)}+\eta\mathrm{poly}(D)\sigma\varepsilon_{\bm{v}}^{(t)}+\eta\zeta,
    \end{align*}
    where $\Delta_{\bm{A}}^{(\mathscr{T})}\leq |\widehat{A}_{i,j}^{(\mathscr{T})}|$, $R^{(t)}= O(C)  (G^{(t)})^{p-1}\widecheck{\Gamma}^{(t)}(\widehat{\alpha}^{(t)})^p$ and $\zeta=\frac{\mathrm{poly}(D)}{\sqrt{N}}.$ 
    %Consequently, as long as $N=\mathrm{poly}(d),$ we have $\Delta_{\bm{A}}^{(t)}\leq1/\mathrm{poly}(d)$ for $t\leq T.$
\end{restatable}

\begin{proof}[Proof of \autoref{lem:epsq}] In this proof, we maintain the hypothesis that $\Delta_{\bm{A}}^{(t)}$ is small. We will eventually prove this statement in \autoref{cor:epsq}.  Let $a,b\in [D]$ such that $a\neq b$. Using GD, $\widehat{A}_{a,b}^{(t+1)}-\widecheck{A}_{a,b}^{(t+1-\mathscr{T})}$ satisfies:

\begin{align}
    &\big|\widehat{A}_{a,b}^{(t+1)}-\widecheck{A}_{a,b}^{(t+1-\mathscr{T})}\big|\leq \big|\widehat{A}_{a,b}^{(t)}-\widecheck{A}_{a,b}^{(t-\mathscr{T})}\big|\nonumber\\
    &\hspace{-.8cm}+\eta\biggr|\frac{D}{N}\sum_{i=1}^N y[i]\mathfrak{S}\big(-y[i]F_{\widehat{\bm{v}},\widehat{\bm{A}}}(\bm{X}[i])\big) \biggr(D\sum_{c=1}^D \widehat{S}_{a,c}^{(t)}\langle \widehat{\bm{v}}^{(t)},\bm{X}_c[i]\rangle\biggr)^{p-1}\hspace{-.5cm}\widehat{S}_{a,b}^{(t)}\sum_{r\neq b}\widehat{S}_{a,r}^{(t)}\langle \widehat{\bm{v}}^{(t)}, \bm{X}_b[i]-\bm{X}_r[i]\rangle\label{eq:wekenew}\\
 -&\mathbb{E}\biggr[ Dy\mathfrak{S}(-yF_{\widehat{\bm{v}},\widehat{\bm{A}}}(\bm{X})) \biggr(D\sum_{c=1}^D \widehat{S}_{a,c}^{(t)}\langle \widehat{\bm{v}}^{(t)},\bm{X}_c\rangle\biggr)^{p-1}\hspace{-.5cm}\widehat{S}_{a,b}^{(t)}\sum_{r\neq b}\widehat{S}_{a,r}^{(t)}\langle \widehat{\bm{v}}^{(t)}, \bm{X}_b-\bm{X}_r\rangle
\biggr]\biggr|\label{eq:oewjwdw}\\
&\nonumber\\
&\hspace{-.8cm}+\eta\biggr|\mathbb{E}\biggr[ Dy\mathfrak{S}(-yF_{\widehat{\bm{v}},\widehat{\bm{A}}}(\bm{X})) \biggr(D\sum_{c=1}^D \widehat{S}_{a,c}^{(t)}\langle \widehat{\bm{v}}^{(t)},\bm{X}_c\rangle\biggr)^{p-1}\hspace{-.3cm}\widehat{S}_{a,b}^{(t)}\sum_{r\neq b}\widehat{S}_{a,r}^{(t)}\langle \widehat{\bm{v}}^{(t)}, \bm{X}_b-\bm{X}_r\rangle
\biggr]\label{eq:odjw}\\
-&\mathbb{E}\biggr[ Dy\mathfrak{S}(-yF_{\widecheck{\bm{v}},\widehat{\bm{A}}}(\bm{X})) \biggr(D\sum_{c=1}^D \widehat{S}_{a,c}^{(t)}\langle \widecheck{\bm{v}}^{(t-\mathscr{T})},\bm{X}_c\rangle\biggr)^{p-1}\hspace{-.3cm}\widehat{S}_{a,b}^{(t)}\sum_{r\neq b}\widehat{S}_{a,r}^{(t)}\langle \widecheck{\bm{v}}^{(t-\mathscr{T})}, \bm{X}_b-\bm{X}_r\rangle
\biggr]\biggr|\label{eq:ofjnewend}\\
&\nonumber\\
%
%
%
%+&\eta\biggr|\mathbb{E}\biggr[ Dy\mathfrak{S}(-yF_{\widehat{\bm{A}}}(\bm{X})) \biggr(D\sum_{c=1}^D \widehat{\bm{S}}_{a,c}^{(t)}\langle \bm{v}^{(t)},\bm{X}_c\rangle\biggr)^{p-1}\hspace{-.3cm}\widehat{\bm{S}}_{a,b}^{(t)}\sum_{r\neq b}\widehat{\bm{S}}_{a,r}^{(t)}\langle \bm{v}^{(t)}, \bm{X}_b-\bm{X}_r\rangle
%\biggr]\label{eq:efrpefr}\\
%-&\mathbb{E}\biggr[ Dy\mathfrak{S}(-yF_{\bm{Q}}(\bm{X})) \biggr(D\sum_{c=1}^D \widehat{\bm{S}}_{a,c}^{(t)}\langle \bm{v}^{(t)},\bm{X}_c\rangle\biggr)^{p-1}\hspace{-.3cm}\widehat{\bm{S}}_{a,b}^{(t)}\sum_{r\neq b}\widehat{\bm{S}}_{a,r}^{(t)}\langle \bm{v}^{(t)}, \bm{X}_b-\bm{X}_r\rangle
%\biggr]\label{eq:jwejwfej}\\
%&\nonumber\\
%
%
%
%
&\hspace{-.8cm}+D\eta\biggr|\mathbb{E}\biggr[ y\mathfrak{S}(-yF_{\widecheck{\bm{v}},\widehat{\bm{A}}}(\bm{X})) \biggr(D\sum_{c=1}^D \widehat{S}_{a,c}^{(t)}\langle \widecheck{\bm{v}}^{(t-\mathscr{T})},\bm{X}_c\rangle\biggr)^{p-1}\hspace{-.5cm}\widehat{S}_{a,b}^{(t)}\sum_{r\neq b}\widehat{S}_{a,r}^{(t)}\langle \widecheck{\bm{v}}^{(t)}, \bm{X}_b-\bm{X}_r\rangle
\biggr]\label{eq:ofnewow}\\
&\hspace{-.8cm}-\mathbb{E}\biggr[ y\mathfrak{S}(-yF_{\widecheck{\bm{v}},\widecheck{\bm{A}}}(\bm{X})) \biggr(D\sum_{c=1}^D \widecheck{S}_{a,c}^{(t-\mathscr{T})}\langle \widecheck{\bm{v}}^{(t-\mathscr{T})},\bm{X}_c\rangle\biggr)^{p-1}\hspace{-.5cm}\widecheck{S}_{a,b}^{(t-\mathscr{T})}\sum_{r\neq b}\widecheck{S}_{a,r}^{(t-\mathscr{T})}\langle \widecheck{\bm{v}}^{(t-\mathscr{T})}, \bm{X}_b-\bm{X}_r\rangle\biggr]\biggr|.\label{eq:fijoffw}
\end{align}

 We now bound the three summands above.  

\paragraph{Summand 1: $ |\eqref{eq:wekenew}-\eqref{eq:oewjwdw}|$.}  We apply the Hoeffding inequality. With high probability, we have: 
$|\eqref{eq:wekenew}-\eqref{eq:oewjwdw}|\leq \eta\frac{\mathrm{poly}(D)}{\sqrt{N}}$.
%$|\eqref{eq:wekenew}-\eqref{eq:oewjwdw}|\leq  \eta\sqrt{ 8N\log(d) \max_{i\in[N]}\;A_i^2}$ %$|\eqref{eq:wekenew}-\eqref{eq:oewjwdw}|\leq  \eta\sqrt{ 8N\log(d) \max_{i\in[N]}\;A_i^2},$
%where\\
%$\Big|\frac{D}{N}\mathfrak{S}(-y[i]F(\bm{X}^{(i)})) \Big(D\sum_{c=1}^D \widehat{S}_{a,c}^{(t)}\langle \widehat{\bm{v}}^{(t)},\bm{X}_c^{(i)}\rangle\Big)^{p-1}\hspace{-.cm}\widehat{S}_{a,b}^{(t)}\sum_{r\neq b}\widehat{S}_{a,r}^{(t)}\langle \widehat{\bm{v}}^{(t)}, \bm{X}_b^{(i)}-\bm{X}_r^{(i)}\rangle\Big|\leq A_i.$ We have  $A_i^2\leq \frac{\mathrm{poly}(D)}{N^2}$ for all $i\in[N].$

%This leads to $ |\eqref{eq:wekenew}-\eqref{eq:oewjwdw}|\leq \frac{\mathrm{poly}(D)}{\sqrt{N}}.$

\paragraph{Summand 2: $|\eqref{eq:odjw}-\eqref{eq:ofjnewend}|$.} 
%%%previously in red simplification
%\textcolor{red}{Since $\varepsilon_{\bm{v}}^{(t)}$ is small,  (\autoref{cor:epsv}), we have $\Big(D\sum_{r=1}^D \widehat{S}_{m,r}^{(t)}\langle \widehat{\bm{v}}^{(t)},\bm{X}_r\rangle\Big)^{p-1}=$
%$\Big(D\sum_{r=1}^D \widehat{S}_{m,r}^{(t)}\langle \widecheck{\bm{v}}^{(t)},\bm{X}_r\rangle\Big)^{p-1}$ and $\mathfrak{S}(-yF_{\widehat{\bm{v}}}(\bm{X}))=\mathfrak{S}(-yF_{\widecheck{\bm{v}}}(\bm{X}))$.} Therefore, we have: 
Since $\varepsilon_{\bm{v}}^{(t)}$ is small  (\autoref{cor:epsv}), we can show that:
\begin{equation}
\begin{aligned}\label{eq:ferpkre}
    \hspace{-.3cm}|\eqref{eq:odjw}-\eqref{eq:ofjnewend}|
     \leq\eta D\mathbb{E}\biggr[    \Big(D\sum_{c=1}^D \widehat{S}_{a,c}^{(t)}\langle \widecheck{\bm{v}}^{(t)},\bm{X}_c\rangle\Big)^{p-1}\hspace{-.1cm}\widehat{S}_{a,b}^{(t)}\cdot\sum_{r\neq b}\widehat{S}_{a,r}^{(t)}|\langle \bm{u}^{(t)}, \bm{X}_b-\bm{X}_r\rangle|\biggr]\varepsilon_{\bm{v}}^{(t)}. 
\end{aligned} 
\end{equation}
With high probability, we have $\langle \bm{u}^{(t)}, \bm{\xi}_b\rangle\leq \sigma\sqrt{\log(d)}$. Using this fact along with\\
$ D\Big|\Big(D\sum_{c=1}^D \widehat{S}_{a,c}^{(t)}\langle \widecheck{\bm{v}}^{(t)},\bm{X}_c\rangle\Big)^{p-1}\hspace{-.1cm}\widehat{S}_{a,b}^{(t)}\cdot\sum_{r\neq b}\widehat{S}_{a,r}^{(t)}\Big|\leq \mathrm{poly}(D) $, we further bound \eqref{eq:ferpkre} as:
\begin{align}
  |\eqref{eq:odjw}-\eqref{eq:ofjnewend}|
    &\leq \eta\mathrm{poly}(D) \varepsilon_{\bm{v}}^{(t)}\sigma\sqrt{\log(d)}.
\end{align}

\vfill

%\begin{align}
%&|\eqref{eq:odjw}-\eqref{eq:ofjnewend}|\\
%   \leq& \eta\varepsilon_{\bm{v}}^{(t)}\mathbb{E}\biggr[ D\mathfrak{S}(-yF(\bm{X})) \Big(D\sum_{c=1}^m \widehat{\bm{S}}_{a,c}^{(t)}\langle \bm{v}^{(t)},\bm{X}_c\rangle\Big)^{p-1}\hspace{-.1cm}\widehat{\bm{S}}_{a,b}^{(t)}\cdot \max_{r\in[D]}\widehat{\bm{S}}_{a,r}^{(t)}\cdot \Big|\langle \bm{u}^{(t)}, \sum_{r\neq b}\bm{\xi}_r\rangle \Big|\biggr]\nonumber\\
%   \leq& O\Big(\eta\varepsilon_{\bm{v}}^{(t)}\sigma\sqrt{D\log(d)}\Big)\mathbb{E}\biggr[ D\mathfrak{S}(-yF(\bm{X})) \Big(D\sum_{c=1}^m \widehat{\bm{S}}_{a,c}^{(t)}\langle \widehat{\bm{v}}^{(t)},\bm{X}_c\rangle\Big)^{p-1}\hspace{-.1cm}\widehat{\bm{S}}_{a,b}^{(t)}\cdot \max_{r\in[D]}\widehat{\bm{S}}_{a,r}^{(t)}\biggr],\label{eq:fewojewof}
%\end{align}
%where we used in \eqref{eq:fewojewof}  
%$\Big|\langle \bm{u}^{(t)}, \sum_{r\neq b}\bm{\xi}_r\rangle\Big| \leq O(\sigma\sqrt{D\log(d)})$ with high probability.
%Using \autoref{indh:lambgamxi}, \autoref{lem:event3} and $\varepsilon_{\bm{v}}^{(t)}$ small, we have
%$\mathbb{E}\Big[ D\mathfrak{S}(-yF(\bm{X})) \Big(D\sum_{c=1}^m \widehat{\bm{S}}_{a,c}^{(t)}\langle \widehat{\bm{v}}^{(t)},\bm{X}_c\rangle\Big)^{p-1}\hspace{-.1cm}\widehat{\bm{S}}_{a,b}^{(t)}\cdot \max_{r\in[D]}\widehat{\bm{S}}_{a,r}^{(t)}\Big]$ $\leq\mathrm{poly}(D).$ Thus, $|\eqref{eq:odjw}-\eqref{eq:ofjnewend}|\leq \eta\mathrm{poly}(D)\varepsilon_{\bm{v}}^{(t)}.$

\paragraph{Summand 3: $|\eqref{eq:ofnewow}-\eqref{eq:fijoffw}|$.}  We have the following decomposition.
\begin{align}
    &|\eqref{eq:ofnewow}-\eqref{eq:fijoffw}|\nonumber\\
    &\hspace{-.4cm}\leq\eta\biggr|\mathbb{E}\biggr[ Dy\mathfrak{S}(-yF_{\widecheck{\bm{v}},\widehat{\bm{A}}}(\bm{X})) \biggr(D\sum_{c=1}^D \widehat{S}_{a,c}^{(t)}\langle \widecheck{\bm{v}}^{(t)},\bm{X}_c\rangle\biggr)^{p-1}\hspace{-.3cm}\widehat{S}_{a,b}^{(t)}\sum_{r\neq b}\widehat{S}_{a,r}^{(t)}\langle \widecheck{\bm{v}}^{(t)}, \bm{X}_b-\bm{X}_r\rangle\biggr]\label{eq:efrpefr}\\
&\hspace{-.2cm}-\mathbb{E}\biggr[ Dy\mathfrak{S}(-yF_{\widecheck{\bm{v}},\widecheck{\bm{A}}}(\bm{X})) \biggr(D\sum_{c=1}^D \widehat{S}_{a,c}^{(t)}\langle \widecheck{\bm{v}}^{(t)},\bm{X}_c\rangle\biggr)^{p-1}\hspace{-.3cm}\widehat{S}_{a,b}^{(t)}\sum_{r\neq b}\widehat{S}_{a,r}^{(t)}\langle \widecheck{\bm{v}}^{(t)}, \bm{X}_b-\bm{X}_r\rangle
\biggr]\label{eq:jwejwfej}\\
&\nonumber\\
    &\hspace{-.4cm}+\eta\biggr|\mathbb{E}\biggr[ Dy\mathfrak{S}(-yF_{\widecheck{\bm{v}},\widecheck{\bm{A}}}(\bm{X})) \biggr(D\sum_{c=1}^D \widehat{S}_{a,c}^{(t)}\langle \widecheck{\bm{v}}^{(t)},\bm{X}_c\rangle\biggr)^{p-1}\hspace{-.5cm}\widehat{S}_{a,b}^{(t)}\sum_{r\neq b}\widehat{S}_{a,r}^{(t)}\langle \widecheck{\bm{v}}^{(t)}, \bm{X}_b-\bm{X}_r\rangle
\biggr]\label{eq:ojdeojedoj}\\
&\hspace{-.2cm}-\mathbb{E}\biggr[ Dy\mathfrak{S}(-yF_{\widecheck{\bm{v}},\widecheck{\bm{A}}}(\bm{X})) \biggr(D\sum_{c=1}^D \widecheck{S}_{a,c}^{(t-\mathscr{T})}\langle \widecheck{\bm{v}}^{(t)},\bm{X}_c\rangle\biggr)^{p-1}\hspace{-.5cm}\widehat{S}_{a,b}^{(t)}\sum_{r\neq b}\widehat{S}_{a,r}^{(t)}\langle \widecheck{\bm{v}}^{(t)}, \bm{X}_b-\bm{X}_r\rangle\biggr]\biggr|\label{eq:weejiwjwei}\\
&\nonumber\\
    &\hspace{-.4cm}+\eta\biggr|\mathbb{E}\biggr[ Dy\mathfrak{S}(-yF_{\widecheck{\bm{v}},\widecheck{\bm{A}}}(\bm{X})) \biggr(D\sum_{c=1}^D \widecheck{S}_{a,c}^{(t-\mathscr{T})}\langle \widecheck{\bm{v}}^{(t)},\bm{X}_c\rangle\biggr)^{p-1}\hspace{-.5cm}\widehat{S}_{a,b}^{(t)}\sum_{r\neq b}\widehat{S}_{a,r}^{(t)}\langle \widecheck{\bm{v}}^{(t)}, \bm{X}_b-\bm{X}_r\rangle
\biggr]\label{eq:kmnkknk}\\
&\hspace{-.2cm}-\mathbb{E}\biggr[ Dy\mathfrak{S}(-yF_{\widecheck{\bm{v}},\widecheck{\bm{A}}}(\bm{X})) \biggr(D\sum_{c=1}^D S_{a,c}^{(t-\mathscr{T})}\langle \widecheck{\bm{v}}^{(t)},\bm{X}_c\rangle\biggr)^{p-1}\hspace{-.5cm}\widehat{S}_{a,b}^{(t)}\sum_{r\neq b}\widecheck{S}_{a,r}^{(t-\mathscr{T})}\langle \widecheck{\bm{v}}^{(t)}, \bm{X}_b-\bm{X}_r\rangle\biggr]\biggr|\label{eq:evfejr}\\\
&\nonumber\\
&\hspace{-.4cm}+\eta\biggr|\mathbb{E}\biggr[ Dy\mathfrak{S}(-yF_{\widecheck{\bm{v}},\widecheck{\bm{A}}}(\bm{X})) \biggr(D\sum_{c=1}^D \widecheck{S}_{a,c}^{(t-\mathscr{T})}\langle \widecheck{\bm{v}}^{(t)},\bm{X}_c\rangle\biggr)^{p-1}\hspace{-.5cm}\widehat{S}_{a,b}^{(t)}\sum_{r\neq b}\widecheck{S}_{a,r}^{(t-\mathscr{T})}\langle \widecheck{\bm{v}}^{(t)}, \bm{X}_b-\bm{X}_r\rangle\biggr]\biggr|\label{eq:rfjrf}\\
&\hspace{-.2cm}-\mathbb{E}\biggr[ Dy\mathfrak{S}(-yF_{\widecheck{\bm{v}},\widecheck{\bm{A}}}(\bm{X})) \biggr(D\sum_{c=1}^D \widecheck{S}_{a,c}^{(t-\mathscr{T})}\langle \widecheck{\bm{v}}^{(t)},\bm{X}_c\rangle\biggr)^{p-1}\hspace{-.5cm}\widecheck{S}_{a,b}^{(t-\mathscr{T})}\sum_{r\neq b}\widecheck{S}_{a,r}^{(t-\mathscr{T})}\langle \widecheck{\bm{v}}^{(t)}, \bm{X}_b-\bm{X}_r\rangle\biggr]\biggr|.\label{eq:ofeojeo}
\end{align}

We need to distinguish two sub-cases: $a,b\in\mathcal{S}_{\ell}$ and $a\in\mathcal{S}_{\ell}$, $b\in\mathcal{S}_m$ with $\ell\neq m.$

\paragraph{\hspace*{.4cm}Subcase 1: $a,b\in\mathcal{S}_{\ell}$.} The proof of \autoref{lem:gd_gamma} highlights that when $a,b\in\mathcal{S}_{\ell}$ the event with largest gradient is event a: "$\ell=\ell(\pi(\bm{X}))$" which happens with probability $1/L.$ Therefore, to simplify the calculations, we will only take into account this event. We first bound $|\eqref{eq:efrpefr}-\eqref{eq:jwejwfej}|$. We successively apply \autoref{lem:lippsi} (Lipschitzness of sigmoid) and \autoref{lem:softmaxlip} (Lipschitzness of softmax) and get: 
\begin{align}
    &|\eqref{eq:efrpefr}-\eqref{eq:jwejwfej}|\nonumber\\
    \leq& \frac{\Theta(\eta D)}{L}\mathbb{E}\biggr[  \biggr(D\sum_{c=1}^D \widehat{S}_{a,c}^{(t)}\langle \widecheck{\bm{v}}^{(t)},\bm{X}_c\rangle\biggr)^{p-1}  \widehat{S}_{a,b}^{(t)}\sum_{c=1}^D\sum_{c'=1}^D  |\widecheck{S}_{c,c'}^{(t-\mathscr{T})}-\widehat{S}_{c,c'}^{(t)}|\cdot |\langle \widecheck{\bm{v}}^{(t)},\bm{X}_{c'}\rangle| \sum_{r\neq b}\widehat{S}_{a,r}^{(t)}|\langle \widecheck{\bm{v}}^{(t)}, \bm{X}_b-\bm{X}_r\rangle|\;\biggr|\; \mathbf{a}\biggr]\nonumber\\
    \leq& \frac{\Theta(\eta D)}{L}\mathbb{E}\biggr[  \biggr(D\sum_{c=1}^D \widehat{S}_{a,c}^{(t)}\langle \widecheck{\bm{v}}^{(t)},\bm{X}_c\rangle\biggr)^{p-1}  \widehat{S}_{a,b}^{(t)}\sum_{c=1}^D\sum_{c'=1}^D  \widecheck{S}_{c,c'}^{(t-\mathscr{T})}|e^{2\Delta_{\bm{A}}^{(t)}}-1| |\langle \widecheck{\bm{v}}^{(t)},\bm{X}_{c'}\rangle|\sum_{r\neq b}\widehat{S}_{a,r}^{(t)}|\langle \widecheck{\bm{v}}^{(t)}, \bm{X}_b-\bm{X}_r\rangle| \;\biggr|\; \mathbf{a}\biggr]\nonumber\\
    \leq& \frac{\Theta(\eta D)}{L}\mathbb{E}\biggr[  \biggr(D\sum_{c=1}^D \widehat{S}_{a,c}^{(t)}\langle \widecheck{\bm{v}}^{(t)},\bm{X}_c\rangle\biggr)^{p-1}  \widehat{S}_{a,b}^{(t)}\sum_{c=1}^D\sum_{c'=1}^D \widecheck{S}_{c,c'}^{(t-\mathscr{T})} |\langle \widecheck{\bm{v}}^{(t)},\bm{X}_{c'}\rangle| \sum_{r\neq b}\widehat{S}_{a,r}^{(t)}|\langle \widecheck{\bm{v}}^{(t)}, \bm{X}_b-\bm{X}_r\rangle|\;\biggr|\; \mathbf{a}\biggr]\Delta_{\bm{A}}^{(t)}\label{eq:fpkekf}\\\
    \leq&  \frac{\Theta(\eta \widehat{\alpha}^{(t)}\lambda_0 D^2)}{L}\mathbb{E}\biggr[  \biggr(D\sum_{c=1}^D \widehat{S}_{a,c}^{(t)}\langle \widecheck{\bm{v}}^{(t)},\bm{X}_c\rangle\biggr)^{p-1}  \widehat{S}_{a,b}^{(t)}\sum_{r\neq b}\widehat{S}_{a,r}^{(t)}|\langle \widecheck{\bm{v}}^{(t)}, \bm{X}_b-\bm{X}_r\rangle|  \;\biggr|\; \mathbf{a}\biggr]\Delta_{\bm{A}}^{(t)}.\label{eq:few[ew}
    %\\
    %\leq& \frac{O(\eta D\lambda_0^2\widehat{\alpha}^{(t)})}{L}\mathbb{E}\biggr[  \biggr(D\sum_{c=1}^D \widehat{S}_{a,c}^{(t)}\langle \widecheck{\bm{v}}^{(t)},\bm{X}_c\rangle\biggr)^{p-1}    \;\biggr|\; \mathbf{a}\biggr]\Delta_{\bm{A}}^{(t)},\label{eq:few[ew}
\end{align}
where we used  $e^{\Delta_{\bm{A}}^{(t)}}-1\leq 2 \Delta_{\bm{A}}^{(t)}$ in \eqref{eq:fpkekf} and \autoref{indh:lambgamxi} in \eqref{eq:few[ew}. Using Lipschitz inequalities, we can further expand \eqref{eq:few[ew} as a function of the coefficients from $\widecheck{S}^{(t)}$ and $\Delta_{\bm{A}}^{(t)}$. However, $\Delta_{\bm{A}}^{(t)}$ is small and we only want terms of order 1 in $\Delta_{\bm{A}}^{(t)}$ in \eqref{eq:few[ew}. Therefore, the only term of order 1 that remains is: 
\begin{equation}
\begin{aligned}\label{eq:wefcerffer}
     &|\eqref{eq:efrpefr}-\eqref{eq:jwejwfej}|\\
     &\hspace{-.7cm}\leq \frac{\Theta(\eta \widehat{\alpha}^{(t)}\lambda_0 D^2)}{L}\mathbb{E}\biggr[  \biggr(D\sum_{c=1}^D \widecheck{S}_{a,c}^{(t-\mathscr{T})}\langle \widecheck{\bm{v}}^{(t)},\bm{X}_c\rangle\biggr)^{p-1} \hspace{-.3cm} \widecheck{S}_{a,b}^{(t-\mathscr{T})} \sum_{r\neq b}\widecheck{S}_{a,r}^{(t-\mathscr{T})}|\langle \widecheck{\bm{v}}^{(t)}, \bm{X}_b-\bm{X}_r\rangle| \;\biggr|\; \mathbf{a}\biggr]\Delta_{\bm{A}}^{(t)}.
\end{aligned}
\end{equation}
Bounding the expectation in \eqref{eq:wefcerffer} as in the proof of \autoref{lem:gd_gamma} yields $|\eqref{eq:efrpefr}-\eqref{eq:jwejwfej}|\leq\eta R^{(t)}\Delta_{\bm{A}}^{(t)}$. We now bound $|\eqref{eq:ojdeojedoj}-\eqref{eq:weejiwjwei}|$. We therefore apply (\autoref{lem:lipepsQ}) and get:
\begin{align}
    &|\eqref{eq:ojdeojedoj}-\eqref{eq:weejiwjwei}|\\
    \leq&  \frac{\Theta(\eta D)}{L}\mathbb{E}\biggr[ \biggr(D\sum_{c=1}^D \widecheck{S}_{a,c}^{(t-\mathscr{T})}\langle \widecheck{\bm{v}}^{(t)},\bm{X}_c\rangle\biggr)^{p-1}   |e^{2(p-1)\Delta_{\bm{A}}^{(t)}}-1|\widehat{S}_{a,b}^{(t)}\sum_{r\neq b}\widehat{S}_{a,r}^{(t)}\langle \widecheck{\bm{v}}^{(t)}, \bm{X}_b-\bm{X}_r\rangle\;\biggr|\; \mathbf{a}\biggr]\\
    \leq& \frac{\Theta(\eta D)}{L}\mathbb{E}\biggr[ \biggr(D\sum_{c=1}^D \widecheck{S}_{a,c}^{(t-\mathscr{T})}\langle \widecheck{\bm{v}}^{(t)},\bm{X}_c\rangle\biggr)^{p-1}\widehat{S}_{a,b}^{(t)}\sum_{r\neq b}\widehat{S}_{a,r}^{(t)}\langle \widecheck{\bm{v}}^{(t)}, \bm{X}_b-\bm{X}_r\rangle   \;\biggr|\; \mathbf{a}\biggr]\Delta_{\bm{A}}^{(t)}.\label{eq:fojwejew} %D\mathfrak{S}(-yF(\bm{X})) 
\end{align}
where we used  $e^{2(p-1)\Delta_{\bm{A}}^{(t)}}-1\leq 4(p-1) \Delta_{\bm{A}}^{(t)}$ in \eqref{eq:fojwejew}. We can further expand \eqref{eq:fojwejew}, keep the terms of first order in $\Delta_{\bm{A}}$ and  get $|\eqref{eq:kmnkknk}-\eqref{eq:evfejr}|\leq \eta R^{(t)}\Delta_{\bm{A}}^{(t)}.$

We now bound $|\eqref{eq:kmnkknk}-\eqref{eq:evfejr}|$. Using the Lipschitz property of the softmax (\autoref{lem:softmaxlip}), we have: 
\begin{align}
    &|\eqref{eq:kmnkknk}-\eqref{eq:evfejr}|\nonumber\\
    %\leq & \frac{\eta D}{L}\mathbb{E}\biggr[\mathfrak{S}(-yF(\bm{X})) \biggr(D\sum_{c=1}^m \bm{S}_{a,c}^{(t)}\langle \bm{v}^{(t)},\bm{X}_c\rangle\biggr)^{p-1}\hspace{-.5cm}\widehat{S}_{a,b}^{(t)}\sum_{r\neq b}e^{2\max_{b\in[D]}|\bm{Q}_{a,b}^{(t)}-\widehat{\bm{A}}_{a,b}^{(t)}|}\bm{S}_{a,r}^{(t)}|\langle \bm{v}^{(t)}, \bm{X}_b-\bm{X}_r\rangle|\Big|\textbf{a}\biggr]\nonumber\\
    \leq&\frac{\eta D}{L} \mathbb{E}\biggr[ \biggr(D\sum_{c=1}^D \widecheck{S}_{a,c}^{(t-\mathscr{T})}\langle  \widecheck{\bm{v}}^{(t)},\bm{X}_c\rangle\biggr)^{p-1}\hspace{-.5cm}\widehat{S}_{a,b}^{(t)}\sum_{r\neq b}|e^{2\Delta_{\bm{A}}^{(t)}}-1|\cdot \widecheck{S}_{a,r}^{(t-\mathscr{T})}\cdot|\langle  \widecheck{\bm{v}}^{(t)}, \bm{X}_b-\bm{X}_r\rangle|\Big|\textbf{a}\biggr]\nonumber\\
    \leq&\frac{\Theta(\eta D)}{L} \mathbb{E}\biggr[ \biggr(D\sum_{c=1}^D \widecheck{S}_{a,c}^{(t-\mathscr{T})}\langle  \widecheck{\bm{v}}^{(t)},\bm{X}_c\rangle\biggr)^{p-1}\hspace{-.5cm}\widehat{S}_{a,b}^{(t)}\sum_{r\neq b}\widecheck{S}_{a,r}^{(t-\mathscr{T})}|\langle  \widecheck{\bm{v}}^{(t)}, \bm{X}_b-\bm{X}_r\rangle|\Big|\textbf{a}\biggr]\Delta_{\bm{A}}^{(t)},\label{eq:ewjowf}
\end{align}
where we used  $|e^{2\Delta_{\bm{A}}^{(t)}}-1|\leq 4 \Delta_{\bm{A}}^{(t)}$ in \eqref{eq:ewjowf}. Using the same arguments as above, we obtain  $|\eqref{eq:kmnkknk}-\eqref{eq:evfejr}|\leq \eta R^{(t)}\Delta_{\bm{A}}^{(t)}.$

The bound on $|\eqref{eq:rfjrf}-\eqref{eq:ofeojeo}|$ can be derived as above. We again use the Lipschitz property of softmax (\autoref{lem:softmaxlip}) which leads to 
\begin{align}
|\eqref{eq:rfjrf}-\eqref{eq:ofeojeo}|&\leq \frac{\Theta(\eta D)}{L} \mathbb{E}\biggr[  \biggr(D\sum_{c=1}^D \widecheck{S}_{a,c}^{(t-\mathscr{T})}\langle \widecheck{\bm{v}}^{(t)},\bm{X}_c\rangle\biggr)^{p-1}\hspace{-.5cm}\widecheck{S}_{a,b}^{(t-\mathscr{T})}\sum_{r\neq b}\widecheck{S}_{a,r}^{(t-\mathscr{T})}|\langle \widecheck{\bm{v}}^{(t)}, \bm{X}_b-\bm{X}_r\rangle|\Big|\textbf{a}\biggr]\Delta_{\bm{A}}^{(t)}\nonumber\\
&\leq \eta R^{(t)}\Delta_{\bm{A}}^{(t)}.
\end{align}
%Therefore, we also have $|\eqref{eq:rfjrf}-\eqref{eq:ofeojeo}|\leq \eta R^{(t)}\Delta_{\bm{A}}^{(t)}.$

\paragraph{\hspace*{.4cm}Subcase 2: $a\in\mathcal{S}_{\ell}$ and $b\in\mathcal{S}_m$ with $\ell\neq m$.} The proof is analogous to the Subcase 1. We only take into account event a: "$\ell=\ell(\pi(\bm{X}))$ and $ \delta_{j}=0$" and \textbf{Event e}: "$ \ell,m\neq\ell(\pi(\bm{X}))$ and $\delta_{j}= 0$ and $\delta_{s}=0$ and show that $|\eqref{eq:ofnewow}-\eqref{eq:fijoffw}|\leq\eta \tilde{R}^{(t)}\Delta_{\bm{A}}^{(t)}$ where $\tilde{R}^{(t)}:=O(\widehat{\alpha}^{(t)}) \left(D\widehat{\alpha}^{(t)}(\Lambda^{(t)}+(C-1)\Gamma^{(t)})\right)^{p-1} \Xi^{(t)} \frac{\lambda_0}{D}+O(\widehat{\alpha}^{(t)}) \left(D\widehat{\alpha}^{(t)}qD\log(d)\Xi^{(t)}\right)^{p-1} \Xi^{(t)}qD\log(d)$. %\autoref{lem:gd_rho} highlights that in this subcase, the most likely events are \textbf{Event a}: "$\ell=\ell(X)$ and $ \delta_{m,j}(X)=0$" and \textbf{Event e}: "$ \ell,m\neq\ell(X)$ and $\delta_{m,j}(X)= 0$ and $\delta_{\ell,s}(X)=0$ for all $s$" which respectively happen with probability $(1-q)/L$ and $(1-2/L)(1-q)^{C+1}.$ To simplify the calculations, we only take into account these two events. We first bound $|\eqref{eq:ojdeojedoj}-\eqref{eq:weejiwjwei}|$. Using \autoref{lem:lipepsQ}, we have: 

\paragraph{\hspace*{.4cm}Putting all the pieces together.} Given the value of the  parameters, we know that $\tilde{R}^{(t)}\leq R^{(t)}$ for all $t\in[T].$ Therefore, Summand 3 is bounded as: 
\begin{align}
    |\eqref{eq:ofnewow}-\eqref{eq:fijoffw}|\leq \eta R^{(t)}\Delta_{\bm{A}}^{(t)}.
\end{align}

\paragraph{Conclusion.} Plugging the bounds on Summands 1, 2 and 3 in the original decomposition of $\big|\widehat{A}_{a,b}^{(t+1)}-\widecheck{A}_{a,b}^{(t+1-\mathscr{T})}\big|$ yields the bound on $\Delta_{\bm{A}}^{(t)}$. The second part of the lemma is obtained using \autoref{cor:epsq}.

\end{proof}

\begin{lemma}\label{cor:epsq} Let $N=\mathrm{poly}(d)$. Then, for all $t\leq T$,  $\Delta_{\bm{A}}^{(t)}\leq \frac{1}{\mathrm{poly}(d)}.$
\end{lemma}
\begin{proof}[Proof of \autoref{cor:epsq}] Let $\mathcal{E}_{\bm{v}}>0$ such that  $\varepsilon_{\bm{v}}^{(t)}\leq \mathcal{E}_{\bm{v}}$ for $t\in[T]$ -- we proved the existence of $\mathcal{E}_{\bm{v}}$ in  \autoref{cor:epsv}.   We bound $\Delta_{\bm{A}}^{(t)}$ when  $t\in[\mathscr{T},\mathscr{T}+\widecheck{\mathcal{T}}_0]$ and $t\in [\mathscr{T}+\widecheck{\mathcal{T}}_0,T].$ 

\paragraph{First phase: $t\in[\mathscr{T},\mathscr{T}+\widecheck{\mathcal{T}}_0]$.} Unraveling \autoref{lem:epsq} for $t=\mathscr{T},\dots,\mathscr{T}+\widecheck{\mathcal{T}}_0$ leads to: 
\begin{equation}
\begin{aligned}\label{eq:jfejfe}
    &\Delta_{\bm{A}}^{(\mathscr{T}+\widecheck{\mathcal{T}}_0)}\\ & \hspace{-.4cm}\leq\Big(\Delta_{\bm{A}}^{(\mathscr{T})}+\eta\widecheck{\mathcal{T}}_0\big(\mathcal{E}_{\bm{v}}\mathrm{poly}(D)\sigma+\zeta) \Big) \hspace{-.2cm}\prod_{\tau=\mathscr{T}}^{\mathscr{T}+\widecheck{\mathcal{T}}_0-1}\hspace{-.1cm}\biggr(1+\frac{\eta D(\widehat{\alpha}^{(\tau)})^p}{L} O\big(D(\widecheck{\Lambda}^{(t)}+(C-1)\widecheck{\Gamma}^{(t)})\big)^{p-1}\widecheck{\Gamma}^{(t)}\biggr).
\end{aligned}
\end{equation}
We now apply \autoref{indh:lambgamxi} to simplify \eqref{eq:jfejfe} and get: 
\begin{align}\label{eq:jfejfevfe}
    \Delta_{\bm{A}}^{(\mathscr{T}+\widecheck{\mathcal{T}}_0)} &\leq \Big(\Delta_{\bm{A}}^{(\mathscr{T})}+\eta\widecheck{\mathcal{T}}_0\big(\mathcal{E}_{\bm{v}}\mathrm{poly}(D)\sigma+\zeta) \Big) \prod_{\tau=\mathscr{T}}^{\mathscr{T}+\widecheck{\mathcal{T}}_0-1}\biggr(1+\frac{O(\eta \lambda_0^p)}{L} (\widehat{\alpha}^{(t)})^p \biggr).
\end{align}
We then apply \autoref{lem:pow_method_prod} to bound the product term in \eqref{eq:jfejfevfe}. We obtain: 
\begin{align}
     \Delta_{\bm{A}}^{(\mathscr{T}+\mathcal{T}_0)} 
     &\leq \Big(\Delta_{\bm{A}}^{(\mathscr{T})}+\eta\widecheck{\mathcal{T}}_0\big(\mathcal{E}_{\bm{v}}\mathrm{poly}(D)\sigma+\zeta) \Big) \biggr(1+\frac{\Theta(1)}{C^{2(p-2)}\lambda_0^{p-2}} \biggr)^{\frac{\eta\lambda_0^{2p}}{L}}\nonumber\\
     &\leq O\Big(\Delta_{\bm{A}}^{(\mathscr{T})}+\eta\widecheck{\mathcal{T}}_0\big(\mathcal{E}_{\bm{v}}\mathrm{poly}(D)\sigma+\zeta) \Big).\label{eq:fwjoofjw}
\end{align}
%Since $\biggr(1+\frac{\Theta(1)}{C^{2(p-2)}\lambda_0^{p-2}} \biggr)^{\frac{\eta\lambda_0^{2p}}{L}}\leq O(1)$, we finally bound $\varepsilon_Q^{(\mathcal{T}_0)}$ as: 
%\begin{align}\label{eq:oewfoifew}
%    \varepsilon_Q^{(\mathcal{T}_0)}&\leq  O(\eta)\big(\mathcal{E}_v\mathrm{poly}(D)+\zeta) \mathcal{T}_0.
%\end{align}

\paragraph{Second phase: $t\in[\mathscr{T}+\widecheck{\mathcal{T}}_0,T]$.}  Unraveling \autoref{lem:epsq} for $t=\mathscr{T}+\widecheck{\mathcal{T}}_0,\dots,\mathscr{T}+\widecheck{\mathcal{T}}_2$ and and using $\widehat{\alpha}^{(t)}\leq \frac{\mathrm{polylog}(d)}{C^2\lambda_0}$ for $t\leq T$ leads to: 
\begin{align}
    \Delta_{\bm{A}}^{(\mathscr{T}+\widecheck{\mathcal{T}}_2)}%&\leq \biggr(1+\frac{\tilde{O}(\eta )}{L}  \biggr)^{\mathcal{T}_2-\mathcal{T}_0}\varepsilon_Q^{(\mathcal{T}_0)} + \eta\big(\mathcal{E}_v\mathrm{poly}(D)+\zeta)\frac{\biggr(1+\frac{\tilde{O}(\eta )}{L}  \biggr)^{\mathcal{T}_2-\mathcal{T}_0}-1}{\frac{\tilde{\Omega}(\eta)}{L}}\nonumber\\
    &\leq \biggr(1+\frac{\eta \mathrm{polylog}(d)}{L}  \biggr)^{\widecheck{\mathcal{T}}_2-\widecheck{\mathcal{T}}_0}\Big(\Delta_{\bm{A}}^{(\mathscr{T}+\widecheck{\mathcal{T}}_0)}+L\cdot\mathrm{polylog}(d)\big(\mathcal{E}_v\mathrm{poly}(D)\sigma+\zeta\big)\Big) .\label{eq:fojerojr}
\end{align}
Using $\widecheck{\mathcal{T}}_2-\widecheck{\mathcal{T}}_0\leq \frac{\lambda_0^p}{\eta  e^{\mathrm{polyloglog}(d)}}$, we have $\Big(1+\frac{\eta \mathrm{polylog}(d)}{L}  \Big)^{\widecheck{\mathcal{T}}_2-\widecheck{\mathcal{T}}_0}\leq O(1)$. We thus bound \eqref{eq:fwjoofjw} as: 
\begin{align}
    \hspace{-.5cm}\Delta_{\bm{A}}^{(\mathscr{T}+\widecheck{\mathcal{T}}_2)}&\leq O(\Delta_{\bm{A}}^{(\mathscr{T})})+ \mathrm{polylog}(d)\biggr[\mathcal{E}_{\bm{v}}\sigma+\frac{1}{\sqrt{N}}\Biggr]\Biggr[ \frac{1}{ C(\widehat{\alpha}^{(\mathscr{T})})^{p-1}e^{\mathrm{polyloglog}(d)}}+L\biggr]\leq O\Big(\frac{\nu}{\sqrt{N}}\Big).\label{eq;jofwojwe}
\end{align}
We deduce that setting $N=\mathrm{poly}(d)$ yields $\Delta_{\bm{A}}^{(T)}=\Delta_{\bm{A}}^{(\mathscr{T}+\widecheck{\mathcal{T}}_2)}\leq1/\mathrm{poly}(d).$  
\end{proof}

\subsubsection{Auxiliary lemmas}

\begin{lemma}\label{lem:lipepsQ}
Let $\psi\colon\mathbb{R}^D\rightarrow\mathbb{R}$ defined as $\psi(\bm{x}):=(D\sum_{m=1}^D S_m \langle \bm{v},\bm{X}_m\rangle)^{p-1}$ where $S_m=(\mathrm{softmax}(x_1,\dots,x_D))_m$,   $\{\bm{X}_m\}_{m=1}^D$ and $\bm{v}$ are fixed vectors and $p\geq 3$ is an odd integer. 
Then, we have:
\begin{align*}
    |\psi(\bm{x})-\psi(\bm{y})|\leq  \psi(\bm{x})  \big|e^{2(p-1)\max_{c\in[D]}|x_c-y_c|}-1\big|. % Theta(1) I removed this
    %(p-1)\Big(D\sum_{m=1}^D A|\langle v,X_m\rangle|\Big)^{p-2}D \max_{j\in[D]}|\langle v,X_j\rangle|\cdot \|x-y\|_1.
\end{align*}
\end{lemma}
\begin{proof}[Proof of \autoref{lem:lipepsQ}] Let $S_c=(\mathrm{softmax}(x_1,\dots,x_D))_c$, $\widehat{S}_c=(\mathrm{softmax}(y_1,\dots,y_D))_c$. We have:
\begin{align}
     | \psi(\bm{x})-\psi(\bm{y}) |
     %&= \psi(\bm{x})\Bigg|\Bigg[ \frac{\sum_{c=1}^D \bm{S}_c\langle v ,\bm{X}_c\rangle}{\sum_{c=1}^D \widehat{\bm{S}}_c\langle \bm{v} ,\bm{X}_c\rangle}\Bigg]^{p-1}-1\Bigg|
     &\leq \psi(\bm{x})\Bigg| \Bigg[ \frac{\sum_{c=1}^D S_c\langle \bm{v} ,\bm{X}_c\rangle}{\sum_{c=1}^D \widehat{S}_c\langle \bm{v} ,\bm{X}_c\rangle}\Bigg]^{p-1}-1\Bigg|\\
     &\leq  
    \psi(\bm{x})\Bigg| \Bigg[ \frac{\sum_{c=1}^D e^{x_c}\langle \bm{v} ,\bm{X}_c\rangle }{\sum_{c=1}^D e^{y_c}\langle \bm{v} ,\bm{X}_c\rangle}\cdot\frac{\sum_{r=1}^D e^{y_r}}{\sum_{r=1}^D e^{x_r}}\Bigg]^{p-1}-1\Bigg|, \label{eq:foejefje} %i removed this \Theta(1)
     %=&\biggr(D\sum_{c=1}^m S_{a,c}^{(t)}\langle v^{(t)},X_c\rangle\biggr)^{p-1} |\widehat{S}_{a,c}^{(t)}-S_{a,c}^{(t)}| \cdot \Biggr( \Biggr|\Bigg[ \frac{\sum_{c=1}^m \widehat{S}_{a,c}^{(t)}\langle v^{(t)},X_c\rangle}{\sum_{c=1}^m S_{a,c}^{(t)}\langle v^{(t)},X_c\rangle}\Bigg]^{p-1}-1\Biggr|\cdot\frac{1}{|\widehat{S}_{a,c}^{(t)}-S_{a,c}^{(t)}|} \Biggr)
\end{align}
%where we used   $\sum_{m=1}^D e^{\bm{x}_m}=\sum_{m=1}^D e^{\bm{y}_m}$ in the last equality of \eqref{eq:foejefje}.
%\begin{align}
%     | \psi(x)-\psi(y) |
%     &\leq \Theta(1) \psi(x)\Bigg[ \frac{\sum_{c=1}^D e^{x_c}\langle v ,X_c\rangle}{\sum_{c=1}^D e^{y_c}\langle v ,X_c\rangle}\Bigg]^{p-1}.
%\end{align}
We finally apply the generalized mediant inequality in \eqref{eq:foejefje} and get:
\begin{align*}  
      | \psi(\bm{x})-\psi(\bm{y}) |
     &\leq \psi(\bm{x})\Big|\Big(\max_{c\in[D]} e^{|x_c-y_c|}\cdot\max_{r\in[D]}e^{|x_r-y_r|}\Big)^{p-1}-1\Big| \\
     &\leq \psi(\bm{x})  \big|e^{2(p-1)\max_{c\in[D]}|x_c-y_c|}-1\big| .
\end{align*}
\end{proof}

\begin{lemma}[Lipschitzness of Softmax]\label{lem:softmaxlip}
Let $\bm{a},\bm{b}\in\mathbb{R}^D$. For all $i\in[D]$, we have: 
\begin{align*}
   |\mathrm{softmax}(\bm{a})_i-\mathrm{softmax}(\bm{b})_i|\leq \big|e^{2 \max_{k\in[D]} |b_k-a_k|}-1\big|\cdot\mathrm{softmax}(\bm{a})_i .
\end{align*}
\end{lemma}
\begin{proof}[Proof of \autoref{lem:softmaxlip}] Let $i\in[D]$. The difference of softmax is bounded as: 
\begin{align}
    \hspace{-.5cm} |\mathrm{softmax}(\bm{a})_i-\mathrm{softmax}(\bm{b})_i|
    %=&\left|\frac{e^{a_i}}{\sum_{j=1}^D e^{a_j}}\left(1-\frac{e^{b_i}}{e^{a_i}}\frac{\sum_{j=1}^D e^{a_j}}{\sum_{j=1}^D e^{b_j}}\right)\right|\nonumber\\
    &\leq \frac{e^{a_i}}{\sum_{j=1}^D e^{a_j}}\cdot\bigg| \frac{e^{b_i}}{e^{a_i}}\frac{\sum_{j=1}^D e^{a_j}}{\sum_{j=1}^D e^{b_j}} -1 \bigg|\nonumber\\
    &\leq \frac{e^{a_i}}{\sum_{j=1}^D e^{a_j}}\cdot\Big| e^{b_i-a_i}\max_{k\in[D]} e^{a_k-b_k}-1\Big|,\label{eq:woddwq}
\end{align}
where we used the mediant inequality in the last inequality of \eqref{eq:woddwq}. Since the exponential function is non-decreasing,  we deduce:
\begin{equation}
\begin{aligned}
   |\mathrm{softmax}(\bm{a})_i-\mathrm{softmax}(\bm{b})_i|&\leq 
     \big|e^{2 \max_{k\in[D]} |b_k-a_k|}-1\big|\cdot\mathrm{softmax}(\bm{a})_i .
\end{aligned}
\end{equation}
\end{proof}

\subsubsection{Dynamics of $\widehat{\bm{v}}^{(t)}$ }\label{sec:dynv}

Lastly, since $\Delta_{\bm{A}}^{(t)}$ remains small and $\widehat{\bm{v}}^{(t)}$ mainly lies in $\mathrm{span}(\bm{w}^*)$, we  show that $\widehat{\alpha}^{(t)}$ satisfies the same updates as the ideal $\alpha^{(t)}$ (up to some constant factors).
\begin{lemma}\label{lem:}
Assume that we run GD on the empirical risk \eqref{eq:empirical} for $T$ iterations with parameters set as in \autoref{param} and the number of samples is $N=\mathrm{poly}(d).$ Then, there exist times $\mathscr{T},\widehat{\mathcal{T}}_0,\widehat{\mathcal{T}}_1>0$ such that
\begin{enumerate}
    \item Analog of Event I (\autoref{lem:event1}): $\widehat{\alpha}^{(t+1)}=\widehat{\alpha}^{(t)}+\Theta(\eta ) ( \widehat{G}^{(t)} )^{p}(\widehat{\alpha}^{(t)})^{p-1}$ for $t\in[\mathscr{T},\mathscr{T}+\widehat{\mathcal{T}}_0]$.
    \item Analog of Event III (\autoref{lem:event3}): $\widehat{\alpha}^{(t+1)}=\widehat{\alpha}^{(t)}+\Theta(\eta ) ( \widehat{G}^{(t)} )^{p}(\widehat{\alpha}^{(t)})^{p-1}$ for $t\in[\mathscr{T}+\widehat{\mathcal{T}}_1,T]$.
\end{enumerate}
Consequently, $\widehat{\alpha}^{(t)}$ is non-decreasing and eventually, $\widehat{\alpha}^{(T)}=\alpha^{(T)}=\frac{\mathrm{polylog}(d)}{C^2\lambda_0}.$
\end{lemma}
These three lemmas imply that \textit{at time $T$}, the realistic iterates are very close to the ideal ones. Therefore, they incur nearby  test loss and thus the realistic model generalizes. We now proceed to the proof of

In order to analyze the dynamics of $\widehat{\bm{v}}^{(t)}$, we first show that the gradient (with respect to $\widehat{\bm{v}}$) in the realistic learning process is very close to the one in the semi-idealized one. 

\begin{lemma}\label{lem:gap_grads}
Let $t\in[\mathscr{T},T]$. With high probability, we have 
\begin{align}
    \| \nabla_{\widehat{\bm{v}}}\widehat{\mathcal{L}}(\widehat{\bm{v}}^{(t)},\widehat{\bm{A}}^{(t)})-\nabla_{\widecheck{\bm{v}}}\widecheck{\mathcal{L}}(\widecheck{\bm{v}}^{(t-\mathscr{T})},\widecheck{\bm{A}}^{(t-\mathscr{T})})\|_2\leq \mathrm{poly}(D)\Big(\frac{1}{\sqrt{N}}+ \varepsilon_{\bm{v}}^{(t)} +\Delta_{\bm{A}}^{(t)}\Big).
\end{align}
By choosing $N=\mathrm{poly}(d)$, we have $\| \nabla_{\widehat{\bm{v}}}\widehat{\mathcal{L}}(\widehat{\bm{v}}^{(t)},\widehat{\bm{A}}^{(t)})-\nabla_{\widecheck{\bm{v}}}\widecheck{\mathcal{L}}(\widecheck{\bm{v}}^{(t-\mathscr{T})},\widecheck{\bm{A}}^{(t-\mathscr{T})})\|_2\leq1/\mathrm{poly}(d).$
\end{lemma}

\begin{proof}[Proof of \autoref{lem:gap_grads}] We have:
\begin{align}
    &\| \nabla_{\widehat{\bm{v}}}\widehat{\mathcal{L}}(\widehat{\bm{v}},\widehat{\bm{A}})-\nabla_{\widecheck{\bm{v}}}\widecheck{\mathcal{L}}(\widecheck{\bm{v}},\widecheck{\bm{A}})\|_2\label{eq:oewdojfwe}\\
    &\hspace{-.55cm}=\Big|\frac{D}{N} \sum_{i=1}^Ny[i]\mathfrak{S}\big(-y[i]F_{\widehat{\bm{A}}^{(t)},\widehat{\bm{v}}^{(t)}}(\bm{X}[i])\big)\sum_{j=1}^D  \sigma'\Big(\sum_{k=1}^D \widehat{S}_{j,k}^{(t)}\langle \widehat{\bm{v}}^{(t)},\bm{X}_k[i]\rangle\Big)\sum_{r=1}^D \widehat{S}_{j,r}^{(t)}\langle \bm{X}_r[i],\bm{w}^*\rangle\label{eq:fnewfwe}\\
    &-D\mathbb{E}\Big[y\mathfrak{S}\big(-yF_{\widehat{\bm{A}}^{(t)},\widehat{\bm{v}}^{(t)}}(\bm{X})\big)  \sum_{j=1}^D  \sigma'\Big(\sum_{k=1}^D \widehat{S}_{j,k}^{(t)}\langle \widehat{\bm{v}}^{(t)},\bm{X}_k\rangle\Big)\sum_{r=1}^D \widehat{S}_{j,r}^{(t)}\langle \bm{X}_r,\bm{w}^*\rangle \Big] \Big|\\
    +&D\Bigg|\mathbb{E}\Big[y\mathfrak{S}\big(-yF_{\widehat{\bm{A}}^{(t)},\widehat{\bm{v}}^{(t)}}(\bm{X})\big)  \sum_{j=1}^D  \sigma'\Big(\sum_{k=1}^D \widehat{S}_{j,k}^{(t)}\langle \widehat{\bm{v}}^{(t)},\bm{X}_k\rangle\Big)\sum_{r=1}^D \widehat{S}_{j,r}^{(t)}\langle \bm{X}_r,\bm{w}^*\rangle \Big]\label{eq:wfojw}\\
    &-\mathbb{E}\Big[y\mathfrak{S}\big(-yF_{\widehat{\bm{A}}^{(t)},\widecheck{\bm{v}}^{(t)}}(\bm{X})\big)  \sum_{j=1}^D  \sigma'\Big(\sum_{k=1}^D \widehat{S}_{j,k}^{(t)}\langle \widecheck{\bm{v}}^{(t)},\bm{X}_k\rangle\Big)\sum_{r=1}^D \widehat{S}_{j,r}^{(t)}\langle \bm{X}_r,\bm{w}^*\rangle\Big]\Big|\label{eq:ojerofjr}\\
    +&D\Big|\mathbb{E}\Big[y\mathfrak{S}\big(-yF_{\widehat{\bm{A}}^{(t)},\widecheck{\bm{v}}^{(t)}}(\bm{X})\big)  \sum_{j=1}^D  \sigma'\Big(\sum_{k=1}^D \widehat{S}_{j,k}^{(t)}\langle \widecheck{\bm{v}}^{(t)},\bm{X}_k\rangle\Big)\sum_{r=1}^D \widehat{S}_{j,r}^{(t)}\langle \bm{X}_r,\bm{w}^*\rangle\Big]\label{eq:fojoe}\\
    &-\mathbb{E}\Big[y\mathfrak{S}\big(-yF_{\widecheck{\bm{A}}^{(t)},\widecheck{\bm{v}}^{(t)}}(\bm{X})\big)  \sum_{j=1}^D  \sigma'\Big(\sum_{k=1}^D \widecheck{S}_{j,k}^{(t)}\langle \widecheck{\bm{v}}^{(t)},\bm{X}_k\rangle\Big)\sum_{r=1}^D \widecheck{S}_{j,r}^{(t)}\langle \bm{X}_r,\bm{w}^*\rangle \Big] \Big|\label{eq:qwdwq}\\
    &\hspace{-.55cm}+ \varepsilon_{\bm{v}}^{(t)}\Big|\frac{D}{N} \sum_{i=1}^Ny[i]\mathfrak{S}\big(-y[i]F_{\widehat{\bm{A}}^{(t)},\widehat{\bm{v}}^{(t)}}(\bm{X}[i])\big)\sum_{j=1}^D  \sigma'\Big(\sum_{k=1}^D \widehat{S}_{j,k}^{(t)}\langle \widehat{\bm{v}}^{(t)},\bm{X}_k[i]\rangle\Big)\sum_{r=1}^D \widehat{S}_{j,r}^{(t)}\langle \bm{X}_r[i],\bm{u}^{(t)}\rangle\Big|.\label{eq:kewffkpd}
\end{align}
We bound each of the terms above using concentration or lipschitz inequalities. Using the same arguments as in the proof of \autoref{lem:epsv}, we have $|\eqref{eq:oewdojfwe}-\eqref{eq:fnewfwe}|\leq\mathrm{poly}(D)/\sqrt{N}$ and $|\eqref{eq:wfojw}-\eqref{eq:ojerofjr}|\leq \mathrm{poly}(D) \varepsilon_{\bm{v}}^{(t)}$. Using the same steps as in the proof of \autoref{lem:epsq},  we have $|\eqref{eq:fojoe}-\eqref{eq:qwdwq}|\leq\mathrm{poly}(D)\Delta_{\bm{A}}^{(t)}$. Lastly, $|\eqref{eq:kewffkpd}|\leq\mathrm{poly}(D)\varepsilon_{\bm{v}}^{(t)}.$ Summing up all these terms yields the aimed result.

\end{proof}

\autoref{lem:gap_grads} shows that we can use the gradient from the semi-idealized process to analyze the dynamics of $\widehat{\alpha}^{(t)}$ in the real process. Therefore, we can derive similar updates for  $\widehat{\alpha}^{(t)}$ as in \autoref{lem:event1}, \autoref{lem:event2} and \autoref{lem:event3}.

\begin{lemma}\label{lem:alphahatupdinit}
Let $\widecheck{\mathcal{T}}_0=\Theta\left(\frac{1}{\eta C(\widehat{\alpha}^{(\mathscr{T})})^{p-1} }\right)$.  Therefore, $\widehat{\alpha}^{(t)}$ is updated as
\begin{align*}
    \widehat{\alpha}^{(t+1)}&=\widehat{\alpha}^{(t)}+\Theta(\eta C) ( \widecheck{G}^{(t)} )^{p}(\widehat{\alpha}^{(t)})^{p-1}.
\end{align*}
Consequently, $\widehat{\alpha}^{(t)}$ is non-decreasing and after $\widecheck{\mathcal{T}}_0$ iterations, we have $\widehat{\alpha}^{(t)}\geq\frac{\Omega(1)}{C^2\lambda_0}$ for $t\geq \widecheck{\mathcal{T}}_0.$
\end{lemma}
\begin{lemma}\label{lem:event2emp}
  Let   $\widecheck{\mathcal{T}}_1=\widecheck{\mathcal{T}}_0+ \Theta\Big( \frac{\lambda_0^p}{\eta  e^{\mathrm{polyloglog}(d)}} \Big)$.
For all $t\in[\widecheck{\mathcal{T}}_1,T]$, we have $\widecheck{\Gamma}^{(t)}\geq \frac{\Omega(\lambda_0)}{D}$. 
\end{lemma}

\begin{lemma}\label{lem:event3emp}
 Let   $\widecheck{\mathcal{T}}_1=\widecheck{\mathcal{T}}_0+ \Theta\Big( \frac{\lambda_0^p}{\eta  e^{\mathrm{polyloglog}(d)}} \Big)$ and $t\in[\widecheck{\mathcal{T}}_1,T]$.  $\widehat{\alpha}^{(t)}$ updates as
\begin{align}\label{eq:widehatupdalphaev3}
    \widehat{\alpha}^{(t+1)}&=\widehat{\alpha}^{(t)}+\Theta(\eta C) ( \widecheck{G}^{(t)} )^{p}(\widehat{\alpha}^{(t)})^{p-1}.
\end{align}
Consequently, $\widehat{\alpha}^{(t)}$ is non-decreasing and eventually $\widehat{\alpha}^{(T)}=\frac{\mathrm{polylog}(d)}{C^2\lambda_0}.$
\end{lemma}

\paragraph{Auxiliary lemma.} The following lemma is useful to prove \autoref{lem:epsv} and \autoref{lem:epsq}.
\begin{lemma}\label{lem:timeT2}
The time at which $\widehat{\alpha}^{(t)}$ stops increasing is $\widecheck{\mathcal{T}}_2=\widecheck{\mathcal{T}}_1+\Theta\left(\frac{\lambda_0^{p-2}}{\eta\mathrm{polylog}(d)}\right).$
\end{lemma}
\begin{proof}[Proof of \autoref{lem:timeT2}] The result is obtained by applying
\autoref{lem:pow_method} to \eqref{eq:widehatupdalphaev3}. We have:
\begin{align*}
   \widecheck{\mathcal{T}}_2=\widecheck{\mathcal{T}}_1+\frac{3 (\lambda_0)^{p-2}}{\eta \mathrm{polylog}(d) e^{\mathrm{polyloglog}(d)} }+\frac{2^{p}\lambda_0^p}{e^{\mathrm{polyloglog}(d)}}\log\log(d).
\end{align*}
\end{proof}

\subsubsection{The realistic model fits the labeling function}\label{sec:general_realistic}

\begin{lemma}\label{lem:testerr_real}
In the realistic case, the model  fits the labeling function i.e. 
\begin{align}
    \mathbb{P}_{\mathcal{D}}[f^*(\bm{X})F_{\widehat{\bm{A}}^{(T)},\widehat{\bm{v}}^{(T)}}(\bm{X})>0]\geq 1-o(1).
\end{align}
\end{lemma}
\begin{proof}[Proof of \autoref{lem:testerr_real}] We bound the population risk $ \mathcal{L}(\widehat{\bm{A}}^{(T)},\widehat{\bm{v}}^{(T)})$. We have: 
\begin{align}\label{eq:wfeewfewfwef}
    \mathcal{L}(\widehat{\bm{A}}^{(T)},\widehat{\bm{v}}^{(T)})&\leq \Big|  \mathcal{L}(\widehat{\bm{A}}^{(T)},\widehat{\bm{v}}^{(T)})- \mathcal{L}(\widecheck{\bm{A}}^{(T)},\widecheck{\bm{v}}^{(T)})\Big| +\mathcal{L}(\widecheck{\bm{A}}^{(T)},\widecheck{\bm{v}}^{(T)}). 
    %\Big| \mathbb{E}\Big[\log\Big(1+e^{-yF_{\widehat{\bm{A}}^{(T)},\widehat{\bm{v}}^{(T)}}(\bm{X})}\Big)\Big]- \mathbb{E}\Big[\log\Big(1+e^{-yF_{\bm{Q}^{(T)},\bm{v}^{(T)}}(\bm{X})}\Big)\Big]\Big|
\end{align}
Using  \autoref{lem:poplosstilde}, we have $\mathcal{L}(\widecheck{\bm{A}}^{(T)},\widecheck{\bm{v}}^{(T)})\leq1/\mathrm{poly}(d).$ We now bound the first summand in \eqref{eq:wfeewfewfwef} using the 1-Lipschitzness of the logistic function and get: 
\begin{align}
    \Big|  \mathcal{L}(\widehat{\bm{A}}^{(T)},\widehat{\bm{v}}^{(T)})- \mathcal{L}(\widecheck{\bm{A}}^{(T)},\widecheck{\bm{v}}^{(T)})\Big|&\leq \mathbb{E}\Big[\big|F_{\widehat{\bm{A}}^{(T)},\widehat{\bm{v}}^{(T)}}(\bm{X})-F_{\widecheck{\bm{A}}^{(T)},\widecheck{\bm{v}}^{(T)}}(\bm{X})\big|\Big]
\end{align}
Using \autoref{lem:gapfunctrealideal}, we have $\mathbb{E}\Big[\big|F_{\widehat{\bm{A}}^{(T)},\widehat{\bm{v}}^{(T)}}(\bm{X})-F_{\widecheck{\bm{A}}^{(T)},\widecheck{\bm{v}}^{(T)}}(\bm{X})\big|\Big]\leq1/\mathrm{poly}(d).$ Therefore, we deduce that $\mathcal{L}(\widehat{\bm{A}}^{(T)},\widehat{\bm{v}}^{(T)})\leq 1/\mathrm{poly}(d).$ Since the 0-1 loss is a convex surrogate, we have
\begin{align}\label{eq:fjwefew}
     \mathbb{P}_{(\bm{X},y)\sim\mathcal{D}}[yF_{\widehat{\bm{A}}^{(T)},\widehat{\bm{v}}^{(T)}}(\bm{X})<0]\leq  \mathcal{L}(\widehat{\bm{A}}^{(T)},\widehat{\bm{v}}^{(T)})\leq 1/\mathrm{poly}(d).
\end{align}
We can further expand \eqref{eq:fjwefew} as in the proof of  \autoref{lem:cv_pop} and  deduce the aimed result.
\end{proof}

To prove \autoref{lem:testerr_real}, we use the following auxiliary lemma.

\begin{lemma}\label{lem:poplosstilde}
After $T$ iterations, the population risk in the semi-idealized case converges i.e.\ $\widecheck{\mathcal{L}}(\widecheck{\bm{A}}^{(T)},\widecheck{\bm{v}}^{(T)})\leq o(1).$
\end{lemma}
\begin{proof}[Proof of \autoref{lem:poplosstilde}] The proof is similar to the one of \autoref{lem:cv_pop}.
\end{proof}

\begin{lemma}\label{lem:gapfunctrealideal}
For all $\bm{X}$ sampled from $\mathcal{D}$, we have 
\begin{align*}
    \big|F_{\widehat{\bm{A}}^{(T)},\widehat{\bm{v}}^{(T)}}(\bm{X})-F_{\widecheck{\bm{A}}^{(T)},\widecheck{\bm{v}}^{(T)}}(\bm{X})\big|\leq1/\mathrm{poly}(d).
\end{align*}
\end{lemma}
\begin{proof}[Proof of \autoref{lem:gapfunctrealideal}]  We have:
\begin{align}
   \big|F_{\widehat{\bm{A}}^{(T)},\widehat{\bm{v}}^{(T)}}(\bm{X})-F_{\bm{Q}^{(T)},\widecheck{\bm{v}}^{(T)}}(\bm{X})\big|&\leq   \big|F_{\widehat{\bm{A}}^{(T)},\widehat{\bm{v}}^{(T)}}(\bm{X})-F_{\widehat{\bm{A}}^{(T)},\widecheck{\bm{v}}^{(T)}}(\bm{X})\big|\label{eq:fweojowjeef}\\
   &+  \big|F_{\widehat{\bm{A}}^{(T)},\widecheck{\bm{v}}^{(T)}}(\bm{X})-F_{\widecheck{\bm{A}}^{(T)},\widecheck{\bm{v}}^{(T)}}(\bm{X})\big|. \label{eq:fwejoojrw}
\end{align}
We now separately bound \eqref{eq:fweojowjeef} and \eqref{eq:fwejoojrw}.
\paragraph{Bound on \eqref{eq:fweojowjeef}.} Since $x\mapsto x^p+\nu x$ is Lipschitz on a bounded domain, we have:
\begin{align}\label{eq:ojwoewd}
    \eqref{eq:fweojowjeef}&\leq  \sum_{m=1}^D\sigma'\Big(D\sum_{r=1}^D \widehat{S}_{m,r}^{(T)}|\langle \widehat{\bm{v}}^{(T)},\bm{X}_r\rangle|\Big)  D\sum_{r=1}^D \widehat{S}_{m,r}^{(T)}|\langle \widehat{\bm{v}}^{(T)}-\widecheck{\bm{v}}^{(T)},\bm{X}_r\rangle| \\
    &\leq \varepsilon_{\bm{v}}^{(T)}\sum_{m=1}^D\sigma'\Big(D\sum_{r=1}^D \widehat{S}_{m,r}^{(T)}|\langle \widehat{\bm{v}}^{(T)},\bm{X}_r\rangle|\Big)  D\sum_{r=1}^D \widehat{S}_{m,r}^{(T)}|\langle \bm{u}^{(t)},\bm{X}_r\rangle|,  
\end{align}
since $\langle \widehat{\bm{v}}^{(T)},\bm{w}^*\rangle=\langle \widecheck{\bm{v}}^{(T)},\bm{w}^*\rangle$. 
Using Cauchy-Schwarz inequality, \eqref{eq:ojwoewd} simplifies as:
\begin{align}
     \eqref{eq:fweojowjeef}&\leq \sum_{m=1}^D\sigma'\Big(D\sum_{r=1}^D \widehat{S}_{m,r}^{(T)}|\langle \widehat{\bm{v}}^{(T)},\bm{X}_r\rangle|\Big)^{p-1}  D\sum_{r=1}^D \widehat{S}_{m,r}^{(T)}\varepsilon_{\bm{v}}^{(T)} \leq  \mathrm{poly}(D)\varepsilon_{\bm{v}}^{(T)}.
\end{align}
Using \autoref{lem:epsv}, we conclude that $ \eqref{eq:fweojowjeef}\leq 1/\mathrm{poly}(d).$

\paragraph{Bound on \eqref{eq:fwejoojrw}.} We again use the Lipschitzness of the power function and get: 
\begin{align}\label{eq:wffjewfw}
   \eqref{eq:fwejoojrw}&\leq \sum_{m=1}^D\sigma'\Big(D\sum_{r=1}^D \max\{\widecheck{S}_{m,r}^{(T)},\widehat{S}_{m,r}^{(T)}\}|\langle \bm{v}^{(T)},\bm{X}_r\rangle|\Big)   D\sum_{r=1}^D |\widehat{S}_{m,r}^{(T)}-\widecheck{S}_{m,r}^{(T)}|\cdot |\langle \widehat{\bm{v}}^{(T)},\bm{X}_r\rangle|.
\end{align}
We apply \autoref{lem:softmaxlip} in \eqref{eq:wffjewfw} to get $\eqref{eq:fwejoojrw}\leq\mathrm{poly}(D)\Delta_{\bm{A}}^{(T)}.$ Finally, we apply \autoref{lem:epsq} to get $\eqref{eq:fwejoojrw}\leq 1/\mathrm{poly}(d).$

\end{proof}

\section{Transfer Learning}\label{sec:app_transfer}

In this section, we show that a transformer that has been pre-trained on a structured dataset require a few samples to generalize in a new dataset sharing the same structure.

\thmtransfer*
\begin{proof}[Proof of \autoref{thm:transfer}]
Actually, even one step of the update using normalized gradient descent on $\bm{v}$ can already achieve test accuracy $\geq 1-o(1).$ We know that for a datum $(\bm{X}, y)$, the gradient of $L(\bm{X})$ with respect to $\bm{v}$ is
\begin{align}\label{eq:grad_trf}
    \nabla_{\bm{v}}L(\bm{X})=-y\mathfrak{S}(-yF(\bm{X}))\sum_{m=1}^D \sigma'(\langle\bm{O}_m^{(t)},\widetilde{\bm{v}}^{(t)}\rangle)\bm{O}_m^{(t)}
\end{align}
Since $\widetilde{\bm{v}}^{(0)}=\bm{0},$ we have $F_{\widetilde{\bm{v}}^{(0)}}(\bm{X})=0$ and $\sigma'(\langle\bm{O}_m^{(0)},\widetilde{\bm{v}}^{(0)}\rangle)=\nu.$ Thus, the gradient \eqref{eq:grad_trf} simplifies to
\begin{align}
     \nabla_{\bm{v}}L(\bm{X})= -\frac{\nu y}{2}\sum_{j=1}^D\bm{X}_j\sum_{m=1}^D \widehat{S}_{m,j}^{(t)}
\end{align}
By symmetry of the $\widehat{S}_{j,m}^{(t)}$  , we know that
\begin{align}
    \Big|\sum_{m=1}^D \widehat{S}_{m,j}^{(t)} -D\Big|\leq \frac{D}{\mathrm{poly}(d)},
\end{align}
where $\frac{1}{\mathrm{poly}(d)}$
comes from the $\varepsilon_{\bm{Q}}$ part in the previous section. Moreover, since the noise and feature noise has mean zero independent of $y$, we know that there
exists some value $c_0>0$ (roughly equal to $\alpha DC$) ) such that:
\begin{align}
    \mathbb{E}[ \nabla_{\bm{v}}L(\bm{X})]&= c_0 \widetilde{\bm{w}}^*.
\end{align}
Now, by standard concentration inequality, we know that for $N$ i.i.d. samples $\bm{X}[i],y[i]$, with high probability
\begin{align}
    \bm{\nabla}_0:=\sum_{i=1}^N \nabla_{\bm{v}}L(\bm{X}[i])=c_0 \widetilde{\bm{w}}^* +\varepsilon_0\widetilde{\bm{w}}^*+\bm{\chi}_{0},
\end{align}
where $\varepsilon_0$ comes from the feature noise
\begin{align}
    |\varepsilon_0|\leq \frac{c_0 qD\log(d)}{\sqrt{N}},
\end{align}
and $\bm{\chi}_{0}$ comes from the noise:
\begin{align}
    \|\bm{\chi}_{0}\|_2\leq\frac{c_0\sigma\sqrt{d}\sqrt{D}}{\sqrt{N}}.
\end{align}
Therefore, if we update using normalized GD:
\begin{align}
    \widetilde{\bm{v}}^{(1)}&=\widetilde{\bm{v}}^{(0)}+\frac{\bm{\nabla}_0}{\|\bm{\nabla}_0\|_2}
\end{align}
we have that:
\begin{align}
    \widetilde{\bm{v}}^{(1)}&= c_1 \widetilde{\bm{w}}^*+\bm{\chi}_{1},
\end{align}
where $\bm{\chi}_{1}\perp\widetilde{\bm{w}}^*$, $\|\bm{\chi}_{1}\|_2=O(1)$ and $c_1\geq \Tilde{\Omega}(1)/\sqrt{D}.$ Now, for a new datum $\bm{X}^{new}$ with noises $\{\bm{\xi}_i^{new}\}_{i=1}^D$, we know that w.h.p
\begin{align}
    |\langle \bm{\xi}_i^{new},\widetilde{\bm{v}}^{(1)}\rangle|\leq\frac{\Tilde{O}(1)}{\sqrt{d}}, \qquad |\langle \widetilde{\bm{w}}^{*},\widetilde{\bm{v}}^{(1)}\rangle| \geq \frac{\Tilde{\Omega}(1)}{\sqrt{D}}.
\end{align}
We can prove the test accuracy is small using the same proof as in  \autoref{lem:cv_pop}, where we show
that:
\begin{align}
    y\sigma(\langle  \widetilde{\bm{v}}^{(1)}, \bm{O}_m^{(1)}\rangle)\geq 0,
\end{align}
for $m\in\mathcal{S}_{\ell(\bm{X})}$  and it dominates the other $ y\sigma(\langle  \widetilde{\bm{v}}^{(1)}, \bm{O}_j^{(1)}\rangle)$ for $j\in\mathcal{S}_{\ell}$ with $\ell\neq\ell(\bm{X}).$  Therefore, we prove that $\mathbb{P}_{\widetilde{\mathcal{D}}}[yF(\bm{X})>0]\geq 1-o(1)$ which implies the aimed result.
\end{proof}

\thmtransfernegative*

\begin{proof}[Proof of \autoref{thm:samplcomplexA}]
Let $\mathcal{A}$ be an algorithm.
Assume that at training time the algorithm $\mathcal{A}$ has access to $D^{o(1)}$ training data, Since each input $X$ is made of $D$ patches, this means that there exist at least $\Omega(1)$ fraction of $k\in[D]$ such that $\mathcal{A}$ has not seen training samples with $\ell(X)=k$. Consider the following two distributions over $\{-1, 0, 1\}^m$:
\begin{enumerate}
    \item  $\mathcal{D}_1$: Sample $z \in \{0, 1\}^m$ where each $z_i$ i.i.d. $1$ w.p. $q/2$, $-1$ w.p. $q/2$ and $0$ otherwise. 
    
    \item  $\mathcal{D}_2$: Sample a set $\mathcal{S}$ uniformly at random from $[m]$ of size $C$, set all $z_i = 1$ for $i \in \mathcal{S}$, and sample other $z_j$  i.i.d. $1$ w.p. $q/2$, $-1$ w.p. $q/2$ and $0$ otherwise. 
\end{enumerate}

We can easily see that as long as $qm = \text{poly}(C)$, then 
$$\textbf{TV}(\mathcal{D}_1, \mathcal{D}_2) = o(1)$$

This implies that $\mathcal{A}$ must have bad generalization error ($\Omega(1)$) on $\tilde{\mathcal{D}}$.

%Since $\ell(X)$ contains the label of the datapoint, this implies that  $\mathcal{A}$ cannot well generalize on test samples with $\ell(X)=k$. Therefore,  $D^{\Omega(1)}$ samples are at least needed to generalize.
\end{proof}

\section{Gradient descent updates in the idealized process}

In this section, we derive the gradient descent updates of $A_{i,j}$ in the idealized learning process.

\subsection{Indices in the same set: $i,j\in\mathcal{S}_{\ell}$}

\begin{lemma}\label{lem:gd_gamma}
Let $T>0$ be the time where the population loss is at most $o(1)$ and $t\in[0,T].$ Then, $\gamma^{(t)}$ satisfies the update
\begin{align*}
    \gamma^{(t+1)}&=\gamma^{(t)}+\eta\Theta(C)\alpha^{(t)}\Big(D\alpha^{(t)}\big[\Lambda^{(t)}+(C-1)\Gamma^{(t)}\big]\Big)^{p-1}\Gamma^{(t)}.
\end{align*}
\end{lemma}
\begin{proof}[Proof of \autoref{lem:gd_gamma}] Let $\ell\in[L]$ and  $i,j\in\mathcal{S}_{\ell}$ with $i\neq j.$ The main idea of the proof is to bound the gradient of $L(\bm{X})$ with respect to $A_{i,j}$. This gradient is given by \autoref{lem:grad_Qij} and is made of two terms: the $\sigma'$ term and the sum outside $\sigma'$. We distinguish the following cases and bound these two terms.

\vspace{.3cm}

\underline{\textbf{1.} $\bm{\ell=\ell(X)}$.} We first bound the outside sum. Using \autoref{lem:chernoff_sumnoise}, we have
\begin{align}\label{eq:surfeec}
    \Big|\sum_{h\neq \ell(\bm{X})}\sum_{r\in\mathcal{S}_h}(1-y\delta_{r})-(D-C)\Big|&\leq \sum_{h\neq \ell(\bm{X})}\sum_{r\in\mathcal{S}_h}|\delta_{r}|\leq qD\log(d).
\end{align}
Since $(D-C)/2-qD\log(d)\geq 0$, we rewrite \eqref{eq:surfeec}  as: 
\begin{align}\label{eq:jfreince}
    \frac{1}{2}(D-C)\leq\sum_{h\neq \ell(\bm{X})}\sum_{r\in\mathcal{S}_h}(1-y\delta_{r})\leq 2(D-C).
\end{align}
Regarding the sum inside $\sigma'$, we use \autoref{lem:bd_signnoisesig} which shows: 
\begin{align}\label{eq:jecdnecj}
\Lambda^{(t)}+(C-1)\Gamma^{(t)}+\Xi^{(t)} \sum_{h\neq \ell(\bm{X})}\sum_{r\in\mathcal{S}_h}\delta_{h,r}=\Theta(1)\left(\Lambda^{(t)}+(C-1)\Gamma^{(t)}\right).
\end{align}
By using \eqref{eq:jfreince} and \eqref{eq:jecdnecj}, we finally obtain:
\begin{align}\label{eq:wcdwe}
    -\frac{\partial L(\bm{X})}{\partial A_{i,j}}&=  \Theta(D\alpha^{(t)})\mathfrak{S}(-yF(\bm{X})) \left(D\alpha^{(t)}(\Lambda^{(t)}+(C-1)\Gamma^{(t)})\right)^{p-1}\Gamma^{(t)}(D-C)\Xi^{(t)}.
\end{align}

\underline{\textbf{2.} $\bm{\ell\neq\ell(X)}$ \textbf{and} $\bm{\delta_{s}=0}$ \textbf{for all} $\bm{s\in\mathcal{S}_{\ell}}$.} We first bound the outside sum. Since $\delta_{s}=0$, the only  non-zero term is the one with factor $\Xi^{(t)}.$ Using triangle inequality, we have:
\begin{align}\label{eq:kappaencede}
    \hspace{-.4cm}\Xi^{(t)}\Big|Cy+ \sum_{h\neq\{\ell(\bm{X}),\ell\}}\sum_{r\in\mathcal{S}_h}\delta_{r}\Big|&\leq \Xi^{(t)}\bigg(C+\sum_{h\neq\{\ell(\bm{X}),\ell\}}\sum_{r\in\mathcal{S}_{h}}|\delta_{r}|\bigg)=\Xi^{(t)}\kappa(\bm{X}).
\end{align}
We now bound the sum inside $\sigma'$. This sum is actually equal to the outside sum and we can therefore use the bound \eqref{eq:kappaencede}. Therefore, the overall gradient is bounded as: 
\begin{equation}
    \begin{aligned}\label{eq:wknwnq}
    \Big|\frac{\partial L(\bm{X})}{\partial A_{i,j}}\Big|
    &\leq \Theta(D\alpha^{(t)})\mathfrak{S}(-yF(\bm{X})) \left(D\Xi^{(t)}\alpha^{(t)}\kappa(\bm{X})\right)^{p-1}\Gamma^{(t)}\Xi^{(t)}\kappa(\bm{X}).
\end{aligned}
\end{equation}

\underline{\textbf{3.} $\bm{\ell\neq\ell(\bm{X})}$ \textbf{and at least one }  $\bm{\delta_{s}\neq0}$ \textbf{ and } $\bm{\delta_{i}=0}$}.  We first bound the outside sum. Using \autoref{lem:chernoff_sumnoise} and $\Lambda^{(t)}+\Gamma^{(t)}+\Xi^{(t)}=1$, we have:
\begin{equation}
\begin{aligned}
    &\alpha^{(t)}\bigg|\Lambda^{(t)}\delta_{j}+\Gamma^{(t)} \sum_{r\in\mathcal{S}_{\ell}\backslash\{i\}}(\delta_{j} 
    -\delta_{r})
    +\Xi^{(t)}C(\delta_{j}-y) +\Xi^{(t)}\sum_{h\neq\{\ell(\bm{X}),\ell\}}\sum_{m\in\mathcal{S}_{h}}(\delta_{j}-\delta_{m})\bigg|\\
    \leq& \alpha^{(t)}\left[ (C-1)\Gamma^{(t)}+\Xi^{(t)}\left(C+\Theta(qD)\log(d)\right)+\mathbf{1}_{\delta_{j}\neq0}\right].\label{eq:jfnecjnewe}
\end{aligned}    
\end{equation}
We lastly apply \autoref{indh:lambgamxi} to show that \eqref{eq:jfnecjnewe} is less or equal to $\Theta(\alpha^{(t)}).$ We now bound the sum inside $\sigma'$.
\begin{equation}
\begin{aligned}\label{eq:jennjsnww}
    \alpha^{(t)}\bigg|\Gamma^{(t)}\sum_{r\in\mathcal{S}_{\ell}\backslash\{i,j\}} \delta_{r}+\Xi^{(t)}\Big(Cy+\sum_{h\neq\{\ell(\bm{X}),\ell\}}\sum_{m\in\mathcal{S}_{h}}\delta_{m}\Big)\bigg|
    &\leq \alpha^{(t)}\bigg( \Gamma^{(t)}\hspace{-.3cm} \sum_{r\in\mathcal{S}_{\ell}\backslash\{i,j\}} |\delta_{r}|+\Xi^{(t)} \kappa(\bm{X})\bigg)\\
    &\leq \alpha^{(t)}\left( (C-1)\Gamma^{(t)} +\Xi^{(t)} \kappa(\bm{X})\right).
\end{aligned}    
\end{equation}
The overall bound on the derivative is: 
\begin{equation}\label{eq:kcdnnqw}
    \begin{aligned}
 \Big|\frac{\partial L(\bm{X})}{\partial A_{i,j}}\Big|& \leq \Theta(D\alpha^{(t)})\mathfrak{S}(-yF(\bm{X})) \Big(D\alpha^{(t)}\big( (C-1)\Gamma^{(t)}+\Xi^{(t)} \kappa(X)\big) \Big)^{p-1} \Gamma^{(t)}.
 %&=\Theta(D\alpha^{(t)})\mathfrak{S}(-yF(X)) \left(D\alpha^{(t)}( \Gamma^{(t)}+\Xi^{(t)} \kappa(X)) \right)^{p-1} \Gamma^{(t)}. 
\end{aligned}
\end{equation}

\underline{\textbf{4.} $\bm{\ell\neq\ell(X)}$ \textbf{ where }  $\bm{\delta_{s}\neq0}$ \textbf{for all $s$}.}  We first bound the outside sum as follows.
\begin{equation}
\begin{aligned}
    &\alpha^{(t)}\biggr| \Lambda^{(t)}(\delta_{j}-\delta_{i})
    +\Gamma^{(t)}\sum_{r\in\mathcal{S}_{\ell}\backslash\{i,j\}}(\delta_{j}-\delta_{r})  +\Xi^{(t)}\Big[C(\delta_{j}-y)
    +\hspace{-.2cm}\sum_{h\neq\{\ell(\bm{X}),\ell\}}\sum_{m\in\mathcal{S}_{h}}(\delta_{j}-\delta_{m})\Big] \bigg|\\
     &\leq \alpha^{(t)}(\Lambda^{(t)}+ (C-1)\Gamma^{(t)}+\Xi^{(t)}\kappa(\bm{X})+\mathbf{1}_{\delta_j\neq 0}).
\end{aligned}    
\end{equation}
We now bound the sum inside $\sigma'$.
\begin{equation}
\begin{aligned}\label{eq:jewnwcdw}
    &\alpha^{(t)}\bigg| \Lambda^{(t)} \delta_{i} +\Gamma^{(t)}\sum_{r\in\mathcal{S}_{\ell}\backslash\{i,j\}} \delta_{r}+\Xi^{(t)}\Big[C y+\sum_{h\neq\{\ell(\bm{X}),\ell\}}\sum_{m\in\mathcal{S}_{h}} \delta_{m}\Big] \bigg|\\
    \leq& \alpha^{(t)}\Big(\Lambda^{(t)}+(C-1)\Gamma^{(t)}+\Xi^{(t)}\kappa(\bm{X})\Big).
\end{aligned}
\end{equation}
We lastly apply \autoref{lem:bd_signnoisesig} to show that \eqref{eq:jewnwcdw} is bounded by $\Theta(\alpha^{(t)})(\Lambda^{(t)}+(C-1)\Gamma^{(t)}).$ Therefore, the overall gradient is bounded as
\begin{equation}
    \begin{aligned}\label{eq:ejwdnjwne}
   \Big|\frac{\partial L(\bm{X})}{\partial A_{i,j}}\Big|& \leq \Theta(D\alpha^{(t)})\mathfrak{S}(-yF(\bm{X})) \left(D\alpha^{(t)}((C-1)\Gamma^{(t)}+\Lambda^{(t)})  \right)^{p-1} \Gamma^{(t)}\cdot\\
   &\hspace{.5cm} \left(\Lambda^{(t)}+(C-1)\Gamma^{(t)}+\Xi^{(t)}\kappa(\bm{X})+\mathbf{1}_{\delta_j\neq 0}\right) .
\end{aligned}
\end{equation}

\paragraph{Putting all the pieces together.} 
 We now  bound the derivative of the population loss. Using Tower property and \autoref{lem:perminv_dist}, we have: 
 \begin{align}
     \mathbb{E}_{\bm{X}}\bigg[\frac{\partial L(\bm{X})}{\partial A_{i,j}}\bigg] = \mathbb{E}_{\bm{X}}\bigg[\frac{\partial L(\pi(\bm{X}))}{\partial A_{i,j}}\bigg]  =\mathbb{E}_{\bm{X}}\bigg[\mathbb{E}_{\pi_1,\pi_2}\Big[\frac{\partial L(\pi(\bm{X}))}{\partial A_{i,j}}\Big|\bm{X}\Big]\bigg].
 \end{align}
 For a \textit{fixed} $\bm{X}$, we now bound the derivative of the loss evaluated in $\bm{X}$. For $i,j\in\mathcal{S}_{\ell}$, we distinguish the four possible events depending on the randomness of $\pi.$

 \begin{itemize}
    \item[--] \textbf{Event a}: "$\ell=\ell(\pi(\bm{X}))$" occurs with probability $1/L.$
    \item[--] \textbf{Event b}: "$\ell\neq \ell(\pi(\bm{X}))$ and $\delta_s=0$ for all $s\in\mathcal{S}_{ \ell}$" occurs with probability $(1-1/L)(1-q)^C$. 
    \item[--] \textbf{Event c$_{\bm{k}}$}: "$\ell\neq\ell(\pi(\bm{X}))$ and $\#\{s:\delta_s\neq0\}=k$ for $1\leq k \leq C-1$  and $\delta_i=0$" occurs with probability \ $(1-1/L) \binom{C-1}{k}  q^{k}(1-q)^{C-k}$.
    \item[--] \textbf{Event d}: "$\ell\neq\ell(\pi(\bm{X}))$ and $\delta_s\neq0$ for all $s$" occurs  with probability $(1-1/L)\cdot q^{C} $.
\end{itemize}

%%%%%%%%%%%%%%%%%%%%%%%%%%%%%%%%%%%%%%%%%%%%%%%%
%%% Old events with permutations. Maybe not the right ones.
\iffalse
\begin{itemize}
    \item[--] \textbf{Event a}: "$\pi_1,\pi_2:\pi_1(\ell)=\pi_1(\ell(X_0))$" occurs with probability $1/L.$
    \item[--] \textbf{Event b}: "$\pi_1,\pi_2:\;\pi_1(\ell)\neq\pi_1(\ell(X_0))$ and $\delta_s(X)=0$ for all $s\in\mathcal{S}_{\pi_1(\ell)}$" occurs with probability $(1-1/L)(1-q)^C$. 
    \item[--] \textbf{Event c$_{\bm{k}}$}: "$\pi_1,\pi_2:\;\pi_1(\ell)\neq\pi_1(\ell(X_0))$ and $\#\{s:\delta_s(X_0)\neq0\}=k$ for $1\leq k \leq C-1$  and $\delta_i(X_0)=0$" occurs with probability \ $(1-1/L) \binom{C-1}{k}  q^{k}(1-q)^{C-k}$.
    \item[--] \textbf{Event d}: "$\pi_1,\pi_2:\pi_1(\ell)\neq\pi_1(\ell(X))$ and $\delta_s(X_0)\neq0$ for all $s$" occurs  with probability $(1-1/L)\cdot q^{C} $.
\end{itemize}
\fi 
%%%%%%%%%%%%%%%%%%%%%%%%%%%%%%%%%%%%%%%%%%%%%%%%%%%%%

Therefore, the derivative of the loss in $\bm{X}$ is: 
\begin{equation}
\begin{aligned}\label{eq:popgradijl1}
   \mathbb{E}_{\pi_1,\pi_2}\Big[\frac{\partial L (\pi(\bm{X}))}{\partial A_{i,j}}\Big|\bm{X}\Big]&=  \frac{1}{L}\mathbb{E}_{\pi_1,\pi_2}\Big[\frac{\partial L(\pi(\bm{X}))}{\partial A_{i,j}}\;\Big|\;\bm{X},\textbf{a}\Big]\\
   +&\left(1-\frac{1}{L}\right)(1-q)^C\mathbb{E}_{\pi_1,\pi_2}\Big[\frac{\partial L(\pi(\bm{X}))}{\partial A_{i,j}}\;\Big|\;\bm{X},\textbf{b}\Big]\\
   +& \left(1-\frac{1}{L}\right)\sum_{k=1}^{C-1}\binom{C-1}{k}  q^{k}(1-q)^{C-k}\mathbb{E}_{\pi_1,\pi_2}\Big[\frac{\partial L(\pi(\bm{X}))}{\partial A_{i,j}}\;\Big|\;\bm{X},\textbf{c}_k\Big]\\
   +&\left(1-\frac{1}{L}\right)q^C\mathbb{E}_{\pi_1,\pi_2}\Big[\frac{\partial L(\pi(\bm{X}))}{\partial A_{i,j}}\;\Big|\;\bm{X},\textbf{d}\Big].
\end{aligned}
\end{equation}
Event a is the event that is the most likely to happen. Therefore, we only take into account $\mathbb{E}_{\pi_1,\pi_2}\left[\frac{\partial L(\pi(\bm{X}))}{\partial A_{i,j}}\;\middle|\;\textbf{a}\right]$ in
\eqref{eq:popgradijl1} and obtain:
%We now plug \eqref{eq:wcdwe}, \eqref{eq:wknwnq}, \eqref{eq:kcdnnqw} and \eqref{eq:ejwdnjwne} in \eqref{eq:popgradijl1} and apply \autoref{lem:perm_inv} which states   $F((\pi_1,\pi_2)(X_0))=F(X_0)$:
%Using \autoref{lem:bd_signnoisesig} and \autoref{indh:lambgamxi}, we obtain that \eqref{eq:popgradijl2} is
\begin{align}\label{eq:final_gradgam}
    \mathbb{E}_{\bm{X}}\bigg[\frac{\partial L(\bm{X})}{\partial A_{i,j}}\bigg]&=\  \frac{\Theta(D\alpha^{(t)})}{L}\mathbb{E}_X[\mathfrak{S}(-yF(\bm{X}))] \left(D\alpha^{(t)}(\Lambda^{(t)}+(C-1)\Gamma^{(t)}))\right)^{p-1}\Gamma^{(t)}.
\end{align}
Since the population loss is a $\Omega(1)$ for $t\leq T$, this implies that $\mathbb{E}_{\bm{X}}[\mathfrak{S}(-yF(\bm{X}))]=\Theta(1)$ (\autoref{lem:sigmfunct}). We thus plug \eqref{eq:final_gradgam} in the update of $\gamma^{(t)}$ to obtain the desired result.
\end{proof}

\begin{corollary}\label{cor:gamupdate}
Let $T>0$ be the time where the population loss is $o(1)$ and $t\leq T.$ Let $G^{(t)}:=D(\Lambda^{(t)}+(C-1)\Gamma^{(t)})$. The update of $\gamma^{(t)}$ satisfies: 
\begin{align*}
    \gamma^{(t+1)}&=\gamma^{(t)}+\Theta(C\eta)(\alpha^{(t)})^{p}\Gamma^{(t)}(G^{(t)})^{p-1}.
\end{align*}
\end{corollary}

\begin{proof}[Proof of \autoref{cor:gamupdate}] 
\autoref{lem:gd_gamma} provides the update rule of $\gamma^{(t)}$. 
\begin{align*}
    \gamma^{(t+1)}&=\gamma^{(t)}+\eta\Theta(C)(\alpha^{(t)})^{p}\Gamma^{(t)}(G^{(t)})^{p-1}.
\end{align*}
%Since $C=\mathrm{polylog}(d)$, we obtain the desired result.
\end{proof}

\subsection{Update for $i\in\mathcal{S}_{\ell}$ and $j\in\mathcal{S}_m$}

\begin{lemma}\label{lem:gd_rho}
Let $T>0$ be the time where the population loss is at most $o(1)$ and $t\in[0,T].$ Then, $\rho^{(t)}$ satisfies the update
\begin{align*}
    \left|\frac{\rho^{(t+1)}-\rho^{(t)}}{\eta}\right|&\leq \Theta(C^2\alpha^{(t)}) \left(D\alpha^{(t)}(\Lambda^{(t)}+(C-1)\Gamma^{(t)})\right)^{p-1} \Xi^{(t)} \frac{\lambda_0}{D} \\
    &+\Theta(\alpha^{(t)}) \left(D\alpha^{(t)}qD\log(d)\Xi^{(t)}\right)^{p-1} \Xi^{(t)}qD\log(d).
\end{align*}
\end{lemma}
\begin{proof}[Proof of \autoref{lem:gd_gamma}] Let $\ell,m\in[L]$ such that $\ell\neq m$ and  $i\in\mathcal{S}_{\ell}$, $j\in\mathcal{S}_m$. The main idea of the proof is to bound the gradient of $L(\bm{X})$ with respect to $A_{i,j}$. This gradient is given by \autoref{lem:grad_Qij} and is made of two terms: the $\sigma'$ term and the sum outside $\sigma'$. We distinguish the following cases and bound these two terms. 

\vspace{.3cm}
\underline{\textbf{1.} $\bm{\ell=\ell(X)}$ \textbf{ and $\bm{\delta_{j}=0}$}.}  We first bound the outside sum. We apply \autoref{lem:bd_signnoisesig} to obtain:
\begin{align}\label{eq:jfeoerew}
     &\hspace{-.3cm}\alpha^{(t)}\Big |\big(\Lambda^{(t)}+(C-1)\Gamma^{(t)}\big) y + \Xi^{(t)}\sum_{h\neq \ell(\bm{X})}\sum_{r\in\mathcal{S}_h}\delta_{r} \Big| \leq 2\alpha^{(t)}\left(\Lambda^{(t)}+(C-1)\Gamma^{(t)}\right).
\end{align}
We now bound the sum inside $\sigma'$. This sum is actually equal to the outside sum and we can therefore use the bound \eqref{eq:jfeoerew}. Therefore, the overall gradient is bounded as: 
\begin{equation}\label{eq:gradrho1}
    \left| \frac{\partial L(\bm{X})}{\partial A_{i,j}}\right|\leq \Theta(D\alpha^{(t)})\mathfrak{S}(-yF(\bm{X}))\left(D\alpha^{(t)}(\Lambda^{(t)}+(C-1)\Gamma^{(t)})\right)^{p-1}\Xi^{(t)} (\Lambda^{(t)}+(C-1)\Gamma^{(t)}).
\end{equation}

\vspace{.3cm}
\underline{\textbf{2.} $\bm{\ell=\ell(X)}$ \textbf{ and $\bm{\delta_{j}\neq 0}$}.} We first bound the outside sum. We successively apply $\Lambda^{(t)}+\Gamma^{(t)}+\Xi^{(t)}=1$, \autoref{lem:chernoff_sumnoise} and \autoref{indh:lambgamxi} to  obtain:
\begin{align}
    &\alpha^{(t)}\bigg|\Lambda^{(t)}(\delta_{j}-y)+(C-1)\Gamma^{(t)}(\delta_{j}-y)+\Xi^{(t)}\sum_{h\neq\ell(\bm{X})}\sum_{r\in\mathcal{S}_h}(\delta_{j}-\delta_{r})\bigg|\nonumber\\
    \leq\;& \alpha^{(t)}\bigg[\Lambda^{(t)}+(C-1)\Gamma^{(t)} +\Xi^{(t)}\sum_{h\neq\ell(\bm{X})}\sum_{r\in\mathcal{S}_h}|\delta_{r}|+\mathbf{1}_{\delta_{j}\neq 0}\bigg]\nonumber\\
     \leq\;& \alpha^{(t)}\left[\Lambda^{(t)}+(C-1)\Gamma^{(t)}+\Xi^{(t)}qD\log(d)+ \mathbf{1}_{\delta_{j}\neq 0} \right]\nonumber\\
      \leq\;& \alpha^{(t)}\left[\frac{e^{\mathrm{polyloglog}(d)}+\lambda_0}{D}+\Theta(q\log(d))+ \mathbf{1}_{\delta_{j}\neq 0} \right]=\Theta(\alpha^{(t)}) .\label{eq:fkwe3fe}
\end{align}
We now bound the sum inside $\sigma'$. We successively apply \autoref{lem:chernoff_sumnoise}, triangle inequality and \autoref{lem:bd_signnoisesig} to obtain: 
\begin{align}\label{eq:jfnewnw}
   \alpha^{(t)} \bigg|y(\Lambda^{(t)}+(C-1)\Gamma^{(t)}) +\Xi^{(t)}\sum_{h\neq\ell(\bm{X})}\sum_{r\in\mathcal{S}_h}\delta_{r}\bigg|\leq 2\alpha^{(t)}(\Lambda^{(t)}+(C-1)\Gamma^{(t)}).
\end{align}
Thus, we use \eqref{eq:fkwe3fe} and \eqref{eq:jfnewnw} to obtain a bound on the derivative.
\begin{align}\label{eq:gradrho2}
   \left| \frac{\partial L(\bm{X})}{\partial A_{i,j}}\right|&\leq \Theta(D\alpha^{(t)})\mathfrak{S}(-yF(\bm{X}))\Big(D\alpha^{(t)}(\Lambda^{(t)}+(C-1)\Gamma^{(t)})\Big)^{p-1}\Xi^{(t)} .
\end{align}

\vspace{.2cm}

\underline{\textbf{3.} $\bm{\ell\neq\ell(X)}$ \textbf{ and $\bm{\delta_{s}=0}$ for all $\bm{s}$ and $\bm{m=\ell(X)}$}.}  We first bound the outside sum. Since $\Lambda^{(t)}+(C-1)\Gamma^{(t)}+(D-C)\Xi^{(t)}=1$, we have:
\begin{align}\label{eq:ejdnoewqs}
    \hspace{-.4cm}\Xi^{(t)}\bigg|y\left(\Lambda^{(t)} +(C-1)\Gamma^{(t)}\right) +\Xi^{(t)}\hspace{-.6cm}\sum_{h\neq\{\ell(\bm{X}),\ell\}}\sum_{r\in\mathcal{S}_h}(y-\delta_{r})\bigg|\leq\Xi^{(t)}\hspace{-.1cm}\left(\Xi^{(t)}\kappa(\bm{X})+\mathbf{1}_{m=\ell(\bm{X})}\right).
\end{align}
We now bound the sum inside $\sigma'$. 
\begin{align}\label{eq:fjcnofew}
   \Xi^{(t)}\bigg|y+\sum_{h\neq\{\ell(\bm{X}),\ell\}}\sum_{r\in\mathcal{S}_h}\delta_{r}(X)\bigg|\leq  \Xi^{(t)}\kappa(\bm{X}).
\end{align}
Using \eqref{eq:ejdnoewqs} and \eqref{eq:fjcnofew}, we obtain a bound on the derivative. 
\begin{align}\label{eq:gradrho3}
    \left| \frac{\partial L(\bm{X})}{\partial A_{i,j}}\right|&\leq \Theta(D\alpha^{(t)})\mathfrak{S}(-yF(\bm{X}))\Big(D\alpha^{(t)}\kappa(\bm{X})\Xi^{(t)}\Big)^{p-1} \Xi^{(t)} \left(\Xi^{(t)}\kappa(\bm{X})+\mathbf{1}_{m=\ell(\bm{X})}\right) .
\end{align}

\vspace{.2cm}

\underline{\textbf{4.} $\bm{\ell\neq\ell(\bm{X})}$ \textbf{ and $\bm{\delta_{s}\neq 0}$ for some $\bm{s}$ and $\bm{m=\ell(X)}$}.}  We first bound the outside sum. We apply  $\Lambda^{(t)}+(C-1)\Gamma^{(t)}+(D-C)\Xi^{(t)}=1$ and \autoref{lem:lamgamxicte} to get: 
\begin{align}
    &\Xi^{(t)}\left|\Lambda^{(t)}(y-\delta_{i} )+\Gamma^{(t)}\sum_{r\in\mathcal{S}_{\ell}\backslash\{i\}}(y-\delta_{r}) +\Xi^{(t)}\sum_{h\neq \{\ell(X),\ell\}}\sum_{r\in\mathcal{S}_h}(y-\delta_{r})\right|\nonumber\\
  \leq  \; &\Xi^{(t)}\left(\Lambda^{(t)} +(C-1)\Gamma^{(t)}+\Xi^{(t)}\sum_{h\neq\{\ell(\bm{X}),\ell\}}\sum_{r\in\mathcal{S}_h}|\delta_{r}|+\mathbf{1}_{m=\ell(\bm{X})}\right)\nonumber\\
  \leq  \; & \Theta(\Xi^{(t)}).\label{eq:wwoenf}
\end{align}
We now bound the sum inside $\sigma'$. We apply \autoref{lem:bd_signnoisesig} to obtain:
\begin{equation}
\begin{aligned}\label{eq:fjqjeq}
    &\left|\Lambda^{(t)}\delta_{i}(X)+\Gamma^{(t)}\sum_{r\in\mathcal{S}_{\ell}\backslash\{i\}} \delta_{r}(X)+\Xi^{(t)}\left(Cy+\sum_{h\in[L]\backslash\{\ell(\bm{X}),\ell\}}\sum_{r\in\mathcal{S}_h}\delta_{h,r}(X)\right)\right|\\
    \leq  &\; \Lambda^{(t)}+(C-1)\Gamma^{(t)}+\Xi^{(t)}\kappa(\bm{X})\\
   \leq  &\; 2\left(\Lambda^{(t)}+(C-1)\Gamma^{(t)}\right).
\end{aligned}
\end{equation}
We combine \eqref{eq:wwoenf} and \eqref{eq:fjqjeq} and obtain: 
\begin{align}\label{eq:gradrho4}
    \left| \frac{\partial L(\bm{X})}{\partial A_{i,j}}\right|&\leq \Theta(D\alpha^{(t)})\mathfrak{S}(-yF(\bm{X}))\Big(D\alpha^{(t)}\big(\Lambda^{(t)}+(C-1)\Gamma^{(t)}\big)\Big)^{p-1} \Xi^{(t)}  .
\end{align}

\underline{\textbf{5.} $\bm{\ell,m\neq\ell(X)}$ \textbf{ and $\bm{\delta_{s}=0}$ \textbf{for all $s$} and $\bm{\delta_{j}= 0}$}.} We first bound the outside sum.
\begin{equation}
\begin{aligned}\label{eq:frekfwew}
  & \Xi^{(t)} \Big| Cy+\sum_{h\neq\{\ell,\ell(\bm{X})\}}\sum_{r\in\mathcal{S}_h}\delta_{r} \Big|\leq \Xi^{(t)}\kappa(\bm{X}).
\end{aligned}
\end{equation} 
We now bound the sum inside $\sigma'$. This sum is actually equal to the outside sum outside and we can therefore use the bound \eqref{eq:frekfwew}.
Thus, the derivative is bounded as: 
\begin{align}\label{eq:gradrho5}
     \left| \frac{\partial L(\bm{X})}{\partial A_{i,j}}\right|&\leq \Theta(D\alpha^{(t)})\mathfrak{S}(-yF(\bm{X}))\left(D\alpha^{(t)}\kappa(\bm{X})\Xi^{(t)}\right)^{p-1} (\Xi^{(t)})^2 \kappa(\bm{X})  .
\end{align}

\underline{\textbf{6.} $\bm{\ell,m\neq\ell(X)}$ \textbf{ and $\bm{\delta_{s}\neq0}$ \textbf{for some $s$} and $\bm{\delta_{j}= 0}$}.}   We first bound the outside sum. We apply \autoref{lem:bd_signnoisesig} and obtain:
\begin{equation}
\begin{aligned}\label{eq:redeww}
  & \alpha^{(t)}\Big|\Lambda^{(t)}\delta_{i}+\Gamma^{(t)}\sum_{r\in\mathcal{S}_{\ell}\backslash\{i\}} \delta_{r} +\Xi^{(t)}\big[Cy+\sum_{h\neq\{\ell,\ell(\bm{X})\}}\sum_{r\in\mathcal{S}_h}\delta_{r}\big]\Big|\\
  \leq&\;\alpha^{(t)}\Big(\Lambda^{(t)}  + (C-1)\Gamma^{(t)}+\Xi^{(t)}\big[C+\sum_{h\neq\{\ell,\ell(\bm{X})\}}\sum_{r\in\mathcal{S}_h}|\delta_{r}|\big]\Big)\\
  \leq&\;2\alpha^{(t)}\left(\Lambda^{(t)}  + (C-1)\Gamma^{(t)}\right).
\end{aligned}
\end{equation}
We now bound the sum inside the power term. This sum is actually equal to the sum outside the power term and we can therefore use the bound \eqref{eq:redeww}. Thus, the derivative is bounded as: 
\begin{align}\label{eq:gradrho6}
     \hspace{-.4cm}\left| \frac{\partial L(\bm{X})}{\partial A_{i,j}}\right|&\leq \Theta(D\alpha^{(t)})\mathfrak{S}(-yF(\bm{X}))\left(D\alpha^{(t)}(\Lambda^{(t)}+(C-1)\Gamma^{(t)})\right)^{p-1} \Xi^{(t)} (\Lambda^{(t)}+(C-1)\Gamma^{(t)})  .
\end{align}

\underline{\textbf{7.} $\bm{\ell,m\neq\ell(X)}$ \textbf{ and $\bm{\delta_{s}=0}$ \textbf{for all $s$} and $\bm{\delta_{j}\neq 0}$}.}  We first bound the outside sum.
\begin{equation}
\begin{aligned}\label{eq:jfnrorjw3}
     &\Xi^{(t)}\bigg|\Lambda^{(t)}\delta_{j}+(C-1)\Gamma^{(t)}\delta_{j} +\Xi^{(t)}\Big[C(\delta_{j}-y)+\sum_{h\neq\{\ell(X),\ell\}}\sum_{r\in\mathcal{S}_h}(\delta_{j}-\delta_{r})\Big]\bigg|\\
     \leq& \Xi^{(t)}\Big(\Xi^{(t)}\kappa(\bm{X})+\mathbf{1}_{\delta_{j}\neq 0}\Big).
 \end{aligned}
 \end{equation}
We now bound the sum inside $\sigma'$.
\begin{align}\label{eq:jfeneew}
   \Xi^{(t)}\Big|Cy+\sum_{h\neq\{\ell(\bm{X}),\ell\}}\sum_{r\in\mathcal{S}_h}\delta_{r}\Big|\leq \Xi^{(t)}\kappa(\bm{X}).
\end{align}
Using \eqref{eq:jfnrorjw3} and \eqref{eq:jfeneew}, the bound on the derivative is: 
\begin{align}\label{eq:gradrho7}
     \Big| \frac{\partial L(\bm{X})}{\partial A_{i,j}}\Big|&\leq \Theta(D\alpha^{(t)})\mathfrak{S}(-yF(\bm{X}))\Big(D\alpha^{(t)}\Xi^{(t)}\kappa(\bm{X})\Big)^{p-1} \Xi^{(t)}\Big(\Xi^{(t)}\kappa(\bm{X})+\mathbf{1}_{\delta_{j}\neq 0}\Big) .
\end{align}
\underline{\textbf{8.} $\bm{\ell,m\neq\ell(X)}$ \textbf{ and $\bm{\delta_{s}(X)\neq0}$ \textbf{for some $s$} and $\bm{\delta_{j}\neq 0}$}:} We first bound the outside sum. We apply \autoref{lem:lamgamxicte} to get:
\begin{equation}
\begin{aligned}\label{eq:fjeiwbf}
    &\Xi^{(t)}\bigg|\Lambda^{(t)}(\delta_{j}-\delta_{i} )+\Gamma^{(t)}\sum_{r\in\mathcal{S}_{\ell}\backslash\{i\}}(\delta_{j}-\delta_{r}) +\Xi^{(t)}\Big[C(\delta_{j}-y)+\sum_{h\neq\{\ell(\bm{X}),\ell\}}\sum_{r\in\mathcal{S}_h}(\delta_{j}-\delta_{r})\Big]\bigg|\\
   \leq\;&\Xi^{(t)}\Big(\Lambda^{(t)}+(C-1)\Gamma^{(t)}+\Xi^{(t)}\kappa(\bm{X})+\mathbf{1}_{\delta_{j}\neq 0}\Big)\\
   \leq\;&  \Theta(\Xi^{(t)}).
\end{aligned}
\end{equation}
We now bound the sum inside the power term. We apply \autoref{lem:bd_signnoisesig} and get:
\begin{equation}
\begin{aligned}\label{eq:fjbweijf}
   \bigg| \Lambda^{(t)}\delta_{i}+\Gamma^{(t)}\sum_{r\in\mathcal{S}_{\ell}\backslash\{i\}} \delta_{r}+\Xi^{(t)}\Big(Cy+\sum_{h\neq\{\ell(\bm{X}),\ell\}}\sum_{r\in\mathcal{S}_h}\delta_{r}\Big) \bigg|&\leq \Lambda^{(t)}+(C-1)\Gamma^{(t)}+\Xi^{(t)}\kappa(\bm{X})\\
   &\leq \Theta\Big(\Lambda^{(t)}+(C-1)\Gamma^{(t)}\Big).
\end{aligned}
\end{equation}
We plug \eqref{eq:fjeiwbf} and  \eqref{eq:fjbweijf} to obtain the derivative.
\begin{align}\label{eq:gradrho8}
     \left| \frac{\partial L(\bm{X})}{\partial A_{i,j}}\right|&\leq \Theta(D\alpha^{(t)})\mathfrak{S}(-yF(\bm{X}))\Big(D\alpha^{(t)} \big(\Lambda^{(t)}+(C-1)\Gamma^{(t)}\big)\Big)^{p-1} \Xi^{(t)}  .
\end{align}

\paragraph{Putting all the pieces together.}  We now  bound the derivative of the population loss. Using Tower property and and \autoref{lem:perminv_dist}, we have: 
 \begin{align}
     \mathbb{E}_X\left[\frac{\partial L(\bm{X})}{\partial A_{i,j}}\right]=\mathbb{E}_{\bm{X}}\left[\mathbb{E}_{\pi_1,\pi_2}\left[\frac{\partial L( \pi(\bm{X}))}{\partial A_{i,j}}\middle|\bm{X}\right]\right].
 \end{align}
 For a \textit{fixed} $\bm{X}$, we now bound the derivative of the loss evaluated in $\bm{X}$. For $i\in\mathcal{S}_{\ell}$ and $j\in\mathcal{S}_m$, we distinguish the eight possible events depending on the randomness of $\pi.$
\begin{itemize}
    \item[--] \textbf{Event a}: "$\ell=\ell(\pi(\bm{X}))$ and $ \delta_{j}=0$" occurs with probability $(1-q)/L.$
    \item[--]   \textbf{Event b}: "$\ell=\ell(\pi(\bm{X}))$ and $ \delta_{j}\neq 0$" occurs with probability $q/L.$
    \item[--] \textbf{Event c}:  "$\ell\neq\ell(\pi(\bm{X}))$ and $\delta_{s}=0$ for all $s$ and $m=\ell(\pi(\bm{X}))$" occurs with probability $(1-q)^C/L$.
     \item[--] \textbf{Event d$_k$}: "$\ell\neq\ell(\pi(\bm{X}))$ and $\#\{s:\delta_{s}\neq0\}=k$ for $1\leq k\leq C$  and $m=\ell(\pi(\bm{X}))$" occurs with probability $\binom{C}{k}q^k(1-q)^{C-k}/L$.
     \item[--] \textbf{Event e}: "$ \ell,m\neq\ell(\pi(\bm{X}))$ and $\delta_{j}= 0$ and $\delta_{s}=0$ for all $s$" occurs with probability $(1-2/L)(1-q)^{C+1}.$
     \item[--] \textbf{Event f}$_k$: "$ \ell,m\neq\ell(\pi(\bm{X}))$  and $\#\{s:\delta_{s}\neq0\}=k$ for $1\leq k\leq C$ and $\delta_{j}= 0$" occurs with probability $(1-2/L)(1-q)\binom{C}{k}q^k(1-q)^{C-k}$.
     \item[--] \textbf{Event g}: "$\ell,m\neq\ell(\pi(\bm{X}))$  and $\delta_{s}=0$ for all $s$ and $\delta_{j}\neq 0$" occurs with probability $(1-2/L)(1-q)^Cq.$
     \item[--] \textbf{Event h}$_k$: "$\ell,m\neq\ell(\pi(\bm{X}))$ and $\#\{s:\delta_{s}\neq0\}=k$ for $0\leq k\leq C$ and $\delta_{j}\neq 0$" occurs with probability $(1-2/L)q\binom{C}{k}q^k(1-q)^{C-k}$
\end{itemize}

Since events a and e are the ones with highest probabilities, the derivative of the loss is bounded by the expectations conditioned on these events. We have:
\begin{equation}
\begin{aligned}\label{eq:jwoeodw}
    & \mathbb{E}_{\pi_1,\pi_2}\Big[\Big|\frac{\partial L(\pi(\bm{X}))}{\partial A_{i,j}}\Big|\Big]  \\
    \leq\;& \frac{\Theta(D\alpha^{(t)})}{L}\Xi^{(t)}\mathfrak{S}(-yF(\bm{X}))\big(D\alpha^{(t)}(\Lambda^{(t)}+(C-1)\Gamma^{(t)})\big)^{p-1} (\Lambda^{(t)}+(C-1)\Gamma^{(t)})  \\
    +\;&\Theta(D\alpha^{(t)})\Xi^{(t)}\mathfrak{S}(-yF(\bm{X}))\mathbb{E}_{\pi_1,\pi_2} \big[\big(D\alpha^{(t)}\kappa(\pi(\bm{X}))\Xi^{(t)}\big)^{p-1} \Xi^{(t)}\kappa(\pi(\bm{X}))\Big|\mathbf{e}\Big].
\end{aligned}
\end{equation}
We now apply \autoref{indh:lambgamxi} and \autoref{lem:chernoff_sumnoise} and finally obtain: %to further bound \eqref{eq:jwoeodw}.
\begin{equation}
\begin{aligned}\label{eq:final_gradrho}
   \mathbb{E}_{\bm{X}}\Big[\Big|\frac{\partial L(\bm{X})}{\partial A_{i,j}}\Big|\Big] &\leq \Theta(C^2\alpha^{(t)}) \mathbb{E}_{\bm{X}}[\mathfrak{S}(-yF(\bm{X}))]\big(D\alpha^{(t)}(\Lambda^{(t)}+(C-1)\Gamma^{(t)})\big)^{p-1} \Xi^{(t)} \frac{\lambda_0}{D} \\
    +\;&\Theta(\alpha^{(t)}) \mathbb{E}_{\bm{X}}[\mathfrak{S}(-yF(\bm{X}))]\big(D\alpha^{(t)}qD\log(d)\Xi^{(t)}\big)^{p-1} \Xi^{(t)}qD\log(d).
\end{aligned}
\end{equation}
Since  $\mathbb{E}_{\bm{X}}[\mathfrak{S}(-yF(\bm{X}))]\leq 1$, we thus plug \eqref{eq:final_gradrho} in the update of $\rho^{(t)}$ to obtain the aimed result.

\end{proof}

%Let $G^{(t)}:=D(\Lambda^{(t)}+(C-1)\Gamma^{(t)})$.
\begin{corollary}\label{cor:rhoupdate}
Let $T>0$ be the time where the population loss is $o(1)$ and $t\leq T.$  The update of $\rho^{(t)}$ satisfies: 
\begin{align*}
    |\rho^{(t+1)}|&\leq|\rho^{(t)}|+\eta\cdot\mathrm{polylog}(d)(\alpha^{(t)})^{p}\left(\frac{1}{D}+\frac{\lambda_0}{D}\Xi^{(t)}(G^{(t)})^{p-1}\right).
\end{align*}
\end{corollary}
\begin{proof}[Proof of \autoref{cor:rhoupdate}] Using \autoref{lem:gd_rho}, $C=\mathrm{polylog}(d)$ and $qD\leq 1$, $\rho^{(t)}$'s update is:
\begin{align}\label{eq:fjonedw}
    | \rho^{(t+1)}| &\leq |\rho^{(t)}|+ \eta(\alpha^{(t)})^p\mathrm{polylog}(d)\left((D\Xi^{(t)})^{p-1}\Xi^{(t)}+\frac{\lambda_0}{D}(G^{(t)})^{p-1}\Xi^{(t)}\right).
\end{align}
Lastly, we apply \autoref{indh:lambgamxi} to replace $\Xi^{(t)}$ by its value in \eqref{eq:fjonedw} and thus obtain the aimed result.
\end{proof}

\subsection{Auxiliary lemmas}

\begin{lemma}\label{lem:bd_signnoisesig}
    Let $t>0.$ In the idealized learning process, with high probability, we have: 
    \begin{align*}
    \Lambda^{(t)}+(C-1)\Gamma^{(t)}+y\Xi^{(t)} \sum_{\ell\neq\ell(\bm{X})}\sum_{r\in\mathcal{S}_{\ell}}  \delta_{r} =\Theta\big(\Lambda^{(t)}+(C-1)\Gamma^{(t)}\big).
    \end{align*}
\end{lemma}
\begin{proof}[Proof of \autoref{lem:bd_signnoisesig}]
We first bound the sum with factor $\Xi^{(t)}$. Using \autoref{indh:lambgamxi} and \autoref{lem:chernoff_sumnoise}, we have: 
\begin{align}\label{eq:xisum43}
   \Big| y\Xi^{(t)} \sum_{\ell\neq\ell(\bm{X})}\sum_{r\in\mathcal{S}_{\ell}}  \delta_{r}\Big|\leq  \Theta(q)\log(d)=\frac{1}{D} \mathrm{poly}(C)\log(d).
\end{align}
We now bound $\Lambda^{(t)}+(C-1)\Gamma^{(t)}$. Using \autoref{indh:lambgamxi}, we have:
\begin{align}\label{eq:signsigsell1}
    \frac{ 1}{D}\left(e^{\mathrm{polyloglog}(d)}+\Omega(C)
   \right)\leq (\Lambda^{(t)}+(C-1)\Gamma^{(t)})\leq \frac{1}{D}\left(e^{\mathrm{polyloglog}(d)}+\lambda_0\right).
\end{align}
Since $\mathrm{poly}(C)\log(d)\leq e^{\mathrm{polyloglog}(d)}/2$, we combine \eqref{eq:xisum43} and \eqref{eq:signsigsell1} to get the aimed result.
\end{proof}

\begin{lemma}\label{lem:lamgamxicte} Let $t>0$. In the idealized learning process, we have with high probability: 
\begin{align*}
    \Lambda^{(t)}+(C-1)\Gamma^{(t)}+\Xi^{(t)}\sum_{h\neq\ell(\bm{X})}\sum_{r\in\mathcal{S}_h}|\delta_{r}| \leq \Theta(1).
\end{align*}
\end{lemma}
\begin{proof}[Proof of \autoref{lem:lamgamxicte}] We successively apply \autoref{lem:chernoff_sumnoise} and \autoref{indh:lambgamxi} to get the desired bound. Indeed, we have: 
\begin{align*}
     \Lambda^{(t)}+(C-1)\Gamma^{(t)}+\Xi^{(t)}\sum_{h\neq\ell(\bm{X})}\sum_{r\in\mathcal{S}_h}|\delta_{r}|&\leq  \Lambda^{(t)}+(C-1)\Gamma^{(t)}+\Xi^{(t)}qD\log(d)\\
     &\leq \frac{1}{D}\left(e^{\mathrm{polyloglog}(d)}+\lambda_0+qD\log(d)\right)\\
     &=\Theta(1).
\end{align*}
\end{proof} 
\section{Gradients}
 
In this section, we present the gradients of the loss $\mathcal{L}$ with respect to $\bm{v}$ and $\bm{A}_{i,j}$. 

\begin{lemma}\label{lem:grad_v} Let $(\bm{X},y)$ be a data-point. Then, the gradient of $L(\bm{X})$ with respect to $\bm{v}$ is: 
\begin{align*}
   -\nabla_v L(\bm{X})&= Dy \mathfrak{S}(-yF(\bm{X}))\sum_{i=1}^D\sigma'\bigg(D\sum_{r\in[D]} \bm{S}_{i,r} \langle \bm{v},\bm{X}_r\rangle\bigg)\sum_{j=1}^D \bm{S}_{i,j}\bm{X}_j.
\end{align*}
\end{lemma}

\begin{lemma}\label{lem:grad_Qij} Let $(\bm{X},y)$ be a data-point and $i,j\in[D]$. The derivative of $L(\bm{X})$ with respect to $A_{i,j}$ is: 
\begin{align*}
   -\frac{\partial L(\bm{X})}{\partial A_{i,j}}&= pDy \mathfrak{S}(-yF(\bm{X}))\bigg(D\sum_{r\in[D]} \bm{S}_{i,r} \langle \bm{v},\bm{X}_r\rangle\bigg)^{p-1}\sum_{r\neq j} \bm{S}_{i,r}\langle \bm{v},\bm{X}_j-\bm{X}_r\rangle.
\end{align*}
\end{lemma}
\section{Invariance of the problem}

\subsection{Invariance of the parameters}

\lemvspanw*

\begin{proof}[Proof of \autoref{lem:v_spanw}] The proof is by induction. Our induction hypothesis is  $\bm{v}^{(t)}=\alpha^{(t)}\bm{w}^*$ for all $t\geq 0.$ For $t=0$, we know that 
$\bm{v}^{(0)}=\alpha^{(0)}\bm{w}^*\in\mathrm{span}(\bm{w}^*).$ Assume that $\bm{v}^{(t)}=\alpha^{(t)}\bm{w}^*$. Let's show that there exists $\alpha^{(t+1)}\in\mathbb{R}$ such that $\bm{v}^{(t+1)}=\alpha^{(t+1)}\bm{w}^*$. From the update rule, we have: 
\begin{align}
    \bm{v}^{(t+1)}&=\bm{v}^{(t)}+\eta\mathbb{E}\Big[y\mathfrak{S}(-yF(\bm{X}))\sum_{i\in[D]}\sigma'(\langle \bm{v}^{(t)},\bm{O}_i^{(t)}\rangle)\sum_{r\in[D]}S_{i,r}\bm{X}_r\Big]\nonumber\\
    &=\bm{v}^{(t)}+\eta\mathbb{E}\Big[y\mathfrak{S}(-yF(\bm{X}))\sum_{i\in[D]}\sigma'(\langle \bm{v}^{(t)},\bm{O}_i^{(t)}\rangle)\sum_{r\in[D]}S_{i,r}\mathbb{E}_{\bm{\xi}}[\bm{X}_r]\Big],\label{eq:vupdt}
\end{align}
where  $\mathbb{E}_{\bm{\xi}}$ is the expectation with respect to noise vectors $\bm{\xi}_{r}$.
Using the definition of the data distribution and $\mathbb{E}[\bm{\xi}_{r}]=0$, we simplify the update \eqref{eq:vupdt}:
\begin{equation}
    \begin{aligned}
     \bm{v}^{(t+1)}&= \bm{v}^{(t)}+\eta\mathbb{E}\Big[\mathfrak{S}(-yF(\bm{X}))\sum_{i\in[D]}\sigma'(\langle \bm{v}^{(t)},\bm{O}_i^{(t)}\rangle)\sum_{r\in\mathcal{S}_{\ell(X)}}S_{i,r}\Big]\bm{w}^* \\
     &+\eta\mathbb{E}\Big[y\mathfrak{S}(-yF(\bm{X}))\sum_{i\in[D]}\sigma'(\langle \bm{v}^{(t)},\bm{O}_i^{(t)}\rangle)\sum_{h\neq \ell(X)}\sum_{r\in\mathcal{S}_{h}}S_{i,r}\delta_{r}\Big]\bm{w}^*.
\end{aligned}
\end{equation}
Since $\bm{v}^{(t)}=\alpha^{(t)}\bm{w}^*$,  there exists $\alpha^{(t+1)}\in\mathbb{R}$ such that  $\bm{v}^{(t+1)}=\alpha^{(t+1)}\bm{w}^*.$

\end{proof}

\lemQreducvar*  

\begin{proof}[Proof of \autoref{lem:Qreducvar}]
We initialize
 $A_{i,i}^{(0)}=\sigma_{\bm{A}}$ and do not update $A_{i,i}$, Thus, for all $t\geq 0,$ we have $A_{i,i}^{(t)}=\sigma_{\bm{A}}.$\newline
%\mic{$\forall i, j \in S_l, \forall k, m \in S_{l'} \partial_{A_{i,j}}L(X) \sim \partial_{A_{k,m}}L(X)$}
 The remaining of the proof is by induction. For $i,j\in\mathcal{S}_{\ell}$, we initialize  $A_{i,j}^{(0)}=0=\gamma^{(0)}$. For $i\in\mathcal{S}_{\ell}$ and $j\in\mathcal{S}_m$ with $\ell\neq m$, $A_{i,j}^{(0)}=0=\rho^{(0)}$. The induction hypothesis is true for $t=0.$ 
 
 We assume that for all $i,j\in\mathcal{S}_{\ell}$, $A_{i,j}^{(t)}=\gamma^{(t)}$ and for $i\in\mathcal{S}_{\ell}$ and $j\in\mathcal{S}_m$, $A_{i,j}^{(t)}=\rho^{(t)}$.  Let's first prove that $A_{i,j}^{(t+1)}=A_{k,n}^{(t+1)}$ for $i,j\in\mathcal{S}_{\ell}$ and $k,n\in\mathcal{S}_{\ell'}$. Since the GD update is $A_{i,j}^{(t+1)}=A_{i,j}^{(t)}-\eta \mathbb{E}_{\bm{X}}\left[\frac{\partial L(\bm{X})}{\partial A_{i,j}}\right] $, it's sufficient to prove that  $\mathbb{E}_{\bm{X}}\left[\frac{\partial L(\bm{X})}{\partial A_{i,j}}\right] =\mathbb{E}_{\bm{X}}\left[\frac{\partial L(\bm{X})}{\partial A_{k,n}}\right]$. Let $\pi_1\colon [L]\rightarrow [L]$, $\pi_2\colon [C]\rightarrow [C]$ and $\pi=(\pi_1,\pi_2)$ be permutations. From \autoref{lem:perminv_dist},
 we know that $\bm{X}$ and $\pi(\bm{X})$ have the same distribution which implies  $\mathbb{E}_{\bm{X}}\left[\frac{\partial L(\bm{X})}{\partial A_{i,j}}\right] = \mathbb{E}_{\pi(\bm{X})}\left[\frac{\partial L(\bm{X})}{\partial A_{i,j}}\right]$. Therefore, we have: 
 \begin{equation}
 \begin{aligned}\label{eq:fewjoejdwe}
     &\mathbb{E}_{\bm{X}}\left[\frac{\partial L(\bm{X})}{\partial A_{i,j}}\right]\\
     =&\mathbb{E}_{\bm{X}}\Bigg[ y\mathfrak{S}(-yF(\pi(\bm{X}))) \sigma'\bigg(\sum_{h=1}^L\sum_{r\in\mathcal{S}_{h}}\frac{e^{\langle \bm{A}^{(t)}\bm{p}_i,\;\bm{p}_r\rangle}}{\sum_{s=1}^De^{\langle \bm{A}^{(t)}\bm{p}_i,\;\bm{p}_s\rangle}}\langle \bm{v}^{(t)},\bm{X}_{\pi_1(h),\pi_2(r)}\rangle\bigg)\cdot\\
    &\hspace{.4cm}\frac{e^{\langle \bm{A}^{(t)}\bm{p}_i,\;\bm{p}_j\rangle}}{\sum_{s=1}^De^{\langle \bm{A}^{(t)}\bm{p}_i,\;\bm{p}_s\rangle}}\sum_{k=1}^L\sum_{m\in\mathcal{S}_{k}}\frac{e^{\langle \bm{A}^{(t)}\bm{p}_i,\;\bm{p}_m\rangle}}{\sum_{s=1}^De^{\langle \bm{A}^{(t)}\bm{p}_i,\;\bm{p}_s\rangle}}\langle \bm{v}^{(t)}, \bm{X}_{\pi_1(k),\pi_2(m)}-\bm{X}_{\pi_1(\ell),\pi_2(j)}\rangle\Bigg]
 \end{aligned}
\end{equation} 
Using \autoref{indh:lambgamxi}, we simplify  \eqref{eq:fewjoejdwe} as
 \begin{equation}
 \begin{aligned}\label{eq:fewjoejvferdwe}
     &\mathbb{E}_{\bm{X}}\left[\frac{\partial\mathcal{L}^{(t)}(\bm{X})}{\partial A_{i,j}}\right]\\
     =&\mathbb{E}_{\bm{X}}\Bigg[ y\mathfrak{S}(-yF(\pi(\bm{X}))) \sigma'\bigg(\Lambda^{(t)}\langle \bm{v}^{(t)},\bm{X}_{\pi_1(\ell),\pi_2(i)}\rangle +\Gamma^{(t)} \hspace{-.4cm}\sum_{m\in\mathcal{S}_{\ell}\backslash\{i\}} \langle \bm{v}^{(t)},\bm{X}_{\pi_1(\ell),\pi_2(m)}\rangle\\
     &\hspace{.3cm}+ \Xi^{(t)}\sum_{r\not\in\mathcal{S}_{\ell}} \langle \bm{v}^{(t)},\bm{X}_{\pi(r)}\rangle\bigg) \Gamma^{(t)}\bigg(\Lambda^{(t)}\langle \bm{v}^{(t)},\bm{X}_{\pi_1(\ell),\pi_2(i)}-\bm{X}_{\pi_1(\ell),\pi_2(j)}\rangle \\
     &\hspace{.3cm}+\Gamma^{(t)} \hspace{-.4cm}\sum_{m\in\mathcal{S}_{\ell}\backslash\{i\}} \langle \bm{v}^{(t)},\bm{X}_{\pi_1(\ell),\pi_2(m)}-\bm{X}_{\pi_1(\ell),\pi_2(j)}\rangle+ \Xi^{(t)}\sum_{r\not\in\mathcal{S}_{\ell}} \langle \bm{v}^{(t)},\bm{X}_{\pi(r)}-\bm{X}_{\pi_1(\ell),\pi_2(j)}\rangle\bigg) \Bigg].
 \end{aligned}
\end{equation} 
We now set $\pi_1$ and $\pi_2$ such that $\pi_1(\ell)=\ell'$, $\pi_2(i)=k$ and $\pi_2(j)=n$. Using this choice along with $F(\bm{X})=F(\pi(\bm{X}))$ (\autoref{lem:perminv_dist}), we finally have in \eqref{eq:fewjoejvferdwe}
\begin{equation}
 \begin{aligned}\label{eq:fdvdvwe}
     &\mathbb{E}_{\bm{X}}\left[\frac{\partial L(\bm{X})}{\partial A_{i,j}}\right]\\
     =&\mathbb{E}_{\bm{X}}\Bigg[ y\mathfrak{S}(-yF(\pi(\bm{X}))) \sigma'\bigg(\Lambda^{(t)}\langle \bm{v}^{(t)},\bm{X}_{\ell',k}\rangle +\Gamma^{(t)} \hspace{-.4cm}\sum_{m\in\mathcal{S}_{\ell'}\backslash\{i\}} \langle \bm{v}^{(t)},\bm{X}_{\ell',m}\rangle\\
     &\hspace{.3cm}+ \Xi^{(t)}\sum_{r\not\in\mathcal{S}_{\ell'}} \langle \bm{v}^{(t)},\bm{X}_{r}\rangle\bigg) \Gamma^{(t)}\bigg(\Lambda^{(t)}\langle \bm{v}^{(t)},\bm{X}_{\ell',k}-\bm{X}_{\ell',n}\rangle \\
     &\hspace{.3cm}+\Gamma^{(t)} \hspace{-.4cm}\sum_{m\in\mathcal{S}_{\ell'}\backslash\{k\}} \langle \bm{v}^{(t)},\bm{X}_{\ell',m}-\bm{X}_{\ell',n}\rangle+ \Xi^{(t)}\sum_{r\not\in\mathcal{S}_{\ell}} \langle \bm{v}^{(t)},\bm{X}_{r}-\bm{X}_{\ell',n}\rangle\bigg) \Bigg]\\
     =&\mathbb{E}_{\bm{X}}\left[\frac{\partial\mathcal{L}^{(t)}(\bm{X})}{\partial A_{k,n}}\right].
 \end{aligned}
\end{equation}  
Therefore, \eqref{eq:fdvdvwe} implies that $A_{i,j}^{(t+1)}=A_{k,n}^{(t+1)}$ thus proving the induction hypothesis. Let's now show that for $i\in\mathcal{S}_{\ell},j\in\mathcal{S}_m$ and $k\in\mathcal{S}_{\ell'},j\in\mathcal{S}_{m'}$, we have $A_{i,j}^{(t+1)}=A_{k,n}^{(t+1)}$. We apply a similar argument as above. Using \autoref{indh:lambgamxi}, we have: 
\begin{equation}
 \begin{aligned}\label{eq:fvfejhhrdwe}
     &\mathbb{E}_{\bm{X}}\left[\frac{\partial L(\bm{X})}{\partial A_{i,j}}\right]\\
     =&\mathbb{E}_{\bm{X}}\Bigg[ y\mathfrak{S}(-yF(\pi(\bm{X}))) \sigma'\bigg(\Lambda^{(t)}\langle \bm{v}^{(t)},\bm{X}_{\pi_1(\ell),\pi_2(i)}\rangle +\Gamma^{(t)} \hspace{-.4cm}\sum_{r\in\mathcal{S}_{\ell}\backslash\{i\}} \langle \bm{v}^{(t)},\bm{X}_{\pi_1(\ell),\pi_2(r)}\rangle\\
     &\hspace{.3cm}+ \Xi^{(t)}\sum_{r\not\in\mathcal{S}_{\ell}} \langle \bm{v}^{(t)},\bm{X}_{\pi(r)}\rangle\bigg) \Xi^{(t)}\bigg(\Lambda^{(t)}\langle \bm{v}^{(t)},\bm{X}_{\pi_1(\ell),\pi_2(i)}-\bm{X}_{\pi_1(m),\pi_2(j)}\rangle \\
     &\hspace{.3cm}+\Gamma^{(t)} \hspace{-.4cm}\sum_{r\in\mathcal{S}_{\ell}\backslash\{i\}} \langle \bm{v}^{(t)},\bm{X}_{\pi_1(\ell),\pi_2(r)}-\bm{X}_{\pi_1(m),\pi_2(j)}\rangle+ \Xi^{(t)}\sum_{r\not\in\mathcal{S}_{\ell}} \langle \bm{v}^{(t)},\bm{X}_{\pi(r)}-\bm{X}_{\pi_1(m),\pi_2(j)}\rangle\bigg) \Bigg].
 \end{aligned}
\end{equation} 
We now set $\pi_1$ and $\pi_2$ such that $\pi_1(\ell)=\ell'$, $\pi_1(m)=m'$,   $\pi_2(i)=k$ and $\pi_2(j)=n$. Using this choice, we finally have in \eqref{eq:fvfejhhrdwe}
\begin{equation}
 \begin{aligned}\label{eq:fe}
     &\mathbb{E}_{\bm{X}}\left[\frac{\partial L(\bm{X})}{\partial A_{i,j}}\right]\\
     =&\mathbb{E}_{\bm{X}}\Bigg[ y\mathfrak{S}(-yF(\pi(\bm{X}))) \sigma'\bigg(\Lambda^{(t)}\langle \bm{v}^{(t)},\bm{X}_{\ell',k}\rangle +\Gamma^{(t)} \hspace{-.4cm}\sum_{r\in\mathcal{S}_{\ell'}\backslash\{k\}} \langle \bm{v}^{(t)},\bm{X}_{\ell',r}\rangle\\
     &\hspace{.3cm}+ \Xi^{(t)}\sum_{r\not\in\mathcal{S}_{\ell'}} \langle \bm{v}^{(t)},\bm{X}_{r}\rangle\bigg) \Xi^{(t)}\bigg(\Lambda^{(t)}\langle \bm{v}^{(t)},\bm{X}_{\ell',k}-\bm{X}_{m',n}\rangle \\
     &\hspace{.3cm}+\Gamma^{(t)} \hspace{-.4cm}\sum_{r\in\mathcal{S}_{\ell'}\backslash\{k\}} \langle \bm{v}^{(t)},\bm{X}_{\ell',r}-\bm{X}_{m',n}\rangle+ \Xi^{(t)}\sum_{r\not\in\mathcal{S}_{\ell}} \langle \bm{v}^{(t)},\bm{X}_{r}-\bm{X}_{m',n}\rangle\bigg) \Bigg]\\
     =&\mathbb{E}_{\bm{X}}\left[\frac{\partial L(\bm{X})}{\partial A_{k,n}}\right].
 \end{aligned}
\end{equation}  
Therefore, \eqref{eq:fe} implies that $A_{i,j}^{(t+1)}=A_{k,n}^{(t+1)}$ thus proving the induction hypothesis.

\end{proof}

\subsection{Invariance by permutation}

%\lemperminvdist*
\begin{lemma}\label{lem:perminv_dist}
 Let $\pi_1\colon[L]\rightarrow[L]$ and $\pi_2\colon[C]\rightarrow [C]$ be two permutations and $\pi=(\pi_1,\pi_2)$. 
Let $(\bm{X},\cdot)\sim\mathcal{D}$. Then, we have: 
 \begin{enumerate}
     \item permutation-invariant distribution:  $\bm{X}$ has the same distribution as $\pi(\bm{X}).$
     \item permutation-invariant model: $F(\bm{X}) =F(\pi(\bm{X})).$   
 \end{enumerate}
\end{lemma}

\begin{proof}[Proof of \autoref{lem:perminv_dist}]

To show that $\bm{X}$ and $\pi(\bm{X})$ have to same distribution, it is sufficient to show that the items 1 to 6 hold in our definition of the data distribution. We still have that the label $y$ is uniformly sampled on $\{-1, 1\}$. Let $\pi(\bm{X})=(\widebar{\bm{X}}_1,\dots,\widebar{\bm{X}}_D)$ where $\widebar{\bm{X}}_i$ for some $j.$ For 2, the number of tokens is still $D$ after permutation $\pi$. For 4, we define the same partition with $\widebar{\mathcal{S}}_l = \widebar{\mathcal{S}}_{\pi_1(l)}$. Besides, we have $\ell(\bm{\widebar{X}}) = \pi_1(\ell(\bm{X}))$ that is also uniformly sampled on $[L]$ since the permutation on a uniform distribution is also uniform. For $i\in\mathcal{S}_{\ell(\bm{\widebar{X}})}$,  we have that $\bm{\widebar{X}}_i$ writes $\bm{X}_{\pi_2(k)}$ for some $k \in \mathcal{S}_{\ell(\bm{X})}$ so that we do have  $\bm{\widebar{X}}_i=y\bm{w}^*+\bm{\xi}_i$. The same goes for 6 when $\ell \neq \ell(\bm{\widebar{X}})$: with $\ell\neq \ell(\bm{\widebar{X}})$, $\bm{\widebar{X}}_j=\delta_{\ell,j}\bm{w}^*+\bm{\xi}_{j}$, where $\delta_{j}=1$ with probability $q/2$, $-1$ with the same probability and $0$ otherwise.

Let $\bm{X}$ be a data-point. Using \autoref{lem:Qreducvar}, we rewrite $F(\bm{X})$ as
\begin{align}
    F(\bm{X})&=\sum_{m=1}^D\sigma\Big(\Lambda^{(t)}\langle \bm{v}^{(t)},\bm{X}_{m}\rangle +\Gamma^{(t)}\hspace{-.4cm}\sum_{r\in\mathcal{S}_{\ell}\backslash\{m\}}\langle \bm{v}^{(t)},\bm{X}_{r}\rangle + \Xi^{(t)}\hspace{-.4cm}\sum_{h\in[L]\backslash\{\ell\}}\sum_{j\in\mathcal{S}_h}\langle \bm{v}^{(t)},\bm{X}_{j}\rangle\Big).
\end{align}
Let $\pi=(\pi_1,\pi_2)$ be a permutation. For a given $\ell\in[L]$ and $m\in\mathcal{S}_{\ell}$, assume that $\pi_1(\ell')=\ell$ and $\pi_2(m')=m.$ Using \autoref{lem:Qreducvar}, we have:
\begin{align*}
    F(\pi(\bm{X}))&=\sum_{\ell=1}^L\sum_{m\in\mathcal{S}_{\ell}}\sigma\Big(\Lambda^{(t)}\langle \bm{v}^{(t)},\bm{X}_{\pi_2(m')}\rangle\\%\bm{X}_{\pi_1(\ell'),\pi_2(m')}
    &+\Gamma^{(t)}\hspace{-.4cm}\sum_{r\in\mathcal{S}_{\pi_1(\ell')}\backslash\{\pi_2(m')\}}\langle \bm{v}^{(t)},\bm{X}_{\pi_2(r)}\rangle %\bm{X}_{\pi_1(\ell'),\pi_2(r)}
    + \Xi^{(t)}\hspace{-.4cm}\sum_{h'\in[L]\backslash\{\ell'\}}\sum_{j\in\mathcal{S}_{h'}}\langle \bm{v}^{(t)},\bm{X}_{\pi_2(j')}\rangle\Big)\\
    &=\sum_{\ell=1}^L\sum_{m\in\mathcal{S}_{\ell}}\sigma\Big(\Lambda^{(t)}\langle \bm{v}^{(t)},\bm{X}_{m}\rangle\\
    &+\Gamma^{(t)}\hspace{-.4cm}\sum_{r''\in\mathcal{S}_{\ell}\backslash\{m\}}\langle \bm{v}^{(t)},\bm{X}_{r''}\rangle + \Xi^{(t)} \sum_{j''\not\in \mathcal{S}_{\ell} }\langle \bm{v}^{(t)},\bm{X}_{j''}\rangle\Big)\\
    &=F(\bm{X}).
\end{align*}
\end{proof}

\section{Justification of our data distribution}\label{sec:data_distrib}

In this section, we justify why the distribution $\mathcal{D}$ (\autoref{def:datadist}) is relevant. We first show that linear classifiers poorly generalize (\autoref{sec:genlin}). We then show that there exists classifiers that generalize without learning patch association (\autoref{sec:nopatchassoc}).

\subsection{Generalized linear models poorly generalize}\label{sec:genlin}

\linearmodel*
\begin{proof}[Proof of \autoref{lem:linear_models}]

For every data point $\bm{X}$, consider $\Delta(\bm{X}) := \sum_{j \in [D]} \delta_j$, it is very easy to see that for every integer $p $, as long as $\Pr[\Delta = p] = \Omega(1/\text{polylog}(d))$, we have that:
\begin{align} \label{eq:fadjofij}
   \frac{ \Pr[\Delta = p] }{\Pr[\Delta = p - 2C]} = 1 - o(1)
\end{align}

Consider two independently sampled data points, $\bm{X}, \bm{X}'$ with label $1, -1$ respectively, consider the event when $\Delta(\bm{X}) = p - 2C, \Delta(\bm{X}') = p$ and all the noises $\xi_{i}, \xi_i'$ of $\bm{X}, \bm{X}'$ satisfies $\xi_{i}= \xi_i'$, then we know that
\begin{align}
    \sum_{j \in [D]} \langle w_j, \bm{X}_j \rangle =     \sum_{j \in [D]} \langle w_j, \bm{X}_j' \rangle
\end{align}

By Eq~\eqref{eq:fadjofij} we also know that the density of $\bm{X}$ and $\bm{X}'$ under the data-generation distribution satisfies
$$p(\bm{X}) = (1 \pm o(1) ) p(\bm{X}')$$

Now, we know that 
\begin{align}
   P_0:= &\mathbb{P}[f^*(\bm{X})g(\bm{X})\leq 0 \mid y(\bm{X}) = 1] = 2 \int_{\bm{X}} 1_{f^*(\bm{X})g(\bm{X}) \leq 0}  p(\bm{X}, y(\bm{X}) = 1) d\bm{X}
    \\
    &= 2\int_{\bm{X}} 1_{f^*(\bm{X})g(\bm{X}) \leq 0}  p(\bm{X}', y(\bm{X}') = -1) d\bm{X} \pm o(1)
    \\
     &= 2\int_{\bm{X}} 1_{f^*(\bm{X})g(\bm{X}') \leq 0}  p(\bm{X}', y(\bm{X}') = -1) d\bm{X} \pm o(1)
      \\
     &=2\int_{\bm{X}'} 1_{f^*(\bm{X}')g(\bm{X}') \geq 0}  p(\bm{X}', y(\bm{X}') = -1) d\bm{X}' \pm o(1)
     \\
     &= 1 -  \Pr[f^*(\bm{X}')g(\bm{X}') < 0 \mid y(\bm{X'}) = -1]\pm o(1)
     \\
     & \geq  1 -  \Pr[f^*(\bm{X}')g(\bm{X}') \leq 0 \mid y(\bm{X'}) = -1]\pm o(1)
     \\
     &=   1 -  \Pr[f^*(\bm{X})g(\bm{X}) \leq 0 \mid y(\bm{X}) = -1]\pm o(1)
\end{align}

Therefore, $\mathbb{P}[f^*(\bm{X})g(\bm{X})\leq 0 \mid y(\bm{X}) = 1] \geq \frac{1}{2} - o(1)$.
\end{proof}

\subsection{Classifiers fitting the labelling function without patch association}\label{sec:nopatchassoc}

 \spurious*
\begin{proof}[Proof of \autoref{thm:nopatchlearning}]
We can consider a transformer in our setting, whose weights are defined as:
$\bm{v} = \bm{w}^\star$, $A_{i, j} = \beta > 0$ for $i, j \in \mathcal{S}_{\ell}$.  $\bm{A}_{i, j} =2 \beta$ for $i \in \mathcal{S}_{\ell}, j \in \mathcal{S}_{ \ell + 1 }$ (We denote $\mathcal{S}_{L + 1} = \mathcal{S}_1$). For a sufficiently large $\beta$, it is easy to check that $\mathbb{P}[f^*(\bm{X})\mathcal{M}(\bm{X})\leq 0] =  d^{-\omega(1)}$ but for all $\ell \in [L]$, $i \in \mathcal{S}_{\ell}$, $ \mathrm{Top}_{C}\; \{\langle \bm{p}_i^{(\mathcal{M})},\bm{p}_j^{(\mathcal{M})}\rangle\}_{j=1}^D = \mathcal{S}_{\ell + 1} \cap \mathcal{S}_{\ell} = \emptyset$.

\end{proof}

\section{Technical lemmas}\label{sec:useful}

In this section, we present the technical lemmas used in the paper.

\subsection{Tensor Power Method}

\begin{lemma}\label{lem:pow_method}
Let $\{z^{(t)}\}_{t\geq 0}$ be a positive sequence defined by the following recursions
\begin{align*}
    \begin{cases}
        z^{(t+1)}\geq z^{(t)} +m(z^{(t)})^{k}\\ % m^{(t)}
         z^{(t+1)}\leq z^{(t)} + M (z^{(t)})^{k} %M^{(t)}
    \end{cases},
\end{align*}
where $z^{(0)}>0$ is the initialization, $k>1$ is an  integer and $m,M>0$.  
Let $\upsilon>0$ such that $z^{(0)}\leq \upsilon.$ Then, the time $\mathcal{T}$ such that $z^{(t)}\geq \upsilon$ for all $t\geq \mathcal{T}$ is: 
\begin{align*}
    \mathcal{T}= \frac{3}{m(z^{(0)})^{k-1}}+\frac{2^{k+1}M}{m}\left\lceil \frac{\log(\upsilon/z^{(0)})}{\log(2)}\right\rceil.
\end{align*}

%\begin{align*}
%    t_0=\frac{\delta+1  }{mz^{(0)}} +\frac{M\upsilon^2}{m  (z^{0})^2 }\left\lceil \frac{\log(\upsilon/z_0)}{\log(1+\delta)}\right\rceil.
%\end{align*}
%
%
%where $M,m>0$ are constants defined as $M^{(t)}\leq M$ and $m^{(t)}\geq m$ for all $t\in [0,t_0]$ and $\delta \in (0,1).$
\end{lemma}
\begin{proof}[Proof of \autoref{lem:pow_method}]
 Let $n\in\mathbb{N}^*$. Let $T_n$ be the time where $z^{(t)}\geq 2^n z^{(0)}$. This time exists because $z^{(t)}$ is a non-decreasing sequence. We want to find an upper bound on this time. We start with the case $n=1.$ By summing the recursion, we have:
 \begin{align}\label{eq:fierihvei}
     z^{(T_1)}\geq z^{(0)}+m\sum_{s=0}^{T_1-1} (z^{(s)})^k.
 \end{align}
 We use the fact that $z^{(s)}\geq z^{(0)}$ in \eqref{eq:fierihvei} and obtain:
 \begin{align}\label{eq:hvefihedw}
     T_1 \leq \frac{z^{(T_1)} - z^{(0)}}{m (z^{(0)})^k}.
 \end{align}
 Now, we want to bound $z^{(T_1)} - z^{(0)}$. Using again the recursion and $z^{(T_1-1)}\leq 2z^{(0)}$, we have:
 \begin{align}\label{eq:hifeihrhie}
     z^{(T_1)}\leq z^{(T_1-1)}+M(z^{(T_1-1)})^k\leq 2z^{(0)} + 2^kM(z^{(0)})^k.
 \end{align}
 Combining \eqref{eq:hvefihedw} and \eqref{eq:hifeihrhie}, we get a bound on $T_1.$
 \begin{align}\label{eq:bdtdewojewd}
     T_1\leq\frac{ 1}{m (z^{(0)})^{k-1}}+\frac{2^kM}{m}.
 \end{align}
 Now, let's find a bound for $T_n$. Starting from the recursion and using the fact that $z^{(s)}\geq 2^{n-1}z^{(0)}$ for $s\geq T_{n-1}$ we have:
 \begin{align}\label{eq:vbjfwcsd}
     z^{(T_n)}\geq z^{(T_{n-1})} +  m\sum_{s=T_{n-1}}^{T_{n}-1}(z^{(s)})^k\geq z^{(T_{n-1})} +2^{k(n-1)}m(z^{(0)})^{k} (T_{n}-T_{n-1}).
 \end{align}
 On the other hand, by using $z^{(T_n-1)}\leq 2^nz^{(0)}$ we upper bound $z^{(T_n)}$ as follows.
 \begin{align}
     z^{(T_n)}&\leq z^{(T_n-1)}+M(z^{(T_n-1)})^k\leq 2^n z^{(0)} + 2^{kn}M(z^{(0)})^k.
 \end{align}
 Besides, we know that $z^{(T_{n-1})}\geq 2^{n-1}z^{(0)}$. Therefore, we upper bound $z^{(T_n)}-z^{(T_{n-1})}$ as
 \begin{align}\label{eq:cneknecfr}
    z^{(T_n)}-z^{(T_{n-1})}\leq  2^{n-1} z^{(0)} + 2^{kn}M(z^{(0)})^k.
 \end{align}
 Combining \eqref{eq:vbjfwcsd} and \eqref{eq:cneknecfr} yields:
 \begin{align}\label{eq:cfedaefjnve}
     T_{n}\leq T_{n-1} +\frac{1 }{2^{(k-1)(n-1)}m(z^{(0)})^{k-1}} + \frac{2^kM }{ m }.
 \end{align}
 We now sum \eqref{eq:cfedaefjnve} for $n=2,\dots,n$, use \eqref{eq:bdtdewojewd} and obtain:
 \begin{align}\label{eq:nfceknenvnervvre}
     T_n\leq T_1+ \frac{2}{m(z^{(0)})^{k-1}} + \frac{2^kMn}{m}\leq \frac{3}{m(z^{(0)})^{k-1}}+\frac{2^kM(n+1)}{m}\leq \frac{3}{m(z^{(0)})^{k-1}}+\frac{2^{k+1}Mn}{m}.
 \end{align}
Lastly, we know that $n$ satisfies $2^nz^{(0)}\geq \upsilon$ which implies  $n=\left\lceil \frac{\log(\upsilon/z_0)}{\log(2)}\right\rceil $ in 
 \eqref{eq:nfceknenvnervvre}.
\end{proof}

\begin{lemma}\label{lem:pow_method_prod}
Let $\{z^{(t)}\}_{t\geq 0}$ be a positive sequence defined by the following recursions
\begin{align*}
    \begin{cases}
        z^{(t+1)}\geq z^{(t)} +m(z^{(t)})^{k}\\ % m^{(t)}
         z^{(t+1)}\leq z^{(t)} + M (z^{(t)})^{k} %M^{(t)}
    \end{cases},
\end{align*}
where $z^{(0)}>0$, $k>1$ is an  integer and $m,M>0$.  
Let $\upsilon>0$ such that $z^{(0)}\leq \upsilon$ and  $\mathcal{T}$ be the time such that $z^{(t)}\geq \upsilon$ for all $t\geq \mathcal{T}$. Assume that $\frac{A \upsilon^2}{m}\ll 1$. Then, we have for $\kappa\in\{1,2\}$: 
\begin{align*}
    \prod_{\tau=0}^{\mathcal{T}-1}\big(1+A(z^{(\tau)})^{k-1+\kappa}\big)&\leq \big(1+\upsilon^{k-1}\big)^{\frac{2^{\kappa+1}MA\upsilon^{\kappa}\log(\upsilon/z^{(0)})}{m}}.
\end{align*}
\end{lemma}
\begin{proof}[Proof of \autoref{lem:pow_method_prod}] Let $n\in\mathbb{N}^*$ and let $T_n$ be the time such that $z^{(t)}\geq 2^n z^{(0)}$ for $t\geq T_n$.  Starting from the recursion, we have:
\begin{align}\label{eq:ferofe}
   \log( z^{(t+1)}) &\geq  \log( z^{(t)}) + \log\big(1+m(z^{(t)})^{k-1}\big).
\end{align}
Since $z^{(t)}$ is a non-decreasing sequence, \eqref{eq:ferofe} satisfies: 
\begin{align}\label{eq:zlogzseq}
     (z^{(t+1)})^{\kappa}\log( z^{(t+1)}) &\geq  (z^{(t)})^{\kappa}\log( z^{(t)}) + (z^{(t)})^{\kappa}\log\big(1+m(z^{(t)})^{k-1}\big).
\end{align}
We now sum \eqref{eq:zlogzseq}
for $t=T_{n-1},\dots,T_n$ and get:
\begin{align}\label{eq:oeroef}
     \frac{A}{m}(z^{(T_n)})^{\kappa}\log( z^{(T_n)}) &\geq   \frac{A}{m}(z^{(T_{n-1})})^{\kappa}\log( z^{(T_{n-1})}) + \sum_{t=T_{n-1}}^{T_n-1}  \frac{A (z^{(t)})^{\kappa}}{m}\log\big(1+m(z^{(t)})^{k-1}\big).
\end{align}
Since $\frac{A (z^{(t)})^{\kappa}}{m}\leq \frac{A (2^nz^{(0)})^{\kappa}}{m}\leq \frac{A \upsilon^\kappa}{m}\ll 1$, we have $\big(1+m(z^{(t)})^{k-1}\big)^{\frac{A (z^{(t)})^{\kappa}}{m}}\geq 1+A(z^{(t)})^{k-1+\kappa}$. We thus  lower bound \eqref{eq:oeroef} as: 
\begin{align}\label{eq:ewjeofefe}
     \frac{A}{m}(z^{(T_n)})^{\kappa}\log( z^{(T_n)}) &\geq   \frac{A}{m}(z^{(T_{n-1})})^{\kappa}\log( z^{(T_{n-1})}) + \sum_{t=T_{n-1}}^{T_n-1}  \log\big(1+A(z^{(t)})^{k-1+\kappa}\big).
\end{align}
On the other hand, by using $z^{(T_n-1)}\leq 2^nz^{(0)}$ and $z^{(T_n)}\leq 2^{n+1}z^{(0)}$, we have the following upper bound. 
\begin{equation}\label{eq:jfejoferf}
\begin{aligned}
    &\frac{A}{m}(z^{(T_n)})^{\kappa}\log( z^{(T_n)})\\
    \leq& \frac{A}{m}(z^{(T_n)})^{\kappa}\log\big( z^{(T_n-1)}) +\frac{MA}{m}(z^{(T_n)})^{\kappa}\log(1+(z^{(T_n-1)})^{k-1}\big)\\
    \leq& \frac{A\cdot (2^{n+1}z^{(0)})^{\kappa} }{m}\log\big( 2^n z^{(0)}) + \frac{MA\cdot (2^{n+1}z^{(0)})^{\kappa}}{m}\log\big(1+(2^{n}z^{(0)})^{k-1}\big).
\end{aligned}    
\end{equation}

Since $z^{(T_{n-1})}\geq 2^{n-1}z^{(0)}$, \eqref{eq:jfejoferf} is finally bounded as: 
\begin{equation}
\begin{aligned}\label{eq:ewjeof}
     &\frac{A}{m}\Big((z^{(T_n)})^{\kappa}\log( z^{(T_n)})-(z^{(T_{n-1})})^{\kappa}\log( z^{(T_{n-1})})\Big)\\
     \leq & \frac{\Theta(A)\cdot (2^{n}z^{(0)})^{\kappa} }{m}\log\big( 2^n z^{(0)})+ \frac{MA\cdot(2^{n+1}z^{(0)})^{\kappa}}{m}\log\big(1+(2^{n}z^{(0)})^{k-1}\big) .
\end{aligned}
\end{equation}
We  combine \eqref{eq:ewjeofefe} and \eqref{eq:ewjeof} to obtain:
\begin{equation}
\begin{aligned}\label{eq:fijfje}
   \sum_{t=T_{n-1}}^{T_n-1}  \log\big(1+m(z^{(t)})^{k}\big)&\leq \frac{\Theta(A)\cdot (2^{n}z^{(0)})^{\kappa} }{m}\log\big( 2^n z^{(0)})\\
     &+ \frac{MA\cdot(2^{n+1}z^{(0)})^{\kappa}}{m}\log\big(1+(2^{n}z^{(0)})^{k-1}\big).
\end{aligned}
\end{equation}
We now sum \eqref{eq:fijfje} and get:
\begin{align}\label{eq:freof}
    \sum_{t=0}^{T_n-1} \log\big(1+m(z^{(t)})^{k}\big)&\leq n(2^{n}z^{(0)})^{\kappa} \left( \frac{\Theta(A) }{m}\log\big( 2^n z^{(0)}) + \frac{2^{\kappa}MA}{m}\log\big(1+(2^{n}z^{(0)})^{k-1}\big)\right).
\end{align}
We replace $2^nz^{(0)}$ by $\upsilon$ and $n$ by $\log(\upsilon/z^{(0)})$ in \eqref{eq:freof} to get the aimed result.
\end{proof}

\subsection{Probabilistic lemmas}

\begin{lemma}\label{lem:chernoff_sumnoise}
Let $\{\delta_{r}\}_{r=1}^{D-C}$ be i.i.d. random variables such that with probability $q$ $\delta_{r}=\pm 1$ and zero otherwise. Then, with probability at least $1-1/\mathrm{poly}(d)$, we have:
\begin{align*}
    \sum_{\ell\neq\ell(\bm{X})}\sum_{r\in\mathcal{S}_{\ell}}  |\delta_{r}|\leq \Theta(q(D-C))\log(d).
\end{align*}
\end{lemma}
\begin{proof}[Proof of \autoref{lem:chernoff_sumnoise}] First, note that $|\delta_{r}|$ is a Bernoulli random variable with parameter $q.$ Therefore, $\sum_{\ell\neq\ell(\bm{X})}\sum_{r\in\mathcal{S}_{\ell}}  |\delta_{r}|$ is a binomial random variable $\mathcal{B}(D-C,q).$ Therefore, we apply a Chernoff bound to obtain: 
\begin{align}\label{eq:chrnff_sumbinom}
    \mathbb{P}\left[\sum_{\ell\neq\ell(\bm{X})}\sum_{r\in\mathcal{S}_{\ell}}  |\delta_{r}|\geq (1+\varepsilon)(D-C)q\right]\leq \exp\left(-\frac{(D-C)q\varepsilon^2}{3}\right).
\end{align}
Setting $\varepsilon=\sqrt{\frac{3}{(D-C)q}\log(\mathrm{poly}(d))}$ in \eqref{eq:chrnff_sumbinom} yields the desired result.
\end{proof}

\begin{lemma}\label{lem:smallsumdeltar}
  Let $\{\delta_{r}\}_{r=1}^{C}$ be i.i.d. random variables such that with probability $q$ $\delta_{r}=\pm 1$ and zero otherwise. Then, with probability at least $1-1/\mathrm{poly}(d)$, we have:
  \begin{align*}
      \sum_{r=1}^C |\delta_r|\leq O(1).
  \end{align*}
\end{lemma}
\begin{proof}[Proof of \autoref{lem:smallsumdeltar}] Let $k\in\mathbb{N}$ and $\Delta:= \sum_{r=1}^C |\delta_r|.$  The tail bound is bounded as: 
\begin{align}
    \mathbb{P}[\Delta\geq k]&=\sum_{j=k}^C \binom{C}{j} q^j(1-q)^{C-j}\leq q^k \sum_{j=k}^C \binom{C}{j} \leq q^k \sum_{j=0}^C \binom{C}{j}= 2^C q^k.
\end{align}
We want to find $k$ such that $2^C q^k\leq1/\mathrm{poly}(d)$ which implies $k\leq \log(d)/\log(D)\leq O(1)$.
\end{proof}

\subsection{Logarithmic inequalities}

\begin{lemma}\label{lem:logsigm}
 Let $a\in \mathbb{R}$ such that $C_-\leq a\leq C_+$, where $C_+,C_->0.$ Let $p\geq 3$ be an odd integer.
Then, the following inequality holds: 
\begin{align*}
 \frac{0.1}{C+}\log\left(1+\exp\left(-a^p\right)\right)\leq    \frac{a^{p-1}}{1+\exp(a^p)}\leq \frac{10}{C_-} \log\left(1+\exp\left(-a^p\right)\right).
\end{align*}
%\begin{align*}
%   0.01 \log(1+\exp(-x))\leq \frac{x}{1+\exp(x)}\leq 100 \log(1+\exp(-x)).
%\end{align*}
\end{lemma}

\begin{proof}[Proof of \autoref{lem:logsigm}]
We first remark that: 
\begin{align}\label{eq:194}
    \frac{a^{p-1}}{1+\exp(a^p)}&=\frac{a^{p}}{a(1+\exp(a^p))}.
\end{align}

\paragraph{Upper bound.} We upper bound \eqref{eq:194} by applying $a\geq C_-$:
\begin{align}
    \frac{a^{p-1}}{1+\exp(a^p)}\leq \frac{a^{p}}{C_-(1+\exp(a^p))}\label{eq:195}
\end{align}
We obtain the final bound by applying \autoref{lem:logsigm2} to \eqref{eq:195}.

\paragraph{Lower bound.} We lower bound \eqref{eq:194} by using $a\leq C_+$: 
\begin{align}\label{eq:199}
    \frac{a^{p-1}}{1+\exp(a^p)}\geq \frac{a^{p}}{C_+(1+\exp(a^p))}
\end{align}
We obtain the final bound by applying \autoref{lem:logsigm2} to \eqref{eq:199}.
\end{proof}

\begin{lemma}[Connection between derivative and loss]\label{lem:logsigm2}
Let $x>0.$ Then, we have: 
\begin{align}
 0.1\log(1+\exp(-x)) \leq \mathfrak
 {S}(-x)\leq 10\log(1+\exp(-x))
\end{align}
%\begin{align*}
%   0.01 \log(1+\exp(-x))\leq \frac{x}{1+\exp(x)}\leq 100 \log(1+\exp(-x)).
%\end{align*}
\end{lemma}

\begin{lemma}\label{lem:logfzd} Let $x,y>0.$   
Assume that $y\leq x.$ Then, we have:
\begin{align*}
    \log(1+xy)\leq(1+y) \log(1+x).
\end{align*}

\end{lemma}

\end{document}